\documentclass{article}
\usepackage{titletoc}

\usepackage{etoolbox}
\newcommand{\arxiv}[1]{\iftoggle{icml}{}{#1}}
\newcommand{\icml}[1]{\iftoggle{icml}{#1}{}}
\newtoggle{icml}
\global\togglefalse{icml}
\newcommand{\loose}{\looseness=-1}

\icml{
\PassOptionsToPackage{dvipsnames}{xcolor} 
\usepackage{icml2025}
}

\arxiv{
\usepackage[letterpaper, left=1in, right=1in, top=1in, bottom=1in]{geometry}
\PassOptionsToPackage{hypertexnames=false}{hyperref}  %
\usepackage{parskip}
\usepackage[dvipsnames]{xcolor}
\usepackage[colorlinks=true, linkcolor=blue!70!black, citecolor=blue!70!black,urlcolor=black,breaklinks=true]{hyperref}
}

\icml{
    \usepackage{parskip}
\usepackage[colorlinks=true, linkcolor=blue!70!black, citecolor=blue!70!black,urlcolor=blue,breaklinks=true]{hyperref}
}

\usepackage{amsmath}
\usepackage{microtype}
\usepackage{hhline}

\usepackage{amsthm}
\usepackage{mathtools}
\usepackage{bbm}
\usepackage{amsfonts}
\usepackage{amssymb}
\usepackage[nameinlink,capitalize]{cleveref}

\makeatletter
\newcommand{\neutralize}[1]{\expandafter\let\csname c@#1\endcsname\count@}
\makeatother

\usepackage{algorithm}

\arxiv{
\usepackage{natbib}
\bibliographystyle{plainnat}
\bibpunct{(}{)}{;}{a}{,}{,}
}

\usepackage{xpatch}

\usepackage{thm-restate}
\usepackage{autonum}
\declaretheorem[name=Theorem,parent=section]{theorem}
\declaretheorem[name=Lemma,parent=section, numberlike=theorem]{lemma}

\declaretheorem[name=Definition, parent=section, numberlike=theorem]{definition}

\declaretheorem[name=Corollary, parent=section, numberlike=theorem]{corollary}

\declaretheorem[name=Remark, parent=section, numberlike=theorem]{remark}
\declaretheorem[name=Proposition, parent=section, numberlike=theorem]{proposition}

\usepackage{crossreftools}
\makeatletter
  \renewenvironment{proof}[1][Proof]%
  {%
   \par\noindent{\bfseries\upshape {#1.}\ }%
  }%
  {\qed\newline}
  \makeatother

\xpatchcmd{\proof}{\itshape}{\normalfont\proofnameformat}{}{}
\newcommand{\proofnameformat}{\bfseries}

\newcommand{\pfref}[1]{Proof of \cref{#1}}

\newcommand{\secshort}[1]{\texorpdfstring{\hyperref[#1]{\mbox{Sec. \ref*{#1}}}}{Sec. \ref*{#1}}}

\renewcommand{\eqref}[1]{\texorpdfstring{\hyperref[#1]{(\ref*{#1})}}{(\ref*{#1})}}
\crefformat{equation}{#2Eq.\,(#1)#3}
\Crefformat{equation}{#2Eq.\,(#1)#3}

\newcommand{\lineref}[1]{\hyperref[#1]{Line~\ref*{#1}}}

\Crefformat{figure}{#2Figure~#1#3}
\Crefformat{assumption}{#2Assumption~#1#3}

\Crefname{assumption}{Assumption}{Assumptions}

\crefname{fact}{Fact}{Facts}

\Crefformat{figure}{#2Figure #1#3}
\Crefformat{assumption}{#2Assumption #1#3}

\usepackage{xparse}

\ExplSyntaxOn
\DeclareDocumentCommand{\XDeclarePairedDelimiter}{mm}
 {
  \__egreg_delimiter_clear_keys: %
  \keys_set:nn { egreg/delimiters } { #2 }
  \use:x %
   {
    \exp_not:n {\NewDocumentCommand{#1}{sO{}m} }
     {
      \exp_not:n { \IfBooleanTF{##1} }
       {
        \exp_not:N \egreg_paired_delimiter_expand:nnnn
         { \exp_not:V \l_egreg_delimiter_left_tl }
         { \exp_not:V \l_egreg_delimiter_right_tl }
         { \exp_not:n { ##3 } }
         { \exp_not:V \l_egreg_delimiter_subscript_tl }
       }
       {
        \exp_not:N \egreg_paired_delimiter_fixed:nnnnn 
         { \exp_not:n { ##2 } }
         { \exp_not:V \l_egreg_delimiter_left_tl }
         { \exp_not:V \l_egreg_delimiter_right_tl }
         { \exp_not:n { ##3 } }
         { \exp_not:V \l_egreg_delimiter_subscript_tl }
       }
     }
   }
 }

\keys_define:nn { egreg/delimiters }
 {
  left      .tl_set:N = \l_egreg_delimiter_left_tl,
  right     .tl_set:N = \l_egreg_delimiter_right_tl,
  subscript .tl_set:N = \l_egreg_delimiter_subscript_tl,
 }

\cs_new_protected:Npn \__egreg_delimiter_clear_keys:
 {
  \keys_set:nn { egreg/delimiters } { left=.,right=.,subscript={} }
 }

\cs_new_protected:Npn \egreg_paired_delimiter_expand:nnnn #1 #2 #3 #4
 {%
  \mathopen{}
  \mathclose\c_group_begin_token
   \left#1
   #3
   \group_insert_after:N \c_group_end_token
   \right#2
   \tl_if_empty:nF {#4} { \c_math_subscript_token {#4} }
 }
\cs_new_protected:Npn \egreg_paired_delimiter_fixed:nnnnn #1 #2 #3 #4 #5
 {
  \mathopen{#1#2}#4\mathclose{#1#3}
  \tl_if_empty:nF {#5} { \c_math_subscript_token {#5} }
 }
\ExplSyntaxOff

\XDeclarePairedDelimiter{\supnorm}{
  left=\lVert,
  right=\rVert,
  subscript=\infty
  }

\usepackage{subfigure}

\usepackage[utf8]{inputenc} %
\usepackage[T1]{fontenc}    %
\usepackage{url}            %
\usepackage{booktabs}       %
\usepackage{amsfonts}       %
\usepackage{nicefrac}       %
\usepackage{microtype}      %

\usepackage[subfigure]{tocloft}            %

\usepackage{makecell}
\usepackage{amsmath}
\usepackage{amsthm}

\usepackage{upgreek}

\usepackage{enumitem}

\usepackage{breakcites}

\definecolor{lightgray}{gray}{0.9}
\usepackage{mathrsfs}

\usepackage{algorithm}
\usepackage{verbatim}
\usepackage[noend]{algpseudocode}
\newcommand{\multiline}[1]{\parbox[t]{\dimexpr\linewidth-\algorithmicindent}{#1}}

\usepackage{multicol}

\usepackage{colortbl}
\usepackage{setspace}

\usepackage{transparent}

\usepackage{inconsolata}
\usepackage[scaled=.90]{helvet}
\usepackage{xspace}

\usepackage{pifont}
\usepackage{bm}

\usepackage{tablefootnote}

\DeclareFontFamily{U}{jkpmia}{}
\DeclareFontShape{U}{jkpmia}{m}{it}{<->s*jkpmia}{}
\DeclareFontShape{U}{jkpmia}{bx}{it}{<->s*jkpbmia}{}
\DeclareMathAlphabet{\mathfrak}{U}{jkpmia}{m}{it}
\SetMathAlphabet{\mathfrak}{bold}{U}{jkpmia}{bx}{it}

\DeclarePairedDelimiter{\abr}{\lvert}{\rvert} %
\DeclarePairedDelimiter{\sbr}{[}{]}
\DeclarePairedDelimiter{\cbr}{\{}{\}}
\DeclarePairedDelimiter{\rbr}{(}{)}

\DeclarePairedDelimiter{\abs}{\lvert}{\rvert} %
\DeclarePairedDelimiter{\brk}{[}{]}
\DeclarePairedDelimiter{\crl}{\{}{\}}
\DeclarePairedDelimiter{\prn}{(}{)}

\let\Pr\undefined

\DeclareMathOperator{\En}{\mathbb{E}}

\DeclareMathOperator{\Pr}{Pr}

\DeclareMathOperator*{\argmin}{arg\,min} %
\DeclareMathOperator*{\argmax}{arg\,max}

\newcommand{\wt}[1]{\widetilde{#1}}
\newcommand{\wh}[1]{\widehat{#1}}

\def\ddefloop#1{\ifx\ddefloop#1\else\ddef{#1}\expandafter\ddefloop\fi}
\def\ddef#1{\expandafter\def\csname bb#1\endcsname{\ensuremath{\mathbb{#1}}}}
\ddefloop ABCDEFGHIJKLMNOPQRSTUVWXYZ\ddefloop
\def\ddefloop#1{\ifx\ddefloop#1\else\ddef{#1}\expandafter\ddefloop\fi}
\def\ddef#1{\expandafter\def\csname b#1\endcsname{\ensuremath{\mathbf{#1}}}}
\ddefloop ABCDEFGHIJKLMNOPQRSTUVWXYZ\ddefloop
\def\ddef#1{\expandafter\def\csname sf#1\endcsname{\ensuremath{\mathsf{#1}}}}
\ddefloop ABCDEFGHIJKLMNOPQRSTUVWXYZ\ddefloop
\def\ddef#1{\expandafter\def\csname c#1\endcsname{\ensuremath{\mathcal{#1}}}}
\ddefloop ABCDEFGHIJKLMNOPQRSTUVWXYZ\ddefloop
\def\ddef#1{\expandafter\def\csname h#1\endcsname{\ensuremath{\widehat{#1}}}}
\ddefloop ABCDEFGHIJKLMNOPQRSTUVWXYZ\ddefloop
\def\ddef#1{\expandafter\def\csname hc#1\endcsname{\ensuremath{\widehat{\mathcal{#1}}}}}
\ddefloop ABCDEFGHIJKLMNOPQRSTUVWXYZ\ddefloop
\def\ddef#1{\expandafter\def\csname t#1\endcsname{\ensuremath{\widetilde{#1}}}}
\ddefloop ABCDEFGHIJKLMNOPQRSTUVWXYZ\ddefloop
\def\ddef#1{\expandafter\def\csname tc#1\endcsname{\ensuremath{\widetilde{\mathcal{#1}}}}}
\ddefloop ABCDEFGHIJKLMNOPQRSTUVWXYZ\ddefloop
\def\ddefloop#1{\ifx\ddefloop#1\else\ddef{#1}\expandafter\ddefloop\fi}
\def\ddef#1{\expandafter\def\csname scr#1\endcsname{\ensuremath{\mathscr{#1}}}}
\ddefloop ABCDEFGHIJKLMNOPQRSTUVWXYZ\ddefloop

\newcommand{\veps}{\varepsilon}

\newcommand{\ldef}{\vcentcolon=}
\newcommand{\rdef}{=\vcentcolon}

\newcommand{\Left}{\textbf{Left}}
\newcommand{\Right}{\textbf{Right}}
\newcommand{\Center}{\textbf{Center}}

\newcommand{\monotonic}{scaling-monotonic\xspace}
\newcommand{\monotonicity}{scaling-monotonicity\xspace}

\newcommand{\pipes}{\pihat_{\texttt{Pes}}}
\newcommand{\pichis}{\pi^{\chi}_{\beta}}
\newcommand{\pikl}{\pi^{\texttt{KL}}_{\beta}}
\newcommand{\pichistil}{\wt{\pi}^{\chi}_{\beta}}

\newcommand{\Cmax}{\cC^{\texttt{max}}}
\newcommand{\Chat}{\Cone[\pihat]}
\newcommand{\pimax}{\pi_{\texttt{max}}}
\newcommand{\Cpar}{C^{\star}}

\newcommand{\querycomp}{query complexity\xspace}

\newcommand{\framework}{sample-and-evaluate\xspace}

\newcommand{\pref}{\texttt{pref}}

\newcommand{\bon}{\texttt{BoN}\xspace} 
\newcommand{\bonlong}{Best-of-N\xspace}
\newcommand{\gsmk}{\texttt{GSM8K}\xspace}
\newcommand{\mmlu}{\texttt{MMLU}\xspace}
\newcommand{\mathk}{\texttt{MATH}\xspace}

\newcommand{\qone}{\textrm{\textbf{Q1}}}
\newcommand{\qtwo}{\textrm{\textbf{Q2}}}

\newcommand{\mainalg}{\texttt{InferenceTimePessimism}\xspace}
\newcommand{\normalg}{\texttt{ComputeNormConstant}\xspace}
\newcommand{\mainalglong}{Inference-Time Pessimism\xspace}

\newcommand{\vepsrm}{\veps_{\texttt{RM}}}
\newcommand{\vepsrms}{\veps^2_{\texttt{RM}}}

\newcommand{\bonalg}{\texttt{BoN-Alignment}\xspace}
\newcommand{\bonalglong}{Best-of-N Alignment\xspace}

\newcommand{\Phihat}{\wh\Phi}

\newcommand{\Mlong}{rejection threshold\xspace}
\newcommand{\Edivlong}{$\cE_M$-divergence\xspace}
\newcommand{\wlong}{importance weight\xspace}

\newcommand{\sig}[2][\beta]{\relu\rbr*{{#1}^{-1}\rbr*{#2}}}

\newcommand{\rejalg}{\texttt{RejectionSampling}\xspace}
\newcommand{\rejalglong}{Rejection Sampling \xspace}
\newcommand{\rejalgNM}[1][w]{\texttt{RejectionSampling}_{N,M}({{#1}}\midsem{}\piref,x)}

\newcommand{\Ediv}[3]{\cE_{#1}\prn*{#2,#3}}
\newcommand{\Epi}[1][M]{\Ediv{{#1}}{\pi}{\piref}}
\newcommand{\Epix}[1][M]{\Ediv{{#1}}{\pi(x)}{\piref(x)}}
\newcommand{\Epistarx}[1][M]{\Ediv{{#1}}{\pistar(x)}{\piref(x)}}
\newcommand{\Epistx}[1][M]{\Epistarx[{#1}]}
\newcommand{\Epist}[1][M]{\Epistar[{#1}]}

\newcommand{\Epistar}[1][M]{\Ediv{{#1}}{\pistar}{\piref}}
\newcommand{\Mstar}[2]{\cM_{{#1}}^{#2}}
\newcommand{\Mpi}[1][\veps]{\Mstar{{#1}}{\pi}}
\newcommand{\Mpistar}[1][\veps]{\Mstar{{#1}}{\pistar}}
\newcommand{\Mxpistar}[1][\veps]{\Mstar{x,{#1}}{\pistar}}
\newcommand{\Mxpi}[1][\veps]{\Mstar{x,{#1}}{\pi}}
\newcommand{\Mxpist}[1][\veps]{\Mxpistar[{#1}]}
\newcommand{\pisrej}{\pistar_{\mathsf{R}}}
\newcommand{\ybad}{y_{\mathsf{bad}}}

\newcommand{\uterm}[1]{\mathrm{(T{#1})}}

\newcommand{\one}[1]{\bbI\sbr*{#1}}
\newcommand{\bad}{\mathsf{bad}}

\newcommand{\Calpha}[1][\pi]{\cC_\alpha^{#1}}

\newcommand{\Dela}{\Delta_\alpha}

\newcommand{\pil}[1][\lambdahat]{\pi_{#1}}
\newcommand{\piN}{\pihat_{\texttt{BoN}}}
\newcommand{\Yhat}[1][N]{\wh\cY_{#1}}
\newcommand{\xiN}[1][N]{\wh\xi_{#1}}
\newcommand{\ystar}{y^\star}
\newcommand{\istar}{i^\star}

\newcommand{\pirej}[1][\pi]{{#1}_{\mathsf{R}}}
\newcommand{\term}{\mathrm{stop}}
\newcommand{\Yout}[1][M]{\cY_{#1}}
\newcommand{\Youtx}[1][M]{\cY_{#1}(x)}

\newcommand{\st}{\text{ s.t. }}
\newcommand{\PrN}{\Pr_{\Yhat,\,\xiN}}

\newcommand{\unif}{\texttt{unif}}

\newcommand{\yhat}{\wh{y}}

\newcommand{\reg}{\texttt{Reg}}

\newcommand{\Rmax}{R_{\mathsf{max}}}

\newcommand{\Cstar}{\Cone[\pistar]}

\newcommand{\piref}{\pi_{\mathsf{ref}}}
\newcommand{\rhat}{\wh{r}}

\newcommand{\Jr}[1][\rhat]{J_{#1}}

\newcommand{\chimix}{\chi_{\textsf{mix}}}

\newcommand{\Jmixg}[1][r]{J_{\beta,\gamma}^{\chimix}}

\newcommand{\pistarb}[1][r]{\pistar_{\sss{\beta}}}

\newcommand{\rstar}{r^\star}

\newcommand{\Cone}[1][\pi]{\cC^{#1}}
\newcommand{\Cinf}[1][\pi]{\cC_{\infty}^{#1}}

\newcommand{\chis}{$\chi^2$}

\newcommand{\chipo}{\texttt{$\chi$PO}\xspace} %
\newcommand{\pitil}{\wt{\pi}}%

\newcommand{\relu}{\mathsf{relu}}

\newcommand{\M}[1]{^{{\scriptscriptstyle M}}}  %

\newcommand{\sss}[1]{{\scriptscriptstyle#1}}

\newcommand{\pistar}{\pi^{\star}}

\newcommand{\pihat}{\wh{\pi}}

\newcommand{\midsem}{\,;}

\newcommand{\approxleq}{\lesssim}
\newcommand{\approxgeq}{\gtrsim}

\newcommand{\bigoh}{O}
\newcommand{\bigoht}{\wt{O}}
\newcommand{\bigom}{\Omega}
\newcommand{\bigomt}{\wt{\Omega}}

\renewcommand{\Pr}{\bbP}

\newcommand{\poly}{\mathrm{poly}}

\newcommand{\Dkl}[2]{D_{\mathsf{KL}}\prn*{#1\,\|\,#2}}

\newcommand{\Dchis}[2]{D_{\chi^2}\prn*{#1\dmid{}#2}}

\newcommand{\Dtv}[2]{D_{\mathsf{TV}}\prn*{#1,#2}}

\newcommand{\Ber}{\mathrm{Ber}}

\newcommand{\dmid}{\;\|\;}

\newcommand{\supp}{\mathrm{supp}}

\newcommand{\mathand}{\quad\text{and}\quad}

\def\multiset#1#2{\ensuremath{\left(\kern-.3em\left(\genfrac{}{}{0pt}{}{#1}{#2}\right)\kern-.3em\right)}}

\newcommand{\grad}{\nabla}

\newcommand{\iid}{i.i.d.\xspace}

\newcommand{\lambdahat}{\wh{\lambda}}
\newcommand{\oasst}{\texttt{OASST}\xspace}
\newcommand{\gemmarm}{\texttt{GEMMA-RM}\xspace}
\newcommand{\llamarm}{\texttt{LLAMA-RM}\xspace}
\newcommand{\armorm}{\texttt{ARMO-RM}\xspace}

\newcommand{\gemma}{\texttt{GEMMA-2-2B}\xspace}
\newcommand{\llama}{\texttt{LLAMA-3-3B}\xspace}
\newcommand{\mistral}{\texttt{Mistral-7B}\xspace}
\newcommand{\phimini}{\texttt{Phi-3-Mini}\xspace}
\newcommand{\phismall}{\texttt{Phi-3-Small}\xspace}
\newcommand{\alpaca}{\texttt{AlpacaEval-2.0}\xspace}

\input{widebar}

\usepackage[suppress]{color-edits}
 \addauthor{df}{ForestGreen}
 \addauthor{ab}{red}
 \addauthor{ak}{BurntOrange}
 \addauthor{ah}{olive}
 \newcommand{\dfc}[1]{}
 \newcommand{\ah}[1]{}

\makeatletter
\let\OldStatex\Statex
\renewcommand{\Statex}[1][3]{%
  \setlength\@tempdima{\algorithmicindent}%
  \OldStatex\hskip\dimexpr#1\@tempdima\relax}
\makeatother

 \usepackage{accents}

\usepackage[most]{tcolorbox}

\tcbset {
  base/.style={
    arc=0mm, 
    bottomtitle=0.5mm,
    boxrule=0mm,
    colbacktitle=!10!white, 
    coltitle=black, 
    fonttitle=\bfseries, 
    left=2.5mm,
    leftrule=1mm,
    right=3.5mm,
    title={#1},
    toptitle=0.75mm, 
  }
}
 
\let\oldparagraph\paragraph

\renewcommand{\paragraph}[1]{\oldparagraph{#1.}}
\arxiv{\renewcommand{\paragraph}[1]{\oldparagraph{#1.}}}
\icml{\renewcommand{\paragraph}[1]{\textbf{#1.}}}

\newcommand{\papertitle}{Is Best-of-N the Best of Them? \\Coverage, Scaling, and
  Optimality in Inference-Time Alignment} %
\newcommand{\papertitleshort}{Coverage, Scaling, and Optimality in Inference-Time Alignment} %

\arxiv{
\title{\papertitle}
}

\icml{
\icmltitlerunning{\papertitleshort}
}

\date{}

  \addtocontents{toc}{\protect\setcounter{tocdepth}{0}}

\begin{document}
\arxiv{\maketitle}

\icml{
\twocolumn[
\icmltitle{\papertitle}

\icmlsetsymbol{equal}{*}

\begin{icmlauthorlist}
\icmlauthor{Audrey Huang}{xxx}
\icmlauthor{Qinghua Liu}{xxx}
\icmlauthor{Adam Block}{xxx}
\icmlauthor{Nan Jiang}{xxx}
\icmlauthor{Akshay Krishnamurthy}{yyy}
\icmlauthor{Dylan J. Foster}{yyy}
\end{icmlauthorlist}

\icmlaffiliation{xxx}{Department of XXX, University of YYY, Location, Country}
\icmlaffiliation{yyy}{Department of XXX, University of YYY, Location, Country}
\icmlaffiliation{comp}{Company Name, Location, Country}
\icmlaffiliation{sch}{School of ZZZ, Institute of WWW, Location, Country}

\icmlkeywords{Machine Learning, ICML}

\vskip 0.3in
]
\printAffiliationsAndNotice{}  %
}

\arxiv{
  \vspace{-4em}
  \begin{center}
    \large
    \setlength{\tabcolsep}{20pt}
    \begin{tabular}{ccc}
    \makecell{Audrey Huang \\ \small{\texttt{audreyh5@illinois.edu}}}
    &
    \makecell{Adam Block \\ \small{\texttt{blockadam@microsoft.com}}}
    &
    \makecell{Qinghua Liu \\ \small{\texttt{qinghualiu@microsoft.com}}}
    \end{tabular}\vspace{1em}
    \begin{tabular}{ccc}
    \makecell{Nan Jiang \\ \small{\texttt{nanjiang@illinois.edu}}}
    &
    \makecell{Akshay Krishnamurthy \\ \small{\texttt{akshaykr@microsoft.com}}}
    &
    \makecell{Dylan J. Foster \\ \small{\texttt{dylanfoster@microsoft.com}}}
    \end{tabular}
    \end{center}
    \vspace{2em}
}

\begin{abstract}
\emph{Inference-time computation} offers a powerful axis for scaling the performance of language models.
However, naively increasing computation in techniques like \bonlong sampling
can lead to performance degradation due to reward hacking.
Toward a theoretical understanding of how to
best leverage additional computation, 
we focus on inference-time alignment, which we formalize
as the problem of improving the quality of responses drawn from a pre-trained policy,
given a prompt of interest and access to an imperfect reward model.
We analyze the performance of inference-time alignment algorithms 
in terms of (i) response quality, and (ii) compute,
and provide new results that highlight the importance of the
pre-trained policy's \emph{coverage} over high-quality responses for
performance and compute scaling:\arxiv{\loose}
 \icml{(1) We show that Best-of-N alignment with an ideal $N$
  can achieve optimal performance under stringent notions of coverage,
  but provably suffers from reward hacking when $N$ is large, and
  fails to achieve tight guarantees under more
  realistic coverage conditions; (2) We introduce $\texttt{InferenceTimePessimism}$, a new algorithm which mitigates reward hacking through deliberate use of inference-time compute, implementing \emph{pessimism in the face of uncertainty}; we prove that its performance is \emph{optimal} and \emph{scaling-monotonic}, i.e., does ot degrade as $N$ increases.}
\arxiv{
\begin{enumerate}
\item We show that Best-of-N alignment with an ideal choice for $N$
can achieve optimal performance under stringent notions of coverage,
but provably suffers from reward hacking when $N$ is large, and
fails to achieve tight guarantees under more
realistic coverage conditions.
\item We introduce $\texttt{InferenceTimePessimism}$, a new algorithm which mitigates reward hacking through
  more sophisticated use of inference-time compute, implementing\arxiv{ the
    principle of} \emph{pessimism in the face of uncertainty}\arxiv{ via rejection
    sampling}; we prove that its performance is \emph{optimal} and
  does not degrade with $N$, a property that we call \emph{scaling-monotonic}.
\end{enumerate}
}
We complement our theoretical results with an experimental evaluation that demonstrate the benefits of \texttt{InferenceTimePessimism} across a variety of tasks and models.

\end{abstract}

\section{Introduction}
\label{sec:intro}

  \emph{Inference-time computation} has emerged as a new axis for scaling in language models,
  which has led to dramatic improvements in their capabilities
  \citep{brown2024large,snell2024scaling,wu2024empirical,openai2024o1,deepseek2025r1}
  and played a central role in recent AI breakthroughs \citep{chollet2024arc}. 
  While there are a multitude of ways in which additional computation can be utilized during inference---e.g., to
  generate long chains of thought \citep{wei2022chain, li2024chain}, 
  have the model rate or correct its own responses \citep{zheng2023judging, wu2024metarewarding}, 
  or implement planning and search \citep{yao2024tree,zhang2024rest}---even exceedingly simple methods like \emph{\bonlong (parallel) sampling} 
  can provide significant performance gains \citep{lightman2023lets, brown2024large},
  and enjoy provable representational benefits \citep{huang2024self}. Such algorithms are also widely used throughout post-training for frontier models \citep{deepseek2025r1,yang2024qwen2},
  and thereby serve as a cornerstone of fine-tuning and inference.\loose

  Toward developing a deeper understanding of the algorithmic landscape for inference-time computation,
  in this paper, we focus on the problem of \emph{inference-time alignment}:
  given a task and a reward model that is a proxy for task performance, 
  how can we best use inference-time computation to improve the quality of a response
  chosen from many candidates generated by the base model?  The \bonlong heuristic is one of the most widely used inference-time alignment methods, and proceeds by generating $N$ candidate responses for a given prompt,
  then returning the response with the highest reward under a reward model
  \citep{stiennon2020learning,nakano2021webgpt,touvron2023llama,gao2023scaling,eisenstein2023helping,mudgal2024controlled}.
  While attractive in its simplicity, 
  \bonlong and related heuristics are known to suffer from \emph{reward overoptimization} or \emph{reward hacking} when $N$ increases
  \citep{gao2023scaling,stroebl2024inference,chow2024inference,frick2024evaluate}.
  While one might hope that increasing computation to generate a larger number of candidates will 
  increase the likelihood of selecting a high-quality response,
  reward model errors at the tail of the response distribution 
  can cause \bonlong to return generations with high (modeled) reward, but poor task performance.  \loose

  This overoptimization phenomenon raises two fundamental questions, 
  which we investigate in this paper: 
    (1) what is the extent to which the imperfect reward model limits the performance of inference-time alignment methods;
    and (2) can more deliberate algorithm design 
    lead to performance gains that scale monotonically with increased computation? 
    Beyond guiding practical interventions for inference-time alignment, 
    answers to these questions contribute to a foundational understanding 
    of inference-time computation more broadly.\loose

\icml{\begin{figure*}[htp]}
\arxiv{\begin{figure*}[tp]}
  \centering
    \includegraphics[width=0.43\textwidth]{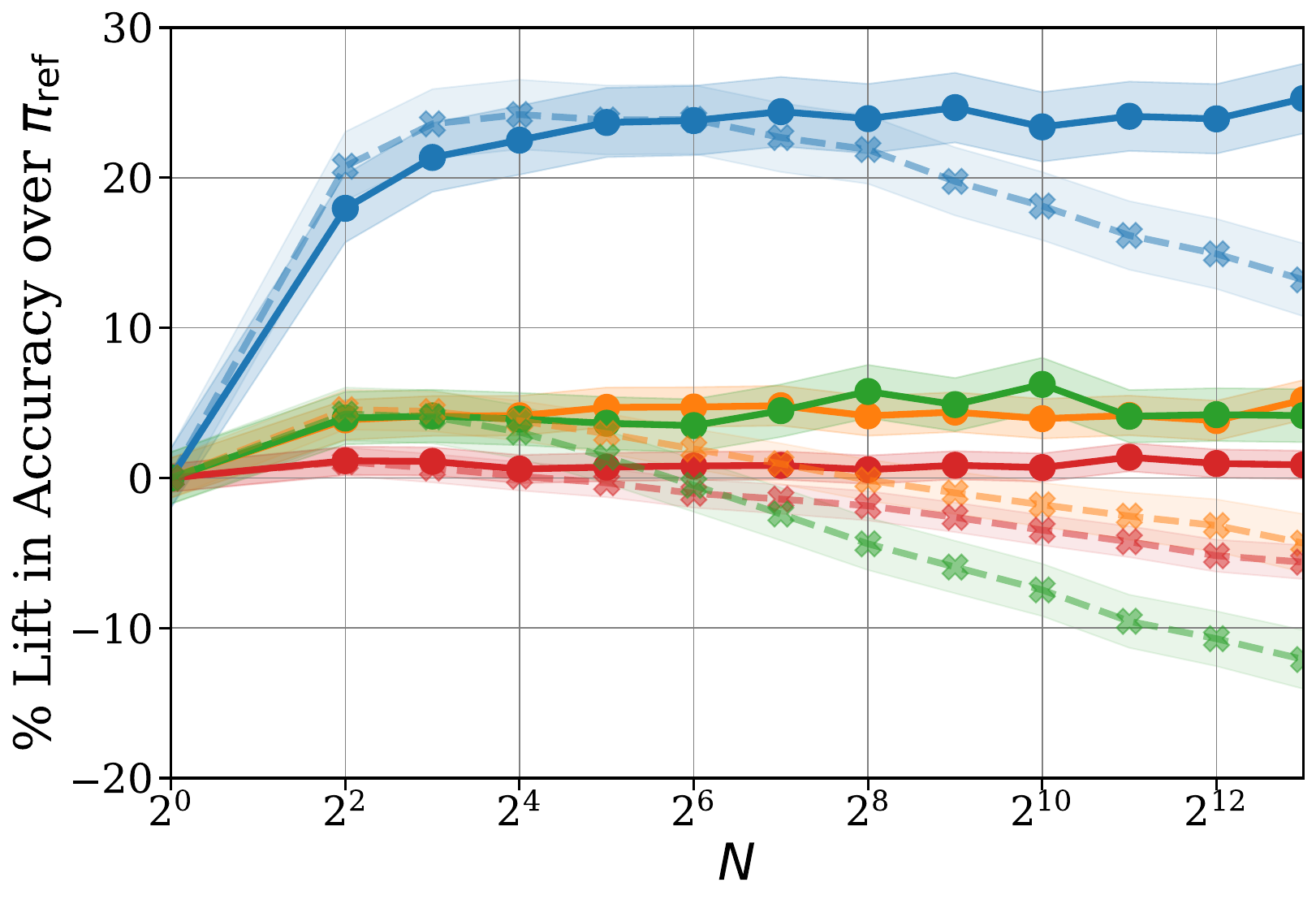}
    \label{sfig:bon-improvement} 
  \hfill
  \includegraphics[width=0.43\textwidth]{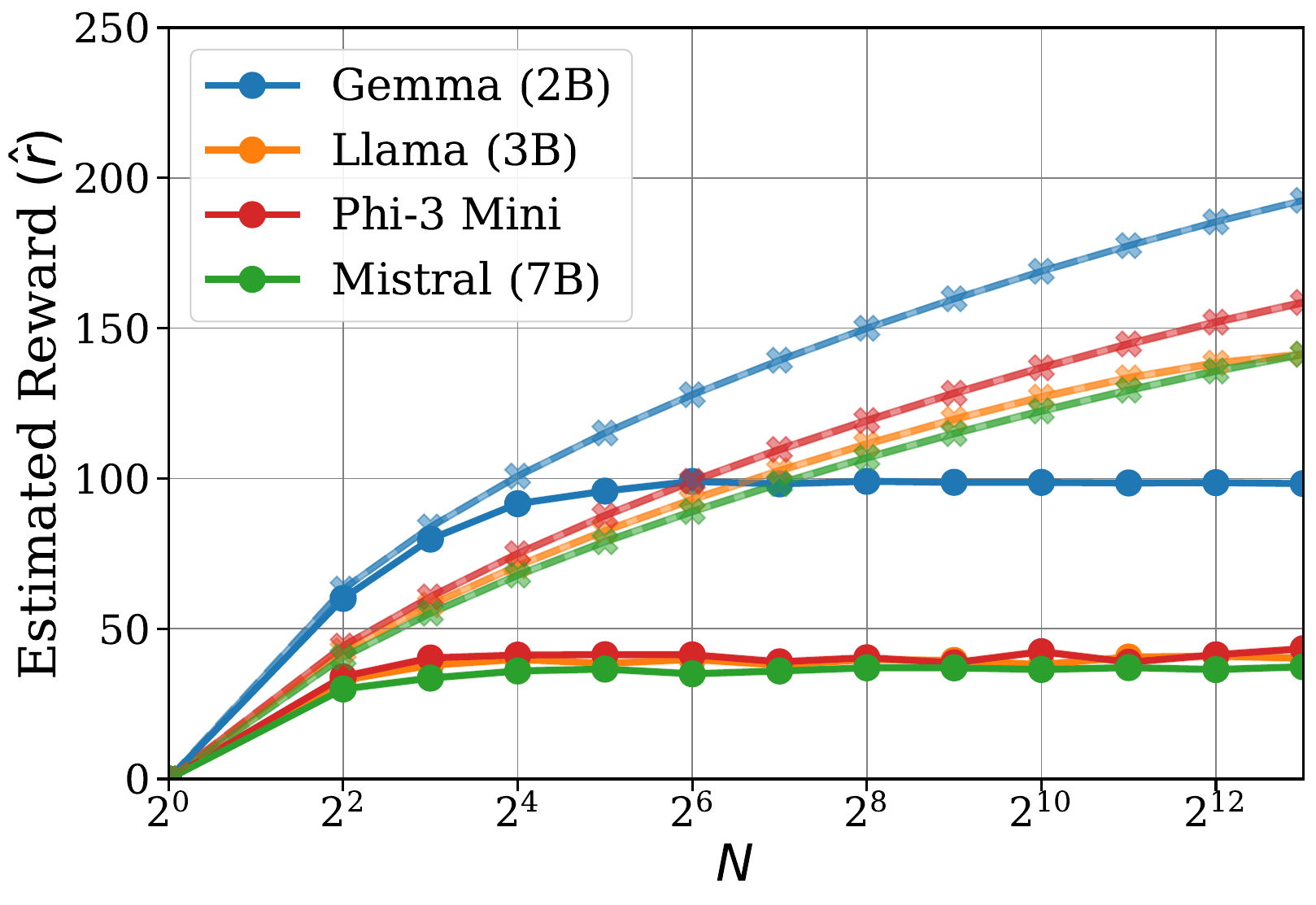}
  \label{sfig:bon-rhat-improvement}
  \icml{\vspace{-0.4cm}}
  \caption{Comparison between performance of \bonalg (dashed lines) and
    \mainalg algorithm (solid lines) on \gsmk with reward model \oasst as $\rhat$ and several different choices of $\piref$.  \Left: %
  $\bonalg$ initially improves accuracy over $\piref$, but eventually
  degrades as $N$ increases, while \mainalg is monotone, as predicted
  by our theory.  \Right: \bonalg overoptimizes the reward model, with
  high $\rhat$ but lower accuracy as $N$ increases, whereas \mainalg
  stops increasing $\rhat$ with $N$ beyond a certain threshold
  determined by a regularization parameter.}
  \icml{\vspace{-0.3cm}}
  \label{fig:intro}
\end{figure*}
  
\icml{\vspace{-0.1cm}}
\arxiv{\vspace{-0.15cm}}
\paragraph{Our framework}
To address the questions above, we pose \emph{inference-time
alignment}\icml{\footnote{Following prior work (e.g.,
\citet{rafailov2024direct,ye2024theoretical}), we adopt contextual
bandit/reinforcement learning terminology, interpreting the language
model as a policy.}} 
as the task of extracting a high-quality response $\yhat$ 
for a prompt $x$ 
from a pre-trained language model, or
\emph{base policy} $\piref: \cX\to\Delta(\cY)$, 
that maps a prompt $x$ to a distribution over responses $y \in \cY$.
The quality of the response $\yhat$ is determined by an underlying 
\emph{true reward function} $\rstar : \cX \times \cY \rightarrow [0, \Rmax]$, that expresses, for example,
the correctness of a math proof or the helpfulness of a chat response.
The algorithm designer does not know $\rstar$
and instead uses an \emph{imperfect reward model $\rhat$} (e.g., one learned from preference-based feedback 
\citep{christiano2017deep,ouyang2022training,wang2024interpretable}) to maximize performance as follows.
\icml{Using black-box access to $\piref$ and $\rhat$, i.e., 
\emph{sampling queries} $y \sim \piref(\cdot \mid{} x)$ and \emph{evaluation queries}
$\rhat(x,y)$ and $\piref(y \mid x)$, we aim to produce a response $\yhat$ that approximately maximizes the true reward (thereby minimizing \emph{regret} to the optimal policy), i.e.,
\begin{align}\label{eq:objective}
  \rstar(x, \yhat) \approx \max_{y \in \cY} \rstar(x,y)
\end{align}
}
\arxiv{
    \begin{tcolorbox}[enhanced,title=Inference-Time Alignment,
        colframe=blue!40!black,
        colback=blue!2!white,
        fonttitle=\bfseries,
      attach boxed title to top text left={xshift=30mm,yshift=-2.5mm},
      boxed title
      style={size=small,colframe=blue!40!black,colback=blue!40!black}]
      \label{box:def}
  Given a prompt $x\in\cX$, use black-box access to $\piref$ and $\rhat$ in the form of
(a) \emph{sampling queries} $y \sim \piref(\cdot \mid{} x)$, and (b) \emph{evaluation queries}
via reward labels $\rhat(x,y)$ and log-probabilities $\log\piref(y \mid x)$, to produce a response $\yhat$ that approximately maximizes the true reward (i.e., minimizes regret to the optimal response):\loose
\begin{align}\label{eq:objective}
  \rstar(x, \yhat) \approx \max_{y \in \cY} \rstar(x,y)
\end{align}
\end{tcolorbox}
}

This formulation presents two central questions:

\begin{itemize}[itemsep=0em]
  \item[\qone:] \textbf{Regret.} \emph{How close to optimal can we make the reward $\rstar(x,\yhat)$ in \cref{eq:objective},
    as a function of the quality of the\arxiv{ reward} model $\rhat$?}\loose
  \item[\qtwo:] \textbf{Compute.} \emph{What is the computational cost---measured
      by the number of sampling queries $y\sim\piref(\cdot\mid{}x)$, and
      evaluation queries $\rhat(x,y)$ and $\log\piref(y\mid{}x)$ used by the algorithm---required for optimal reward?}\loose
  \end{itemize}
  Note that while our goal is to maximize the true reward $\rstar$, 
  the quality of the reward model $\rhat$ creates an 
  inherent, information-theoretic barrier to optimizing $\rstar$, since the latter unobserved. 
  Intuitively, we should not be able to maximize quality under $\rstar$ if $\rhat$ is highly inaccurate.

\subsection{Contributions}

We show that \bonlong alignment and related heuristics may fail to achieve optimal regret, 
but that more sophisticated use of inference-time computation---namely, to extract additional information from the reward model and
quantify its uncertainty---can mitigate overoptimization,
and achieve optimal regret and compute scaling.\loose

\arxiv{\vspace{-2pt}}

\paragraph{Statistical framework and necessity of coverage \arxiv{(\cref{sec:framework})}\icml{(\secshort{sec:framework})}}

Our formal framework, summarized above, reformulates inference-time
alignment as a \emph{statistical problem} via query complexity. 
This allows us to derive fundamental limits on the performance of any inference-time alignment algorithm. 
We show that the best possible reward one can achieve, irrespective of computational cost, 
is determined by the base policy's \emph{coverage} over high-quality responses, 
along with the mean-squared error of $\rhat$. 
This serves as a skyline for our investigation into
improved algorithmic interventions.

\arxiv{\vspace{-2pt}}

\paragraph{Tight analysis of \bonalg \icml{(\secshort{sec:bon})}\arxiv{(\cref{sec:bon})}}
Within our framework, we offer the first theoretical analysis of the regret of 
Best-of-N alignment (\bonalg). We show that \bonalg can achieve optimal regret under a stringent
  notion of coverage (``uniform'' or $L_\infty$-type) when $N$ is tuned appropriately, but:
  \loose
  \begin{enumerate}%
\item provably suffers from overoptimization, degrading in
  performance once $N$ scales past a critical threshold; and\loose
\item fails to achieve tight guarantees under weaker
notions of coverage   (``average-case'' or $L_1$-type), 
thereby falling short of the skyline established in \cref{sec:framework}.
\end{enumerate}

\arxiv{\vspace{-2pt}}

  \paragraph{Optimal algorithm: \mainalg \arxiv{(\cref{sec:algorithm})}\icml{(\secshort{sec:algorithm})}}
Motivated by the shortcomings of \bonlong, we introduce an improved algorithm, \mainalg. We prove that \mainalg:
\begin{enumerate}
\item is \emph{regret-optimal}, in the sense that it can achieve the best possible reward in our framework, 
thereby matching the skyline in \cref{sec:framework}; and
\item is \emph{\monotonic}, in the sense that it is guaranteed to avoid
  overoptimization beyond a certain point (determined by the
  regularization parameter), even as $N\to\infty$. 
\end{enumerate}
To achieve this, our algorithm uses a novel rejection
sampling scheme to implement \chis-regularization---which is known to mitigate overoptimization via the principle of   
\emph{pessimism in the face of uncertainty} \citep{huang2024correcting}---purely at inference time. 
Beyond achieving optimal regret, we show that \mainalg uses near-optimal compute under our framework.
\loose

\arxiv{\vspace{-2pt}}

\paragraph{Empirical evaluation \arxiv{(\cref{sec:experiments})}\icml{(\secshort{sec:experiments})}}
  To demonstrate the benefits of \mainalg, we compare the algorithm to \bonalg across several tasks, base policies, and reward models.  
  As predicted by our theory, \bonalg degrades in performance as computation increases (via $N$), 
  while \mainalg is scaling-monotonic and does not suffer
  from this characteristic overoptimization phenomenon (\cref{fig:intro}). 
  We also observe several instances where \mainalg outperforms
  \bonalg in terms of maximal reward achieved when the computational budget is untuned, demonstrating the robustness of our algorithm.
  \Cref{app:experiments} contains further
  empirical results, including examinations of the reward model's tail behavior
  and demonstrations of \mainalg's robustness to choice of hyperparameter.\loose

\subsection{Related Work}
\label{sec:related_body}

To our knowledge, our work provides the first theoretical framework for understanding and mitigating
overoptimization in inference-time alignment.
Various works have analyzed
specific properties of the \bonlong alignment algorithm
\citep{yang2024asymptotics,beirami2024theoretical,mroueh2024information}
such as tradeoffs between reward and KL-divergence, but do not
ultimately provide guarantees on \arxiv{downstream }performance
when the estimated reward model and true reward are mismatched.
We focus on analyzing \bonlong specifically because it is the most widely used and foundational inference-time
alignment technique, but other algorithms~\citep{welleck2024decoding} including variants of
rejection sampling~\citep{khanov2024args,chen2024pad,shi2024decoding,liu2024decoding,jinnai2024regularized}
and Monte-Carlo Tree Search~\citep{feng2023alphazero,yao2024tree,zhang2024rest}, are also used in practice.\loose %

  Our work is also closely related to research on \emph{offline alignment at training time} \citep{christiano2014online,ouyang2022training,rafailov2024direct}, which involves fitting a reward model $\rhat$ (typically through pairwise feedback), and then maximizing it using \arxiv{policy optimization techniques}\icml{RL}. A growing body of theoretical works on offline
  alignment provide theoretical guarantees for mitigating reward
  overoptimization that---like our work---depend on notions of \emph{coverage} for the base policy~\citep{zhu2023principled,zhan2023provable,li2023reinforcement,xiong2024iterative,liu2024provably,cen2024value,
fisch2024robust,ji2024selfplay,huang2024correcting}. Our formulation for inference-time alignment can be viewed as a variant of this problem that abstracts away the reward model training; in particular, we take $\rhat$ as given and ask how to achieve the best possible performance on a \emph{per-instance} basis, with respect to the true reward $\rstar$ and the prompt $x \in \cX$. %

Our algorithm, \mainalg, is closely related to \chipo \citep{huang2024correcting}, a training-time offline alignment algorithm that aims to mitigate
overoptimization via regularization with the
\chis-divergence. \mainalg also leverages \chis-regularization,
but implements this purely at inference-time via a novel rejection sampling
scheme.\icml{\textbf{See \cref{sec:offline} for a detailed discussion of connections to
offline alignment.}\loose}
Finally, \arxiv{our analysis of \mainalg, as well as \bonalg,
  draws}\icml{our analysis draws} on the treatment of \emph{approximate rejection sampling} in \citet{block2023sample}, which tightly characterizes the performance of rejection sampling with unbounded likelihood ratios. %
\arxiv{Please see \cref{sec:related} for an extended discussion
of related work.}

\arxiv{\section{A Statistical Framework for Inference-Time Alignment}}
\icml{\section{Inference-Time Alignment Framework}}
\label{sec:framework}

In this section, we formally introduce our inference-time alignment
framework, then use it to prove lower bounds that highlight the
necessity of coverage. In our framework, the algorithm designer begins with a \emph{base policy}
$\piref:\cX\to\Delta(\cY)$, which is a conditional distribution
mapping \emph{prompts} $x\in\cX$ to
distributions over \emph{responses} $y\in\cY$,
and $\piref(y\mid{}x)$ is the probability that the base policy generates
a response $y$ given the prompt $x$.\footnote{The motivating special case is
autoregressive language modeling, where $\cY = \cV^H$ for a
  vocabulary space $\cV$ and sequence length $H$, and $\piref$ has the
  autoregressive structure $\piref(y_{1:H}\mid{}x) = \prod_{h=1}^H
  \piref(y_h \mid{}y_{1:h-1},x)$ for $y = y_{1:H} \in
  \cY$.} 
  The algorithm designer also has access to an \emph{imperfect reward model} 
  $\rhat : \cX \times \cY \rightarrow [0, \Rmax]$, 
  where $\Rmax \ge 1$, 
  and the reward label $\rhat(x,y)$ estimates the quality of the response $y$ for the prompt $x$.
  This serves as a proxy for the \emph{true reward model}
  $\rstar : \cX \times \cY \rightarrow [0, \Rmax]$, 
  which is unknown.

  Given the base policy $\piref$, the reward model $\rhat$, and a prompt $x\in\cX$,
  the goal is to produce a high quality response $\yhat\in\cY$
  as measured by the true reward $\rstar$, 
  in the sense that the following \emph{inference-time regret} is small.\loose
\begin{align}
  \label{eq:regret}
\reg_{\pistar}(\yhat\midsem{}x) \ldef  J(\pistar\midsem{}x) - J(\yhat\midsem{}x) \leq \veps,
\end{align}
Here, $\pistar$ can be viewed as a comparator policy with high reward,
and $J(\pi;x)\ldef{}\En_{y\sim\pi(\cdot\mid{}x)}\brk*{\rstar(x,y)}$
denotes the expected reward under $\rstar$ for any policy $\pi$. 
When the response $\yhat$ is chosen by a randomized algorithm $\cA$, 
we use $\pihat_\cA(\cdot\mid{}x)$ to denote the conditional distribution over 
responses induced by $\cA$, and write regret as 
$\reg_{\pistar}(\pihat\midsem{}x) \ldef
J(\pistar\midsem{}x) - J(\pihat_\cA\midsem{}x)$.\loose

In applications, the base policy $\piref$ represents a language model trained in some manner\arxiv{ (e.g., through next-token prediction and  post-training
steps such as SFT or RLHF)}. 
The true reward $\rstar$ can represent the extent to which $y$
agrees with human preference, passes a proof checker, or passes a unit test suite.
The reward model $\rhat$ can be an open
source reward model, a model trained by the algorithm designer
themselves, or even derived directly from $\piref$ itself
\citep{wang2022self,huang2024self,song2024understanding}.\loose

\paragraph{Reward model quality versus regret}
There is no hope of making the regret in \cref{eq:regret}
small without requirements on the fidelity of the reward model $\rhat$.
For a prompt $x\in\cX$, we measure the quality of $\rhat$ via 
the expected squared error with respect to $\rstar$, 
where responses are drawn by the base policy $\piref$:\loose
\begin{align}
  \label{eq:rm}
\vepsrms(x)\ldef{}  \En_{y\sim\piref(\cdot\mid{}x)}\brk*{\prn*{\rhat(x,y)-\rstar(x,y)}^2}.
\end{align}
Empirically and in theory, one can minimize
\cref{eq:rm} by fitting $\rhat$ to rewards or human
preference data for responses generated from $\piref$
\arxiv{\citep{zhu2023principled,zhan2023provable,li2023reinforcement,xiong2024iterative,liu2024provably,cen2024value, fisch2024robust,ji2024selfplay,huang2024correcting}}\icml{(cf. \cref{sec:offline})},
which is already standard in post-training pipelines.
As our focus is on \emph{inference-time} interventions,
we simply abstract the reward model training step away
and assume we have a reward model $\rhat$ with error $\vepsrms$,
as defined in \cref{eq:rm}.
We aim to understand how small we can make the regret in \cref{eq:regret}
as a function of the reward model quality $\vepsrms$,
and what algorithmic interventions are required to achieve this (cf. \qone).

\begin{remark}
  In practice, the reward estimation error in \cref{eq:rm}
  may not be the tightest notion of reward error, and thus our bounds may be conservative. 
  Nevertheless, it is arguably the most common formulation of reward mismatch \citep{zhu2023principled,zhan2023provable, xiong2024iterative,gao2024rebel, huang2024correcting},
  and serve as a foundation for future investigations 
  into other notions of error. 
\end{remark}

\paragraph{\mbox{Measuring computational efficiency via query complexity}}
The next question we consider (cf. \qtwo) is how much computation is required to
minimize the inference-time regret in \cref{eq:regret} for a given
prompt \arxiv{$x\in\cX$}\icml{$x$}.
To allow for computational understanding that is agnostic to the
architecture or description of $\piref$ and $\rhat$,
we draw inspiration from \citet{huang2024self} and abstract it as a \emph{statistical problem},
where computational efficiency is quantified via 
the number of samples and labels queried from the base policy and reward model.\loose

\begin{definition}[Sample-and-evaluate framework]\label{def:oracle-model}
  In the \textbf{\framework} framework, the algorithm designer does
  not have explicit access to the base policy $\piref$ or the reward
  model $\rhat$. Instead, they
access $\piref$ and $\rhat$ through \emph{sample-and-evaluate queries}: For a
given prompt $x\in\cX$, they can sample $N$
responses $y_1,y_2,\dots y_N \sim \piref(\cdot \mid x)$ and observe
the likelihood $\piref(y_i\mid{}x)$ and reward model value
$\rhat(x,y_i)$ for each such response. The efficiency
\arxiv{(\emph{\querycomp}) }of the algorithm is measured by the total number
of queries $N$.\loose
\end{definition}
This is a natural statistical abstraction for computation---analogous
to oracle/query complexity in optimization
\citep{nemirovski1983problem,traub1988information,raginsky2011information,agarwal2012information},
and learning theory \citep{blum1994weakly,kearns1998efficient,feldman2012complete,feldman2017general}---and encompasses \bonlong
alignment and other schemes such as
rejection sampling
\citep{khanov2024args,chen2024pad,shi2024decoding,liu2024decoding,jinnai2024regularized}.\loose

  \subsection{A Skyline\arxiv{ for Performance}: The Necessity of Coverage}
  \label{sec:coverage}

To motivate our main algorithmic results, and as a step toward
answering question \qone, we begin by proving a lower bound on the
best possible regret one can hope to achieve,
as determined by the reward model's error $\vepsrms(x)$.
This lower bound highlights the importance of the base policy's \emph{coverage},
defined formally for a comparator policy $\pistar$ via\loose
\begin{align}
  \label{eq:coverage}
  \Cone[\pistar](x) \ldef{} \En_{y\sim\pistar(\cdot\mid{}x)}\brk*{\frac{\pistar(y\mid{}x)}{\piref(y\mid{}x)}}.
\end{align}

\begin{proposition}[Necessity of coverage]
  \label{prop:coverage}
  Fix a prompt $x\in\cX$,
  and\arxiv{ base policy} $\piref:\cX\to\Delta(\cY)$. For any alignment
  algorithm $\cA$ and any
  $16\leq{}\Cpar\leq\max_{\pi:\cX\to\Delta(\cY)}\Cone$, there exists a
  reward function $\rstar$ and reward model $\rhat$ satisfying \cref{eq:rm} and comparator
  policy $\pistar$ with $\Cone[\pistar]\leq\Cpar$ such that\loose
  \begin{align}
    \label{eq:coverage_lower}
    J(\pistar;x) - J(\pihat_\cA;x) \geq \frac{1}{4}\cdot\sqrt{\Cpar\cdot{}\vepsrms(x)}.
  \end{align}
\end{proposition}
Coverage in the sense of \cref{eq:coverage} plays a central role in the
theoretical study of offline alignment \citep{zhu2023principled,zhan2023provable,li2023reinforcement,xiong2024iterative,liu2024provably,cen2024value,
  fisch2024robust,ji2024selfplay,huang2024self,huang2024correcting}, and
\cref{prop:coverage} shows that it plays a similar role for
inference-time alignment.
Recent works show that standard language models can indeed exhibit favorable coverage
over desirable responses in standard tasks of interest, 
which translate to performance gains at inference time \citep{brown2024large,snell2024scaling,wu2024empirical}.
In the remainder of the paper, we will explore the extent to which\arxiv{ the skyline in \cref{eq:coverage}} \icml{this skyline }can be achieved---through
existing techniques, or through new interventions.\loose

\paragraph{Additional notation}
\arxiv{For $n\in\bbN$, we let $[n]$ denote the set
$\{1,\dots,n\}$. For a set $\cX$, we let $\Delta(\cX)$ denote the
set of all probability distributions over $\cX$.} We adopt standard
big-oh notation, and write  
$f\approxleq{}g$ as shorthand for $f=\bigoh(g)$ and $f=\bigoht(g)$ to denote 
$f = \bigoh(g\cdot{}\max\crl*{1,\mathrm{polylog}(g)})$. \loose

\arxiv{\section{Understanding \bonalglong: Guarantees and
    Limitations}}
\icml{\section{Understanding \bonalglong}}
\label{sec:bon}

As our \arxiv{first step toward understanding of how to best leverage computation for
inference-time alignment}\icml{first result}, we give a sharp analysis of 
\bonlong alignment, highlighting the role of coverage in determining its scaling
  properties and (sub-)  optimality.\loose

\subsection{A Sharp Analysis of \bonalglong}

\begin{algorithm}[tp]
  \caption{\bonalglong (\bonalg)}
  \label{alg:bon}
  \begin{minipage}{\linewidth}
  \begin{algorithmic}[1]
    \Statex[0] \mbox{{\bfseries input:}
      Prompt $x$, base policy $\piref$,
      reward model $\rhat$,
      query size $N$.}
    \State Draw $\Yhat =(y_1,\ldots,y_N)\sim \piref(\cdot\mid{}x)$ \iid
    \State Query $\rhat$ for reward labels $\rbr*{\rhat(x, y_1), \ldots, \rhat(x,y_N)}$ 
    \State \textbf{return} response $y = y_{\istar}$, where $\istar \in  \argmax_{i \in [N]}{} \rhat(x, y_{i})$. 
    \label{line:rej-noterm}
\end{algorithmic}
\end{minipage}
\end{algorithm}

First, in \cref{alg:bon} we display the \bonlong algorithm for an input prompt $x \in \cX$.
\bonalg draws $N$ candidate responses $y_1,\ldots,y_N\sim\piref(\cdot\mid{}x)$ \iid,
and returns the response $y \in \Yhat$ that has the largest reward under the reward model $\rhat$.

Next, for a fixed prompt $x$ and sample size parameter $N$,
let $\piN(x)\in\Delta(\cY)$ denote the policy corresponding to the distribution over
responses output by \bonalg,
which is a random variable depending on draws of candidate $\Yhat \sim \piref$.
\icml{which, given $x\in\cX$, draws $N$ responses
$y_1,\ldots,y_N\sim\piref(\cdot\mid{}x)$ \iid, and returns the
response $\yhat=\argmax_{y_i}\rhat(x,y_i)$.}
Our main guarantee for \bonalg bounds the regret in terms of the
sample size $N$,
the coverage coefficient $\Cone[\pistar](x)$,
and the reward model error $\vepsrms(x)$.\loose
\begin{theorem}[Guarantee for \bonalg]
  \label{thm:bon}
  For any prompt $x\in\cX$, reward model error $\vepsrms(x) \in (0, 1]$,
  and comparator\arxiv{ policy} $\pistar$, whenever
  $N \ge c \cdot \Cone[\pistar](x) \cdot \log\rbr*{\Rmax / \vepsrm(x)}$ for a sufficiently large constant $c$,
  the \bonalg policy $\piN$ satisfies
  \icml{  \begin{align}
    \label{eq:bon1}
            &J(\pistar;x) - J(\piN;x) \\
            &\lesssim 
    \Rmax \cdot
    \frac{\Cone[\pistar](x)\log\rbr*{\nicefrac{\Rmax}{\vepsrm(x)}}}{N}
    + \sqrt{N \cdot \vepsrms(x)}.
  \end{align}} 
  \arxiv{\begin{align}\label{eq:bon1}
    J(\pistar;x) - J(\piN;x)
    &\lesssim 
    \Rmax \cdot \frac{\Cone[\pistar](x)\cdot\log\rbr*{\nicefrac{\Rmax}{\vepsrm(x)}}}{N}
    + \sqrt{N \cdot \vepsrms(x)}.
  \end{align}} 
  In particular, as long as
  $N \asymp \rbr*{\frac{\Rmax\cdot\Cone[\pistar]\log\rbr*{\Rmax / \vepsrm(x)}}{\vepsrm(x)}}^{\frac{2}{3}}$,\arxiv{ we have that}\footnote{We use $N \asymp \square$ to indicate that $C_1\square \leq N \leq C_2\cdot
  \square$ for any sufficiently large absolute constants $C_1\leq{}C_2$. }
  \icml{\begin{align}\label{eq:bon2}
    &J(\pistar;x) - J(\piN;x) \\
    &\lesssim
        \prn*{\Rmax\cdot \Cone[\pistar](x)\cdot \vepsrms(x)\cdot\log\rbr*{\nicefrac{\Rmax}{\vepsrm(x)}}}^{1/3}.
  \end{align}}
  \arxiv{\begin{align}\label{eq:bon2}
           J(\pistar;x) - J(\piN;x) 
           \lesssim
        \prn*{\Rmax\cdot \Cone[\pistar](x)\cdot \vepsrms(x)\cdot\log\rbr*{\nicefrac{\Rmax}{\vepsrm(x)}}}^{1/3}.
  \end{align}}
\end{theorem}
The main regret bound in \cref{eq:bon1} has two terms of interest.
The first, which scales as roughly $\frac{\Cone[\pistar](x)}{N}$,
decreases as the sample size increases,
and reflects the extent to which the set of responses
$y_1,\ldots,y_N\sim\piref(\cdot\mid{}x)$
contains enough information to compete with $\pistar$.
Intuitively, as $N$ grows, there is a higher probability of
drawing a response that, under the true reward $\rstar$,
is at least as good as one from the comparator $\pistar$.
However, $\rstar$ is not available to the learner,
who instead evaluates response quality using the imperfect proxy $\rhat$.
There becomes a greater risk of overfitting to errors in $\rhat$ as $N$ increases,
and candidates are drawn from the tail of the base distribution
where $\rhat$ is more error-prone.
The cost of this overoptimization is expressed in the second term of \cref{eq:bon1},
$\sqrt{N\cdot\vepsrms(x)}$,
which increases with sample size and leads to arbitrarily large regret as $N$ is scaled.\loose

Because sample size is the only parameter in \bonalg,
it plays dual, but opposing, roles in both performing regularization (smaller $N$ to stay on-support)
and increasing response quality (larger $N$ to draw good responses).
The optimal choice of $N$ must balance gains in quality with risk of overoptimization
and be large but not \emph{too} large,
which leads to the second guarantee in \cref{thm:bon}.
\cref{eq:bon2} resembles the idealized lower bound in \cref{prop:coverage} but has a slower rate,
scaling with $\vepsrm^{2/3}(x)$ instead of $\vepsrm(x)$.\footnote{The bound in \cref{eq:bon2} is larger than the
  bound in \cref{prop:coverage} in the non-trivial regime where $(\Cone[\pistar](x)\cdot\vepsrms(x))^{1/2}\leq\Rmax$.\arxiv{  A secondary
  drawback is that the bound in
  \cref{eq:bon2} scales with the reward range $\Rmax$, which
  \cref{prop:coverage} does not.}\loose}
The following result shows that this dependence is in fact tight.\loose
\begin{theorem}[Lower bound for \bonalg]\label{thm:lower-bon}
  For any $\vepsrm\in(0,\tfrac{1}{4}]$ and $N \gtrsim 1$,
  there exists a problem instance\arxiv{\footnote{We use the term ``problem instance'' to
    refer to a tuple $(\piref, \rhat, \rstar)$ of base policy, reward
    model, and true reward function.}}
  with $\vepsrm(x) \le \vepsrm$
  and comparator policy $\pistar$ with $\Cone[\pistar](x) = \wt\bigoh(1)$,
  such that, for any $x \in \cX$,
  \bonalg has regret\loose
  \begin{align}\label{eq:lb1}
    J(\pistar\midsem x) - J(\piN\midsem x) \ge \wt\Omega \rbr*{\sqrt{N\cdot\vepsrms}}.
  \end{align}
Further, for all $N\in\bbN$,
there exists a problem instance\arxiv{ with $\vepsrm(x)\leq\vepsrm$
and comparator policy $\pistar$ with $\Cone[\pistar]=\bigoht(1)$ such that,
for any $x\in\cX$, \bonalg has regret}\icml{ satisfying the same conditions such that \bonalg has\loose}
\begin{align}\label{eq:lb2}
    J(\pistar\midsem x) - J(\piN\midsem x) \ge \wt\Omega \rbr*{\vepsrm^{2 / 3 } }~.
  \end{align}
\end{theorem}

\cref{eq:lb1} shows that the regret indeed grows as $N$ is scaled,
and is an algorithm-dependent lower bound that reflects the consequences of overfitting in \bonalg.
This is a serious concern since, when scaling $N$ in practice,
it may be impossible to know the sample size beyond which overoptimization occurs,
especially on a per-prompt basis. 
Then, building on \cref{eq:lb1},
the bound in \cref{eq:lb2} shows that $\vepsrm^{2/3}(x)$ is the best possible error achievable by \bonalg,
which matches the upper bound in \cref{eq:bon2}.
In other words, \emph{irrespective of the
  amount of computation, \bonalg may fail to achieve the optimal skyline for
  regret in \cref{prop:coverage}.}

The suboptimality stems from the dual role of $N$, which serves as the regularizer,
but is only a weak one at best.
In particular, $N$ cannot be safely increased to sample better candidates
without disproportionately increasing the risk of overfitting.
This insight motivates the algorithms we develop in \cref{sec:algorithm},
which mitigate overoptimization through a more refined form of regularization and use of  inference-time computation.

\begin{remark}[Proof technique]
  The lower bound in \cref{eq:lb1} is algorithm-specific,
  and the construction exposes the cost of overfitting to errors in $\rhat$,
  specifically, for responses with small probability under $\piref$
  that are drawn when $N$ is relatively large.
  To prove \cref{eq:lb2}, we first leverage an information-theoretic argument showing that if
  $N\ll\mathrm{poly}\rbr*{\Cone[\pistar](x),\,\frac{1}{\vepsrms(x)}}$,
  \emph{no algorithm} can extract enough information to compete with $\pistar$.
  Combining this regime with the one in \cref{eq:lb1} yields \cref{eq:lb2}.\loose
\end{remark}

\arxiv{\subsection{Stronger Guarantees for \bonalglong under Uniform
    Coverage}}
\icml{\subsection{Stronger Guarantees under Uniform Coverage}}

Per \cref{thm:lower-bon}, the regret of \bonalg must scale with $\vepsrm^{2 / 3}(x)$ in general.
However, it does enjoy tighter guarantees when a stronger form of coverage---the \emph{uniform coverage coefficient}---is bounded:
\icml{$\Cinf[\pistar](x)\ldef{}\sup_{y\in\cY}\frac{\pistar(y\mid{}x)}{\piref(y\mid{}x)}$.\loose}
\arxiv{\loose\[
\Cinf[\pistar](x)\ldef{}\sup_{y\in\cY}\frac{\pistar(y\mid{}x)}{\piref(y\mid{}x)}.
\]\loose}
\begin{theorem}[\arxiv{Guarantee for }\bonalg under uniform coverage]%
  \label{thm:bon_uniform}%
  For any $x\in\cX$ and comparator policy $\pistar$, if  $N
  \ge \Cinf[\pistar](x)$, the \bonalg policy $\piN$
  satisfies\loose
  \icml{\begin{align}
          & J(\pistar;x) - J(\piN;x) \\
          & \lesssim
    \Rmax \cdot \exp\prn*{ -\nicefrac{N}{\Cinf[\pistar](x)}}
    + \sqrt{N \cdot \vepsrms(x)} ~.
         \end{align}
         }
  \arxiv{\begin{align}
          J(\pistar;x) - J(\piN;x)
          & \lesssim
          \Rmax \cdot \exp\prn*{ -\nicefrac{N}{\Cinf[\pistar](x)}}
          + \sqrt{N \cdot \vepsrms(x)} ~.
         \end{align}
         }
  In particular, as long as $N \asymp \Cinf[\pistar]\log\rbr*{\nicefrac{\Rmax}{\vepsrm(x)}}$,\arxiv{ we have that}
  \icml{\begin{small}}
  \begin{align}
    J(\pistar;x) - J(\piN;x) 
    \lesssim 
    \sqrt{ \Cinf[\pistar](x) \cdot \vepsrms(x) \cdot\log\rbr*{\nicefrac{\Rmax}{\vepsrm(x)}}}.
  \end{align}
  \icml{\end{small}}
\end{theorem}

When we are willing to pay for uniform coverage,
the regret of \bonalg scales as $\sqrt{\vepsrms(x)}$,
which matches the statistical rate in the skyline of \cref{prop:coverage}.
This suggests that \bon may already be sufficient in some cases, at least when
  it is possible to tune $N$ optimally; we revisit
  this point empirically in \Cref{sec:experiments}.
Generally, however, we expect that $\Cinf[\pistar](x) \gg \Cone[\pistar](x)$,
making \cref{thm:bon_uniform} highly suboptimal relative to the skyline.
For example, softmax policies, which are normalized exponentials of the logits,
can have exponentially large $\Cinf[\pistar]$, while $\Cone[\pistar]$ is $\wt{O}(1)$.
This is (loosely) the case in the lower bound construction of \cref{thm:lower-bon},
where $\piref$ is an exponential policy,
and is upweighted by an exponential multiplier to form a $\pistar$ that is more sharply peaked on $\rstar$.

\icml{Lastly, \cref{sec:bon_iid} extends the results in this section
to the setting where prompts are drawn \iid from a distribution,
and shows that similar conclusions hold there for \bonalg with a fixed $N$.}

\arxiv{\section{An Optimal Algorithm: \mainalglong}}
\icml{\section{\mainalglong}}
\label{sec:algorithm}
We now present \mainalg, an optimal inference-time alignment method
that implements a statistically sound regularizer purely at inference time.
In contrast to \bonalg,\newline
\mainalg separates the scaling parameter ($N$) from the regularization parameter,
and is thereby able to achieve the performance skyline in \cref{prop:coverage}.
It is, as a result, also \emph{scaling-monotone}---as $N$ is increased,
it does not overfit or lose performance. 
We first give an \arxiv{overview of the
algorithm}\icml{algorithm overview}, then state the main theoretical guarantee.\loose

\subsection{The \mainalglong Algorithm}

\icml{\begin{algorithm}[tp]}
  \arxiv{\begin{algorithm}[tp]}
  \caption{\arxiv{\mainalglong (\mainalg)}\icml{\mainalg}}
  \label{alg:main}
  \begin{minipage}{\linewidth}
  \begin{algorithmic}[1]
    \Statex[0] \multiline{{\bfseries input:}
      Prompt $x$, 
      reference policy $\piref$, 
      reward model $\rhat$,
      query size $N$, 
      regularization \arxiv{coefficient}\icml{coeff.} $\beta>0$.}
    \State Draw $\Yhat\ldef{}(y_1,\ldots,y_N)\overset{\mathrm{\iid}}{\sim} \piref(\cdot\mid{}x)$.
    \State Using \normalg (\cref{alg:norm}), compute normalization
    constant $\lambdahat(x)$ such that
    \icml{\begin{small}}
    \begin{align}
      \label{eq:normalization_approx}
      \frac{1}{N}\sum_{ y \in \Yhat }\relu\rbr*{ \beta^{-1}\rbr*{ \rhat(x,y) - \lambdahat(x)} } = 1.
    \end{align}\icml{\end{small}}
    \State \mbox{Sample \icml{resp. }\arxiv{response }$y\sim \rejalg_{N,M}(\ahreplace{\pi}{w} \midsem{}\piref,x)$}
    (\cref{alg:rej}), where $M\ldef\frac{\Rmax-\lambdahat(x)}{\beta}$ and
    \icml{\begin{small}}
    \ahreplace{
      \begin{equation}
        \pi(y\mid{}x) \ldef \piref(y\mid{}x) \cdot \relu\rbr*{ \beta^{-1}\rbr*{ \rhat(x,y) - \lambdahat(x)} }.
      \end{equation}
    }{
      \begin{equation}
        w(y\mid{}x) \ldef \relu\rbr*{ \beta^{-1}\rbr*{ \rhat(x,y) - \lambdahat(x)} }.
      \end{equation} 
    }
    \icml{\end{small}}
    \State \textbf{return:}
    response $y$. 
\end{algorithmic}
\end{minipage}
\end{algorithm}

The \mainalg algorithm is displayed in \cref{alg:main}. For a
regularization parameter $\beta>0$, the algorithm is designed
to sample from the distribution
\icml{\begin{small}}
  \begin{align}
    \label{eq:chis_ideal}
    \pichis(y\mid{}x)\ldef{}\piref(y\mid{}x)\cdot\relu\prn*{
    \beta^{-1}(\rhat(x,y)-\lambda(x))
    },
  \end{align}%
\icml{\end{small}}%
where $\relu(z)\ldef{}\max\crl*{z,0}$, and $\lambda(x)$ is a
normalization constant chosen such that
\begin{align}
  \label{eq:normalization}
  \sum_{y\in\cY}\piref(y\mid{}x)\cdot\relu\prn*{
\beta^{-1}(\rhat(x,y)-\lambda(x))
  } = 1.
\end{align}
\arxiv{The distribution in \cref{eq:chis_ideal} is the exact solution to the
\emph{\chis-regularized reinforcement learning} objective,}
\icml{The distribution in \cref{eq:chis_ideal} is the exact solution to the
\emph{\chis-regularized reinforcement learning} objective: defining
$\Dchis{\pi(x)}{\piref(x)}=\frac{1}{2}\En_{y\sim\piref(\cdot\mid{}x)}\brk[\Big]{\prn[\big]{\frac{\pi(y\mid{}x)}{\piref(y\mid{}x)}-1}^2}=\frac{1}{2}\prn*{\Cone[\pi](x)-1}$
as the \chis-divergence, we have}
\begin{align}
  \label{eq:chis_objective}
  \pichis(x) = \argmax_{p\in\Delta(\cY)}\crl*{
  \En_{y\sim{}p}\brk*{\rhat(x,y)}
  - \beta\cdot\Dchis{p}{\piref(x)}
  }\arxiv{,}\icml{.}
\end{align}
\arxiv{which is regularized with the \chis-divergence,
$\Dchis{\pi(x)}{\piref(x)}=\frac{1}{2}\En_{y\sim\piref(\cdot\mid{}x)}\brk[\Big]{\prn[\big]{\frac{\pi(y\mid{}x)}{\piref(y\mid{}x)}-1}^2}=\frac{1}{2}\prn*{\Cone[\pi](x)-1}$.}
As we discuss in the sequel, the \chis-regularizer
adapts to the uncertainty in the reward model $\rhat(x,y)$,
so that the regularized policy in \cref{eq:chis_ideal} implements \emph{pessimism in the face of uncertainty},
a principle from offline reinforcement learning that carries strong guarantees
\citep{jin2021pessimism,rashidinejad2021bridging,huang2024correcting}.
In particular, \citet{huang2024correcting} showed that \chis-regularization at training time
can achieve the skyline in \cref{prop:coverage};
our method implements similar regularization but at inference-time, and attains those same guarantees. 

The pessimistic \chis-regularization prevents the algorithm from
overfitting to potentially misleading responses, e.g.,
those in the tail of the base distribution for which
$\piref(y\mid{}x)$ is small and $\rhat(x,y)$ is erroneously estimated to be large.
The parameter $\beta>0$ controls the degree of pessimism,
with small values inducing a greedier policy,
and large values inducing a conservative, heavy-tailed distribution over responses.
For any choice of $\beta$, however,
there is a set performance level the algorithm will never drop below, even as $N\to\infty$.\loose

To (approximately) sample from the optimal \chis-regularized policy in \cref{eq:chis_ideal},
\cref{alg:main} proceeds in two steps.
First, since the normalization constant $\lambda(x)$ in \cref{eq:normalization} is unknown,
the algorithm computes an approximate normalizer $\lambdahat(x)$.
It draws $N$ responses $\Yhat\ldef{}(y_1,\ldots,y_N)\overset{\mathrm{\iid}}{\sim} \piref(\cdot\mid{}x)$
and uses them to solve \cref{eq:normalization_approx},
an empirical approximation to \cref{eq:normalization}.
This is accomplished via the dynamic programming subroutine in \normalg (\cref{alg:norm}),
in $\bigoh(N\log N)$ time. \icml{See \cref{app:normalization} for details.}\arxiv{Further details are provided in \cref{app:normalization}.}

\arxiv{

\arxiv{\begin{algorithm}[tp]}
  \icml{\begin{algorithm}[tp]}
  \caption{\arxiv{Normalization constant computation (\normalg)}\icml{\normalg}}
  \label{alg:norm}
  \begin{minipage}{\linewidth}
  \begin{algorithmic}[1]
    \Statex[0] \multiline{{\bfseries input:}
      Prompt $x$, 
      reward model $\rhat$,
      regularization coefficient $\beta>0$,
      set of responses $\Yhat= (y_1,\ldots,y_N)$.
    }
    \State Sort and bucket $\Yhat$ into bins $\cY_{1},\ldots,\cY_{M}$ 
    according to the value of $\rhat(x,y)$, in ascending order, such that 
    \begin{align}
        \rhat(x,y) = \rhat_i, \quad\forall y \in \cY_{i},\forall i\in[M],\mathand
            \rhat_{i} < \rhat_{i+1}, \quad\quad\forall i\in[M]
    \end{align}
    \State Initialize $\rhat_0 = -\infty$, $J\leftarrow \sum_{i=1}^M \rhat_{i}\cdot\frac{\abr{\cY_{i}}}{N}$ and $Z\leftarrow 1$.
    \For{$i=1\ldots M$}
      \vspace{0.5em}
      \State Set $\lambda \leftarrow \frac{J - \beta}{Z}$.
      \vspace{0.5em}
      \If{$\rhat_{i-1} \leq \lambda < \rhat_{i}$ or $i = M$}
        \State \textbf{return} $\lambda$.
      \Else
        \State Update $J \leftarrow J -  \rhat_{i}\cdot\frac{\abr{\cY_{i}}}{N}$ and $Z \leftarrow Z - \frac{\abr{\cY_{i}}}{N}$.
      \EndIf
    \EndFor
\end{algorithmic}
\end{minipage}
\end{algorithm}

}

Next, \mainalg generates samples from (approximately) the policy in \cref{eq:chis_ideal} using
classical \emph{rejection sampling}
\citep{von1963various,block2023sample}\arxiv{, shown in \cref{alg:rej}}\icml{; see \cref{alg:rej} in \cref{app:rejection}}.
Given a rejection sampling threshold $M>0$, the algorithm draws another set of responses $\Yhat = y_1,\ldots,y_N$.
For each response $i$, it samples a Bernoulli random variable
$\xi_i\sim\Ber\prn*{\frac{\relu(\beta^{-1}(\rhat(x,y_i)-\lambdahat(x)))}{M}\wedge{}1}$,
and returns $y_i$ as the final response if $\xi_i=1$.
By the heavy-tailed nature of the idealized \chis-regularized distribution in \cref{eq:chis_ideal},
a rejection threshold of $M\approx\frac{\Rmax}{\beta}$ is sufficient.
Accounting for both the normalization constant computation and rejection sampling phase,
a total of $N=\bigoht\prn*{\frac{\Rmax}{\beta}}$ samples ensures that
the response distribution of \cref{alg:main} is a good approximation of \cref{eq:chis_ideal}.\loose

\arxiv{
\icml{\begin{algorithm}[htp]}
\arxiv{\begin{algorithm}[tp]}
  \caption{\rejalglong ($\rejalgNM$)}
  \label{alg:rej}
  \begin{minipage}{\linewidth}
  \begin{algorithmic}[1]
    \Statex[0] \mbox{{\bfseries input:}
      Prompt $x$, base policy $\piref$,
      \ahreplace{
      target policy $\pi$, 
      query sample size $N$,
      }{
        importance weight $w$, 
      } 
      truncation level $M$.} 
    \State Draw $\Yhat =(y_1,\ldots,y_N, y_{N+1})\sim \piref(\cdot\mid{}x)$ \iid
    \For{$i=1\ldots N$}
      \State Sample Bernoulli random variable $\xi_i$ such that 
      \ahreplace{
        $\Pr\rbr*{\xi_i = 1\mid{} y_i} = \min\cbr*{\frac{\pi(y_i\mid{}x)}{M \cdot \piref(y_i\mid{}x)}, 1}$
      }{
        $\Pr\rbr*{\xi_i = 1\mid{} y_i} = \min\cbr*{\frac{w(y_i\mid{}x)}{M}, 1}$
      }
      \If{$\xi_i=1$}
        \State \textbf{return} response $y=y_i$
        \label{line:rej-term}
      \EndIf
    \EndFor
    \State \textbf{return} response $y=y_{N+1}$. 
    \label{line:rej-noterm}
\end{algorithmic}
\end{minipage}
\end{algorithm}

}

\subsection{Theoretical Guarantees}

We now bound the regret for \mainalg,
which achieves the skyline rate of \cref{prop:coverage} with coverage coefficient $\Cone[\pistar](x)$.\loose
\begin{theorem}[Guarantee for \mainalg]
  \label{thm:main}
  For any $\beta > 0$ and $\vepsrms(x) \in (0, 1]$,
  if $N \geq c\cdot\frac{\Rmax}{ \beta}
  \log\rbr*{ \frac{\Rmax}{ \beta\cdot\vepsrm(x)} }$ for a sufficiently
    large \arxiv{absolute} constant $c$, \mainalg satisfies\loose
  \icml{  \begin{align}
    \label{eq:pes1}
            &J(\pistar ; x) - J(\pipes;x) \\&\lesssim \sqrt{ \Cone[\pistar](x)\cdot \vepsrms(x) } + \beta \cdot\Cone[\pistar](x) + \beta^{-1} \cdot\vepsrms(x).
  \end{align}}
\arxiv{  \begin{align}
    \label{eq:pes1}
    J(\pistar ; x) - J(\pipes;x) \lesssim \sqrt{ \Cone[\pistar](x)\cdot \vepsrms(x) } + \beta \cdot\Cone[\pistar](x) + \beta^{-1} \cdot\vepsrms(x).
  \end{align}}
Setting $\beta \asymp \sqrt{\frac{\vepsrms(x)}{\Cone[\pistar](x)}}$,
as long as $N \gtrsim \wt\Omega\rbr*{\sqrt{\frac{\Rmax^2\cdot \Cone[\pistar](x)}{\vepsrms(x)}}}$,
we have
\begin{equation}
  \label{eq:pes2}
    J(\pistar ; x) - J(\pipes;x) \lesssim \sqrt{ \Cone[\pistar](x) \cdot \vepsrms(x) }  .
  \end{equation}
\end{theorem}

\icml{\icml{\begin{figure*}[th]}
  \arxiv{\begin{figure*}[tp]}
    \centering
      \includegraphics[width=0.3\textwidth]{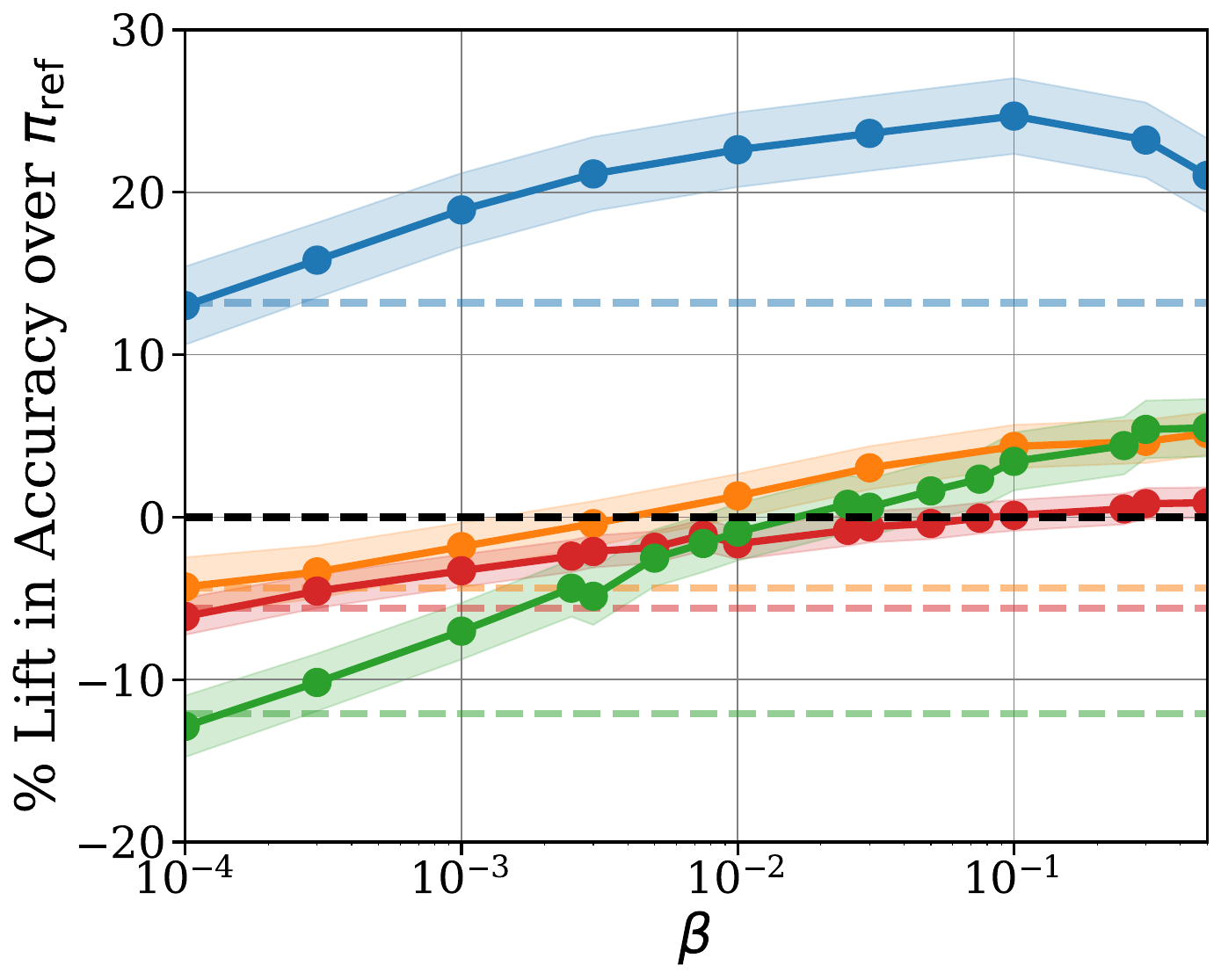}
    \hfill %
        \includegraphics[width=0.3\textwidth]{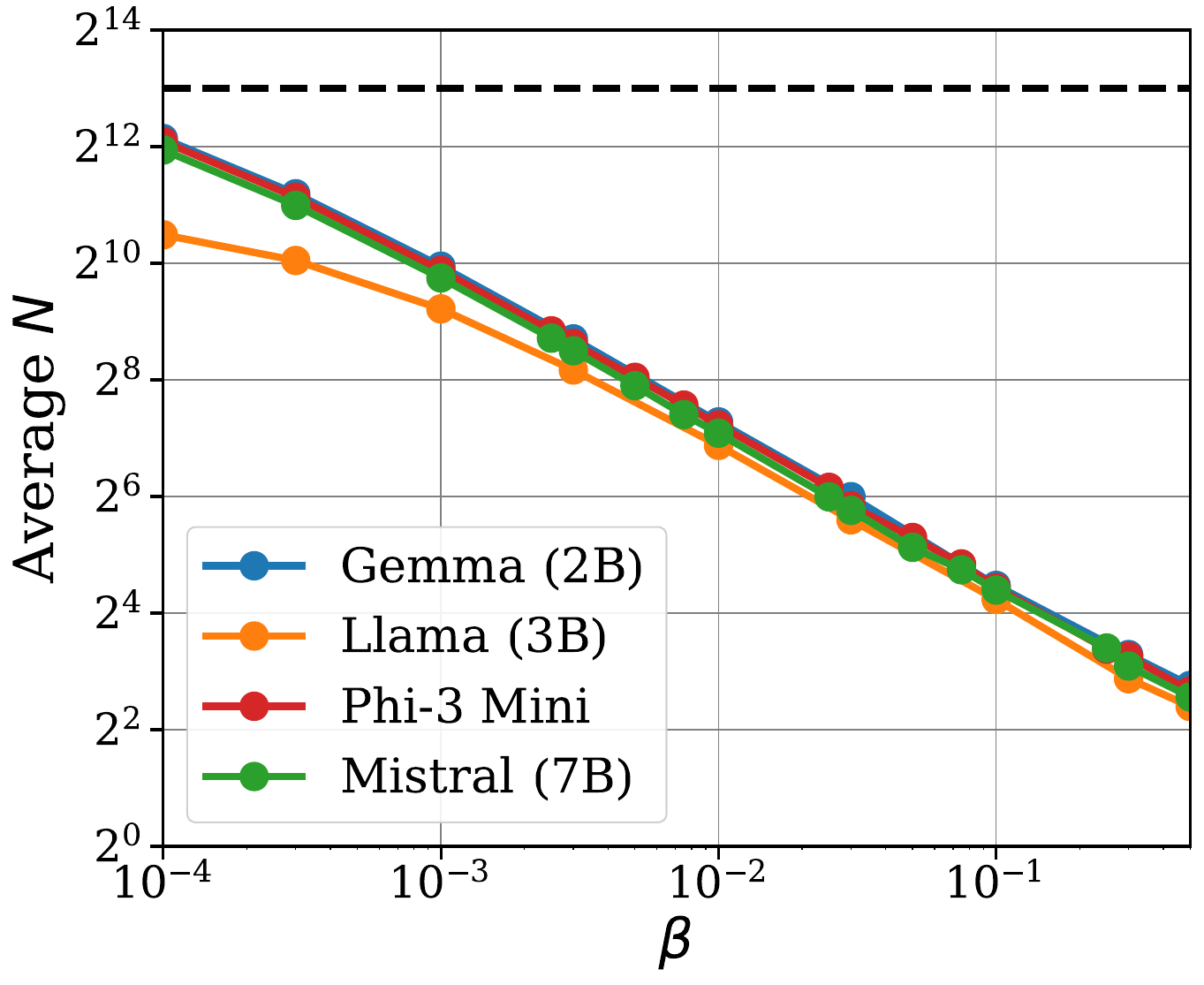}
    \hfill %
      \includegraphics[width=0.3\textwidth]{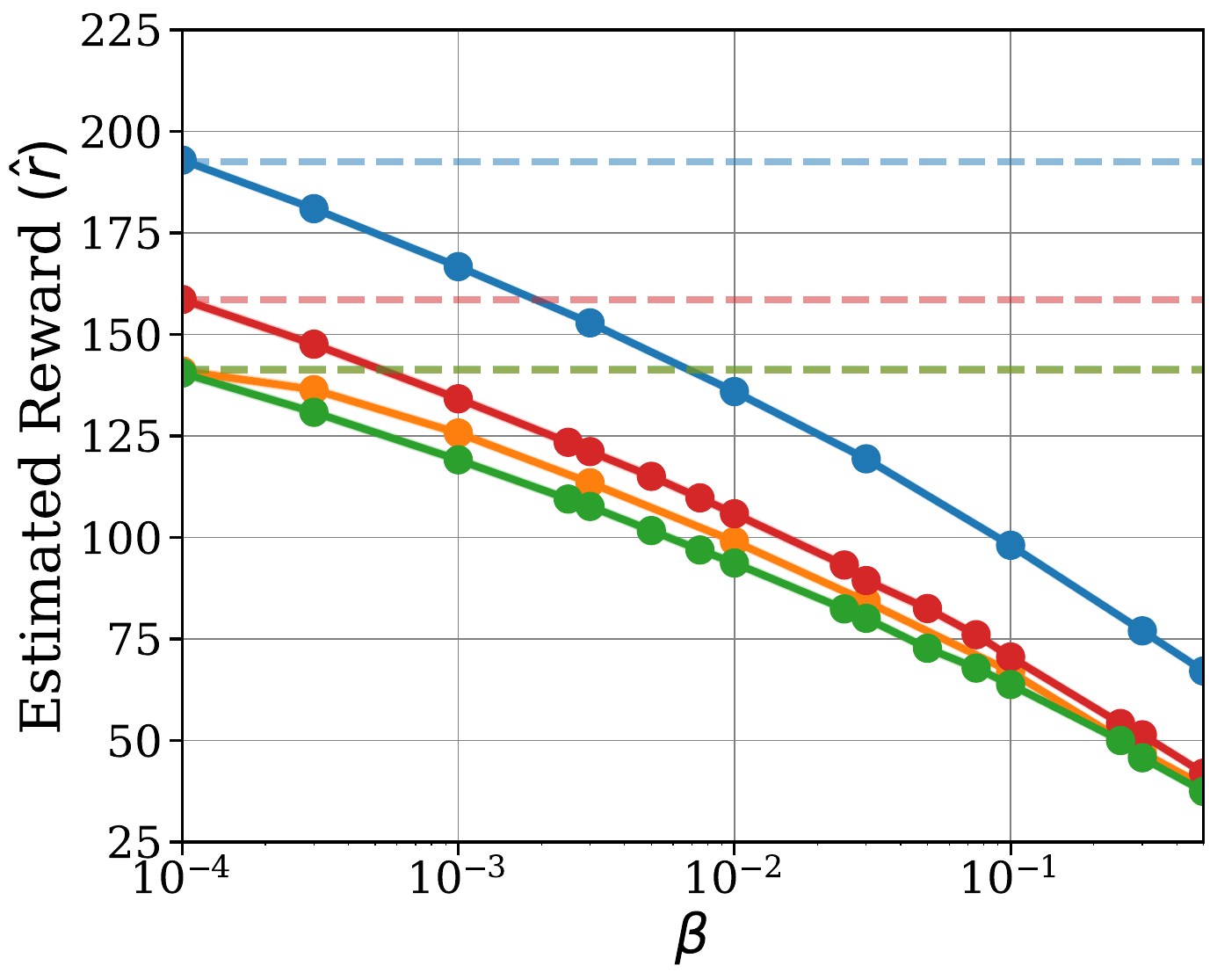}
    \icml{\vspace{-0.4cm}}
    \caption{Compute-normalized comparison of \mainalg (solid lines)
      and \bonalg (dashed lines) on \gsmk with \oasst as $\rhat$, as a
      function of regularization parameter $\beta$. We run \bonalg
      with $N=2^{13}$ and run \mainalg until
      rejection sampling accepts (capped to $N=2^{13}$).
      \Left: \arxiv{We see that }\mainalg can improve significantly
      over\arxiv{ na{\"i}ve} \bonalg for large $N$ due to the reward
      overoptimization.  \Center: Number of responses required for
      \mainalg to accept an answer decreases as $\beta$ increases, as predicted by our theory.  \Right: Estimated reward $\rhat$ for \mainalg decreases as $\beta$ increases.\loose}
      \vspace{-0.3cm}
    \label{fig:beta_varying}
  \end{figure*}

}

On the statistical side, for any choice of the regularization
parameter $\beta>0$, the regret bound in \cref{eq:pes1} balances
overoptimization (reflected in the term $\beta^{-1}\cdot\vepsrms(x)$)
with bias (reflected in the term
$\beta\cdot\Cone[\pistar](x)$). Choosing $\beta$ to balance these
terms leads to the regret bound in \cref{eq:pes2}, which matches the
lower bound in \cref{prop:coverage} up to absolute constants, showing
that \mainalg is \emph{regret-optimal}. 

Computationally, achieving the regret bound in \cref{eq:pes1} requires
$N\geq\bigomt\prn*{\frac{\Rmax}{\beta}}$. In contrast to
\bonlong (\cref{thm:bon}), \mainalg is robust to
overoptimization, in the sense that for any fixed $\beta>0$, the
guarantee in \cref{eq:pes1} holds \emph{for all $N$ sufficiently large}, and
there is no risk of dropping below the bound on the right-hand side of
 \cref{eq:pes1} as we scale computation; we refer to this property
 as \emph{\monotonicity}.\footnote{In practice, the performance of
   \mainalg may decay with $N$ if
   the algorithm initially out-performs the bound in \cref{eq:pes1},
   but it will never drop below the threshold in \cref{eq:pes1}.}

\paragraph{Lower bounds and compute-optimality}
 \mainalg requires
 $N \geq \bigomt\prn*{\sqrt{\Cone[\pistar](x)}\cdot
   \frac{\Rmax}{\vepsrm(x)}}$ samples
 to achieve the optimal regret bound \cref{eq:pes2},
 where $\beta \asymp \sqrt{\frac{\vepsrms(x)}{\Cone[\pistar](x)}}$.
 The following result shows that the $\vepsrm(x)^{-1}$ dependence is necessary for any algorithm in the sample-and-evaluate framework.
 \begin{theorem}[Query complexity lower bound]\label{thm:lower}
  
   For any $\vepsrm\in(0,1/4]$ and $N \lesssim \frac{1}{\vepsrm}$,
   there exists a problem instance with $\Rmax=1$,
   $\vepsrm(x)\leq\vepsrm$,
   and a comparator policy $\pistar$ with $\Cone[\pistar](x) = \Cone[\pistar] =\bigoht(1)$ such that
  any algorithm $\cA$ using at most $N$ \framework queries must have
    \begin{align}
      \icml{J(\pistar) - J(\pihat_\cA) = \wt\Omega\rbr*{1/N}.}
      \arxiv{J(\pistar) - J(\pihat_\cA) = \wt\Omega\rbr*{\frac{1}{N}}.}
    \end{align}
\end{theorem}
This result implies that \emph{any} algorithm achieving the lower bound in \cref{prop:coverage} requires
$N\approxgeq \frac{1}{\vepsrm(x)}$,
which is matched by the guarantee for \mainalg (\cref{thm:main}).
Moreover, when $N \approxleq \frac{1}{\vepsrm(x)}$,
the regret can be even larger than the reward estimation error $\vepsrm$,
since $\frac{1}{N} \ge \vepsrm$.
We remark that the
computational cost here is comparable (though slightly larger) than
the cost of achieving the sub-optimal bound in \cref{eq:bon2} using
\bonlong (roughly $N\approxgeq\vepsrm^{-1}(x)$ versus $N\approxgeq \vepsrm^{-2/3}(x)$).\loose

 \paragraph{Parameter tuning and practicality}
As presented in \cref{alg:main},
\mainalg has two parameters, $\beta$ and $N$, compared to the single parameter $N$ used by \bonalg.
This separation enables the optimal and scale-monotonic guarantees for \mainalg,
and we also view it as a beneficial feature from a practical perspective:
by using two parameters, we achieve a clear separation between tuning the computational budget (through $N$)
and tuning the statistical performance (through $\beta$).
While to achieve \cref{eq:pes2} $\beta$ must be chosen based on potentially unknown parameters
($\Cone[\pistar](x)$ and $\vepsrms(x)$),
this represents a meaningful improvement over \bonalg,
which cannot be tuned to achieve \cref{eq:pes2} even when these parameters are known (\cref{thm:lower-bon}).
This is because \bonalg
conflates computational and statistical considerations in its single parameter $N$.\loose

Empirically, we find that for any fixed choice of $\beta$,
\mainalg is robust to overfitting when computation and sample size is increased,
with performance essentially monotonic as $N$ grows (\cref{fig:beta_robustness}),
as predicted by the guarantee in \cref{eq:pes1}.
Because of this robustness, it is easier to tune the parameter $\beta$,
which we also find is important to achieve strong performance;
further discussion and practical guidance is provided in \cref{sec:experiments}.
Altogether, we believe it is most natural to interpret \mainalg as a
``single-parameter'' algorithm where only $\beta$ needs to be tuned,
and $N$ is as large as the computational budget allows.\loose

\icml{  \begin{table*}[t]
    \centering 
    \caption{Performance of $\piref=$\phimini (\% Lift in Accuracy over $\piref$).}
    \begin{small}
      \begin{tabular}{lcccc}
        \toprule 
        \textbf{Task} & \textbf{\oasst} & \textbf{\gemmarm} & \textbf{\llamarm} & \textbf{\armorm}  \\ \midrule 
        \gsmk (Pessimism) & $0.87 \pm 0.94$ & $5.61 \pm 0.92$ & $12.03 \pm 0.90$ & $12.44 \pm 0.88$ \\ 
        \gsmk (BoN) & $-5.61 \pm 1.13$ & $4.10 \pm 1.08$ & $12.12 \pm 0.95$ & $13.12 \pm 0.96$ \\ 
        \rowcolor{lightgray} \mmlu (Pessimism) & $-0.71 \pm 4.76$ & $14.29 \pm 5.24$ & $24.48 \pm 5.55$ & $21.12 \pm 5.58$ \\ 
        \rowcolor{lightgray} \mmlu (BoN) & $-5.61 \pm 5.50$ & $7.57 \pm 6.05$ & $25.41 \pm 6.10$ & $16.20 \pm 6.42$ \\ 
        \mathk (Pessimism) & $3.41 \pm 3.84$ & $18.36 \pm 4.02$ & $41.47 \pm 4.30$ & $26.72 \pm 3.98$ \\ 
        \mathk (BoN) & $3.32 \pm 3.98$ & $15.36 \pm 4.27$ & $41.74 \pm 4.35$ & $21.72 \pm 4.19$ \\ 
        \bottomrule
      \end{tabular}
    \end{small}
    \vspace{-0.35cm}
    \label{tab:phi3mini}
  \end{table*}
  
}

\paragraph{Overview of analysis}
\arxiv{The analysis of \mainalg}\icml{The proof of \cref{thm:main}}, and has three parts.
First, we show that the idealized \chis-regularized distribution in
\cref{eq:chis_ideal} achieves\icml{ the regret bound in
  \cref{eq:pes1}.}\arxiv{ the regret bound in
\cref{thm:main}, i.e. \icml{$J(\pistar ; x) - J(\pichis;x) \lesssim$}\loose
\begin{equation}
  \label{eq:chis_regret_ideal}
  \arxiv{J(\pistar ; x) - J(\pichis;x) \lesssim} \sqrt{ \Cone[\pistar](x)\cdot \vepsrms(x) } + \beta \cdot\Cone[\pistar](x) + \beta^{-1} \cdot\vepsrms(x).
  \end{equation}}
This follows the same reasoning as the analysis of training-time
interventions based on \chis-regularization in
\citet{huang2024correcting}, 
  and uses the
property that for any function $\Delta(x,y)$ (we use $\Delta(x,y)=\abs*{\rhat(x,y)-\rstar(x,y)}$) and
  policy $\pi$, \icml{$\En_{\pi}\brk*{\Delta(x,y)}\approxleq$} \loose \[\arxiv{\En_{\pi}\brk*{\Delta(x,y)}\approxleq}\sqrt{(1+\Dchis{\pi(x)}{\piref(x)})\cdot\En_{y\sim\piref(x)}\brk*{\Delta^2(x,y)}}.\]%
  Of course, we cannot sample from 
  $\pichis$ itself because the distribution depends on the ``true''
  normalization constant $\lambda(x)$ in \cref{eq:normalization}.
  To address this, we prove a robustness result showing that,
  given a $\lambdahat$ that approximately normalizes the distribution in~\cref{eq:chis_ideal},
the \arxiv{approximate \chis-regularized }policy
$\pichistil(y\mid{}x) = \frac{\piref(y\mid{}x)\cdot\relu\prn*{
\beta^{-1}(\rhat(x,y)-\lambdahat)}}{  \sum_{y'\in\cY}\piref(y'\mid{}x)\cdot\relu\prn*{
\beta^{-1}(\rhat(x,y')-\lambdahat)
  }}$ achieves a regret bound that matches \arxiv{\cref{eq:chis_regret_ideal}}\icml{\cref{eq:pes1}} up to
absolute constants. From here, a concentration argument
implies that the $\lambdahat(x)$ computed in
\cref{eq:normalization_approx} is an approximate normalizer whenever
$N\approxgeq \frac{\Rmax}{\beta}$, with high probability.\loose

Finally, we leverage analysis for rejection sampling to show that, as long as
$N\approxgeq
\max_{y\in\cY}\frac{\pichistil(y\mid{}x)}{\piref(y\mid{}x)}\log(1/\delta)$,
the rejection sampling procedure will terminate and return
$y\sim\pichistil(\cdot\mid{}x)$ with probability at least $1-\delta$.
Critically, due to the heavy-tailed nature of the \chis-regularizer,
this density ratio is bounded as
$\frac{\pichistil(y\mid{}x)}{\piref(y\mid{}x)}\approxleq\frac{\Rmax}{\beta}$,
  which yields the claimed query complexity bound. This observation highlights an important computational benefit of
  \chis-regularization that goes beyond its 
  statistical benefits, illustrated below.\loose

  \begin{remark}[Comparison to KL-regularization]
    \citet{liu2023statistical,li2024cascade}
    use  
    rejection sampling to simulate samples from the
    KL-regularized distribution: \[\pikl(y\mid{}x)\ldef \argmax_{p\in\Delta(\cY)}\crl*{
  \En_{y\sim{}p}\brk*{\rhat(x,y)}
  - \beta\cdot\Dkl{p}{\piref(x)}
    },\] which satisfies
\arxiv{\loose\[
  \pikl(y\mid{}x)\propto\piref(y\mid{}x)\cdot\exp\prn*{\frac{\rhat(x,y)}{\beta}}.
\]}
This has two issues. First, as shown by \citet{huang2024correcting},
this distribution can fail to achieve the guarantee in \cref{eq:pes2},
no matter how $\beta$ is chosen. Second, the density ratio is
exponential in general, i.e.
$\frac{\pikl(y\mid{}x)}{\piref(y\mid{}x)}\geq\exp\prn*{\frac{\Rmax}{\beta}}$,
which means that $N\approxgeq \exp\prn*{\frac{\Rmax}{\beta}}$
\framework queries are required to simulate it with rejection sampling.
\end{remark}

\section{Experiments}
\label{sec:experiments}

\arxiv{}

\arxiv{  \begin{table*}[t]
    \centering 
    \caption{Performance of $\piref=$\phimini (\% Lift in Accuracy over $\piref$).}
    \begin{small}
      \begin{tabular}{lcccc}
        \toprule 
        \textbf{Task} & \textbf{\oasst} & \textbf{\gemmarm} & \textbf{\llamarm} & \textbf{\armorm}  \\ \midrule 
        \gsmk (Pessimism) & $0.87 \pm 0.94$ & $5.61 \pm 0.92$ & $12.03 \pm 0.90$ & $12.44 \pm 0.88$ \\ 
        \gsmk (BoN) & $-5.61 \pm 1.13$ & $4.10 \pm 1.08$ & $12.12 \pm 0.95$ & $13.12 \pm 0.96$ \\ 
        \rowcolor{lightgray} \mmlu (Pessimism) & $-0.71 \pm 4.76$ & $14.29 \pm 5.24$ & $24.48 \pm 5.55$ & $21.12 \pm 5.58$ \\ 
        \rowcolor{lightgray} \mmlu (BoN) & $-5.61 \pm 5.50$ & $7.57 \pm 6.05$ & $25.41 \pm 6.10$ & $16.20 \pm 6.42$ \\ 
        \mathk (Pessimism) & $3.41 \pm 3.84$ & $18.36 \pm 4.02$ & $41.47 \pm 4.30$ & $26.72 \pm 3.98$ \\ 
        \mathk (BoN) & $3.32 \pm 3.98$ & $15.36 \pm 4.27$ & $41.74 \pm 4.35$ & $21.72 \pm 4.19$ \\ 
        \bottomrule
      \end{tabular}
    \end{small}
    \vspace{-0.35cm}
    \label{tab:phi3mini}
  \end{table*}
  
}

In this section, we complement our theoretical results with a suite of experiments that investigate the practicality of \mainalg,
and compare its performance to that of \bonalg. 
We consider three standard tasks: the test split of the \arxiv{well-known }elementary school math dataset \gsmk \citep{cobbe2021training}; math and chemistry splits of \mmlu \citep{hendrycks2020measuring}, and the test split of the advanced math problems dataset \mathk \citep{hendrycks2021measuring}. We also present a preliminary study with \arxiv{the \alpaca task}\icml{\alpaca} \citep{alpaca_eval} in \cref{app:experiments}.  We consider four reward models in increasing order of size: \oasst (1.4B) \citep{kopf2024openassistant}, \gemmarm (2B) \citep{dong2023raft}, \llamarm (3B) \citep{yang2024regularizing}, and \armorm (7B) \citep{ArmoRM}.  Finally, for each task we consider a subset of four policies for the base model: \gemma \citep{team2024gemma}, \llama \citep{dubey2024llama}, \mistral \citep{jiang2023mistral}, and \phimini \citep{abdin2024phi}.  For each task-policy-reward triplet, and each prompt, we generate a large number of responses (20K) and conduct \arxiv{our }experiments by bootstrapping subsamples of this large set.  In all cases, we reuse the same samples for normalization constant estimation as for the rejection sampling step. \loose

To produce \Cref{fig:intro}, we compare the performance of \mainalg (with tuned $\beta$) and \bonalg for a range of $N$ in terms of both true reward (accuracy) and estimated reward ($\rhat$ is \oasst) on \gsmk, defaulting to \bonalg if \mainalg does not terminate.  As we scale $N$, we observe the characteristic dip (e.g., \citet{gao2023scaling}) in accuracy of \bonalg in the left panel; from the right panel, we can infer that this is caused by reward overoptimization.
On the other hand, we find that the accuracy of \mainalg monotonically increases with $N$, as predicted by \cref{thm:main}. To further investigate the effect that regularization has on \mainalg, in \Cref{fig:beta_varying}, we fix a compute budget of $N = 2^{13}$ and compare the performance of \bonalg to \mainalg as we vary the regularization parameter $\beta$.  As we increase $\beta$, we see that \mainalg leads to improved true reward (Left), a smaller computational budget (Center) and significantly less reward overoptimization (Right).\loose

\Cref{tab:phi3mini} collects similar results across all tasks \gsmk, \mmlu, and \mathk and reward models \oasst, \gemmarm, \llamarm, and \armorm, with \phimini as the base policy $\piref$. We use a fixed computational budget, and compare the na{\"i}ve \bonalg for $N = 2^{13}$ with \mainalg for the best $\beta$ (see \Cref{app:experiments} for further details). Here, we find that \mainalg tends to have higher average performance than \bonalg, although in many instances this difference is not statistically significant; we suspect this is because, when $\rstar$ is binary (as it is in all tasks we use), there is no separation between $\Cone[\pistar]$ and $\Cinf[\pistar]$ when $\pistar$ is the optimal policy (uniform over the set of correct answers); thus, we are in the regime where \cref{thm:bon_uniform} predicts near-optimal performance for \bon.
However, a different story may emerge under more refined evaluation metrics,
such as the correctness of proofs in addition to the final answer.
Here we should expect $\Cinf[\pistar] \gg \Cone[\pistar]$, and
we leave evaluation in more realistic environments to future work.
For the sake of space, we defer further empirical results to \cref{app:experiments}, including plots and tables analogous to \Cref{fig:intro,fig:beta_varying} and \Cref{tab:phi3mini} for other policies, tasks, and rewards (\cref{ssec:more_policies}) as well as additional experiments (\cref{ssec:further_exp}).\loose

\section{Conclusion}
\label{sec:conclusion}
Our results reveal the interplay between coverage, scaling, and
optimality in inference-time alignment,
\arxiv{and highlight the benefits of scaling computation at inference-time in more deliberate and sophisticated forms.}
\icml{and highlight the benefits of deliberate compute scaling.}
\arxiv{We believe that, in addition to the sample-and-evaluate framework,
the following technical contributions may be of greater interest,
e.g., for drawing samples from other desirable distributions at inference time:
(i) our methods for efficiently approximating samples from regularized distributions,
(ii) lower bounds on query complexity, and
(iii) regret analysis based on approximate rejection sampling.
}
Beyond providing optimal
algorithms (\mainalg) and insights into the performance of \bonalg,
our framework can serve as starting point toward a foundational
understanding of inference-time computation more broadly. \icml{In particular, our work raises a number of interesting
directions for future research, including moving beyond the worst-case assumption on $\rhat$ and designing inference-aware training procedures that optimize for \mainalg at generation time.}
\arxiv{In
particular, our work raises a number of interesting
directions for future research:
\begin{itemize}
\item \emph{Opening the black box.} Our formulation of the
  inference-time alignment problem treats the base policy $\piref$ as a
  black box, and does not take advantage of any specific properties of
  the policy outside of coverage. An important direction is to
  understand whether we can open this black box and use more refined
  properties of the policy (e.g., the learned representations) to
  derive improved inference-time alignment schemes.\loose
\item \emph{Co-design of training-time and inference-time alignment
    algorithms.} Recent work
  \citep{balashankar2024infalign,chow2024inference} has explored the
  design of ``inference-aware'' training-time procedures that train a policy attempting to optimize \bonlong performance at
  inference-time. Since our work shows that inference-time \bonlong can be suboptimal in general, a natural direction is to co-design better
  inference-time and training-time interventions in parallel.
\item \emph{Inference-time exploration.} Our formulation---where
  the learning algorithm must make decisions based on an imperfect
  reward model---is inspired by the offline reinforcement learning
  framework. An interesting direction, inspired by \emph{online}
  reinforcement learning, is whether one can develop inference-time algorithms
  that leverage exploration to learn as quickly as possible in the
  online setting where the learner has access to the true reward
  function or a high-quality verifier.
\end{itemize}
}

\newpage

\icml{
\subsection*{Impact Statement}
This paper presents work whose goal is to advance the field of Machine
Learning. There are many potential societal consequences of our work,
none which we feel must be specifically highlighted here.
}

\bibliography{refs}
\icml{
\bibliographystyle{icml2025}
}

\clearpage

\appendix

\icml{
\onecolumn
}

\renewcommand{\contentsname}{Contents of Appendix}

\renewcommand{\contentsname}{Contents of Appendix}
\addtocontents{toc}{\protect\setcounter{tocdepth}{2}}
{\hypersetup{hidelinks}
\tableofcontents
}
\newpage
\crefalias{section}{appendix} 
\part{Additional Discussion and Results}

\section{Additional Related Work}
\label{sec:related}

In this section, we discuss additional related work in greater detail.

\paragraph{Theoretical analysis of inference-time alignment}
Inference-time alignment has received limited theoretical
investigation so far. Most notably, various works have analyzed
specific properties of the \bonlong alignment algorithm
\citep{yang2024asymptotics,beirami2024theoretical,mroueh2024information}
such as tradeoffs between reward and KL-divergence, but do not
ultimately provide guarantees on downstream performance in the
presence of mismatch between the estimated reward model and true
reward. Our theoretical framework---which abstracts the role of the
base policy through \framework access---is inspired by
\citet{huang2024self}, who used a similar framework to give guarantees
for the complementary problem of language model self-improvement, but
our specific problem formulation and techniques are quite different.

\paragraph{Empirical algorithms for inference-time alignment}
Empirically,
the \bonlong alignment heuristic
\citep{stiennon2020learning,nakano2021webgpt,touvron2023llama,gao2023scaling,eisenstein2023helping,mudgal2024controlled}
is perhaps the most widely used inference-time alignment heuristic; specific works that have observed the overoptimization phenomenon for
\bonlong include \citet[Figure 1]{gao2023scaling}, \citet[Figure
3]{chow2024inference}, \citet[Figure 7]{frick2024evaluate}, and \citet{stroebl2024inference}.
There are also other algorithms based on more
sophisticated variants of \bonlong or other techniques such as
rejection sampling \citep{liu2023statistical,chen2024pad,chakraborty2024transfer,xu2024genarm,
  shi2024decoding,qiu2024treebon,jinnai2024regularized,zhao2024probabilistic},
though few of these works are explicitly designed to address these
issue of over-optimization. For example, most algorithms based on inference-time search, such as
Monte-Carlo Tree Search (MCTS) and
relatives \citep{feng2023alphazero,yao2024tree,zhang2024rest}, are designed with the complementary goal of maximizing a
fixed reward function of interest given exact access, and do not account for
overoptimization or mismatch between the reward function and task performance. Specific algorithms that
make use of rejection sampling include
\citet{liu2023statistical,li2024cascade,xiong2024iterative,zhao2024probabilistic,khaki2024rs},
though the specific distributions these works aim to sample from are
quite different from that used in \mainalg. See also~\citet{welleck2024decoding} for a survey of inference-time algorithms. \loose

Other related but complementary line of work include (1) distilling
inference-time procedures into policies, thereby giving training time
procedures
\citep{amini2024variational,sessa2024bond,gui2024bonbon,pace2024west},
and (2) designing ``inference-aware'' training procedures which change the
training process to optimize performance of downstream inference-time
such as \bonlong \citep{balashankar2024infalign,chow2024inference}.

  \subsection{Connection to Offline (Training-Time) Alignment}
  \label{sec:offline}

As discussed in \cref{sec:related_body}, our problem formulation and
algorithms are closed related to a growing body of research on
theoretical algorithms for \emph{offline alignment} \citep{zhu2023principled,zhan2023provable,li2023reinforcement,xiong2024iterative,liu2024provably,cen2024value,
  fisch2024robust,ji2024selfplay,huang2024correcting,rashidinejad2024sail},
which give training-time interventions that enjoy robustness to reward
model overoptimization under various notions of coverage. In
particular, our \mainalg algorithm can be viewed as implementing an
``idealized'' version of the \chis-regularized RLHF algorithms
introduced by \citet{huang2024correcting} at inference-time.

Our formulation of inference-time alignment can be viewed
as a variant of the offline alignment problem that abstracts away the
process of training the reward model. Rather than concerning ourselves
with the details of training $\rhat$ from $\cD_{\pref}$ to minimize
\cref{eq:rm}, we take $\rhat$ as a given (in the process, abstracting
away the dataset $\cD_{\pref}$), and ask how to achieve the best
possible regret on a \emph{per-instance} basis, both with respect to
the reward model $\rhat$ itself, and with respect to the prompt
$x\in\cX$ (which is arbitrary and fixed, rather than assumed to be
\iid as $x\sim\rho$). Naturally, our algorithms and analyses can be
combined with any reward estimation procedure that minimizes
$\En_{x\sim\rho}\brk*{\vepsrms(x)}$ from $\cD_{\pref}$ to derive
end-to-end sample complexity guarantees for offline alignment.

\paragraph{Online alignment}
A complementary line of theoretical research which is somewhat less related to our work
studies alignment with \emph{online feedback}
\citep{xu2020preference,novoseller2020dueling,pacchiano2021dueling,wu2023making,zhan2023query,chen2022human,wang2023rlhf,du2024exploration,das2024provably,ye2024theoretical,xie2024exploratory,cen2024value,
  xiong2024iterative,gao2024rebel,chang2024dataset,song2024understanding},
where feedback from the true reward model $\rstar(x,y)$ is available.\loose

\section{Further Empirical Results}
\label{app:experiments}
In this section, we expand on the experiments discussed in \Cref{sec:experiments}.  We begin by providing a complete description of our expeirmental setup in \Cref{ssec:exp_details}, before proceeding to expand the breadth of empirical results reported in the main body to more policies, tasks, and reward models in \Cref{ssec:more_policies}.  We continue by conducting a further investigation into the robustness of \mainalg to the regularization parameter $\beta$ as well as a distributional study of the estimated rewards $\rhat$ sampled from $\piref$ in \Cref{ssec:further_exp}.  Finally, we present preliminary results for the \alpaca task in \Cref{ssec:alpaca}.

\subsection{Further Experimental Details}\label{ssec:exp_details}

As we summarized in the main text, we conduct an extensive empirical suite by considering many tasks, reference policies, and estimated reward models.  We now detail each of these in turn.  The four tasks we consider are the following:
\begin{enumerate}
    \item \gsmk: We consider the test split of the popular
      grade-school math dataset introduced in
      \citet{cobbe2021training}.  This dataset consists of about 1K
      short math word problems.  We prompt all of our policies with
      Chain of Thought (CoT) prompting \citep{wei2022chain} but do not
      include any example demonstrations, i.e., we are zero-shot.  We
      measure correctness in the sense that we assign $\rstar(x,y)=1$
      if the resulting policy gets the correct mathematical answer and
      $\rstar(x,y)=0$ otherwise.
    \item \mmlu: We consider the college math and college chemistry
      splits of the \mmlu dataset introduced in
      \citet{hendrycks2020measuring}.  This dataset consists about 100
      questions each of math and chemistry at the college level with
      multiple choice answers.  Again, we use CoT with zero-shot prompting for all policies.  Correctness is measured in the sense that the resulting policy gets the correct multiple-choice answer.
    \item \mathk: We consider the randomly sampled set of 512 questions from the test split of the \mathk dataset introduced in \citet{hendrycks2021measuring}.  This dataset consists of hard mathematics problems.  As above, we use CoT and zero-shot prompting, with correctness measured in the sense that the resulting policy gets the correct mathematical answer.
    \item \alpaca: For a small subset of our policies and rewards, we
      consider the \alpaca task introduced in \citet{alpaca_eval},
      with 128 randomly sampled questions.  This task is a challenging
      LM benchmark where we compare a policy's generation to that of a
      benchmark LM, and define $\rstar$ according to \emph{win rate}
      against an
      evaluator LM.  In order to collect a denser signal, we compare
      win rate against generations sampled from $\piref$ for each
      $\piref$ we evaluate.  We use GPT-4o-mini \citep{openai2024gpt4omini} as our evaluator.
\end{enumerate}
For each of our tasks, we consider a subset of the following four reward models for use as the estimated reward $\rhat$:
\begin{enumerate}
    \item \oasst: the OpenAssistant reward model based on Pythia-1.4b \citep{kopf2024openassistant}.
    \item \gemmarm: a reward model based on Gemma-2-2b \citep{dong2023raft}.
    \item \llamarm: a reward model based on Llama-3-3b \citep{yang2024regularizing}.
    \item \armorm: a reward model based on Llama-3-8b  \citep{ArmoRM}.
\end{enumerate}
Finally, we consider the following policies for $\piref$:
\begin{enumerate}
    \item \gemma: the Gemma-2-2b model introduced in \citet{team2024gemma}.
    \item \llama: the Llama-3-3b model introduced in \citet{dubey2024llama}.
    \item \mistral: the Mistral-7b model introduced in \citet{jiang2023mistral}.
    \item \phimini: the Phi-3-mini model (3.8b parameters) introduced in \citet{abdin2024phi}.
    \item \phismall: the Phi-3-small model (7b parameters) introduced in \citet{abdin2024phi}.
\end{enumerate}
In all of our experiments, for each prompt in each task and each
policy, we generate about 20K responses sampled with temperature 1
from the chosen $\piref$.  For a given number $M$ of replicates ($M=50$ in all tasks except for \alpaca, where $M=5$ due to resource constraints), we then bootstrap $M$ subsets of $N$ samples each 
from this large set and run our algorithm on these subsampled responses.
We define the \emph{accuracy} of a given algorithm on a prompt as the
average number of correct answers produced over the number of
replicates; an exception to this is \alpaca, where we measure the
\emph{win rate} according to the evaluator LM.  The reported accuracy of a policy is given by the average accuracy over all prompts in the task, and the standard error is estimated by marginalizing over the prompts in the task.  As stated in the main body, in all cases for fixed $N$, we use the same $N$ samples to estimate the normalization constant (\Cref{alg:norm}) as we do to run the rejection sampling.

In what follows, we first present additional plots
for the experiments described in \cref{sec:experiments}, omitted from
the main body for the sake of space (\cref{ssec:more_policies}), then  describe results of additional experiments
(\cref{ssec:further_exp,ssec:alpaca}).

\subsection{Results for Further Policies and
  Tasks}\label{ssec:more_policies}

We complement \Cref{fig:intro,fig:beta_varying} as well as
\Cref{tab:phi3mini} with analogous figures and tables for the
remaining tasks, reward models, and policies described above.  First,
we investigate how the compute budget $N$ affects performance and
estimated reward of the response in \gsmk (\Cref{fig:gsm8k-ns}), \mmlu
(\Cref{fig:mmlu-ns}), and \mathk (\Cref{fig:math-ns}).  In all
cases, we see that \mainalg is essentially monotonic in compute
budget, as predicted by our theory.  In some cases, we see
\emph{monotonicity} in the performance of \bonalg, which is consistent
with the observation (e.g., \cref{fig:histograms}) that some (task,
policy) pairs appear to be more in-distribution for some rewards (such
as \armorm) than others; for such cases, we expect \bonalg to perform well.\loose

We also display the effect of the regularization parameter $\beta$ on
performance, average compute, and estimated $\rhat$ for each of \gsmk
(\Cref{fig:gsm8k-betas}), \mmlu (\Cref{fig:mmlu-betas}), and \mathk
(\Cref{fig:math-betas}); we again compare our performance to \bonalg with the same compute budget of $N = 2^{13}$.  To be precise, we fix $N$ and compare the na{\"ive} \bonalg approach with this large $N$ to \mainalg where we use all $N$ samples to estimate the normalization constant and then run rejection sampling until acceptance for each fixed $\beta$ and prompt; in this case all prompts across all tasks, policies, reward models, and $\beta$'s terminate within the $N$ samples.
As discussed in \Cref{sec:experiments}, increasing $\beta$ leads to a smaller average required responses before rejection sampling terminates as well as a smaller estimated reward $\rhat$.  \loose

Finally, we display analogues of \Cref{tab:phi3mini} for the remaining
policies we consider: For each of \phismall (\Cref{tab:phi3small}),
\mistral (\Cref{tab:mistral}), and \llama (\Cref{tab:llama}), we
compare the performance of compute-normalized \bonalg with $N =
2^{13}$ to \mainalg on \gsmk, \mmlu, and \mathk (with the exception of
\llama, where we only consider the first two tasks) for our four
reward models.  We continue to find that the performance of \bonalg with properly tuned $N$ and \mainalg is similar, which we believe is caused by the fact that $\Cone[\pistar] = \Cinf[\pistar]$ when $\rstar$ is binary and $\pistar$ is uniform over the set of correct answers; by \cref{thm:bon_uniform}, \bonalg will perform near-optimally in this regime.  \loose

\begin{figure*}[htp]
    \centering
    \subfigure[\oasst]{
        \includegraphics[width=0.2\textwidth]{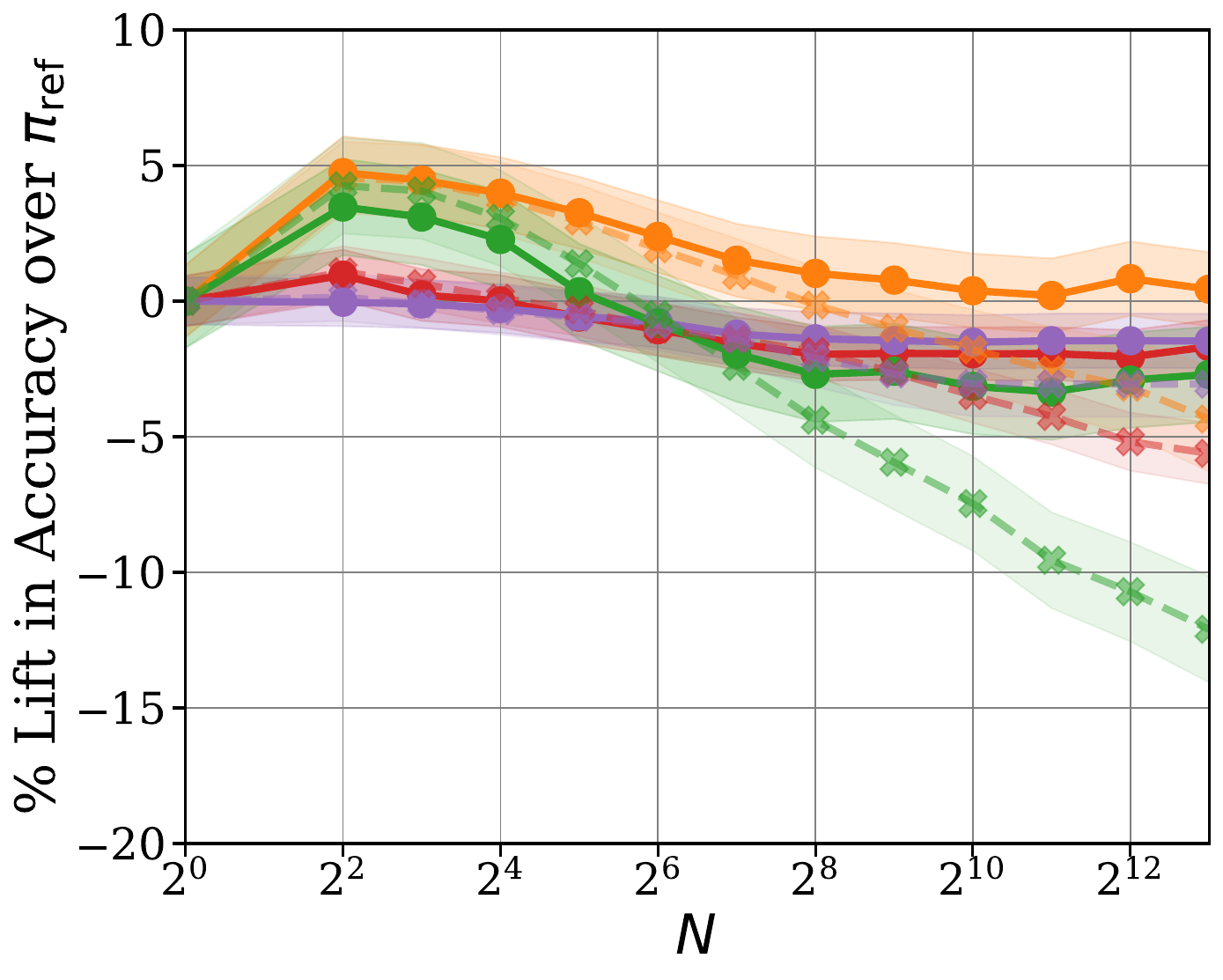}
        \label{sfig:n-gsm8k-oasst-rm} 
      }
   \hfill \subfigure[\gemmarm]{
        \includegraphics[width=0.2\textwidth]{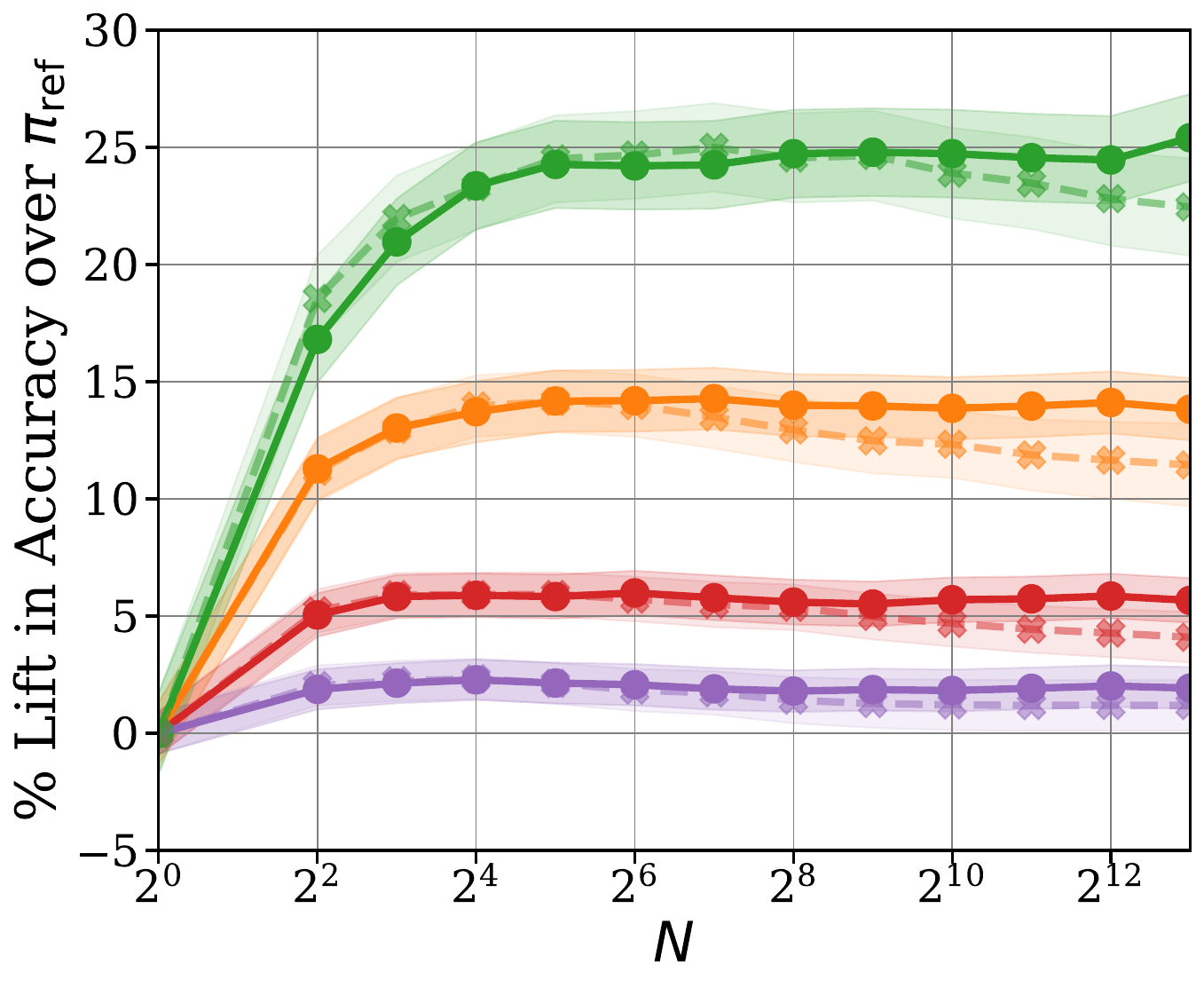}
        \label{sfig:n-gsm8k-gemma-rm} 
      }
      \hfill \subfigure[\llamarm]{
        \includegraphics[width=0.2\textwidth]{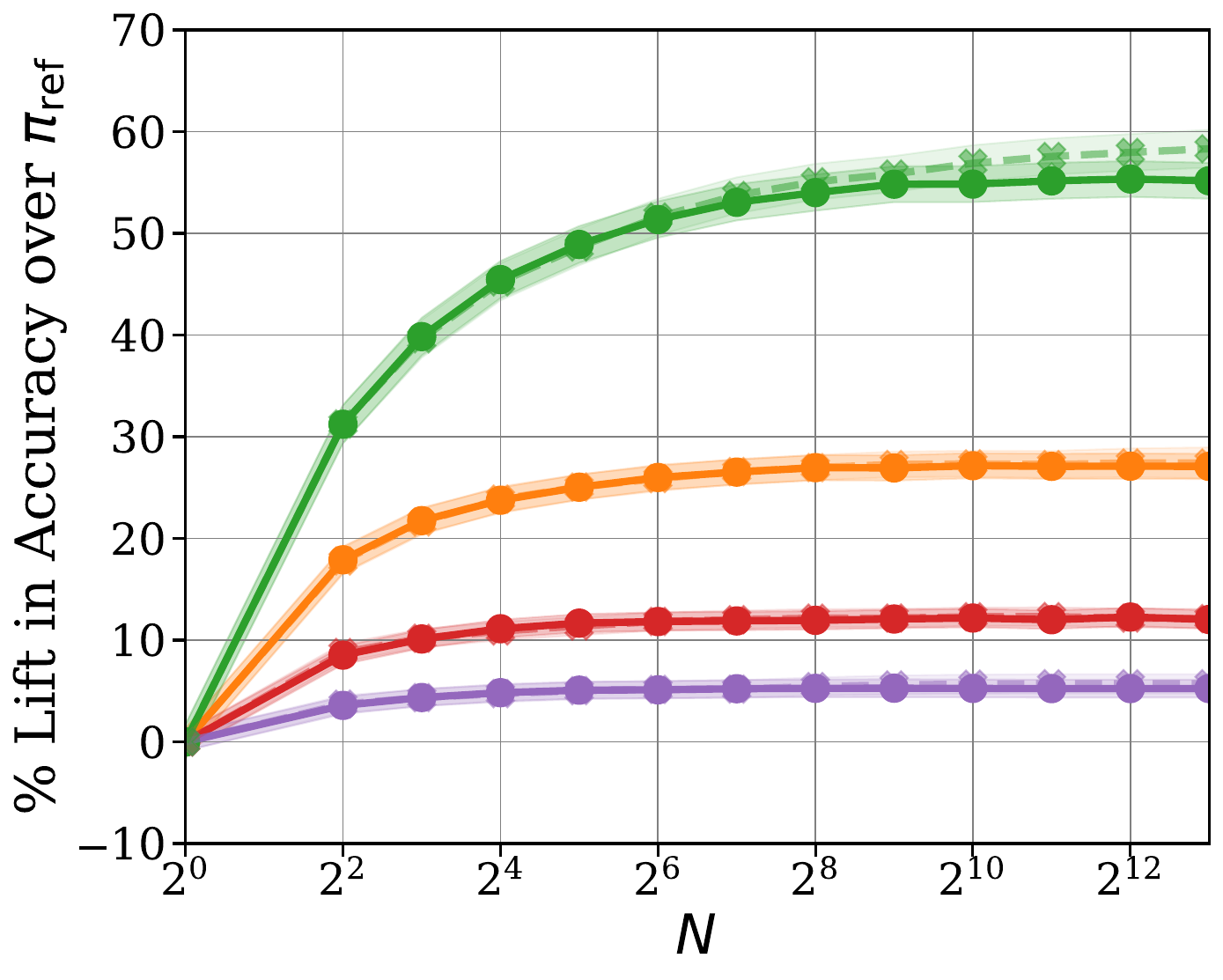}
        \label{sfig:n-gsm8k-llama-3b} 
      }
      \hfill \subfigure[\armorm]{
        \includegraphics[width=0.2\textwidth]{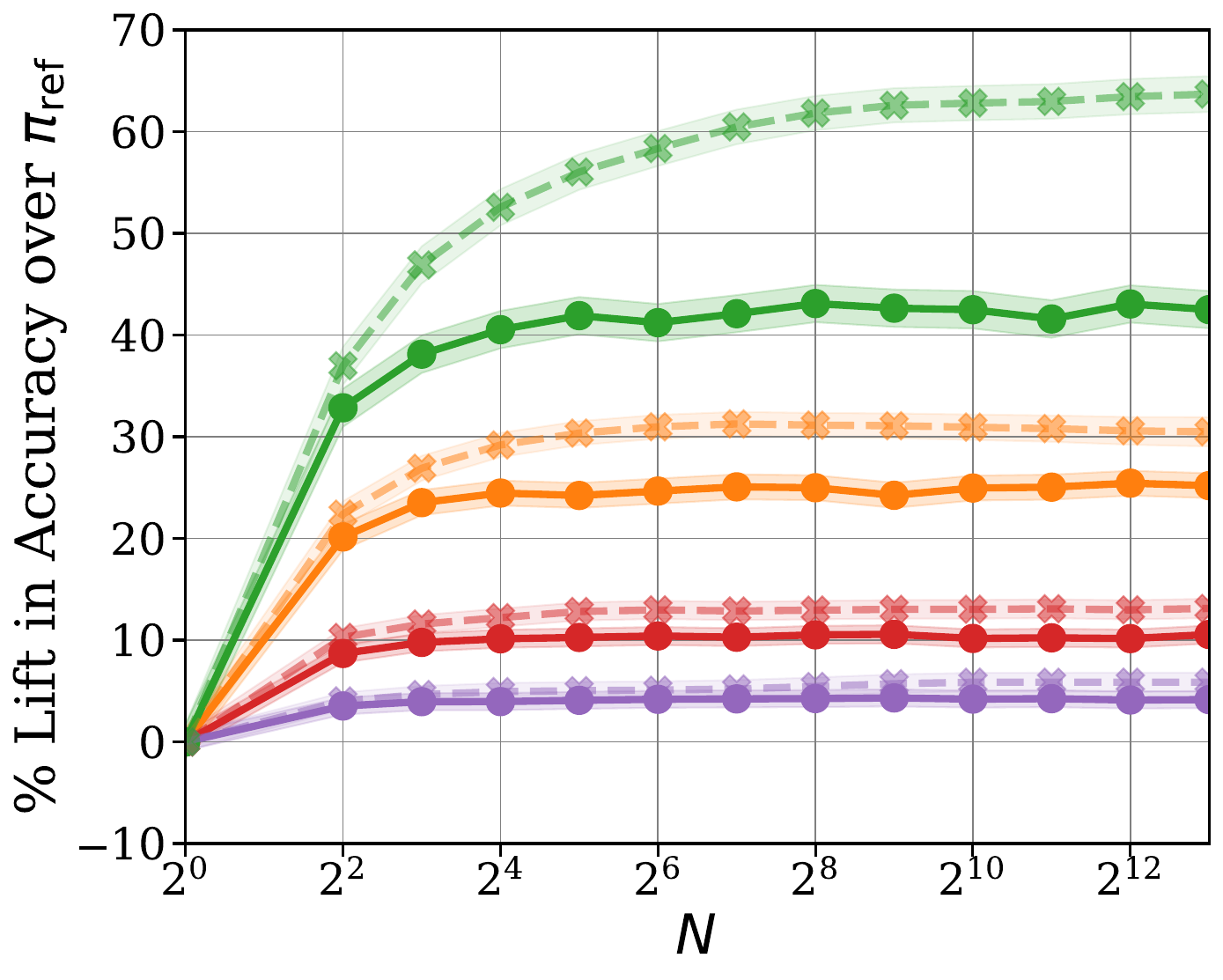}
        \label{sfig:n-gsm8k-armo-rm} 
      }
      \subfigure[\oasst]{
        \includegraphics[width=0.2\textwidth]{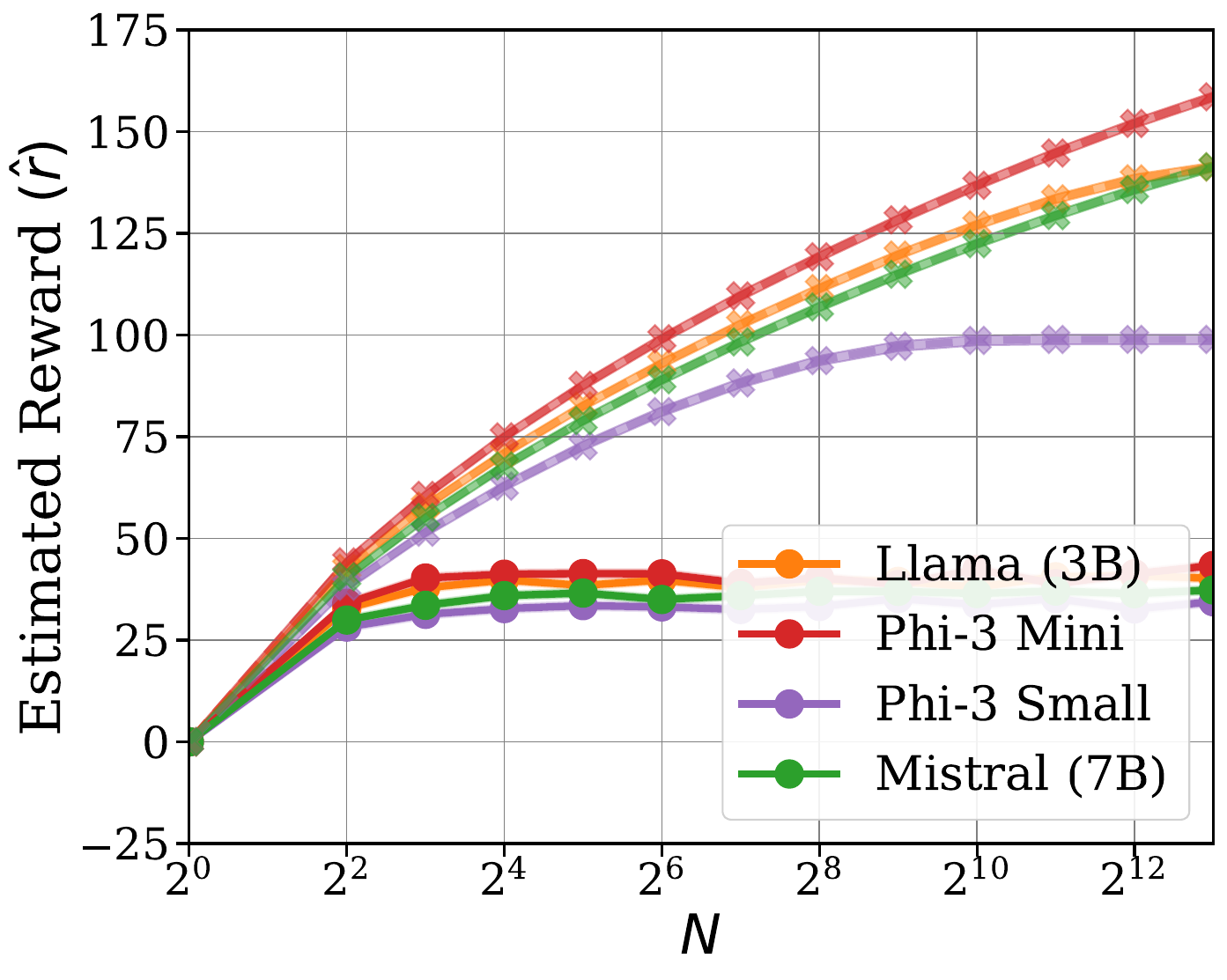}
        \label{sfig:n-gsm8k-oasst-rm-rhat} 
      }
   \hfill \subfigure[\gemmarm]{
        \includegraphics[width=0.2\textwidth]{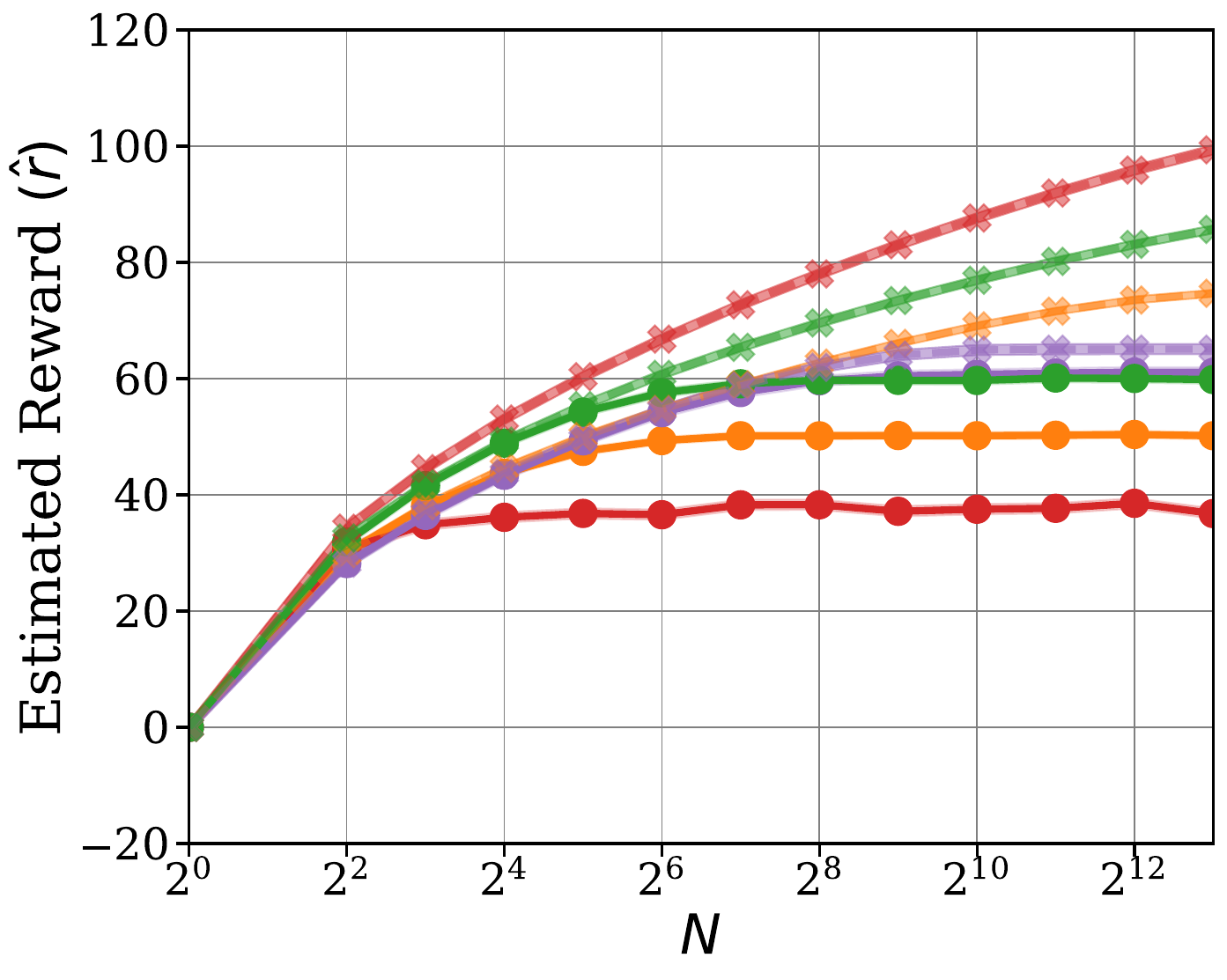}
        \label{sfig:n-gsm8k-gemma-rhat} 
      }
      \hfill \subfigure[\llamarm]{
        \includegraphics[width=0.2\textwidth]{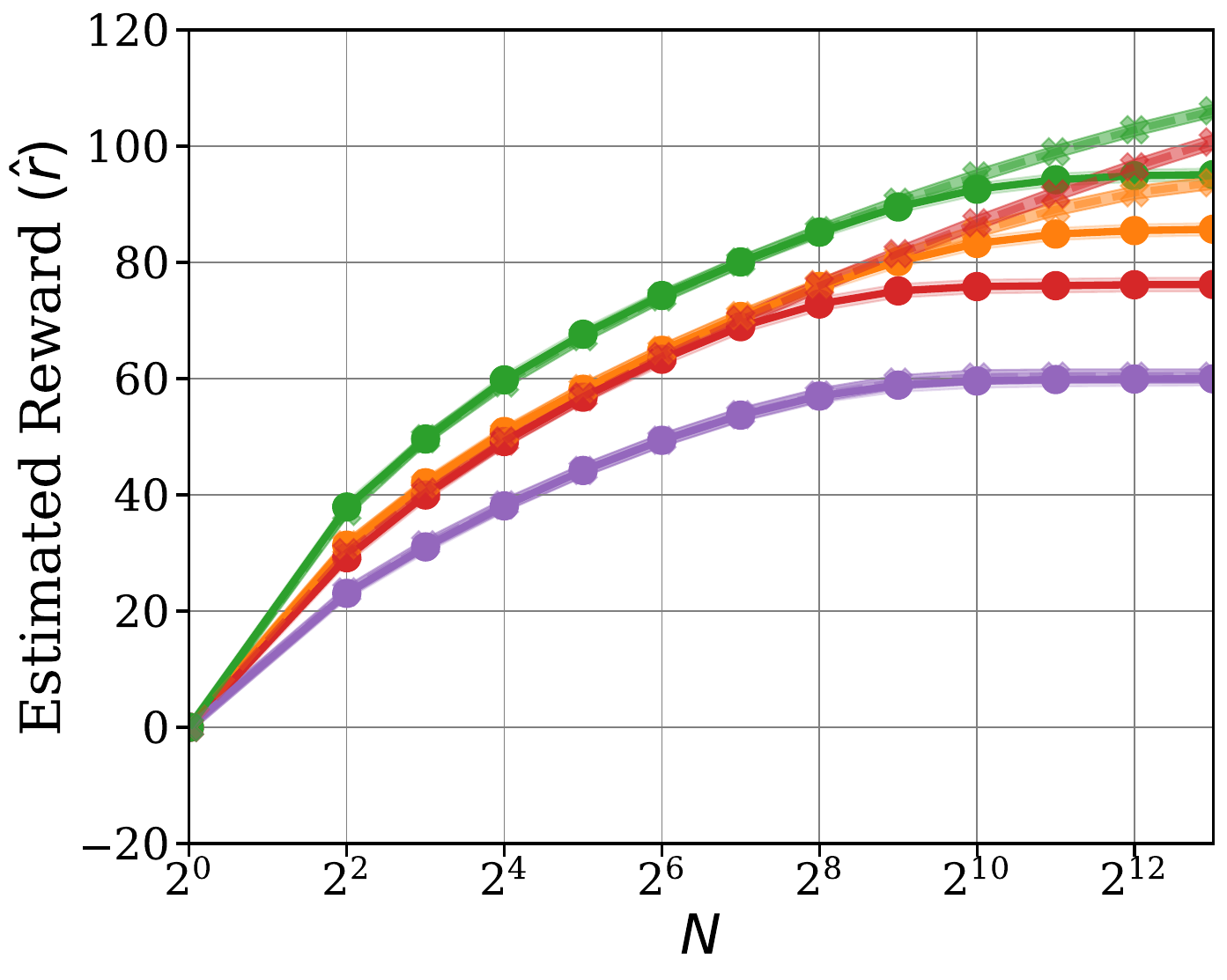}
        \label{sfig:n-gsm8k-llama-rhat} 
      }
      \hfill \subfigure[\armorm]{
        \includegraphics[width=0.2\textwidth]{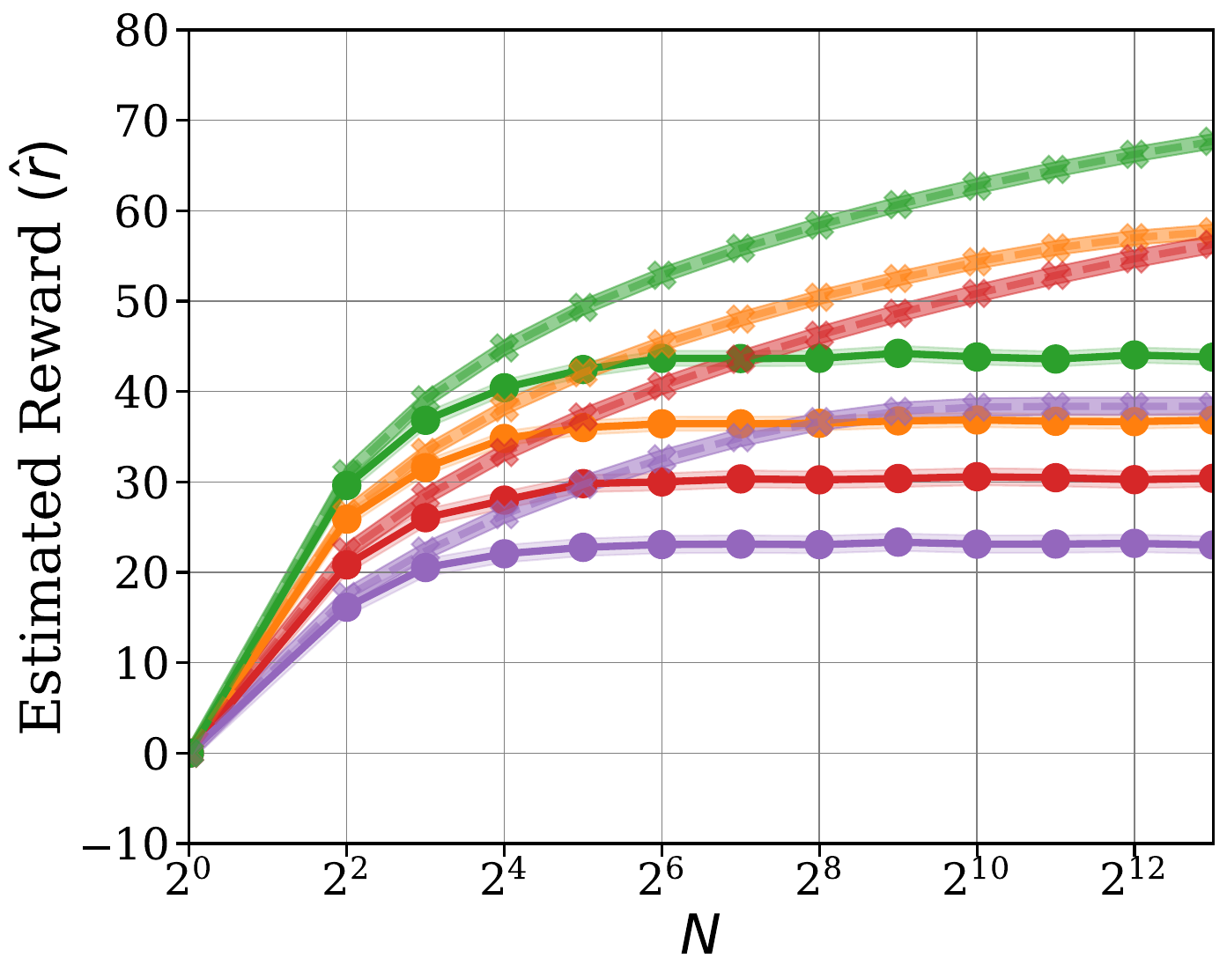}
        \label{sfig:n-gsm8k-armo-rm-rhat} 
      }
    \caption{Comparison of \mainalg (solid lines) and \bonalg (dashed lines) in accuracy and estimated reward $\rhat$ for \gsmk for four reward models and choices of $\piref$.}
    \label{fig:gsm8k-ns}
  \end{figure*}

  \begin{figure*}[htp]
    \centering
    \subfigure[\oasst]{
        \includegraphics[width=0.2\textwidth]{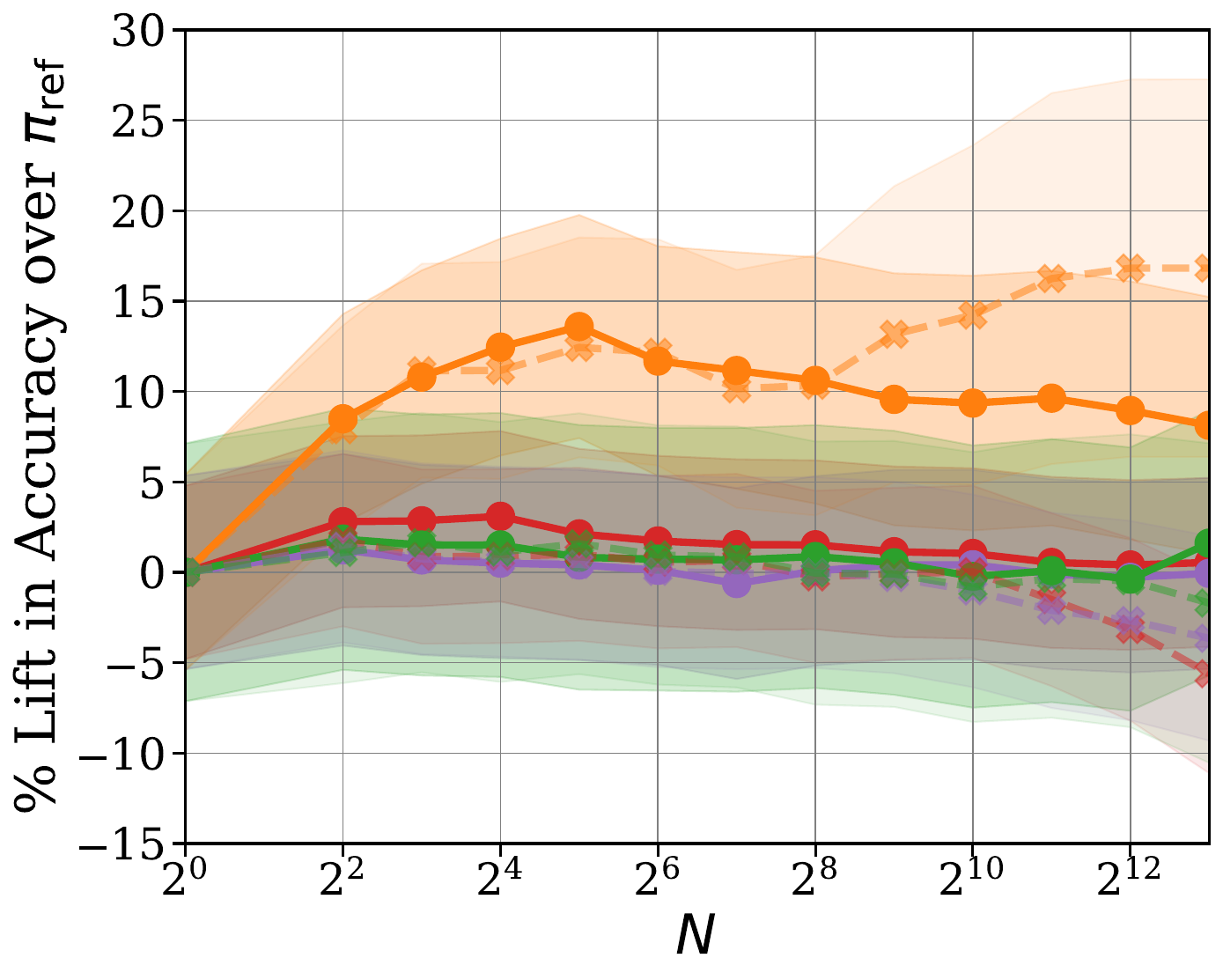}
        \label{sfig:n-mmlu-oasst-rm} 
      }
   \hfill \subfigure[\gemmarm]{
        \includegraphics[width=0.2\textwidth]{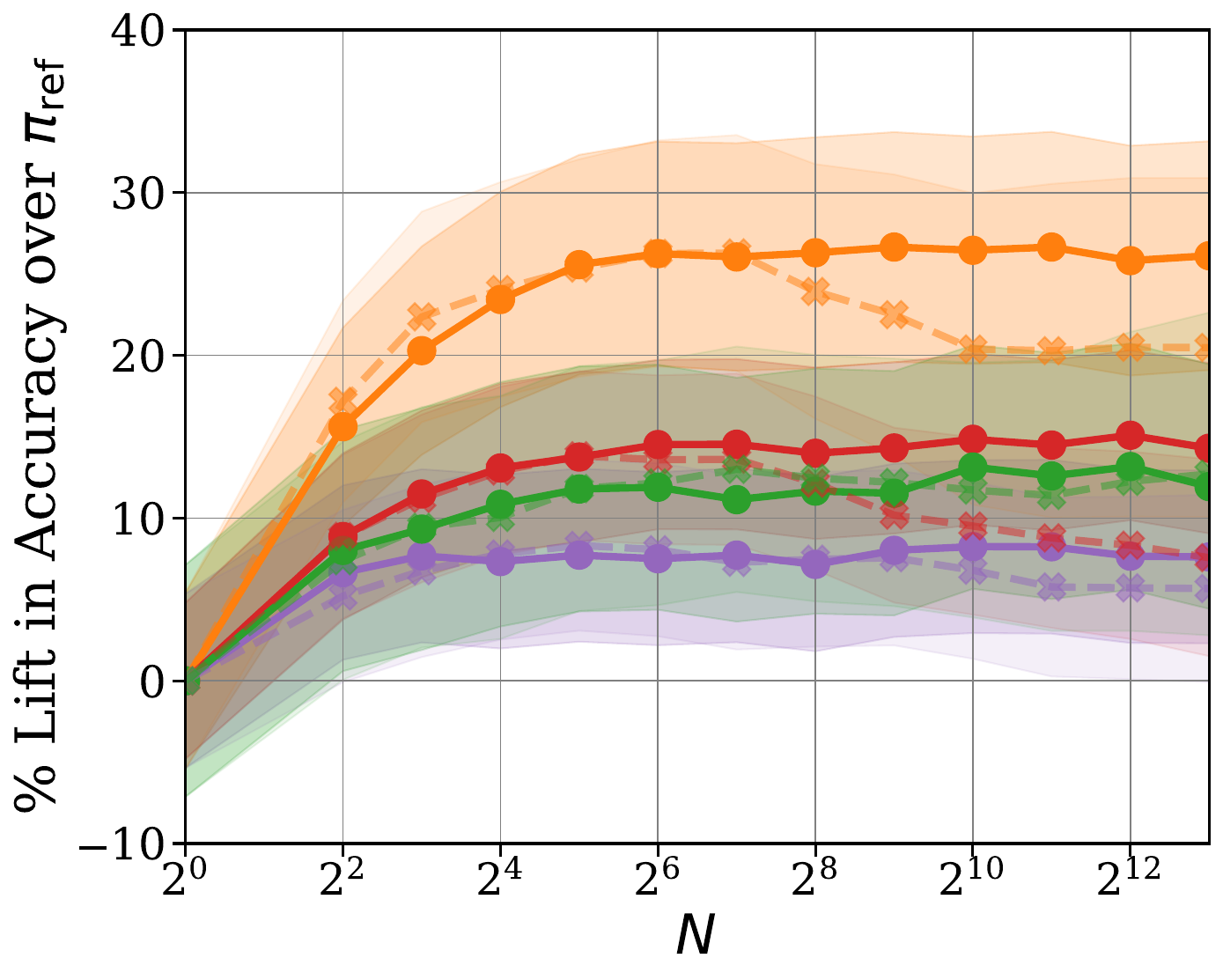}
        \label{sfig:n-mmlu-gemma-rm} 
      }
      \hfill \subfigure[\llamarm]{
        \includegraphics[width=0.2\textwidth]{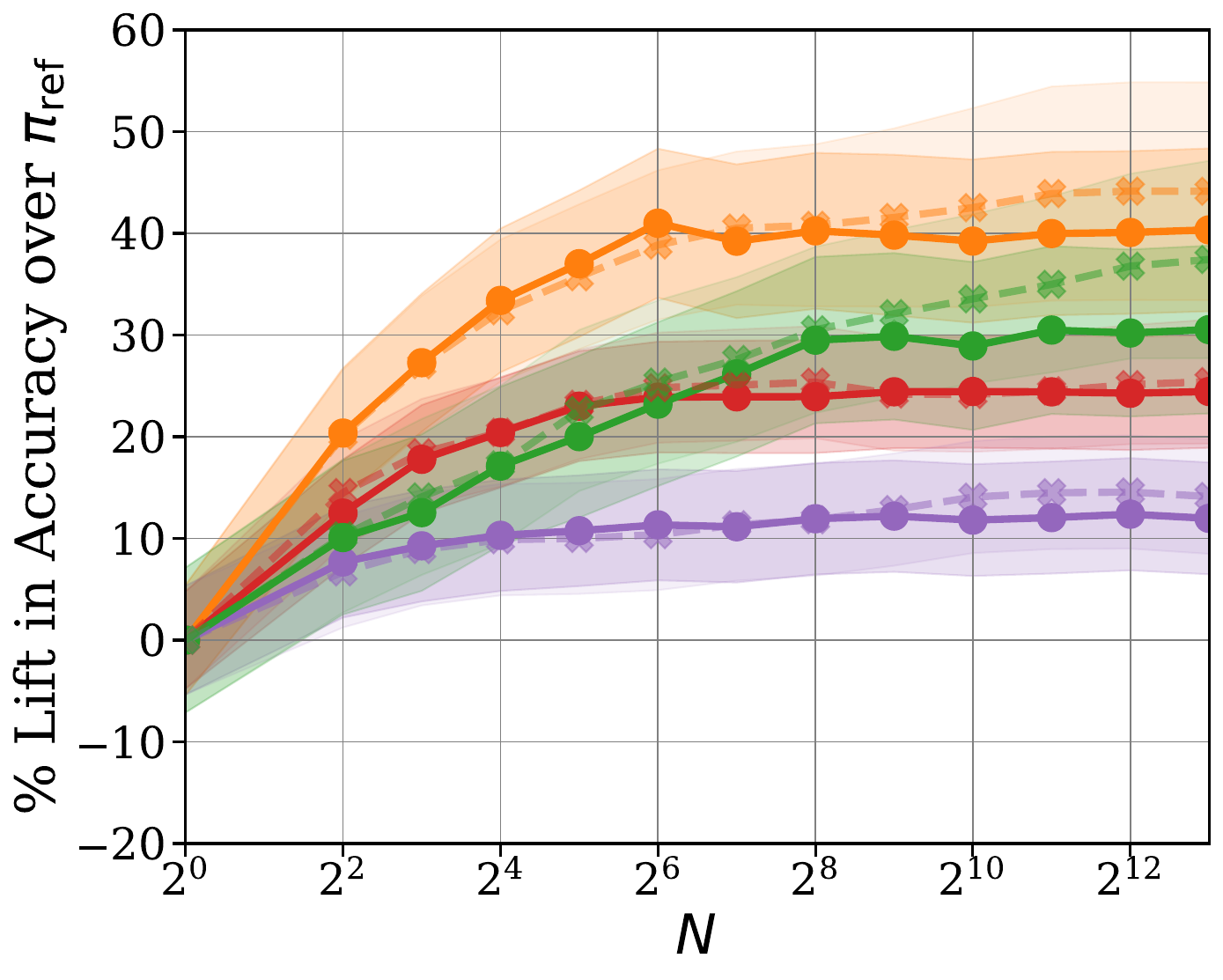}
        \label{sfig:n-mmlu-llama-3b} 
      }
      \hfill \subfigure[\armorm]{
        \includegraphics[width=0.2\textwidth]{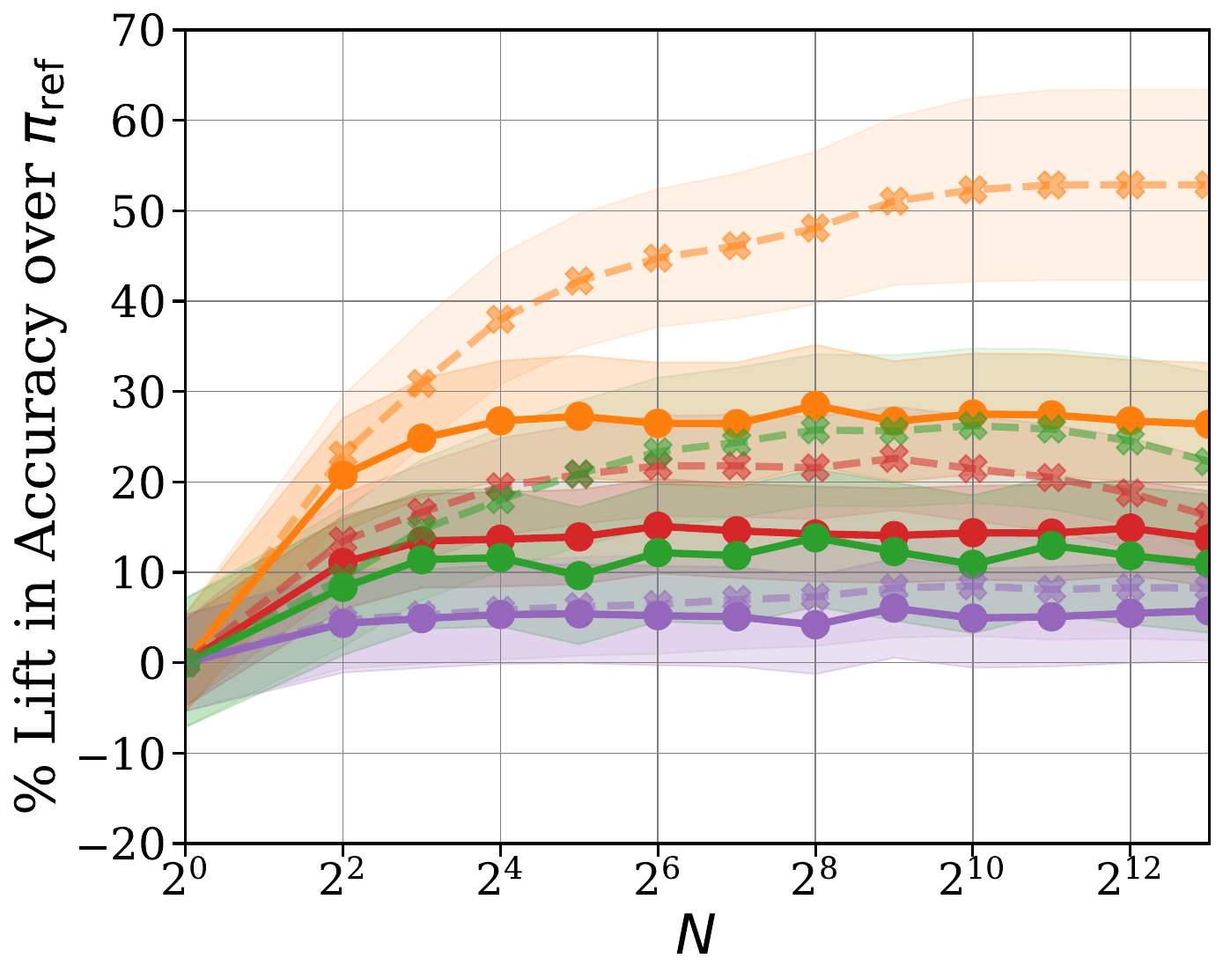}
        \label{sfig:n-mmlu-armo-rm} 
      }
      \subfigure[\oasst]{
        \includegraphics[width=0.2\textwidth]{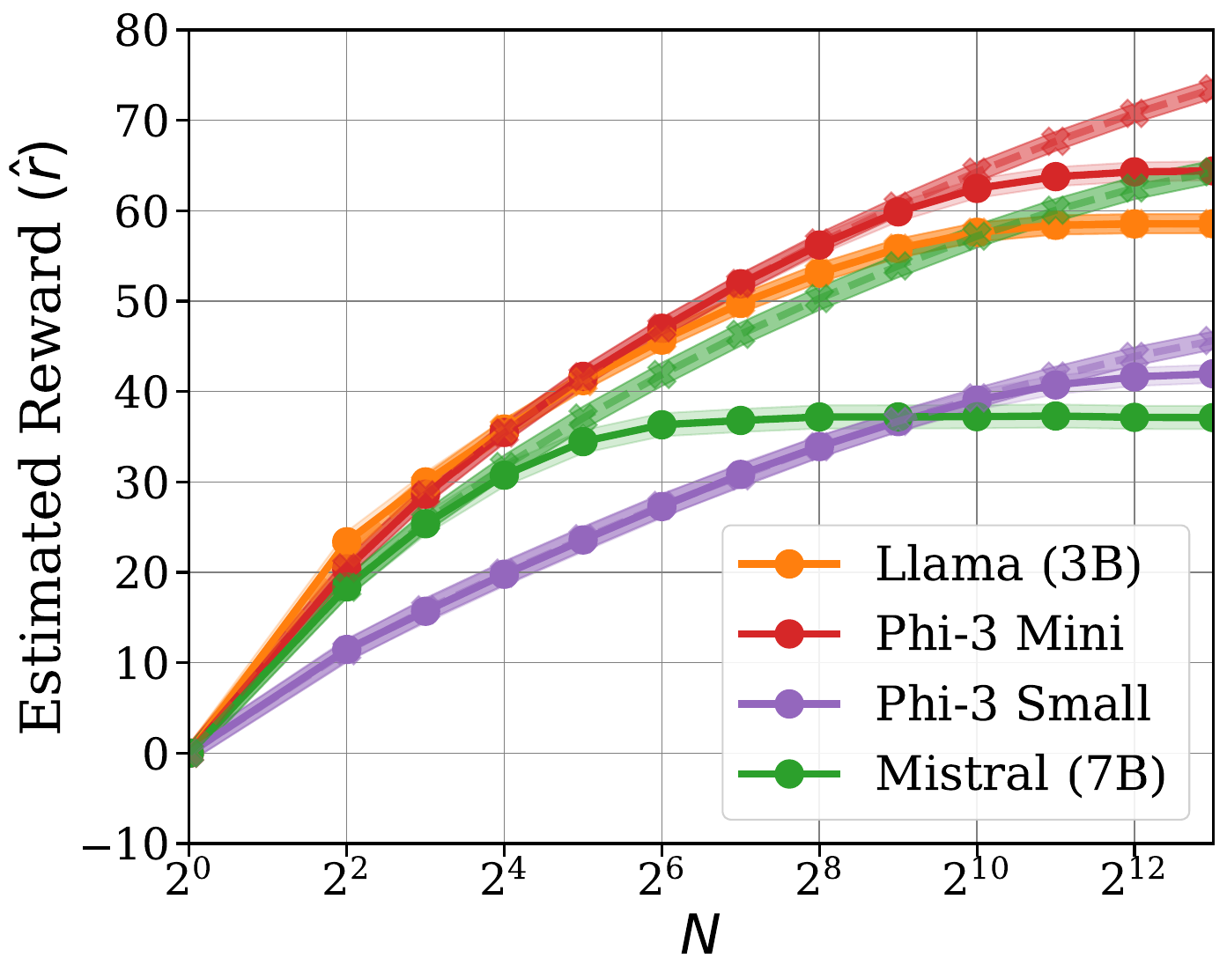}
        \label{sfig:n-mmlu-oasst-rm-rhat} 
      }
   \hfill \subfigure[\gemmarm]{
        \includegraphics[width=0.2\textwidth]{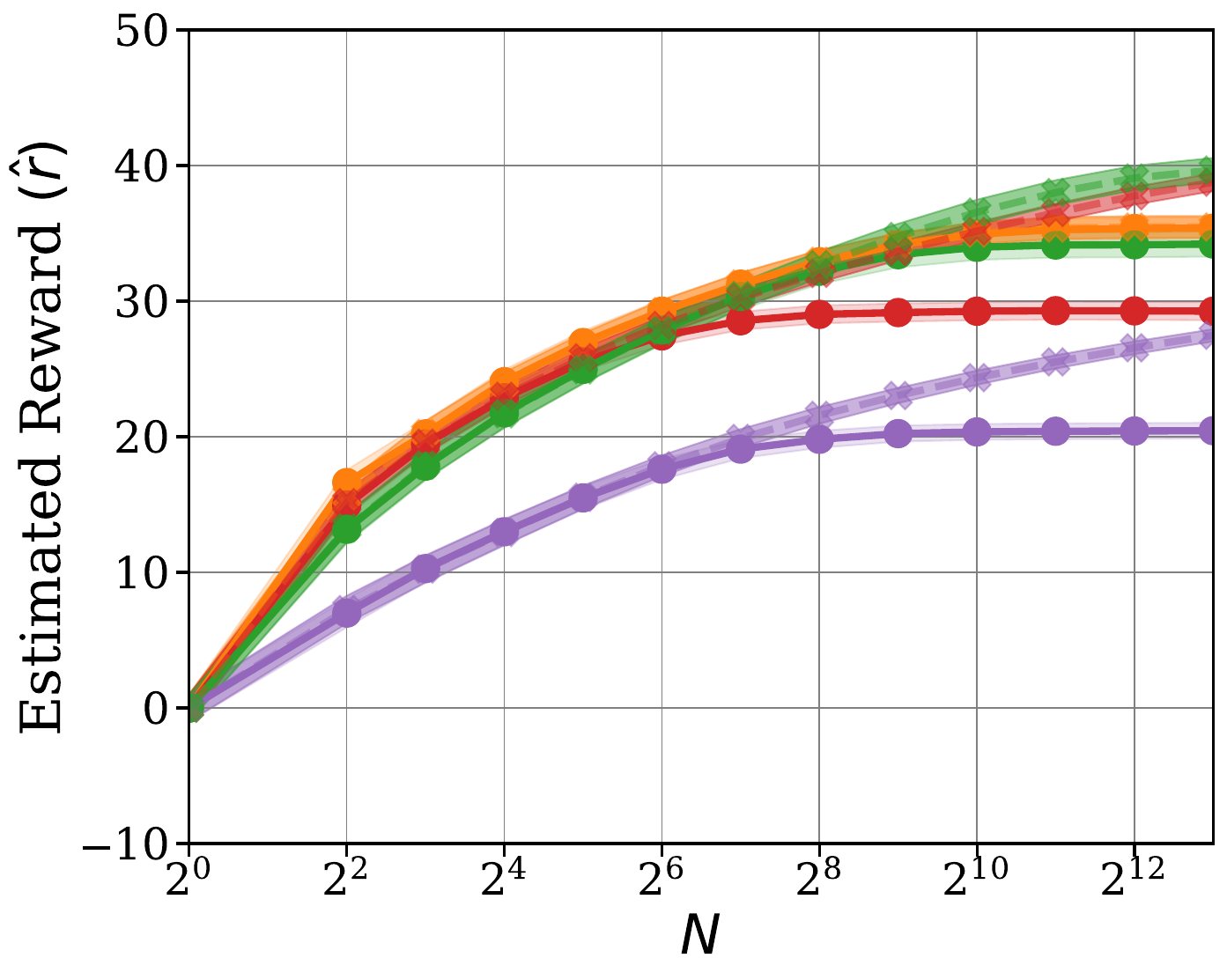}
        \label{sfig:n-mmlu-gemma-rhat} 
      }
      \hfill \subfigure[\llamarm]{
        \includegraphics[width=0.2\textwidth]{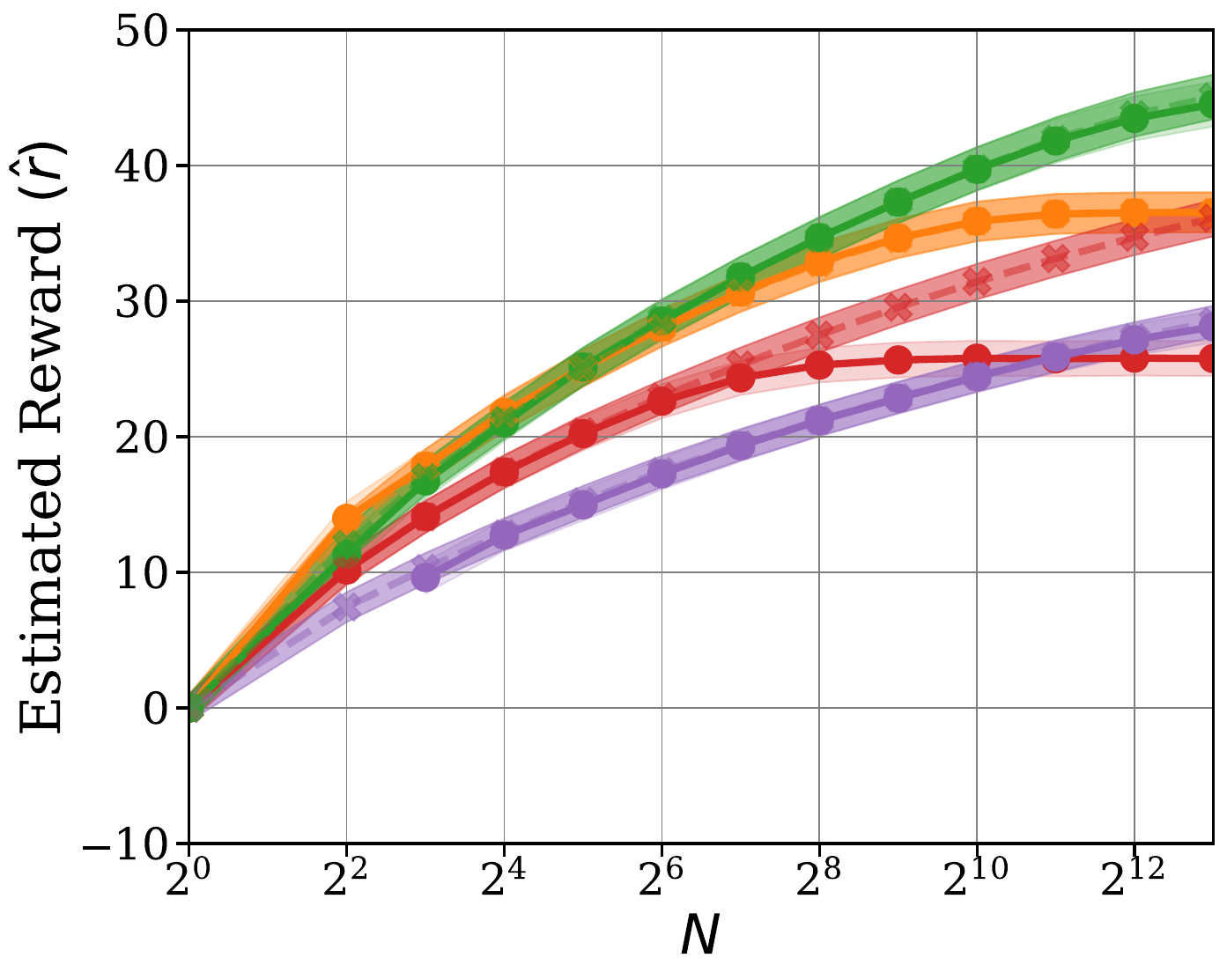}
        \label{sfig:n-mmlu-llama-rhat} 
      }
      \hfill \subfigure[\armorm]{
        \includegraphics[width=0.2\textwidth]{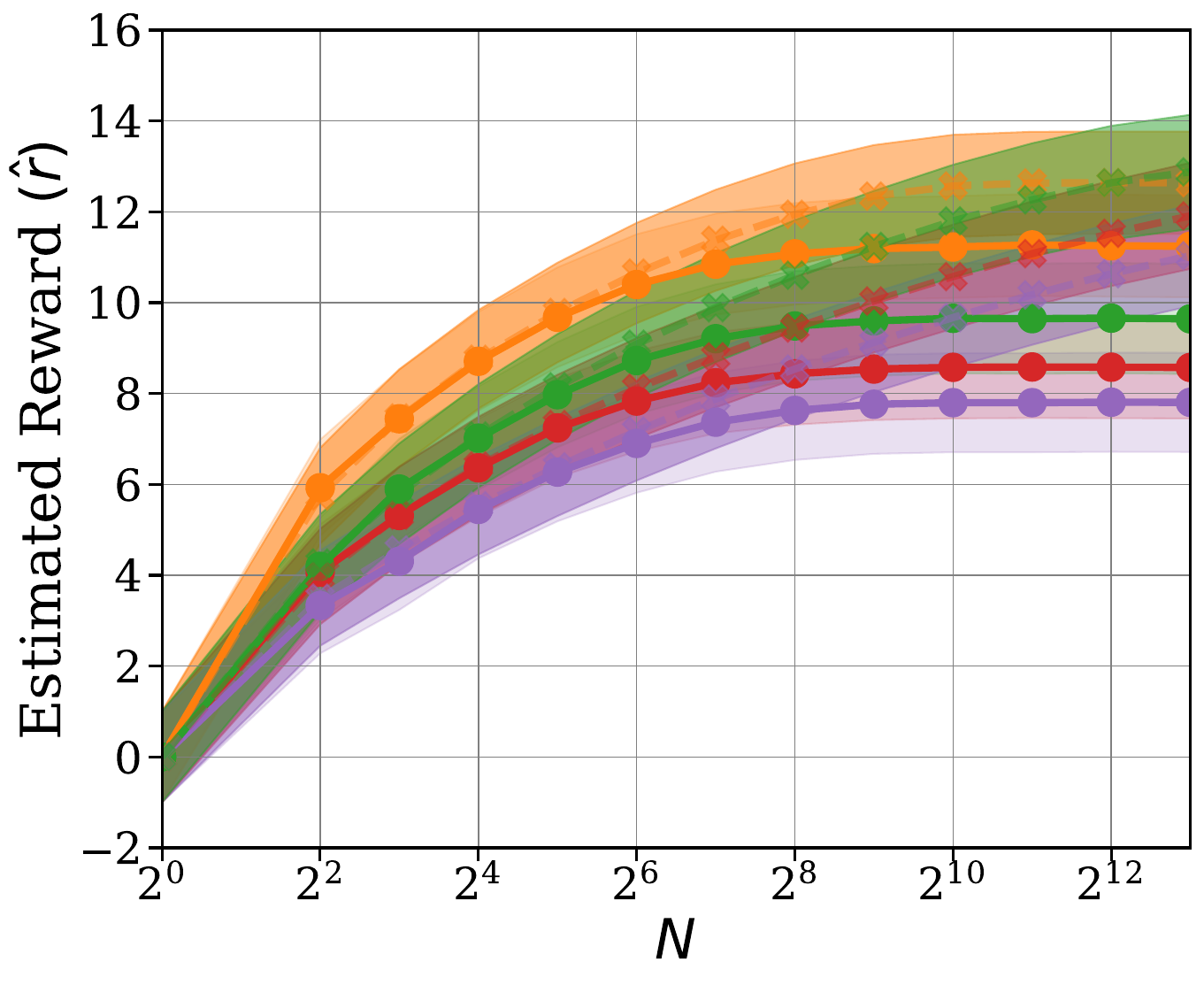}
        \label{sfig:n-mmlu-armo-rm-rhat} 
      }
    \caption{Comparison of \mainalg (solid lines) and \bonalg (dashed lines) in accuracy and estimated reward $\rhat$ for \mmlu for four reward models and choices of $\piref$.}
    \label{fig:mmlu-ns}
  \end{figure*}

  \begin{figure*}[htp]
    \centering
    \subfigure[\oasst]{
        \includegraphics[width=0.2\textwidth]{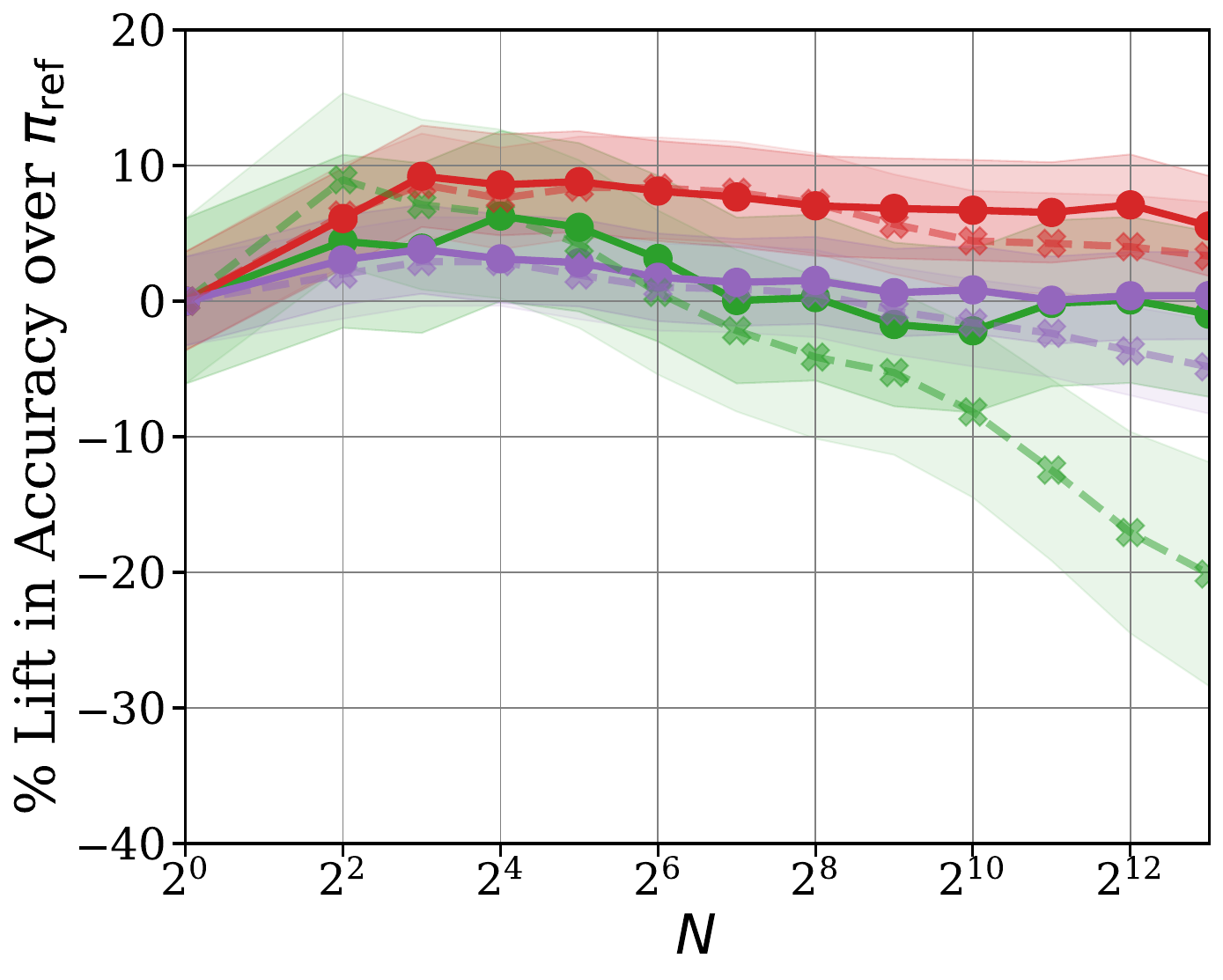}
        \label{sfig:n-math-oasst-rm} 
      }
   \hfill \subfigure[\gemmarm]{
        \includegraphics[width=0.2\textwidth]{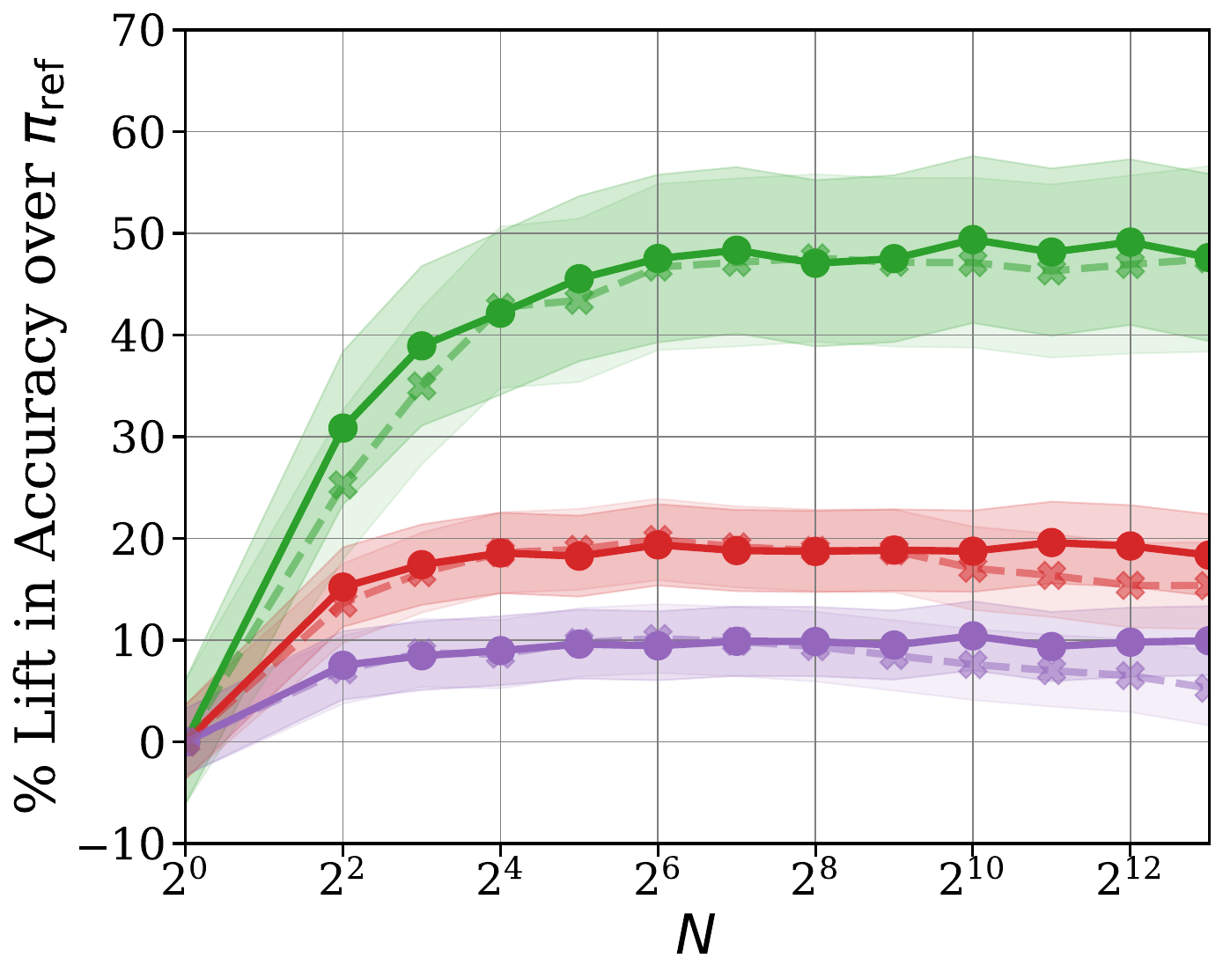}
        \label{sfig:n-math-gemma-rm} 
      }
      \hfill \subfigure[\llamarm]{
        \includegraphics[width=0.2\textwidth]{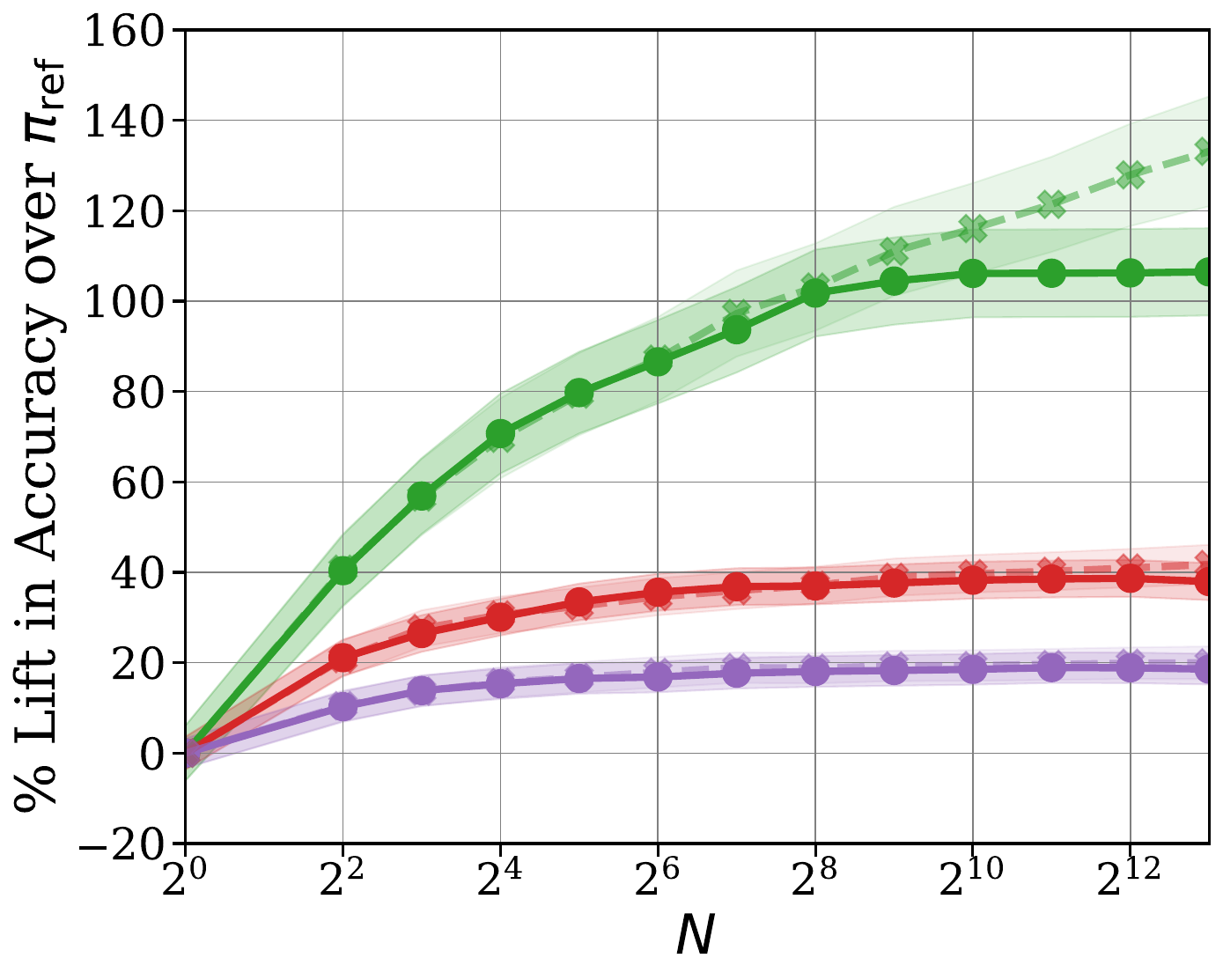}
        \label{sfig:n-math-llama-3b} 
      }
      \hfill \subfigure[\armorm]{
        \includegraphics[width=0.2\textwidth]{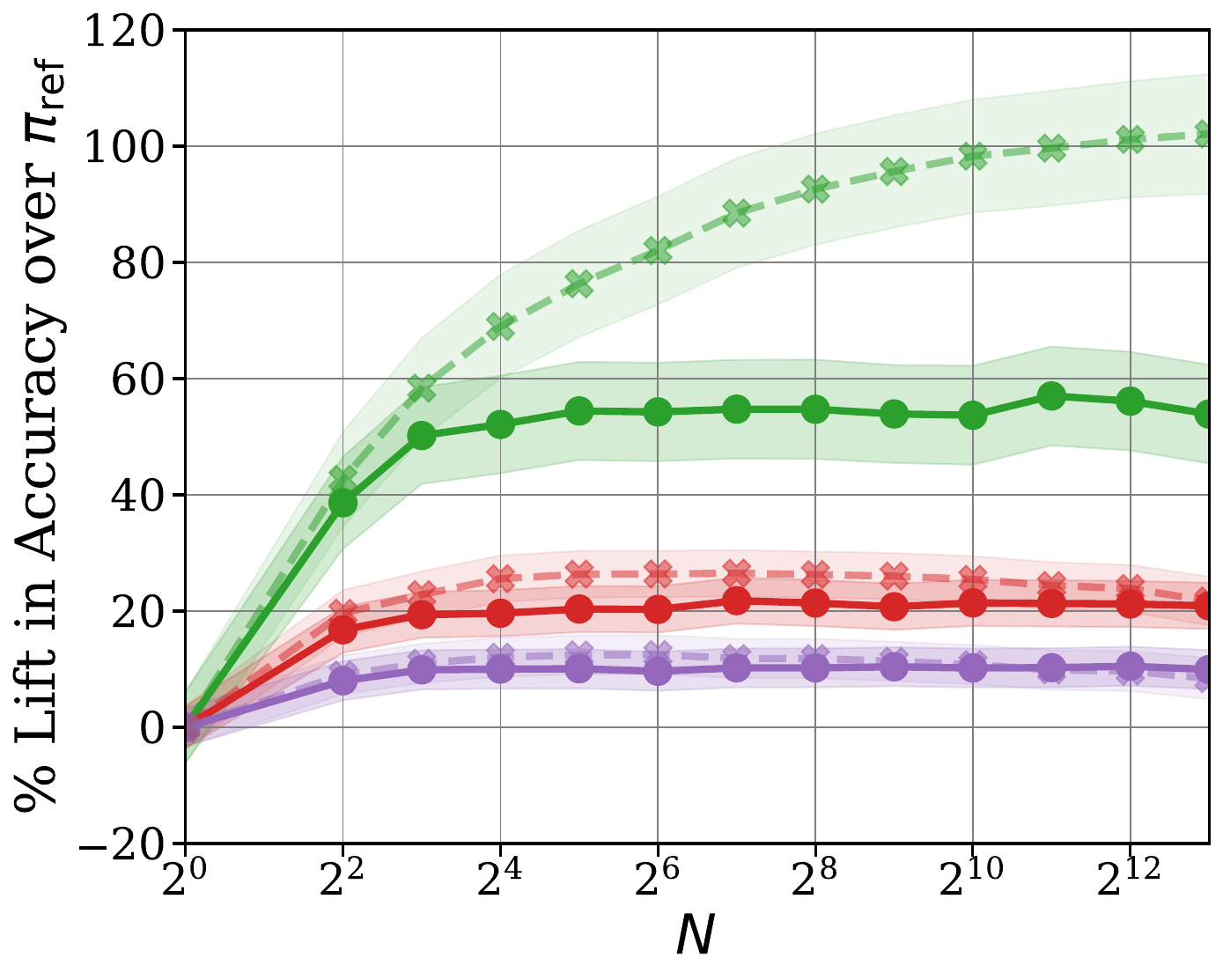}
        \label{sfig:n-math-armo-rm} 
      }
      \subfigure[\oasst]{
        \includegraphics[width=0.2\textwidth]{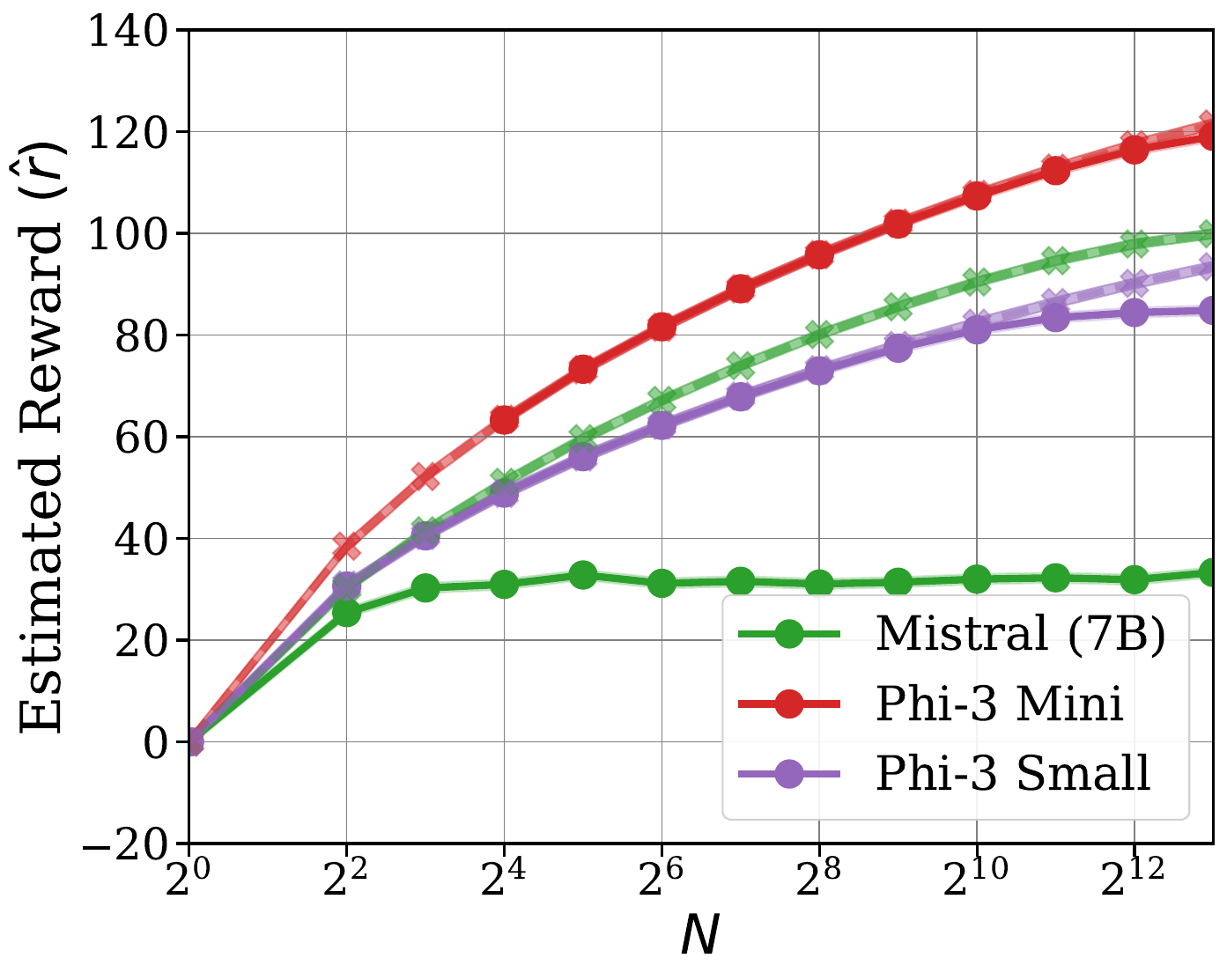}
        \label{sfig:n-math-oasst-rm-rhat} 
      }
   \hfill \subfigure[\gemmarm]{
        \includegraphics[width=0.2\textwidth]{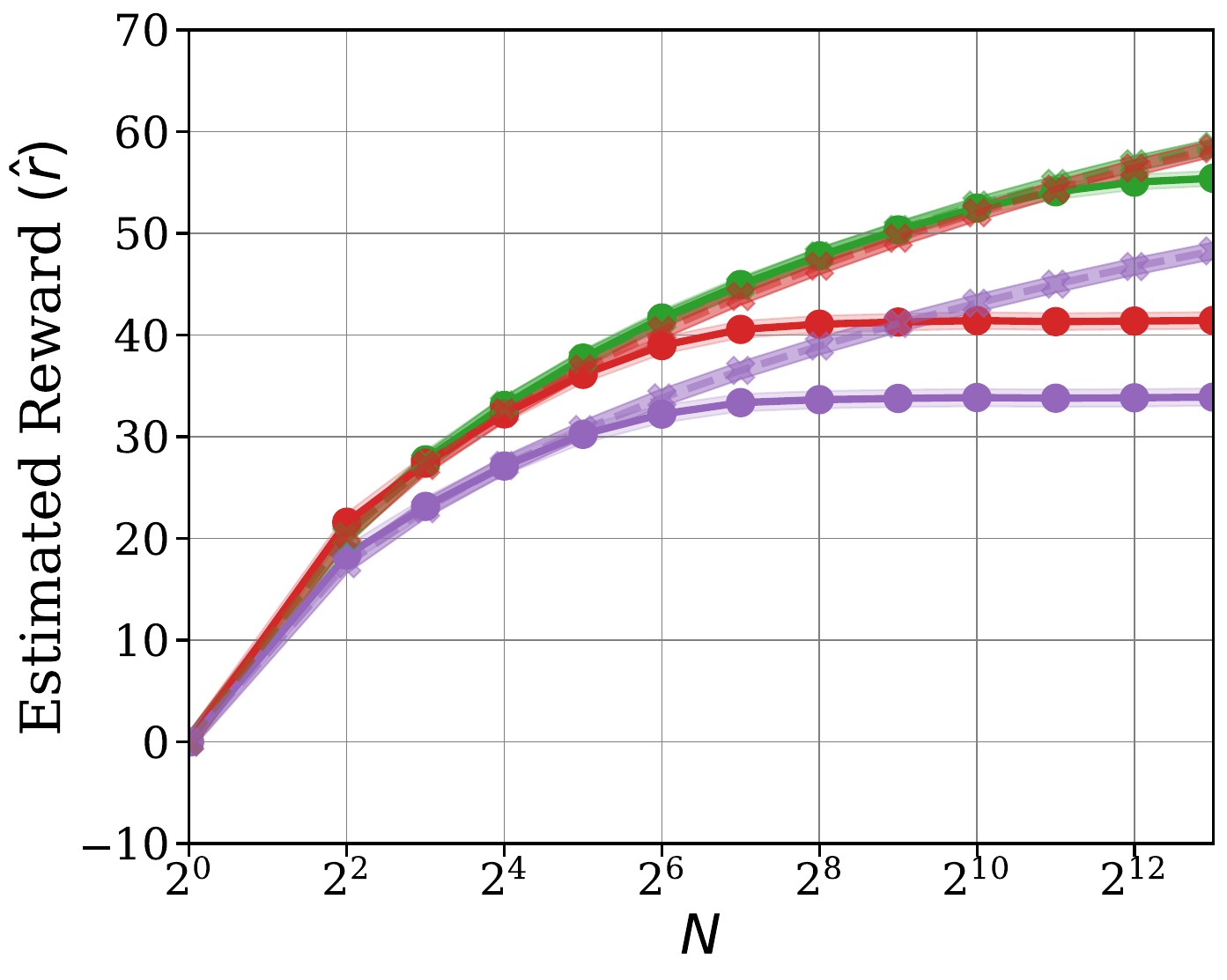}
        \label{sfig:n-math-gemma-rhat} 
      }
      \hfill \subfigure[\llamarm]{
        \includegraphics[width=0.2\textwidth]{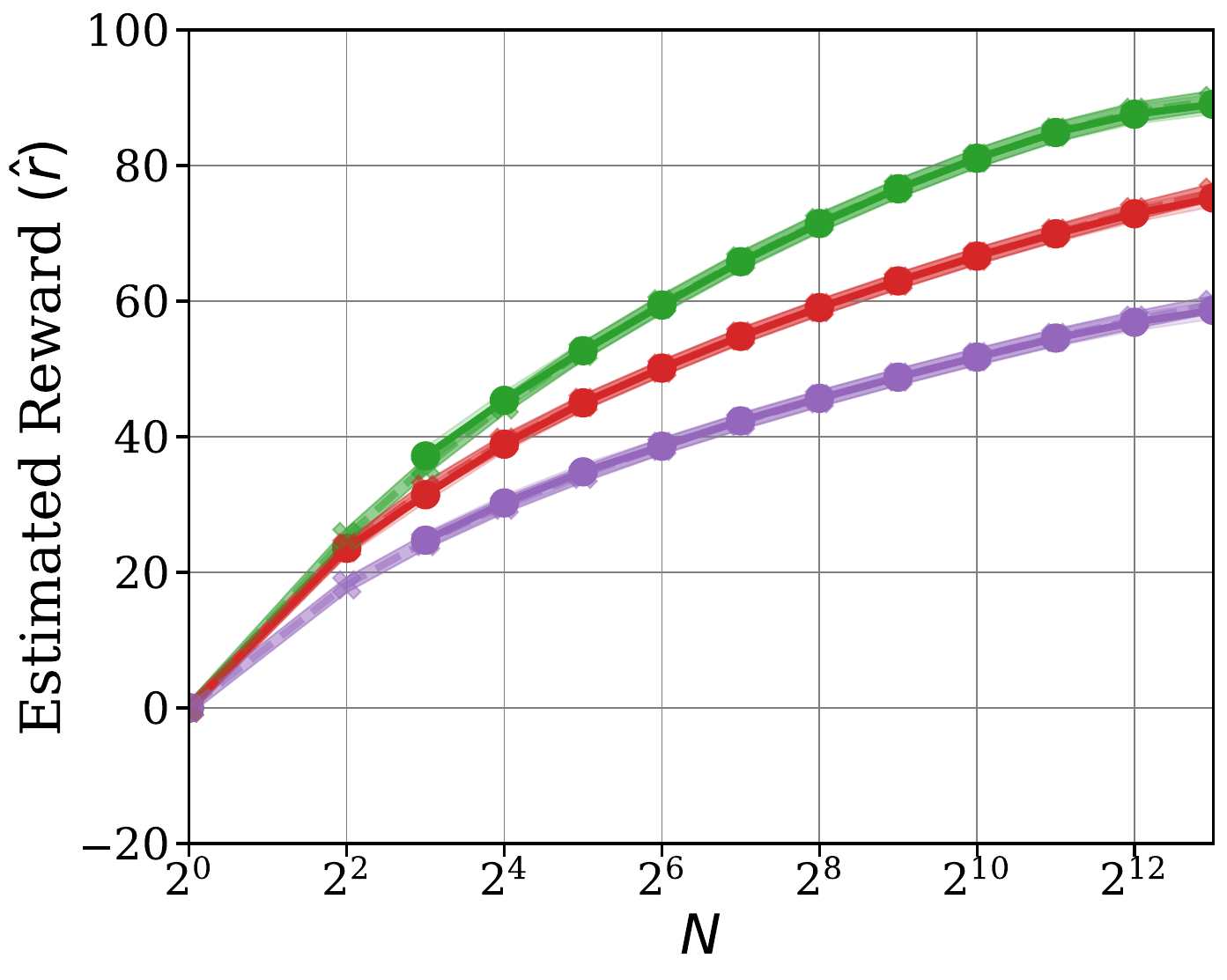}
        \label{sfig:n-math-llama-rhat} 
      }
      \hfill \subfigure[\armorm]{
        \includegraphics[width=0.2\textwidth]{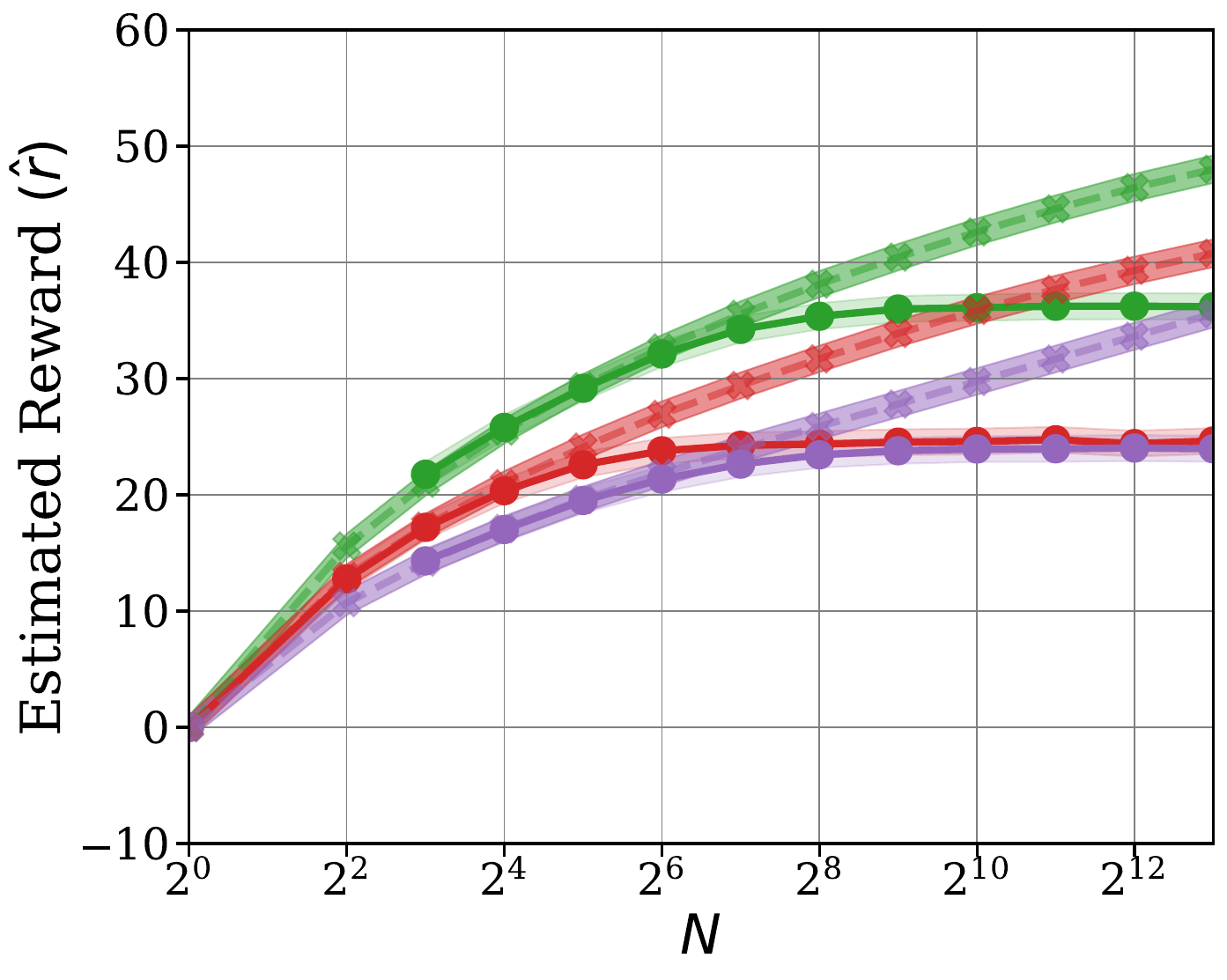}
        \label{sfig:n-math-armo-rm-rhat} 
      }
    \caption{Comparison of \mainalg (solid lines) and \bonalg (dashed lines) in accuracy and estimated reward $\rhat$ for \mathk for four reward models and choices of $\piref$.}
    \label{fig:math-ns}
  \end{figure*}

  \begin{figure*}[htp]
    \centering
    \subfigure[\oasst]{
        \includegraphics[width=0.2\textwidth]{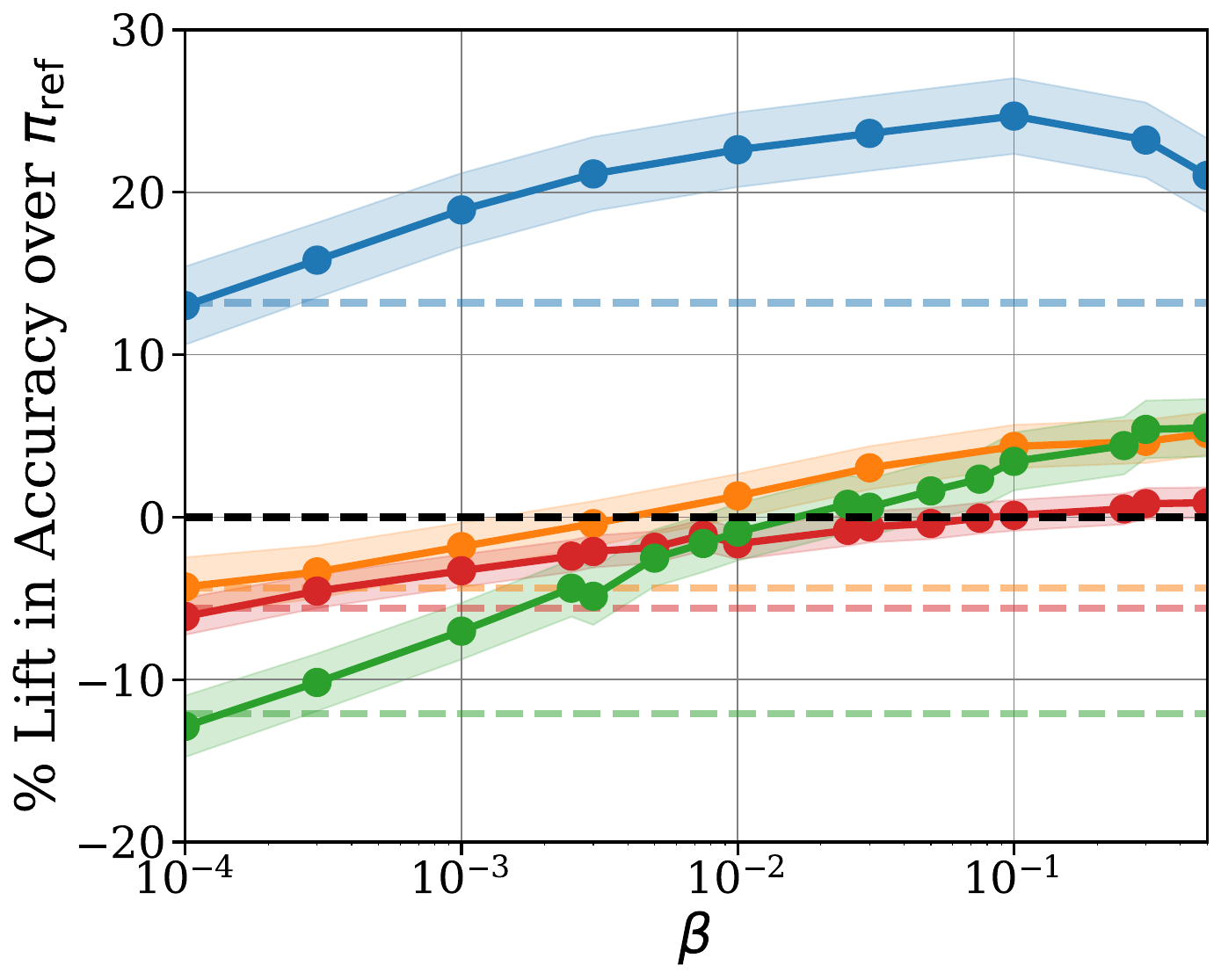}
        \label{sfig:betas-gsm8k-oasst-rm} 
      }
   \hfill \subfigure[\gemmarm]{
        \includegraphics[width=0.2\textwidth]{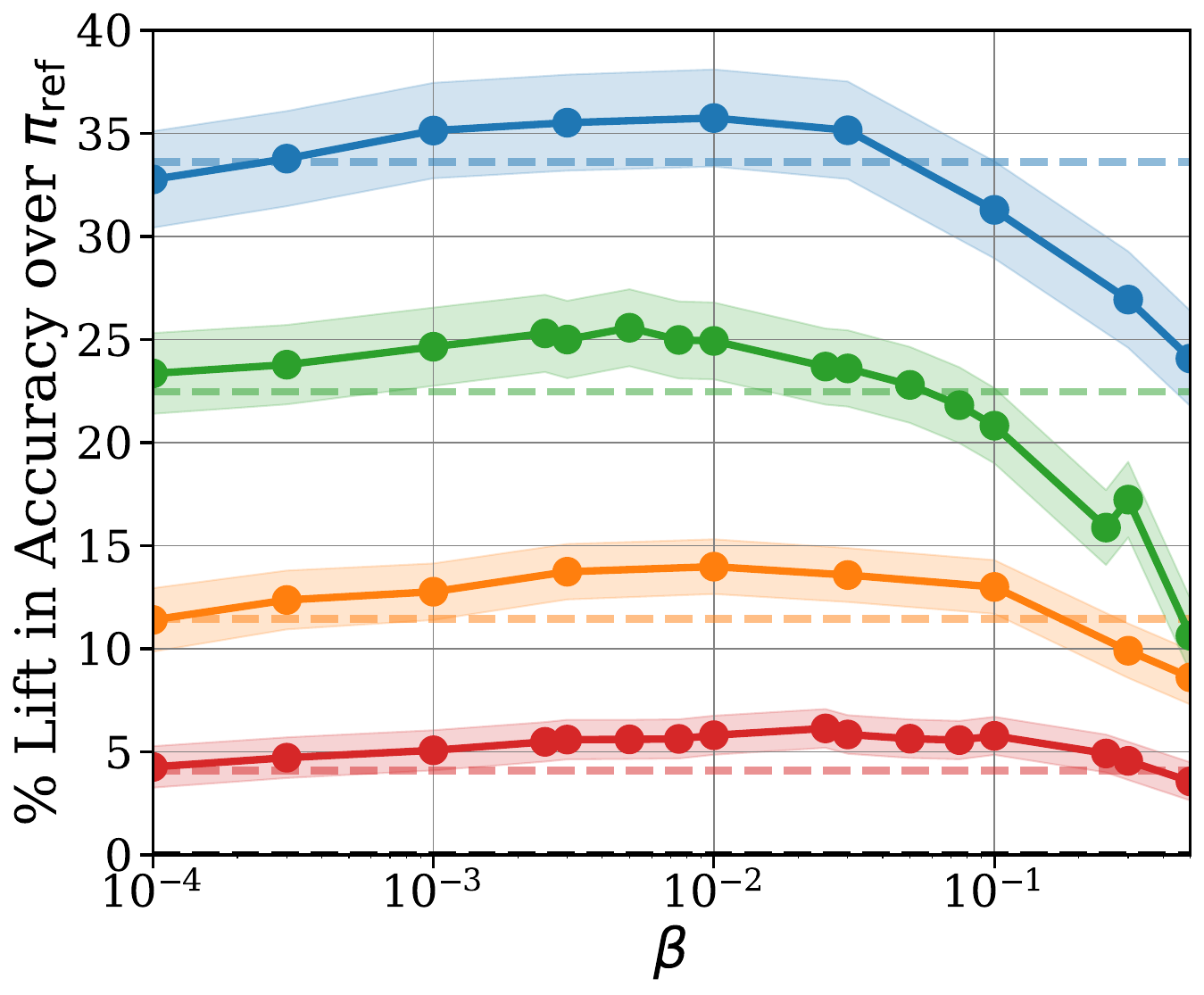}
        \label{sfig:betas-gsm8k-gemma-rm} 
      }
      \hfill \subfigure[\llamarm]{
        \includegraphics[width=0.2\textwidth]{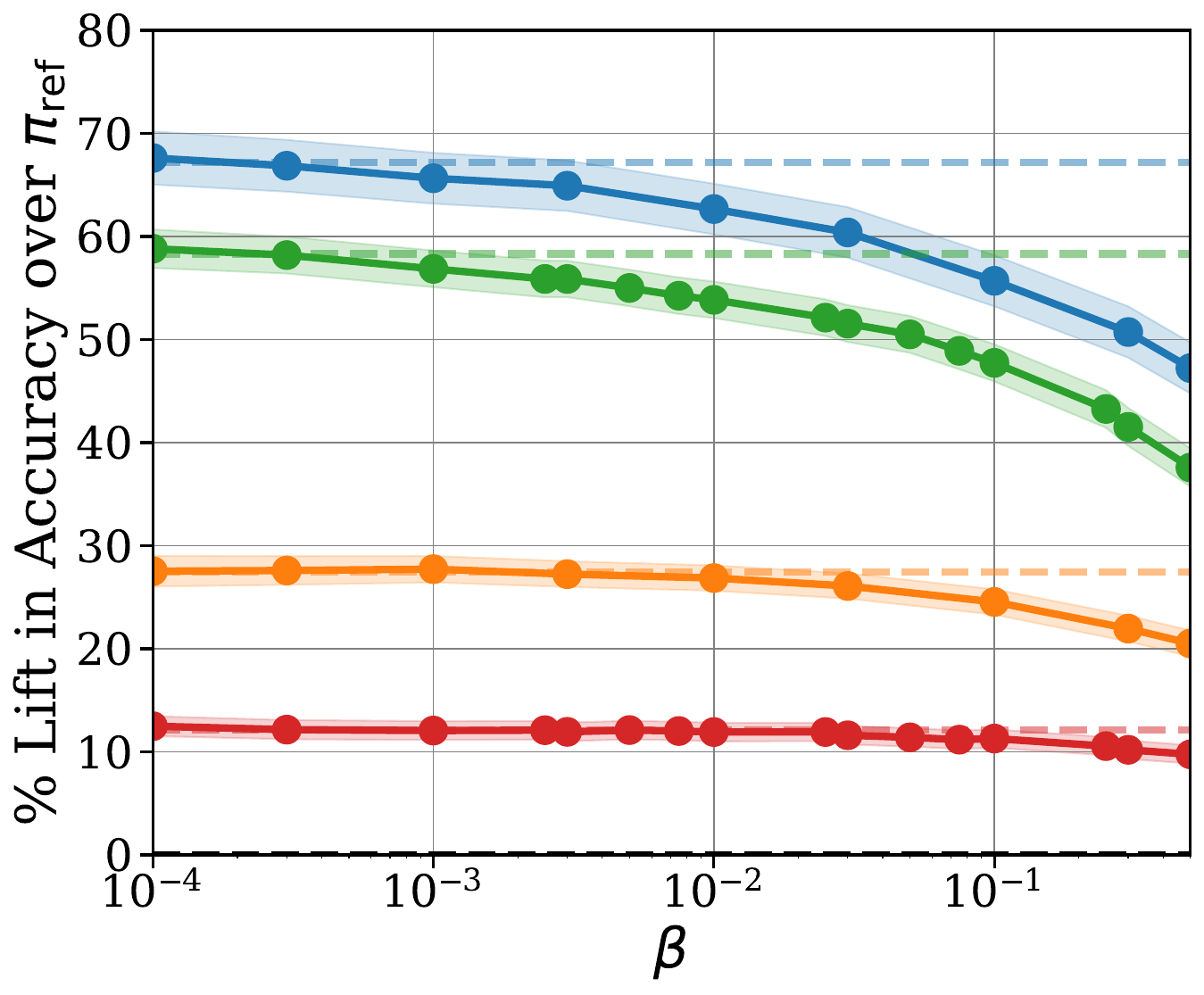}
        \label{sfig:betas-gsm8k-llama-3b} 
      }
      \hfill \subfigure[\armorm]{
        \includegraphics[width=0.2\textwidth]{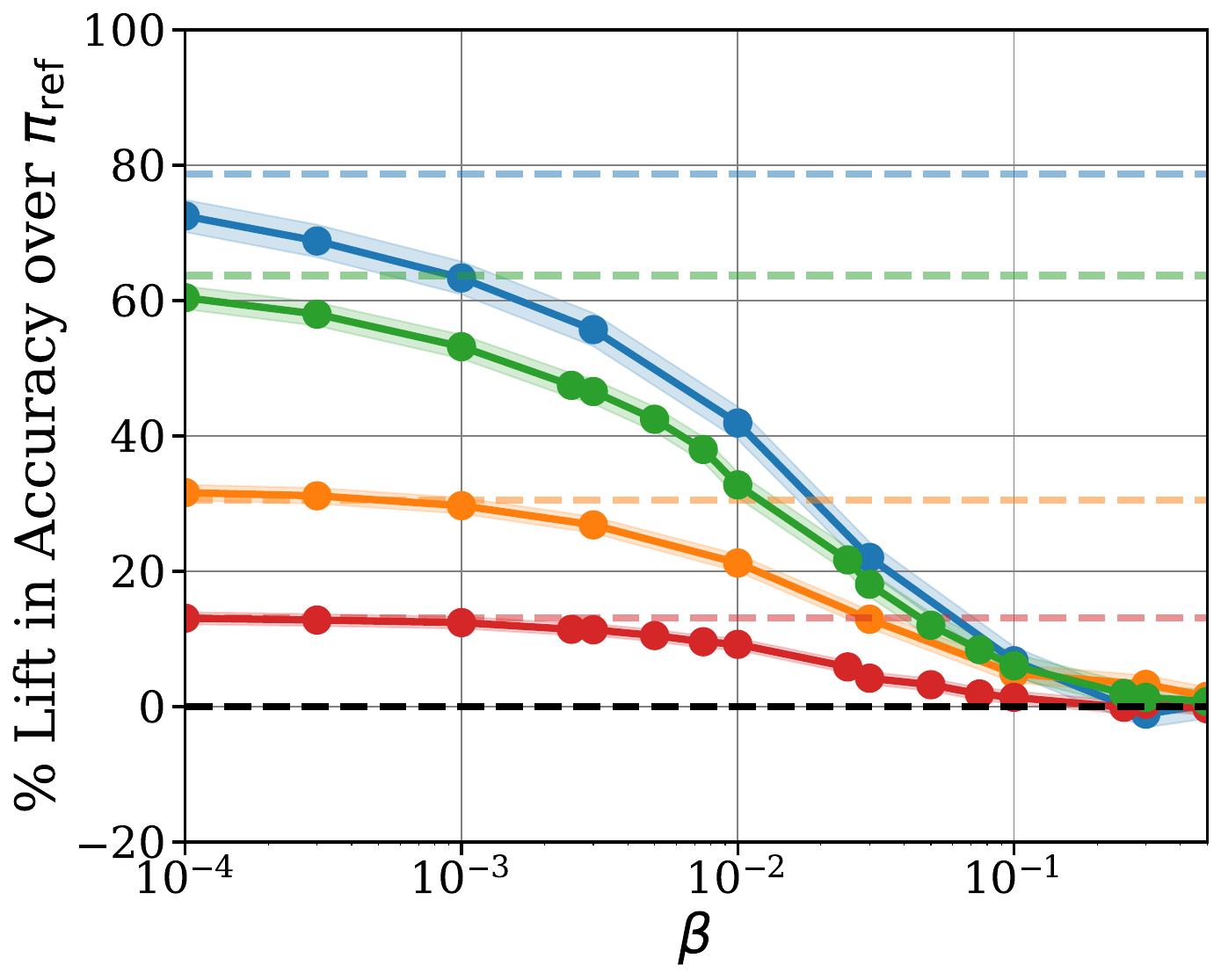}
        \label{sfig:betas-gsm8k-armo-rm} 
      }
      
      \subfigure[\oasst]{
        \includegraphics[width=0.2\textwidth]{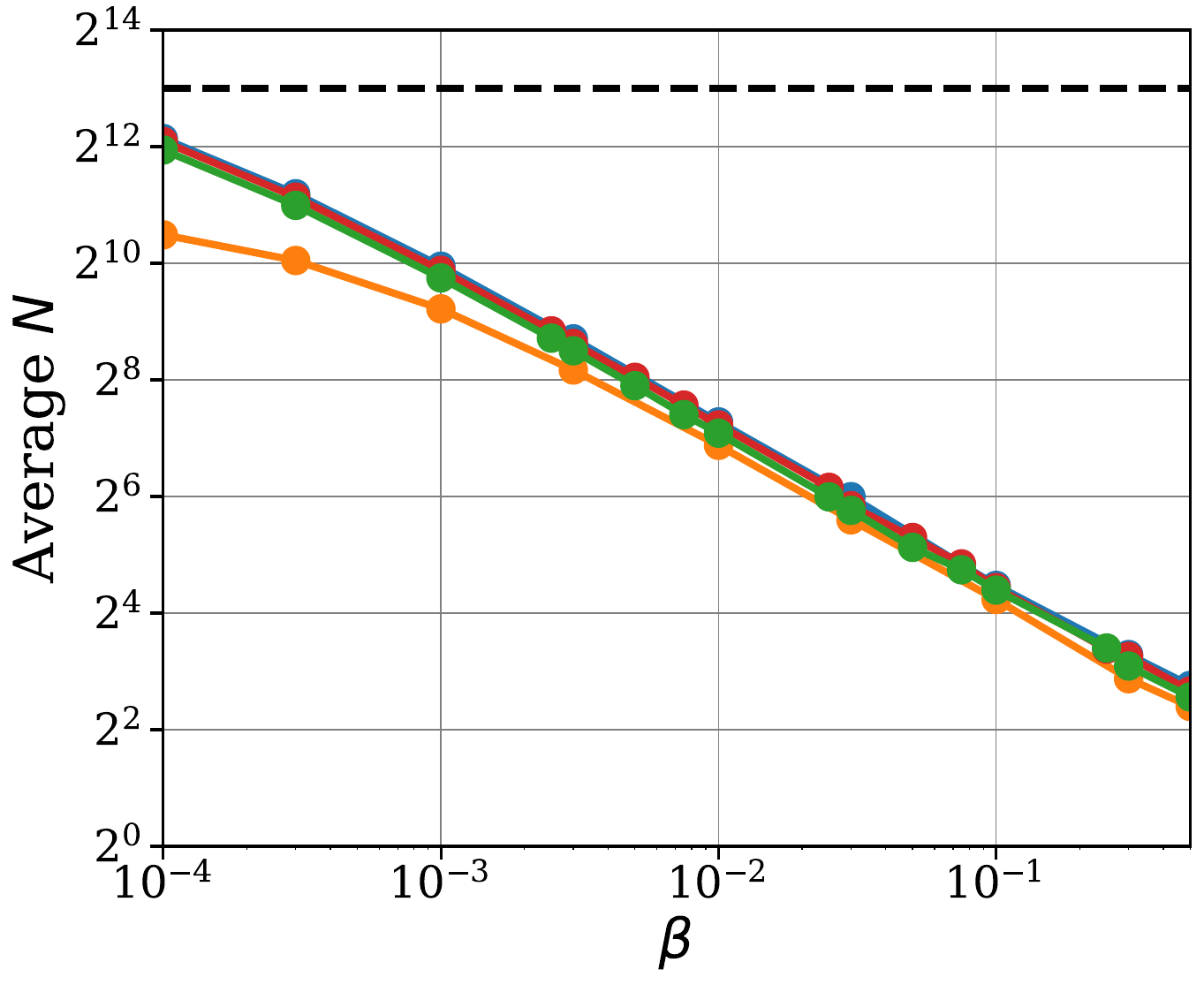}
        \label{sfig:betas-gsm8k-oasst-rm-N} 
      }
   \hfill \subfigure[\gemmarm]{
        \includegraphics[width=0.2\textwidth]{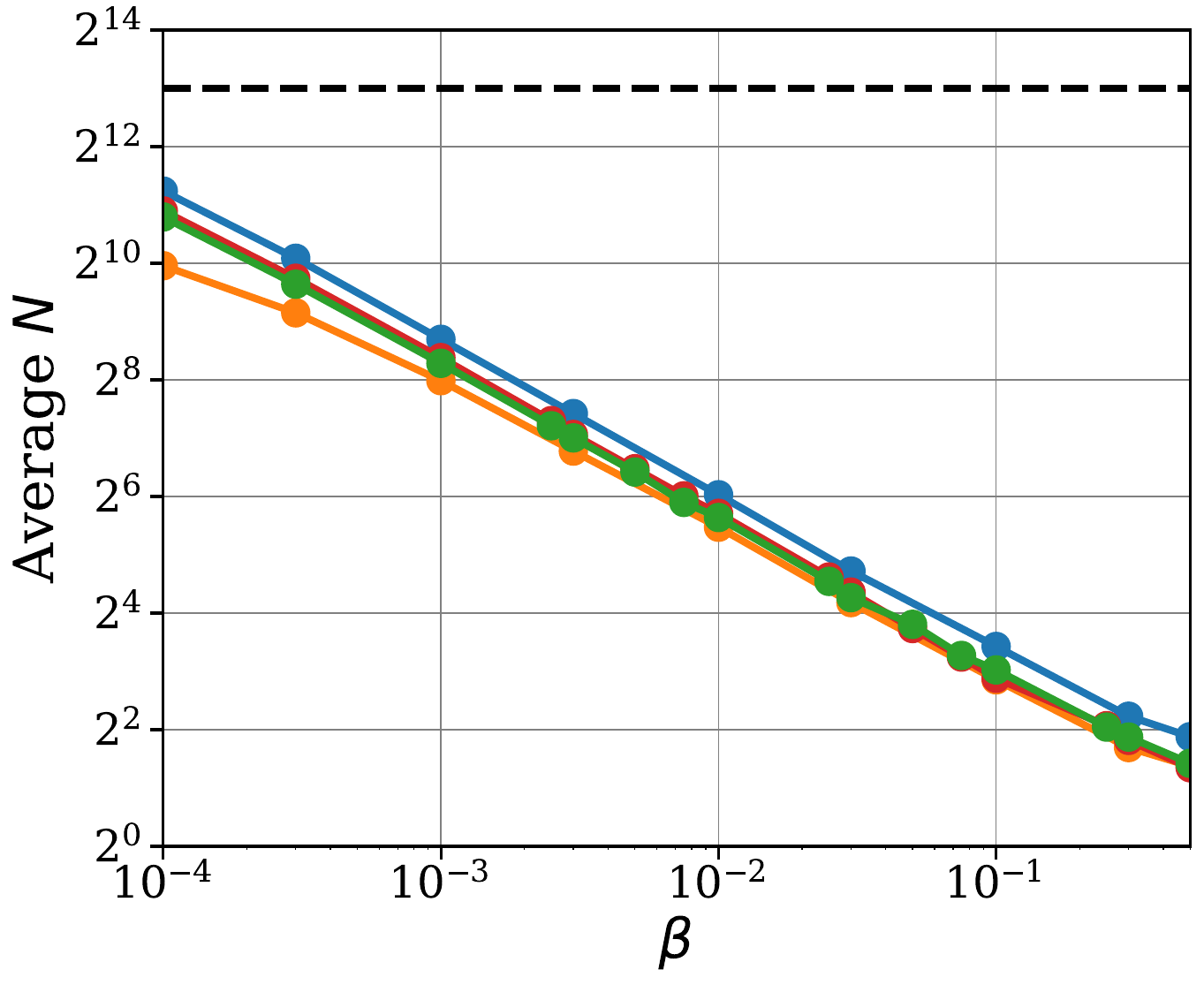}
        \label{sfig:betas-gsm8k-gemma-rm-N} 
      }
      \hfill \subfigure[\llamarm]{
        \includegraphics[width=0.2\textwidth]{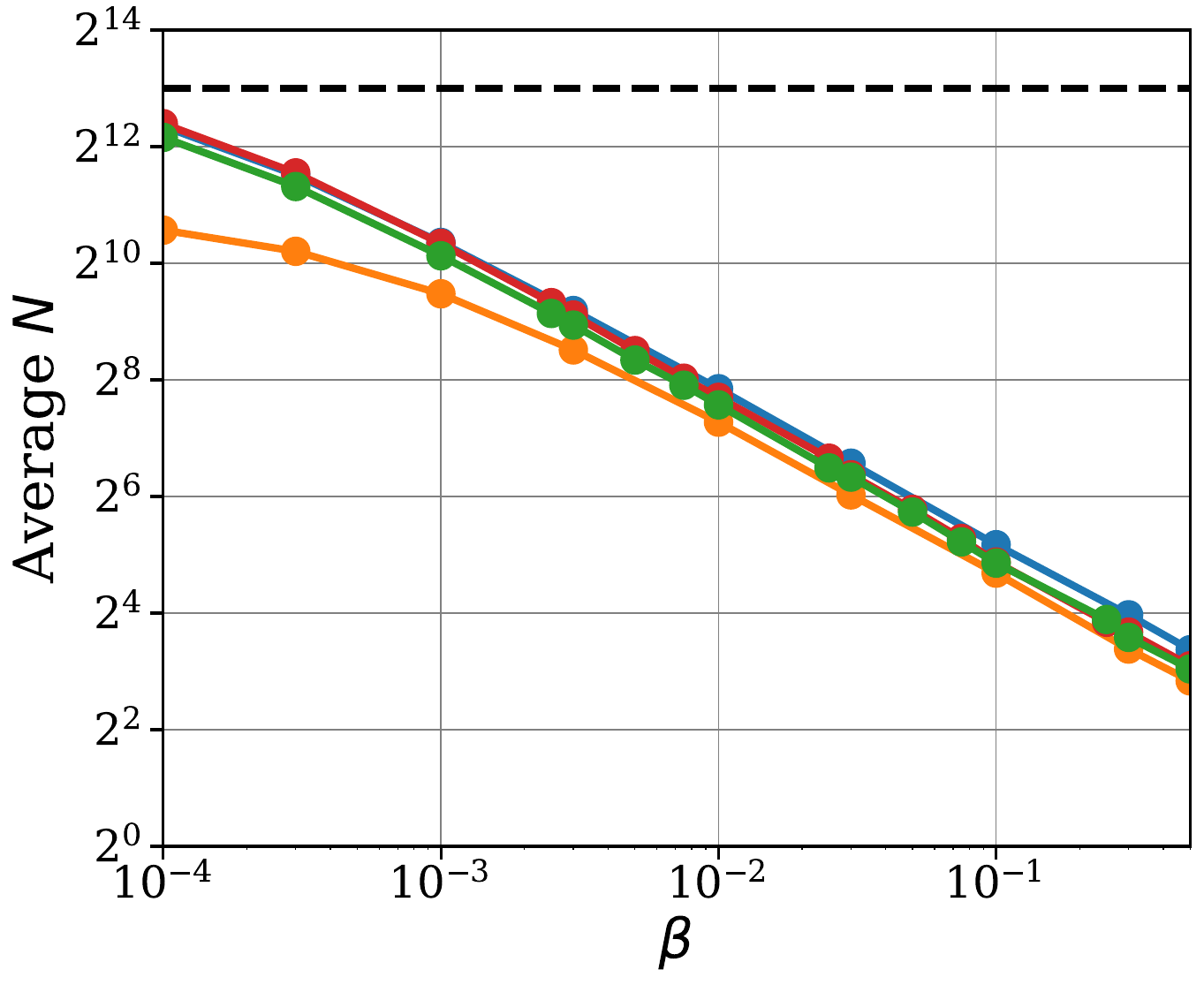}
        \label{sfig:betas-gsm8k-llama-3b-N} 
      }
      \hfill \subfigure[\armorm]{
        \includegraphics[width=0.2\textwidth]{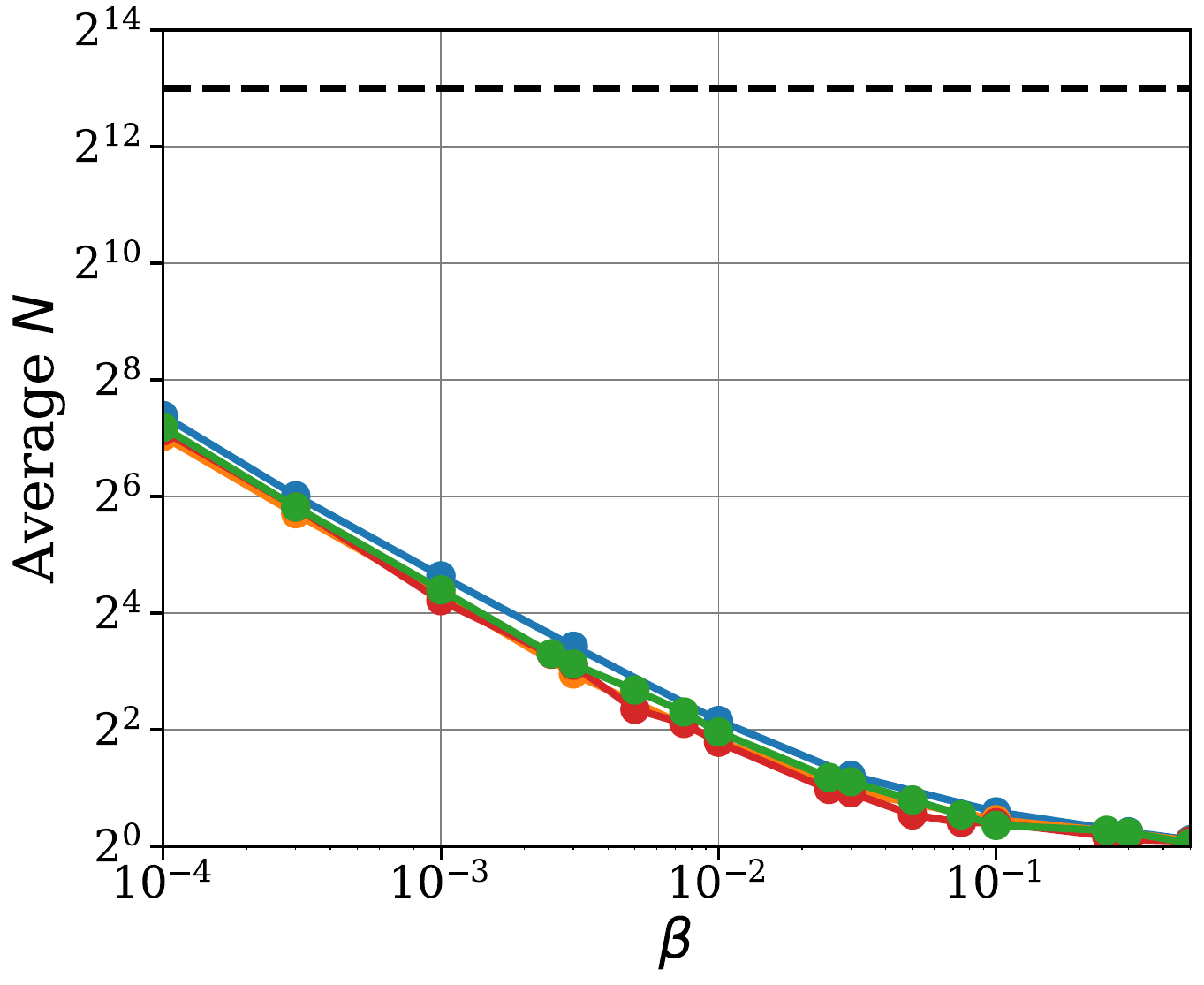}
        \label{sfig:betas-gsm8k-armo-rm-N} 
      }

      \subfigure[\oasst]{
        \includegraphics[width=0.2\textwidth]{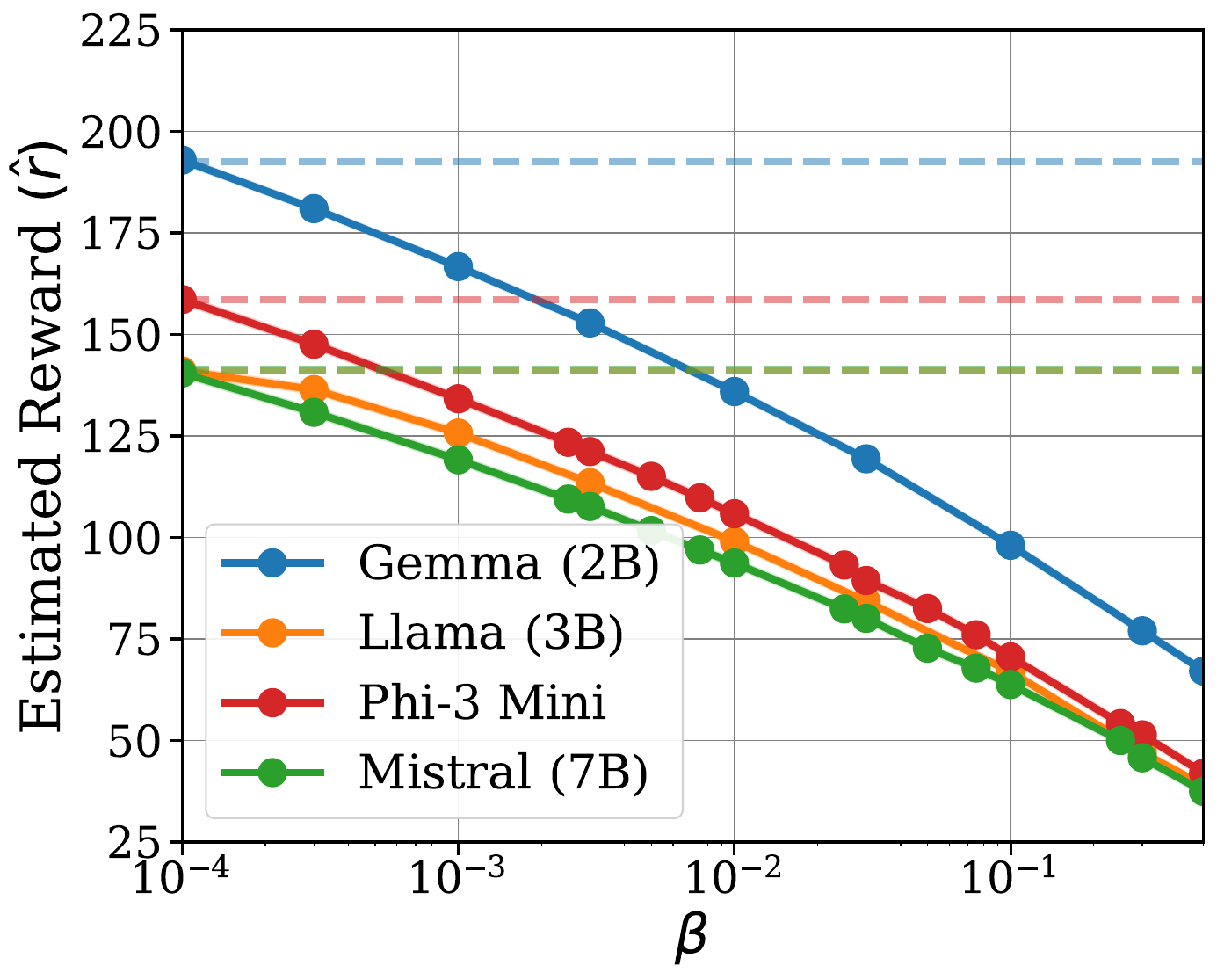}
        \label{sfig:betas-gsm8k-oasst-rm-rhat} 
      }
   \hfill \subfigure[\gemmarm]{
        \includegraphics[width=0.2\textwidth]{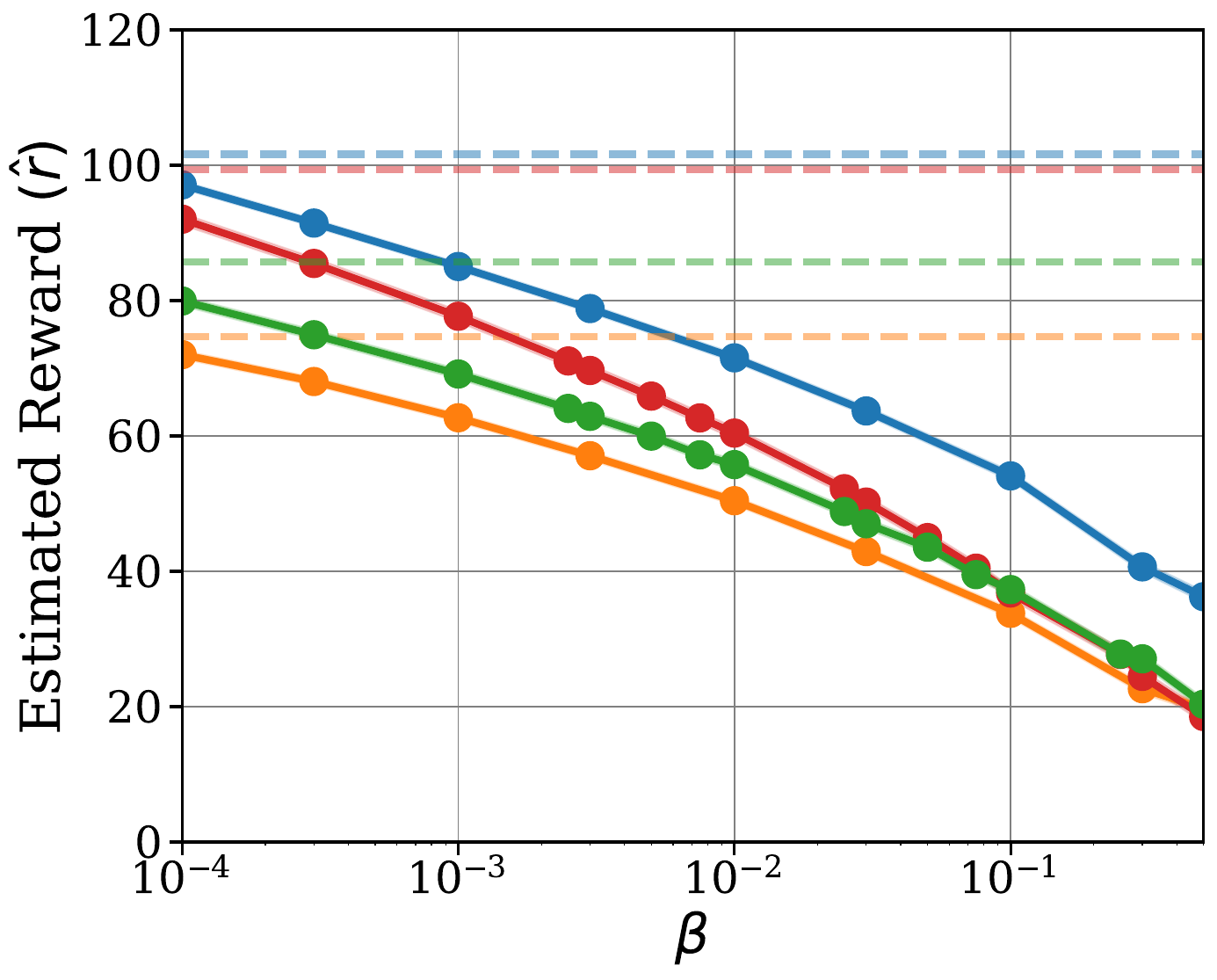}
        \label{sfig:betas-gsm8k-gemma-rhat} 
      }
      \hfill \subfigure[\llamarm]{
        \includegraphics[width=0.2\textwidth]{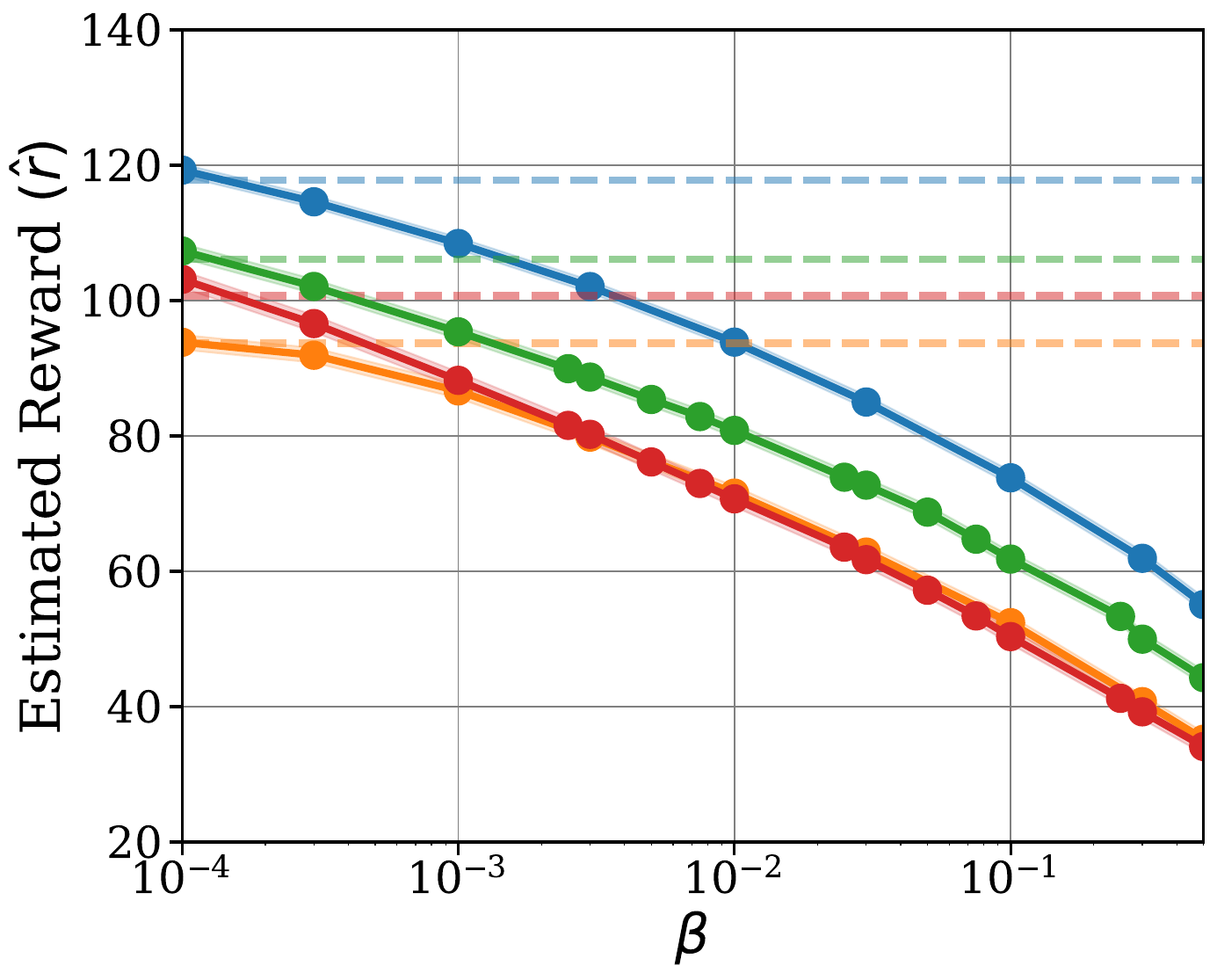}
        \label{sfig:betas-gsm8k-llama-rhat} 
      }
      \hfill \subfigure[\armorm]{
        \includegraphics[width=0.2\textwidth]{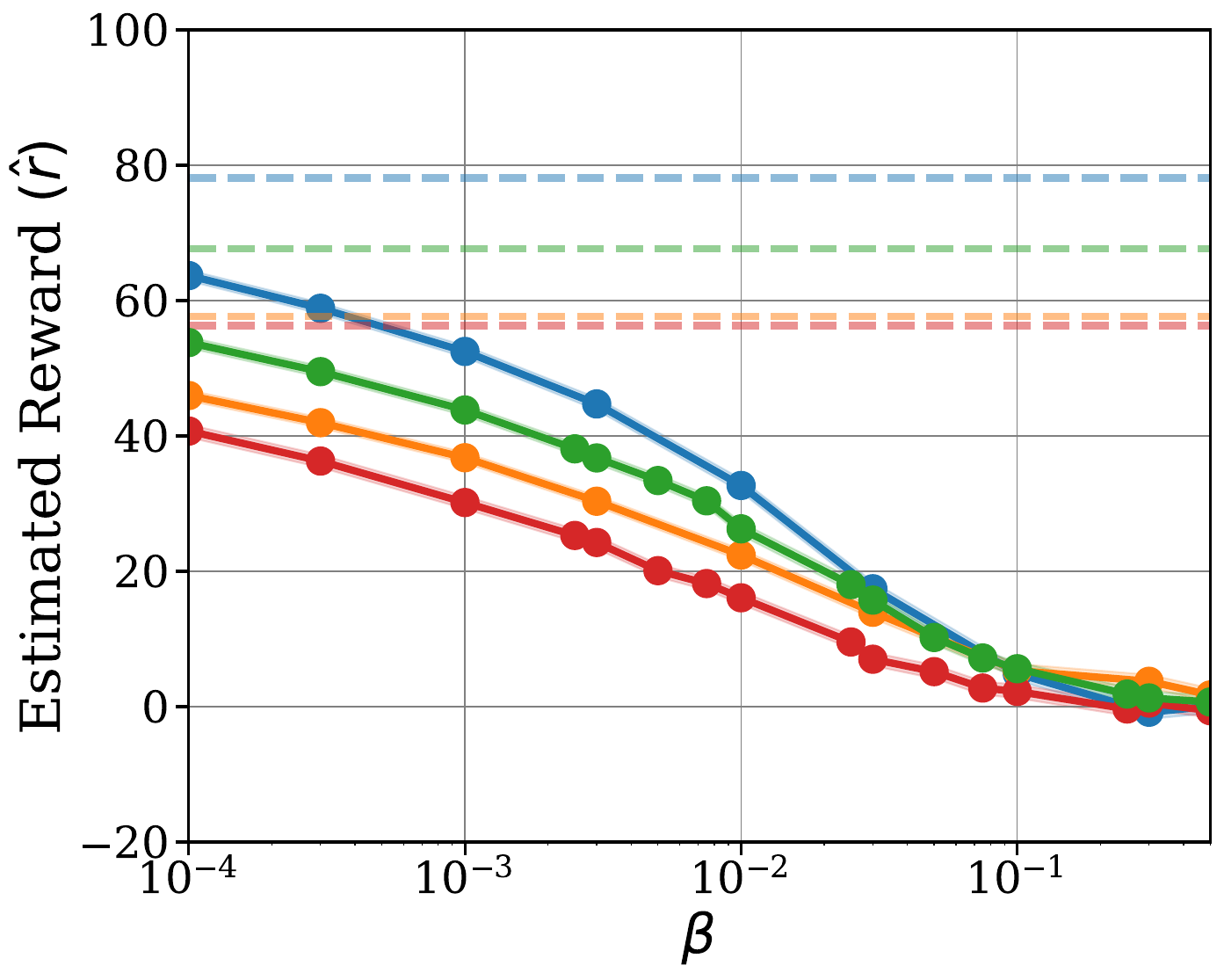}
        \label{sfig:betas-gsm8k-armo-rm-rhat} 
      }
    \caption{Compute-normalized comparison for $N = 2^{13}$ between \bonalg and \mainalg on \gsmk for four reward models and choices of $\piref$, as a function of regularization $\beta$.}
    \label{fig:gsm8k-betas}
  \end{figure*}

  \begin{figure*}[htp]
    \centering
    \subfigure[\oasst]{
        \includegraphics[width=0.2\textwidth]{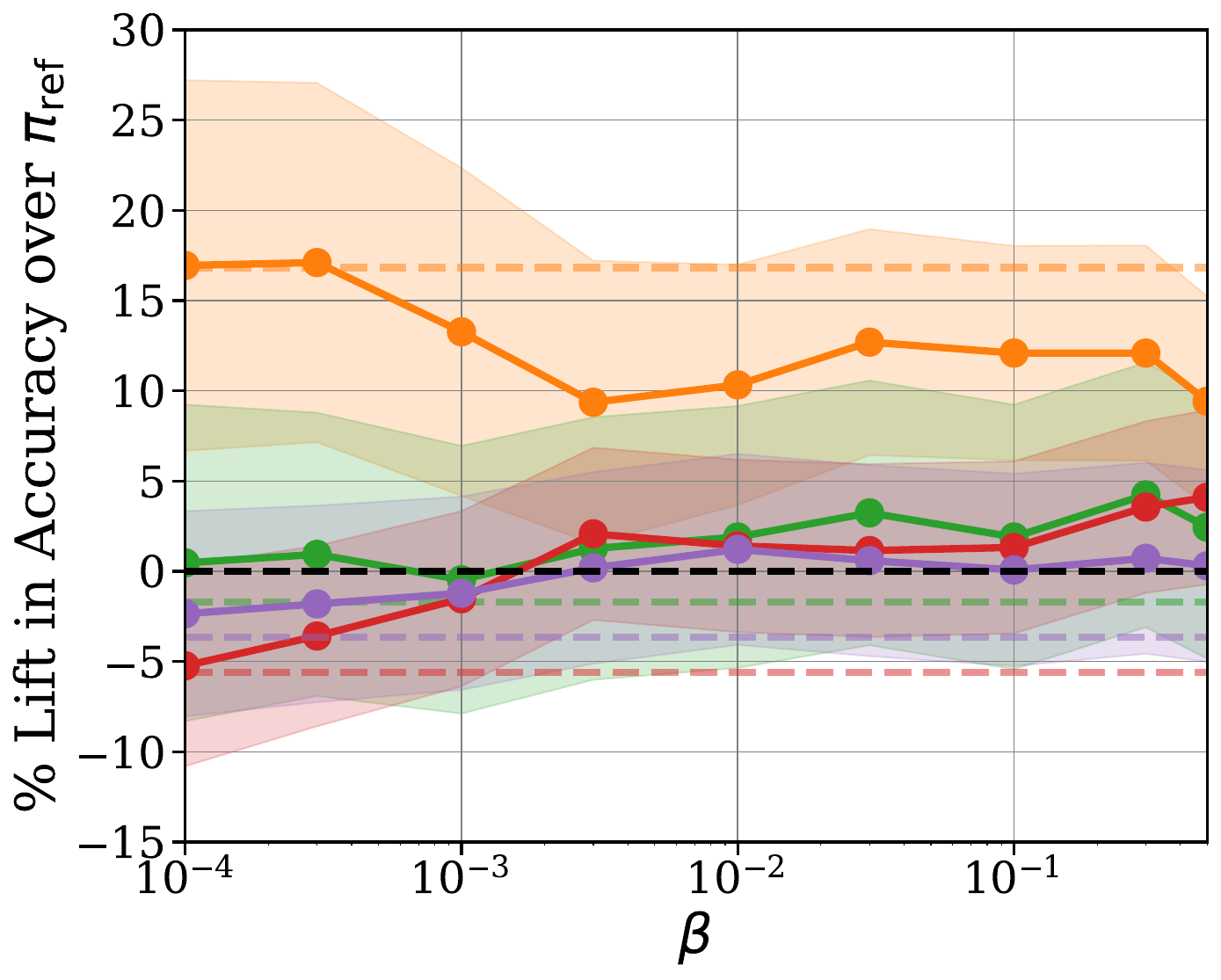}
        \label{sfig:betas-mmlu-oasst-rm} 
      }
   \hfill \subfigure[\gemmarm]{
        \includegraphics[width=0.2\textwidth]{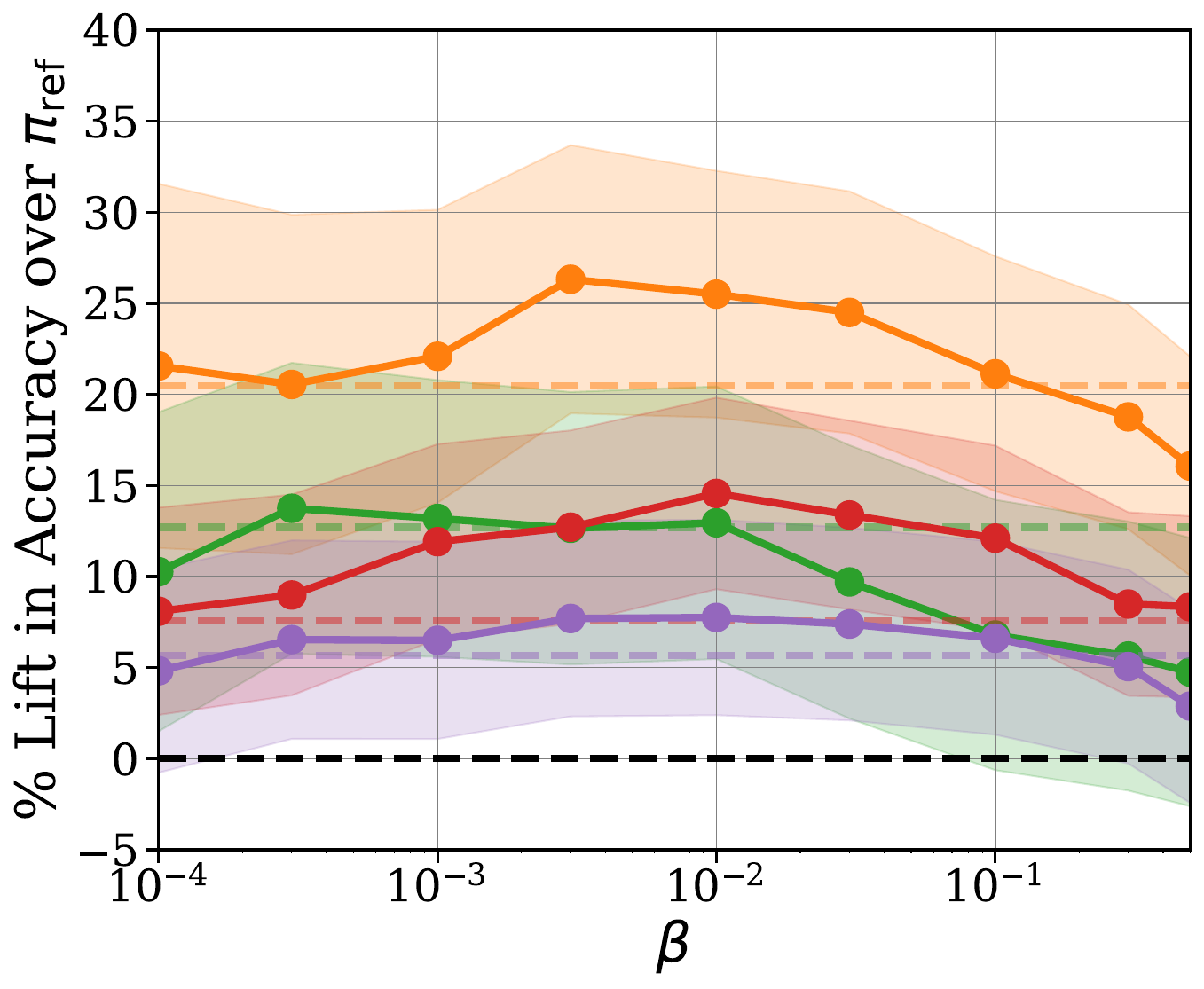}
        \label{sfig:betas-mmlu-gemma-rm} 
      }
      \hfill \subfigure[\llamarm]{
        \includegraphics[width=0.2\textwidth]{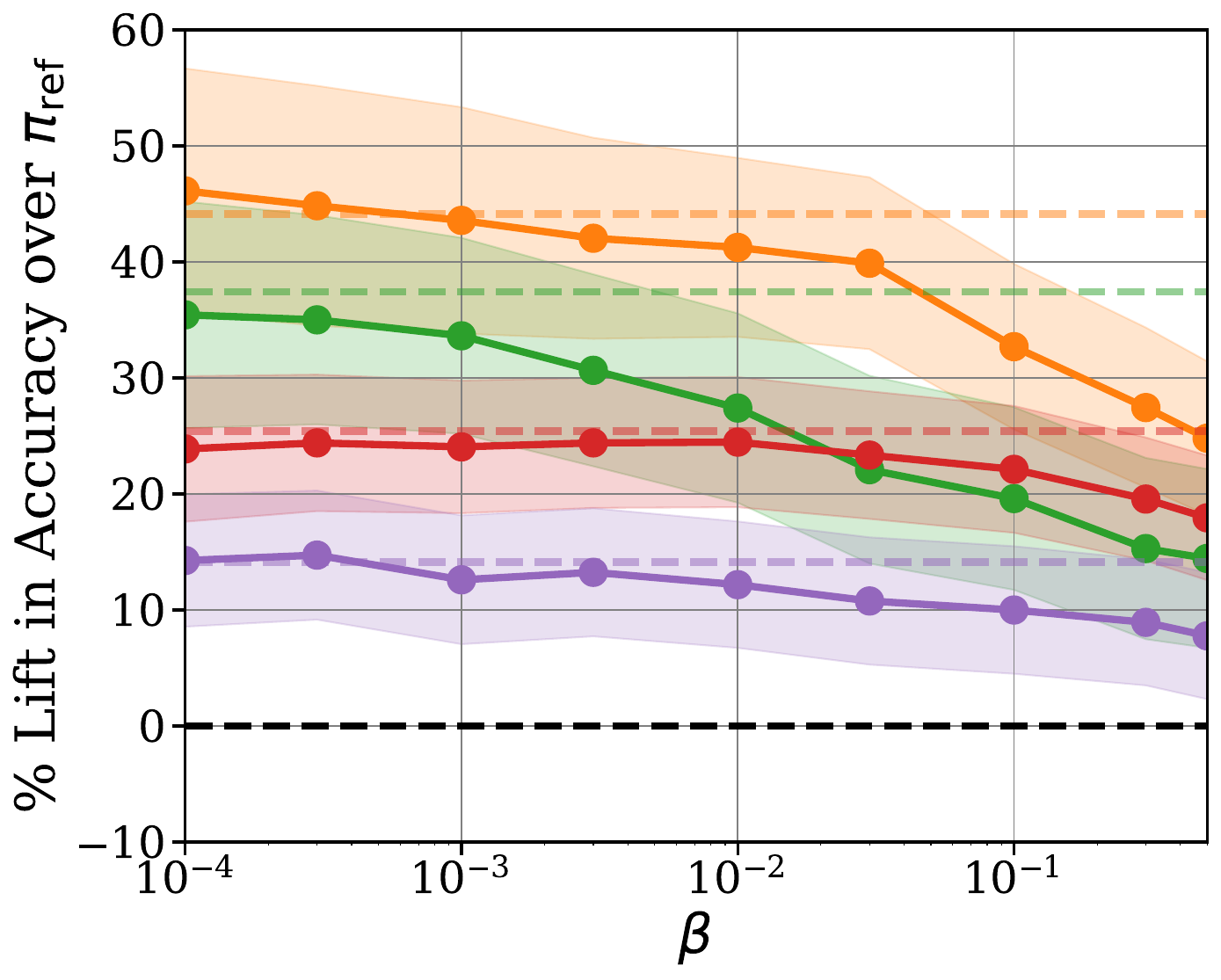}
        \label{sfig:betas-mmlu-llama-3b} 
      }
      \hfill \subfigure[\armorm]{
        \includegraphics[width=0.2\textwidth]{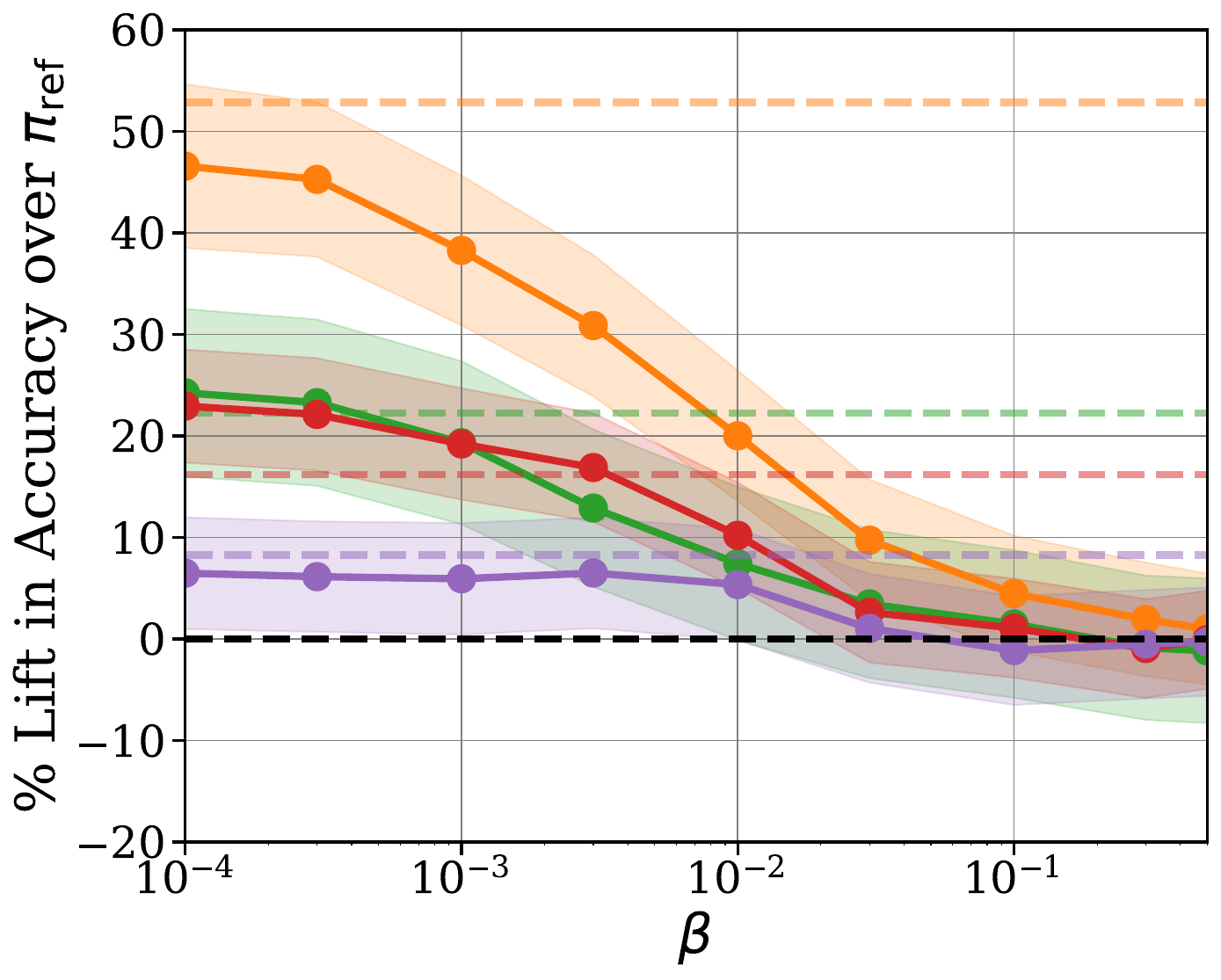}
        \label{sfig:betas-mmlu-armo-rm} 
      }
      
      \subfigure[\oasst]{
        \includegraphics[width=0.2\textwidth]{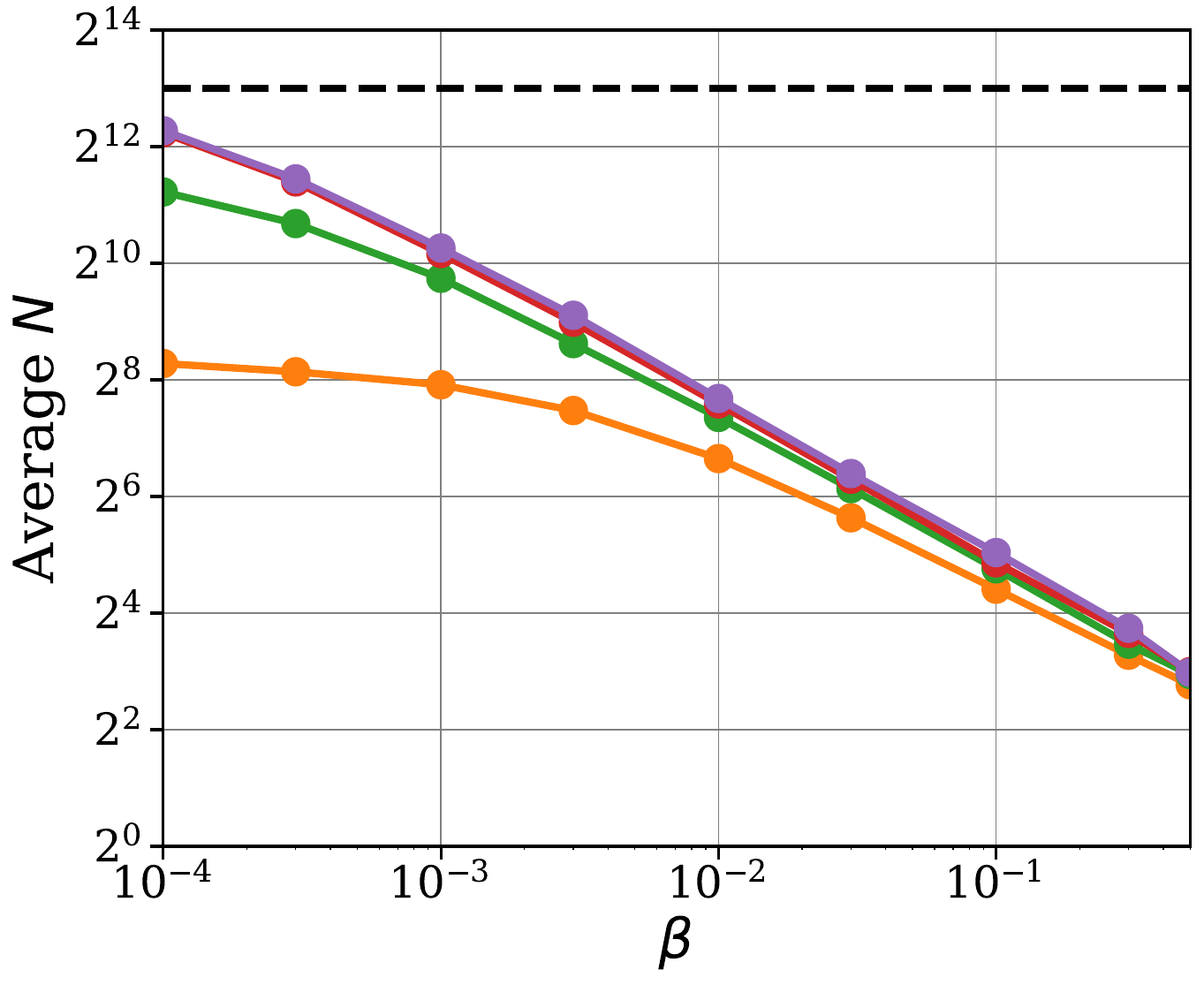}
        \label{sfig:betas-mmlu-oasst-rm-N} 
      }
   \hfill \subfigure[\gemmarm]{
        \includegraphics[width=0.2\textwidth]{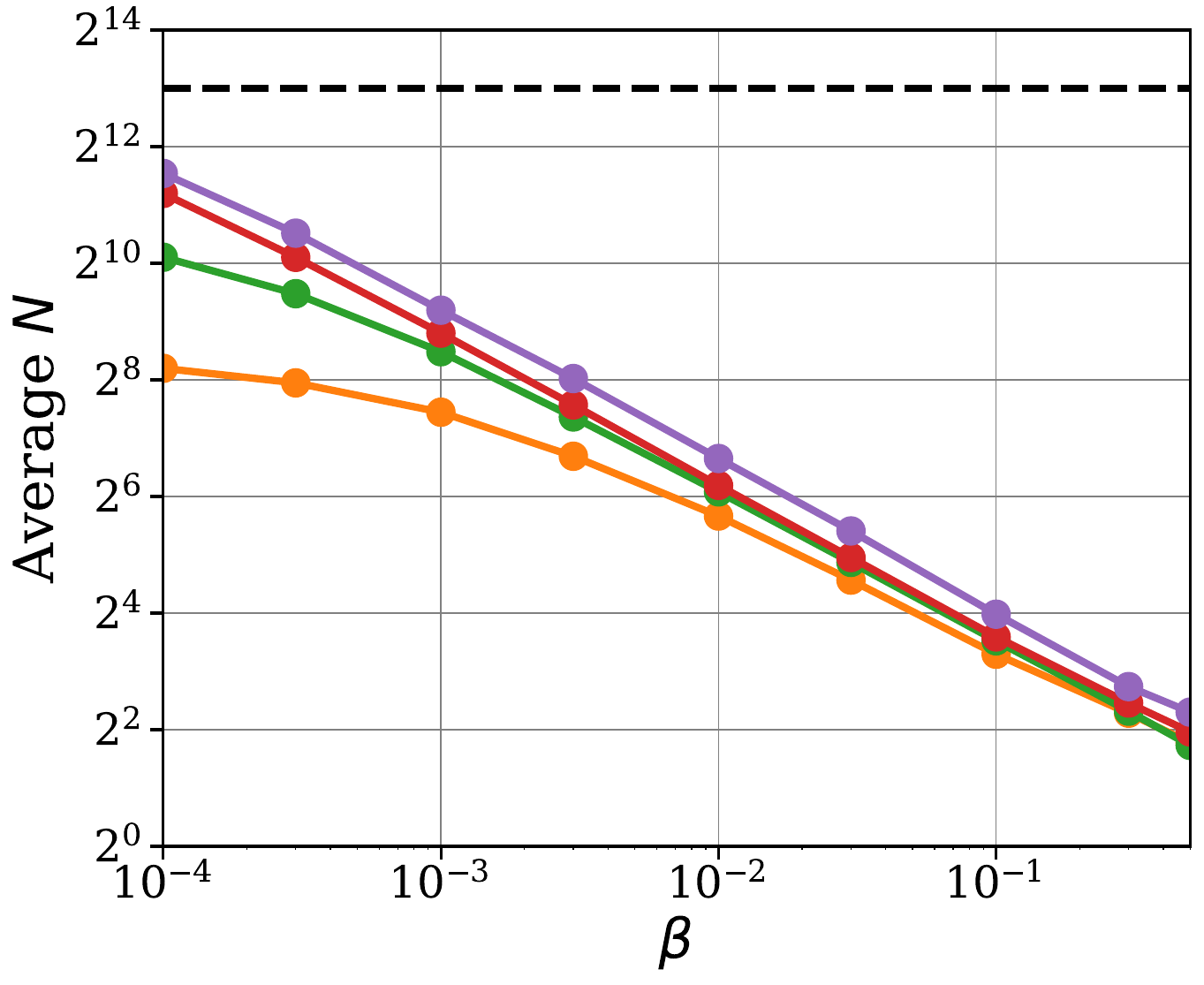}
        \label{sfig:betas-mmlu-gemma-rm-N} 
      }
      \hfill \subfigure[\llamarm]{
        \includegraphics[width=0.2\textwidth]{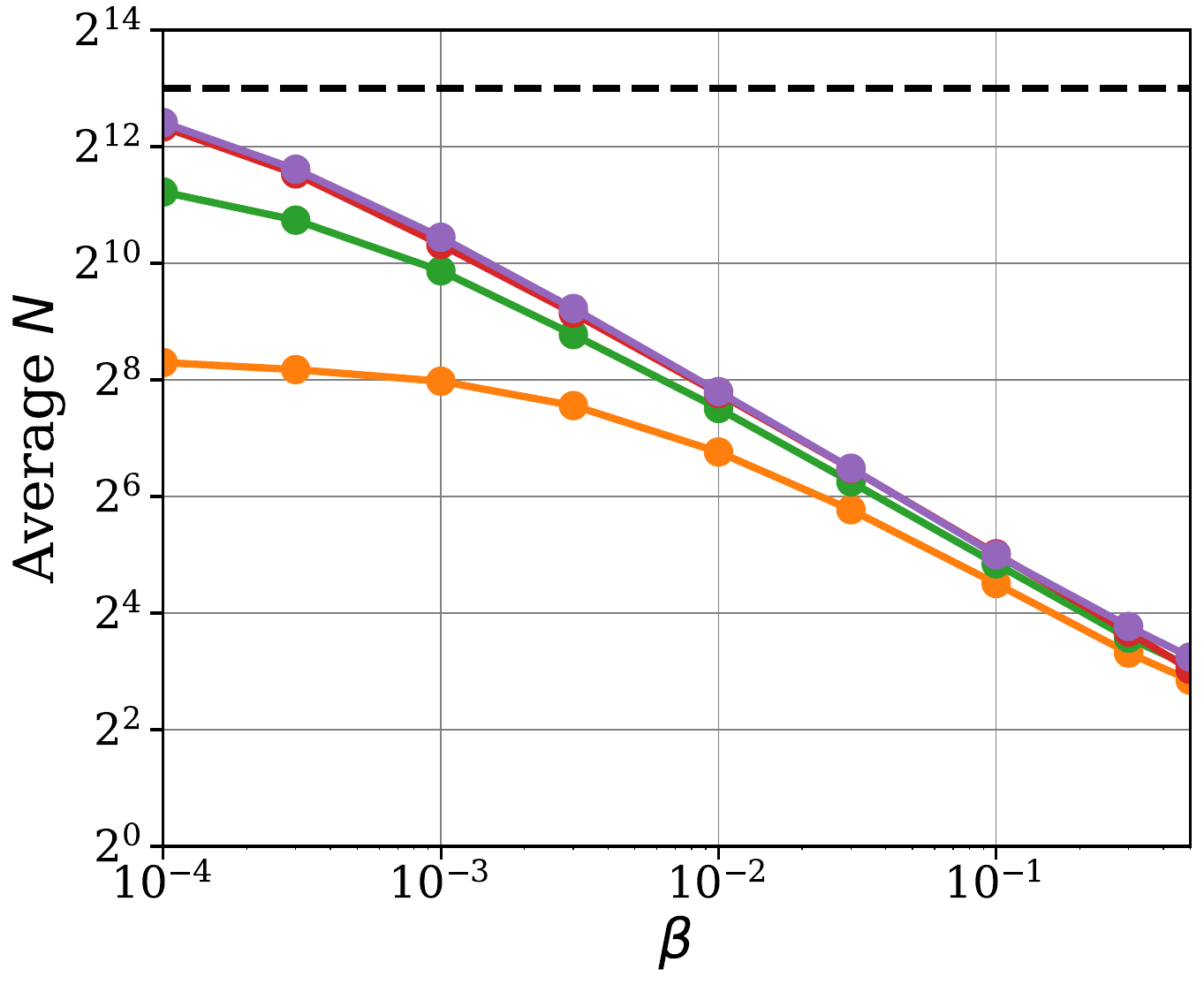}
        \label{sfig:betas-mmlu-llama-3b-N} 
      }
      \hfill \subfigure[\armorm]{
        \includegraphics[width=0.2\textwidth]{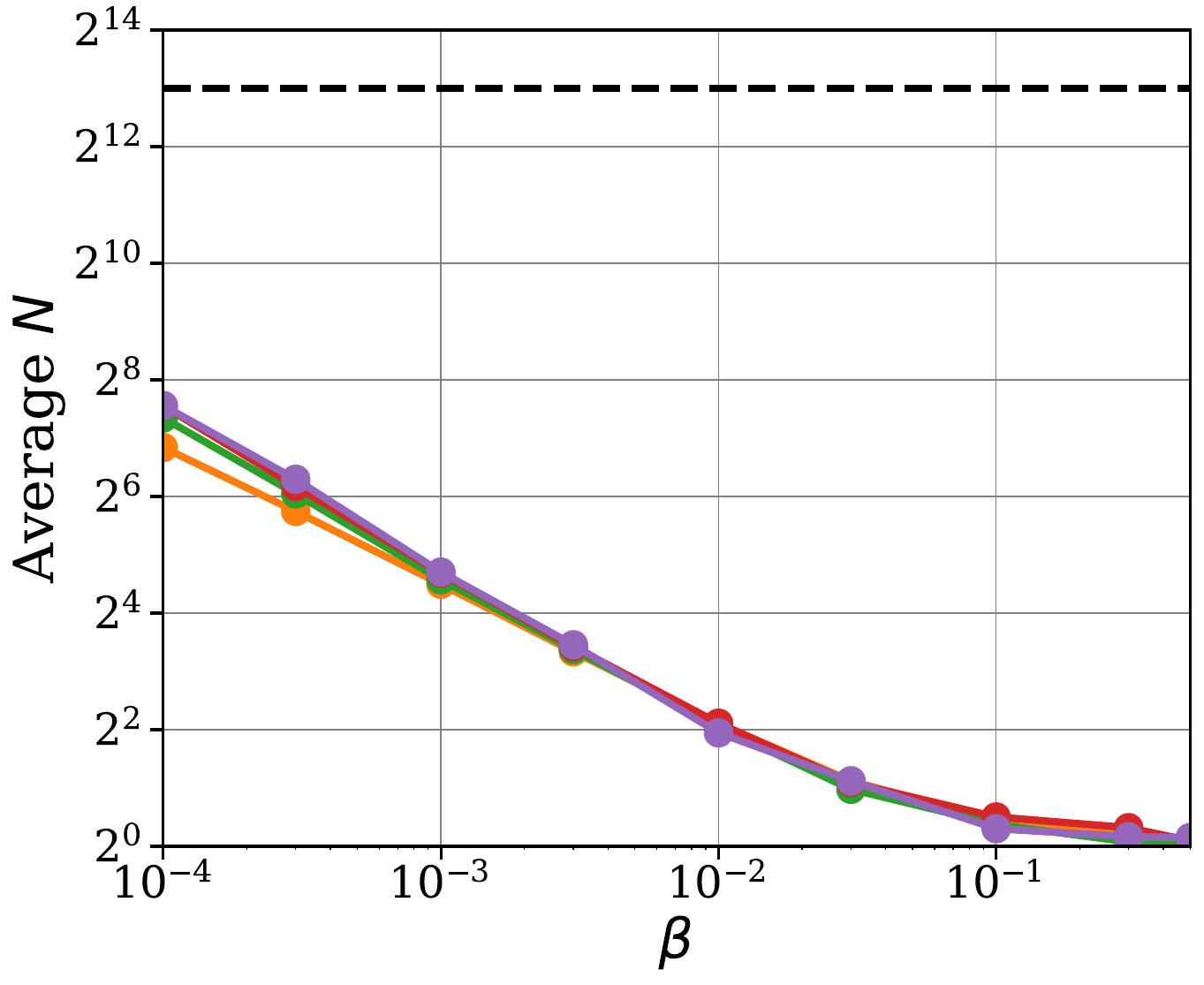}
        \label{sfig:betas-mmlu-armo-rm-N} 
      }

      \subfigure[\oasst]{
        \includegraphics[width=0.2\textwidth]{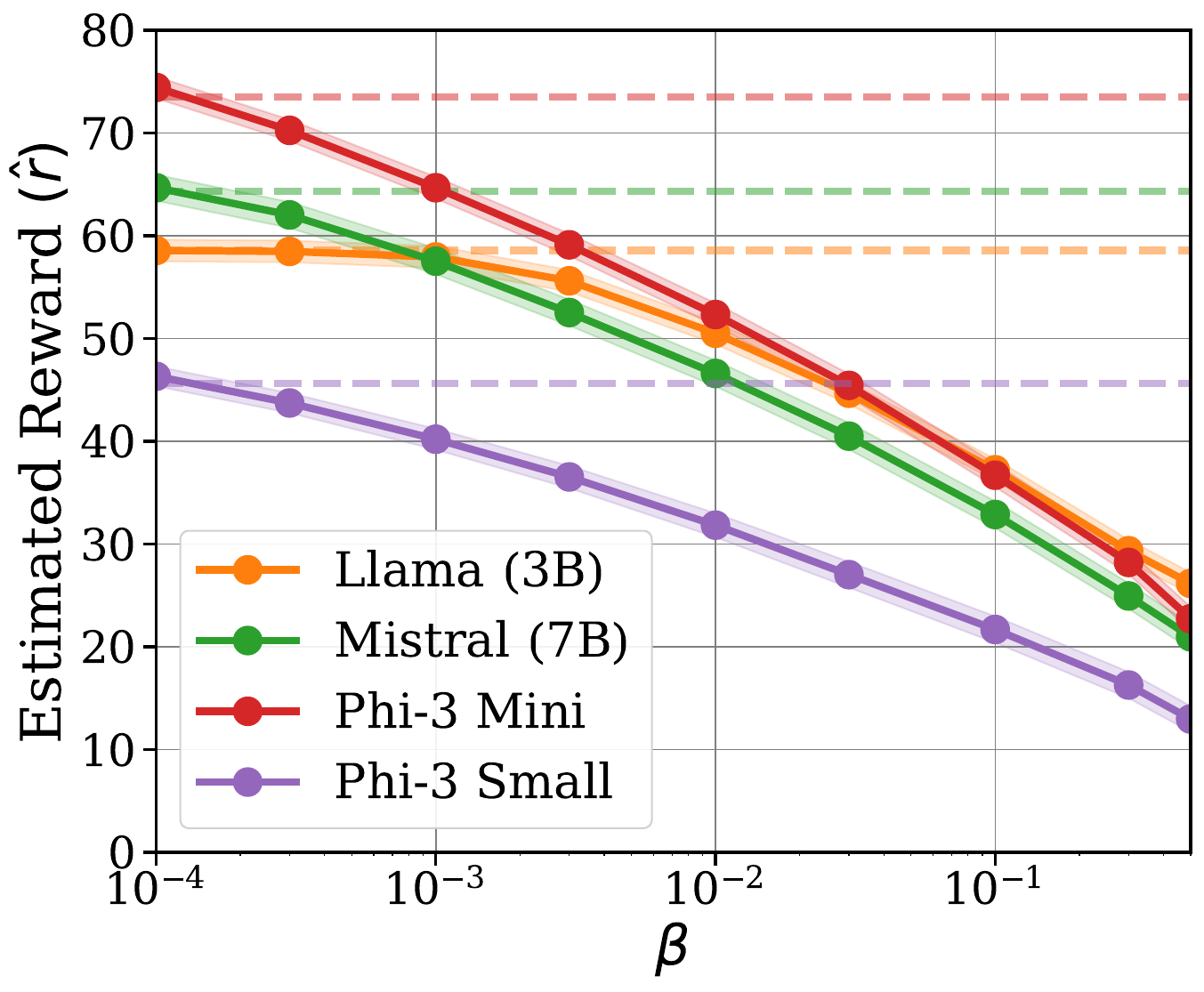}
        \label{sfig:betas-mmlu-oasst-rm-rhat} 
      }
   \hfill \subfigure[\gemmarm]{
        \includegraphics[width=0.2\textwidth]{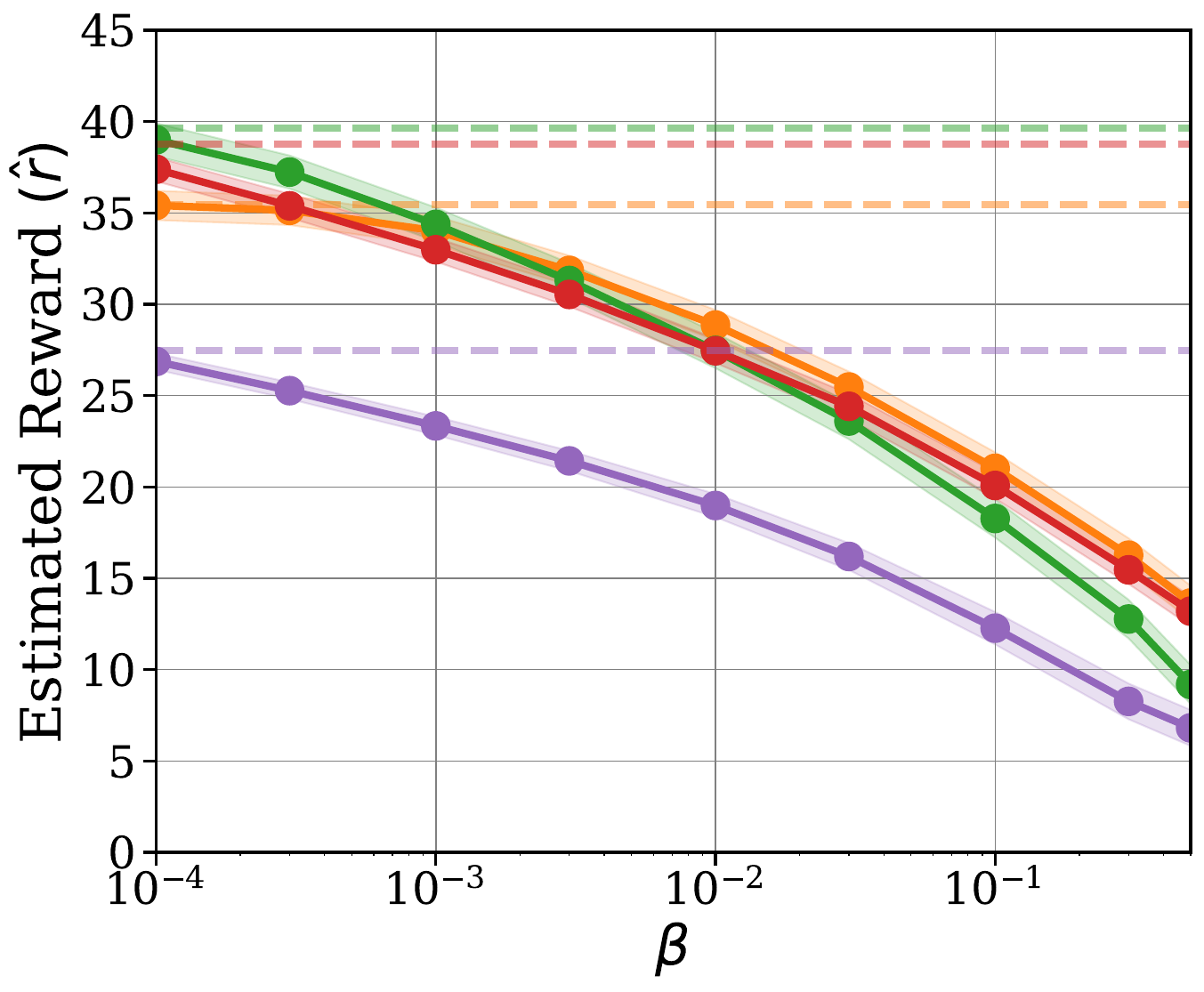}
        \label{sfig:betas-mmlu-gemma-rhat} 
      }
      \hfill \subfigure[\llamarm]{
        \includegraphics[width=0.2\textwidth]{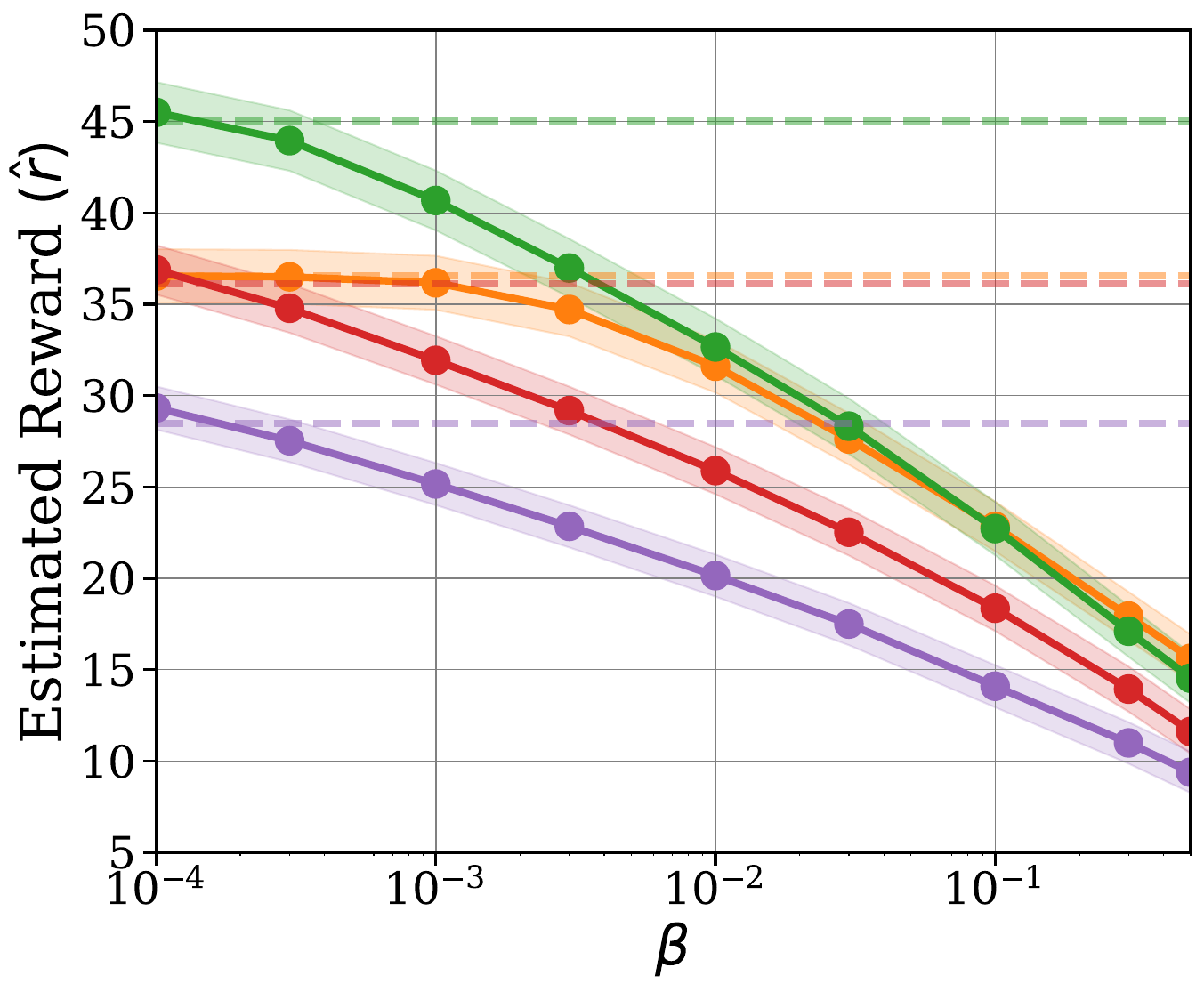}
        \label{sfig:betas-mmlu-llama-rhat} 
      }
      \hfill \subfigure[\armorm]{
        \includegraphics[width=0.2\textwidth]{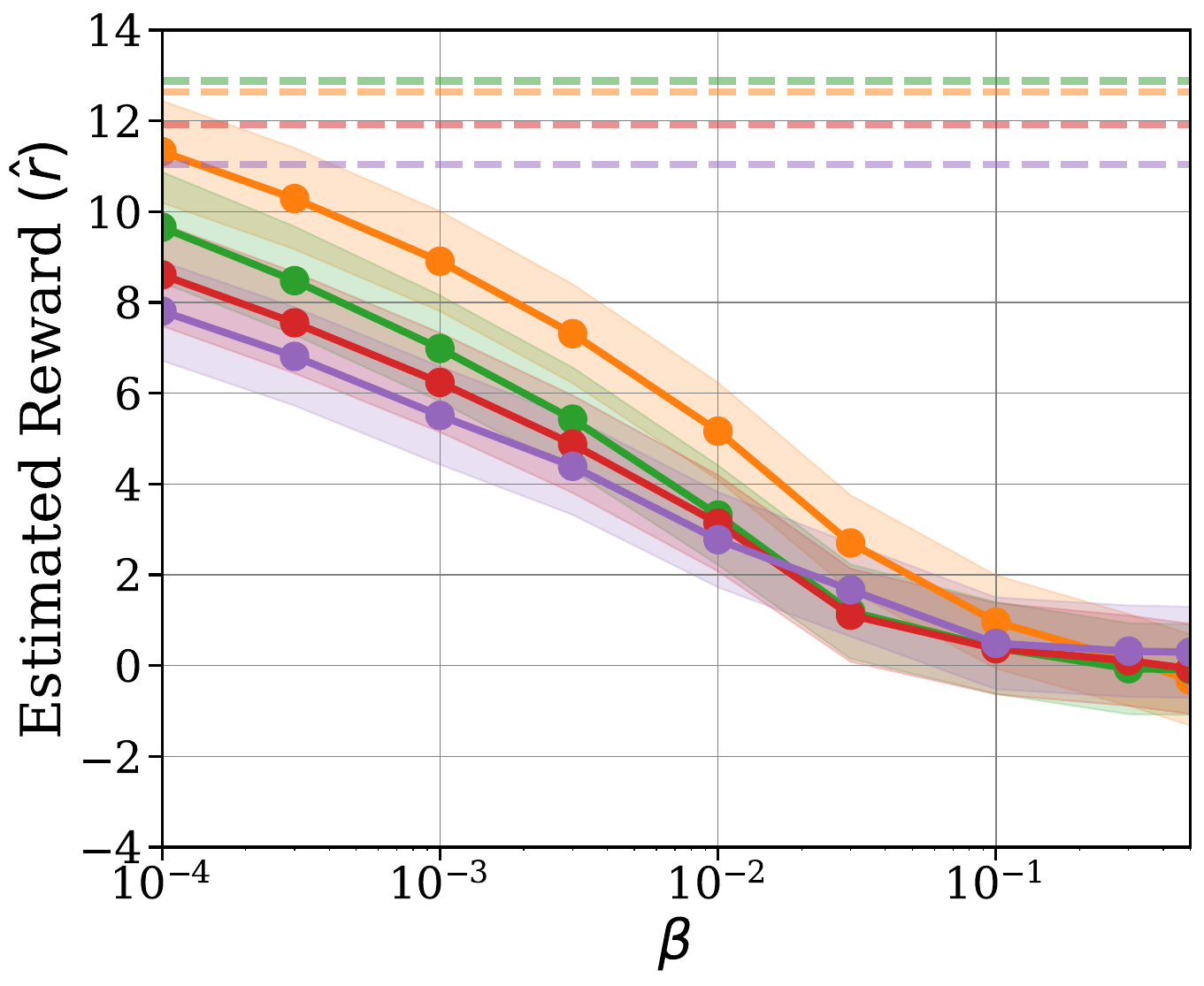}
        \label{sfig:betas-mmlu-armo-rm-rhat} 
      }
    \caption{Compute-normalized comparison for $N = 2^{13}$ between \bonalg and \mainalg on \mmlu for four reward models and choices of $\piref$, as a function of regularization $\beta$.}
    \label{fig:mmlu-betas}
  \end{figure*}

  \begin{figure*}[htp]
    \centering
    \subfigure[\oasst]{
        \includegraphics[width=0.2\textwidth]{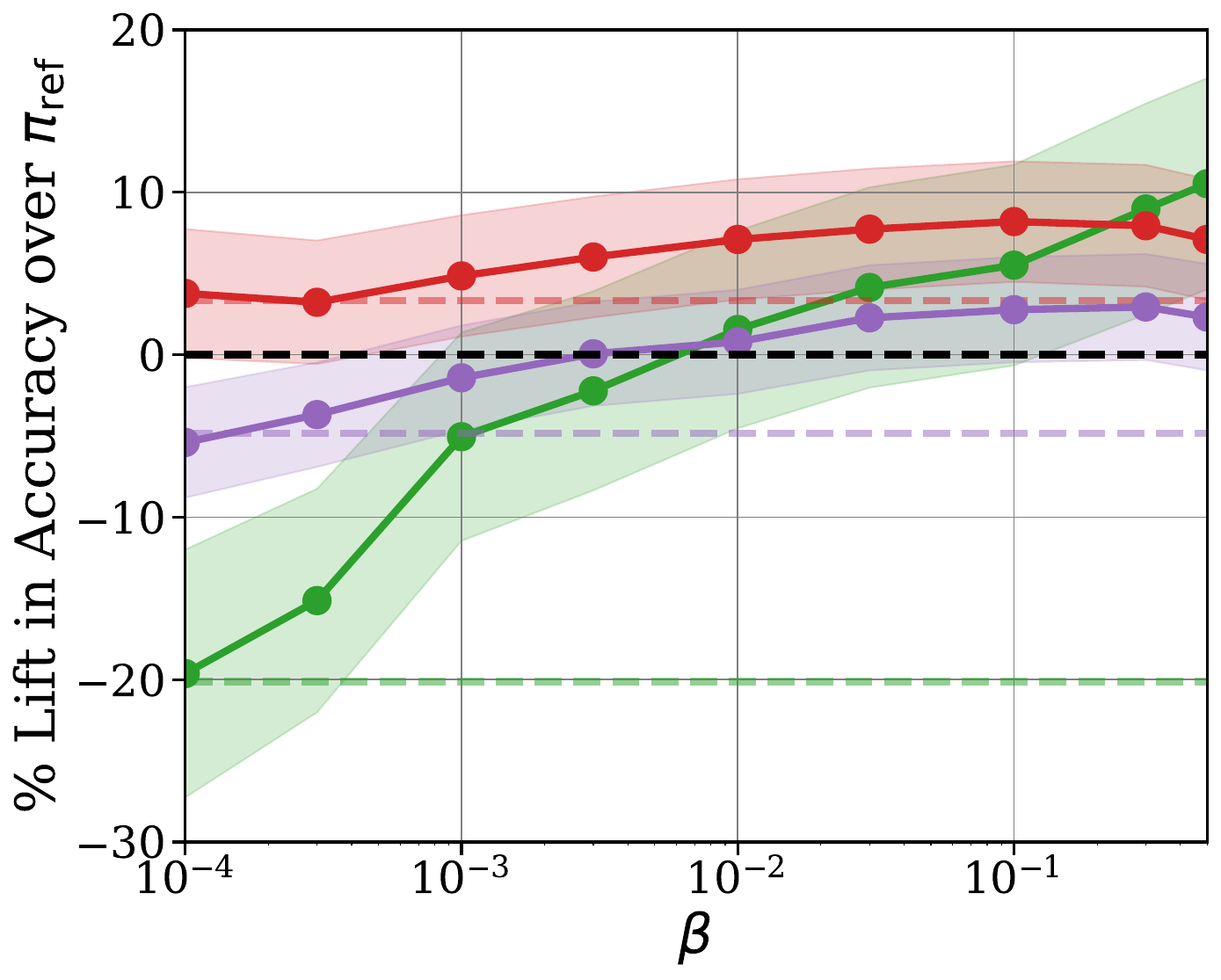}
        \label{sfig:betas-math-oasst-rm} 
      }
   \hfill \subfigure[\gemmarm]{
        \includegraphics[width=0.2\textwidth]{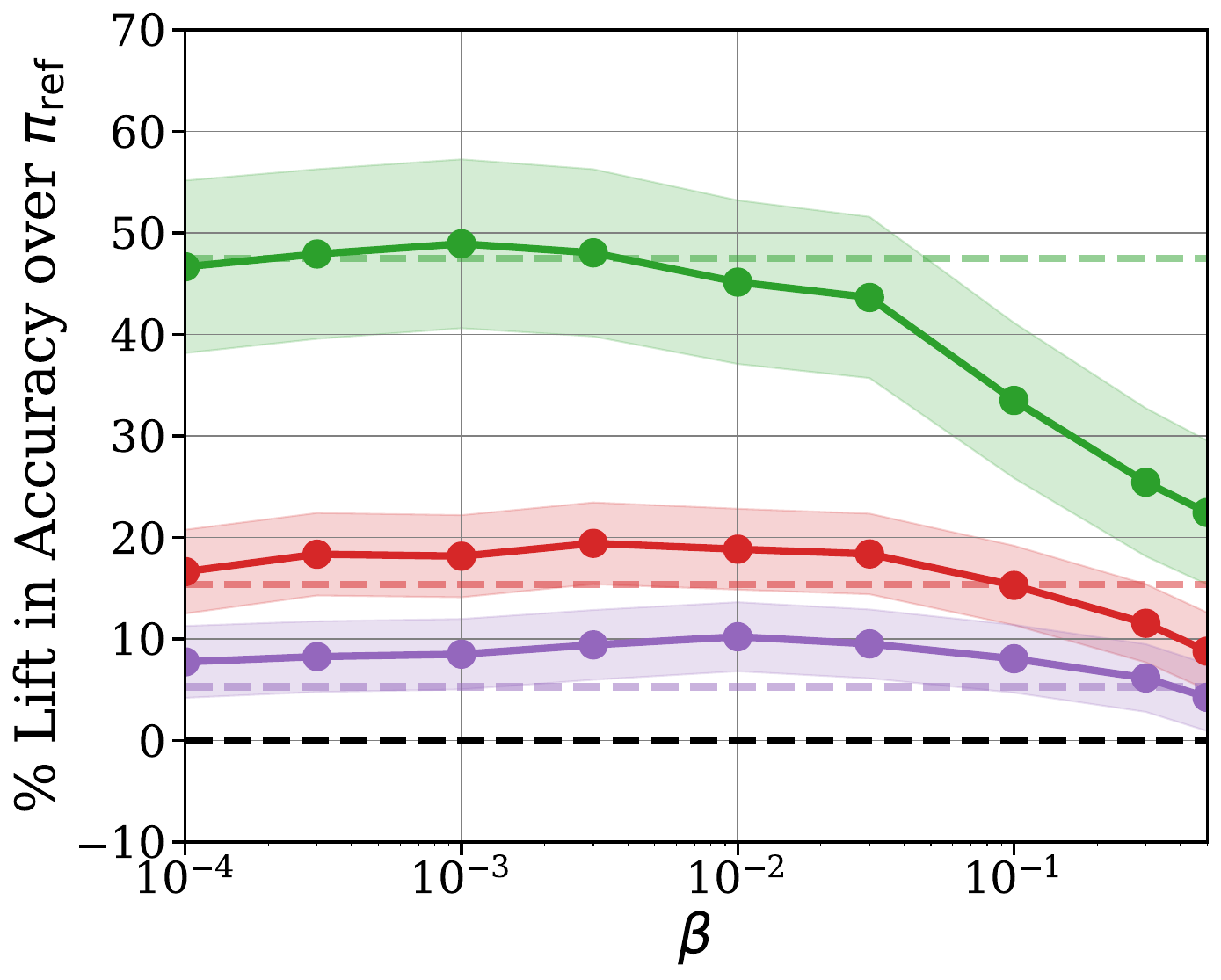}
        \label{sfig:betas-math-gemma-rm} 
      }
      \hfill \subfigure[\llamarm]{
        \includegraphics[width=0.2\textwidth]{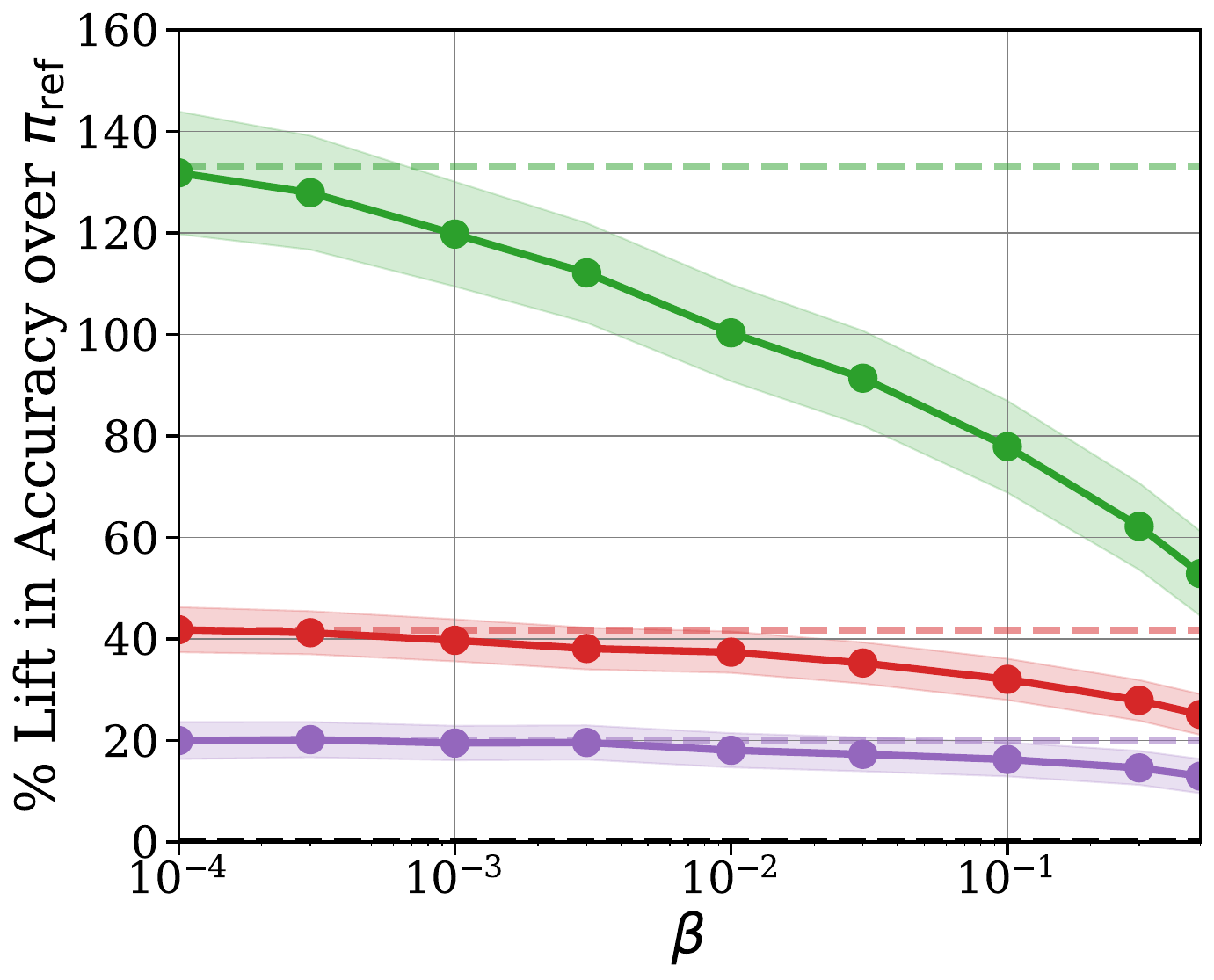}
        \label{sfig:betas-math-llama-3b} 
      }
      \hfill \subfigure[\armorm]{
        \includegraphics[width=0.2\textwidth]{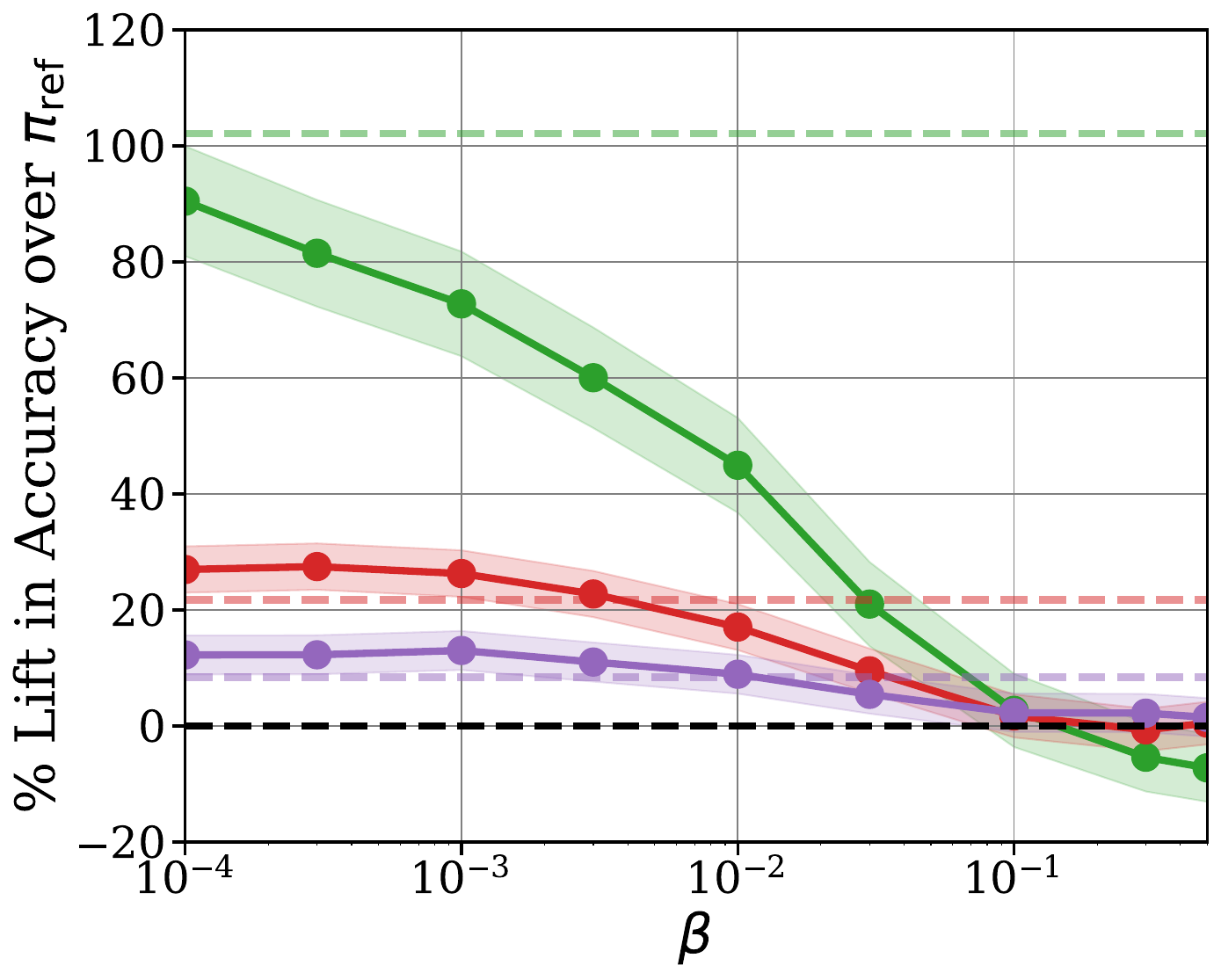}
        \label{sfig:betas-math-armo-rm} 
      }
    \hfill
      \subfigure[\oasst]{
        \includegraphics[width=0.2\textwidth]{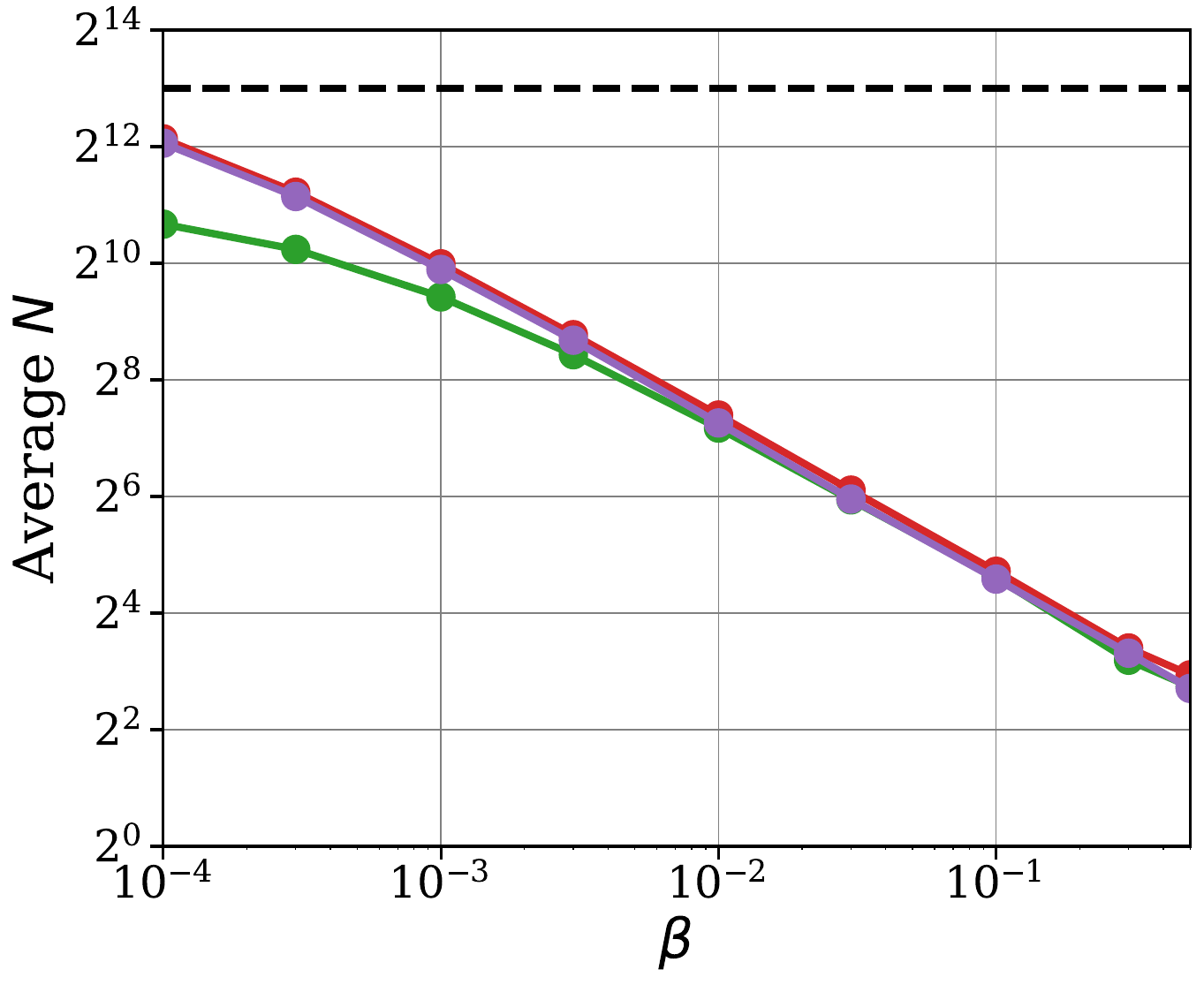}
        \label{sfig:betas-math-oasst-rm-N} 
      }
   \hfill \subfigure[\gemmarm]{
        \includegraphics[width=0.2\textwidth]{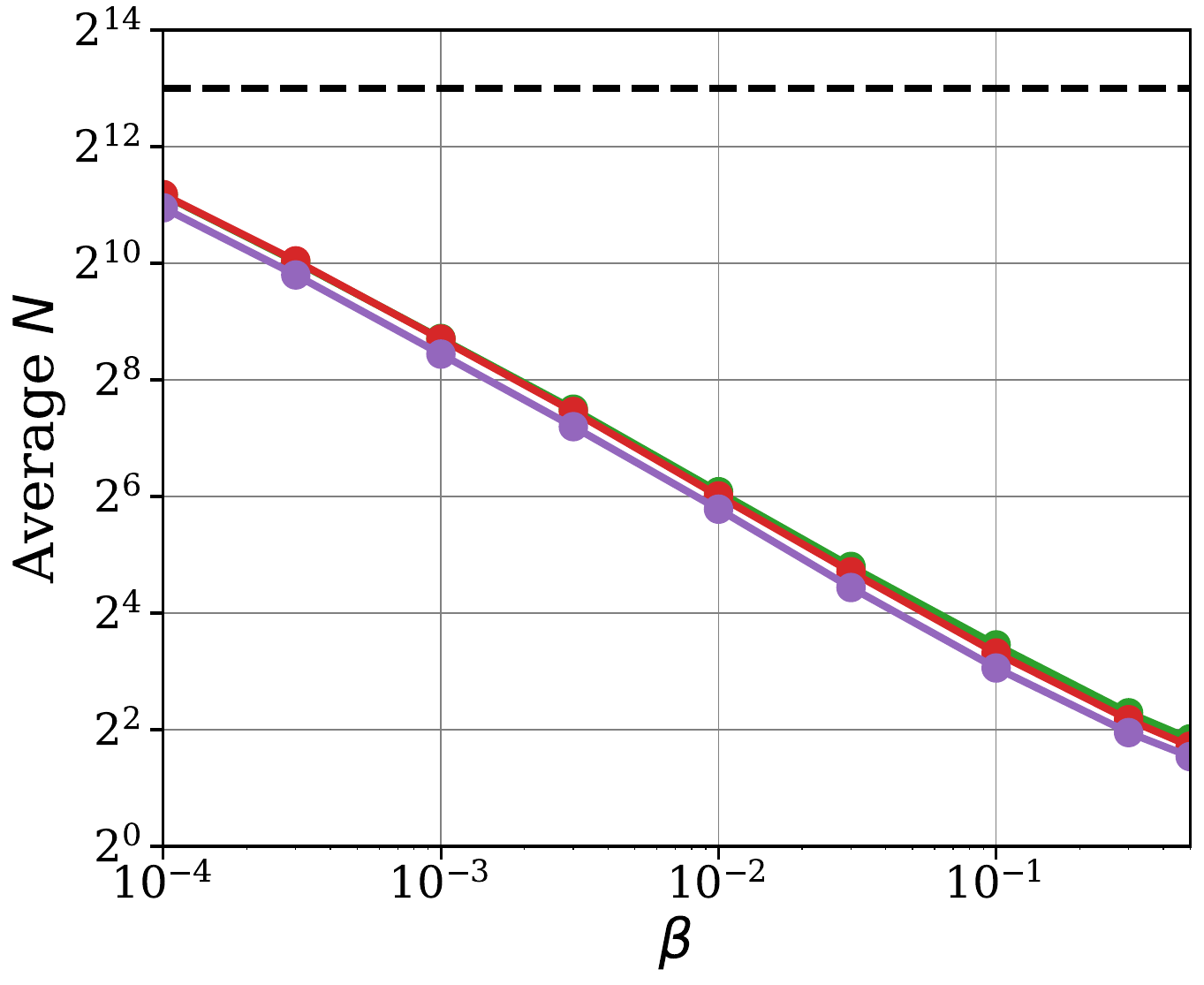}
        \label{sfig:betas-math-gemma-rm-N} 
      }
      \hfill \subfigure[\llamarm]{
        \includegraphics[width=0.2\textwidth]{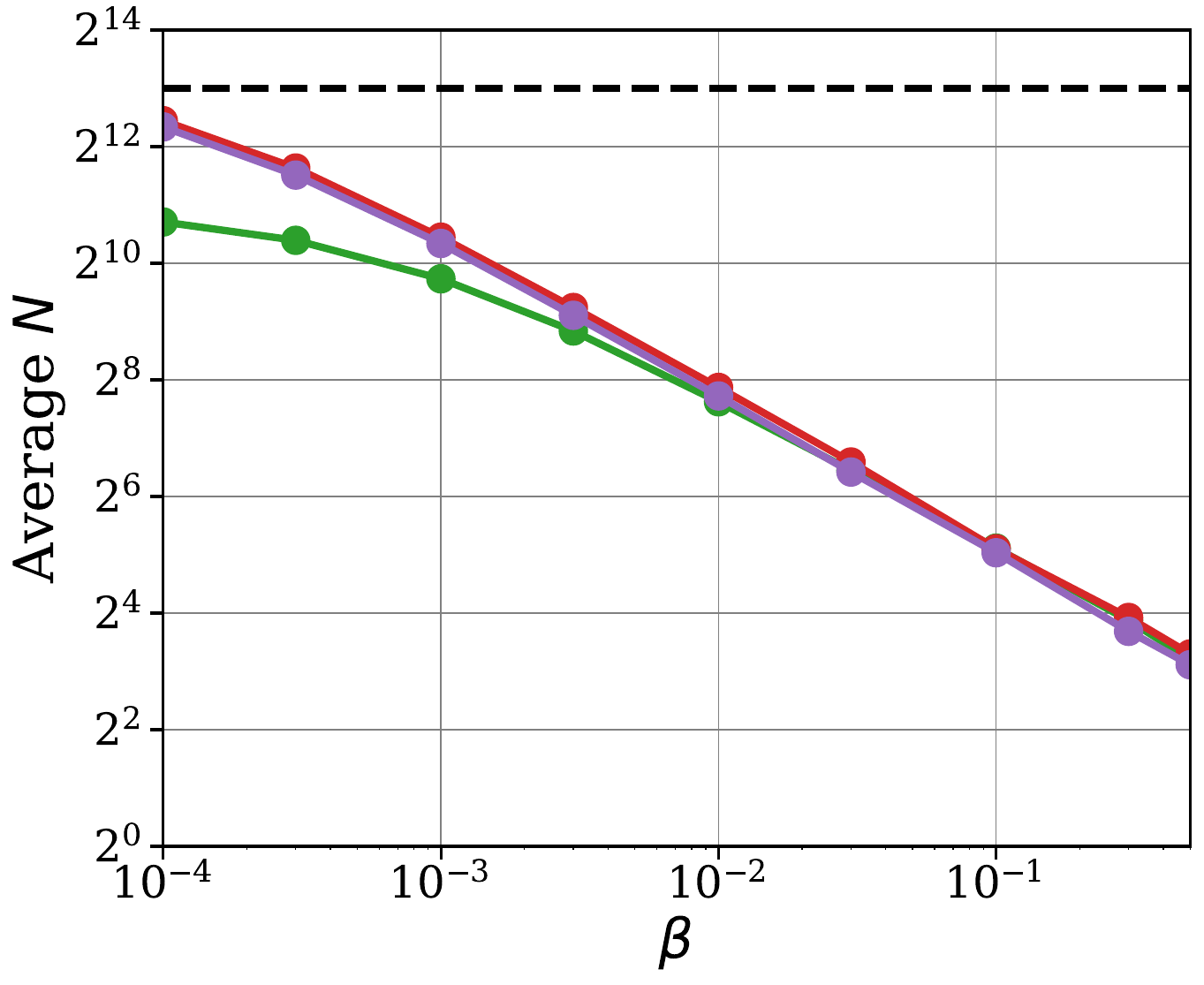}
        \label{sfig:betas-math-llama-3b-N} 
      }
      \hfill \subfigure[\armorm]{
        \includegraphics[width=0.2\textwidth]{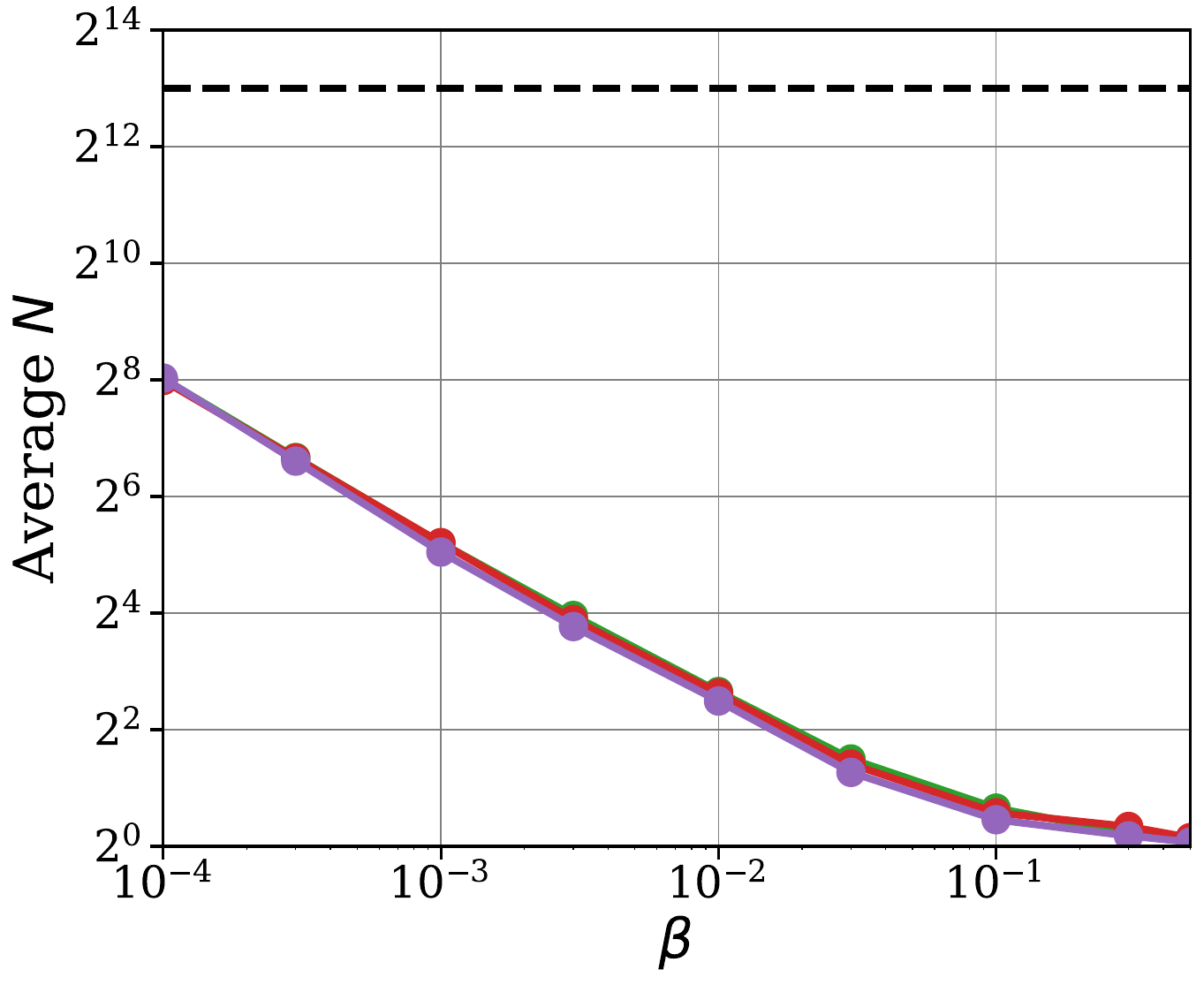}
        \label{sfig:betas-math-armo-rm-N} 
      }
      \hfill
      \subfigure[\oasst]{
        \includegraphics[width=0.2\textwidth]{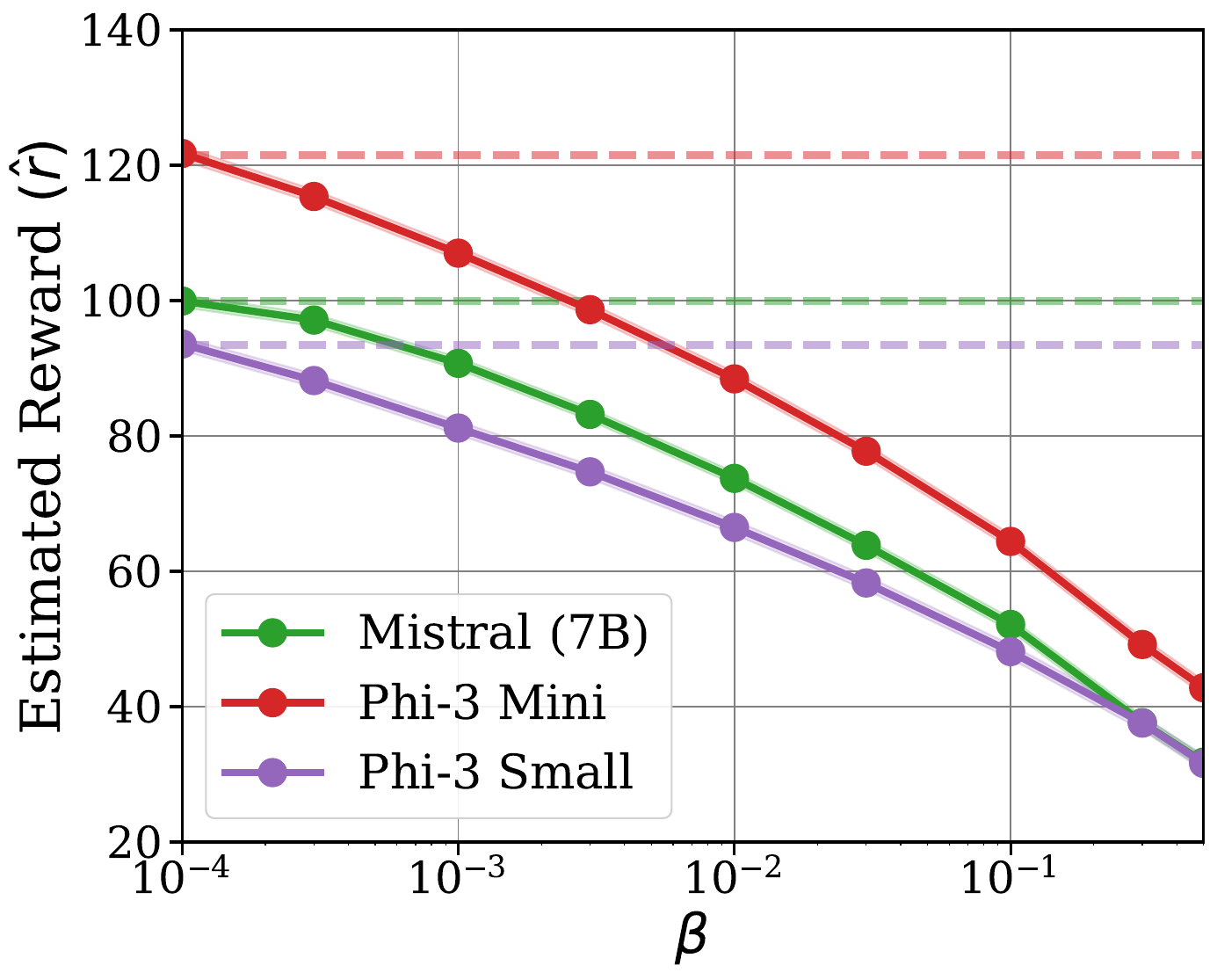}
        \label{sfig:betas-math-oasst-rm-rhat} 
      }
   \hfill \subfigure[\gemmarm]{
        \includegraphics[width=0.2\textwidth]{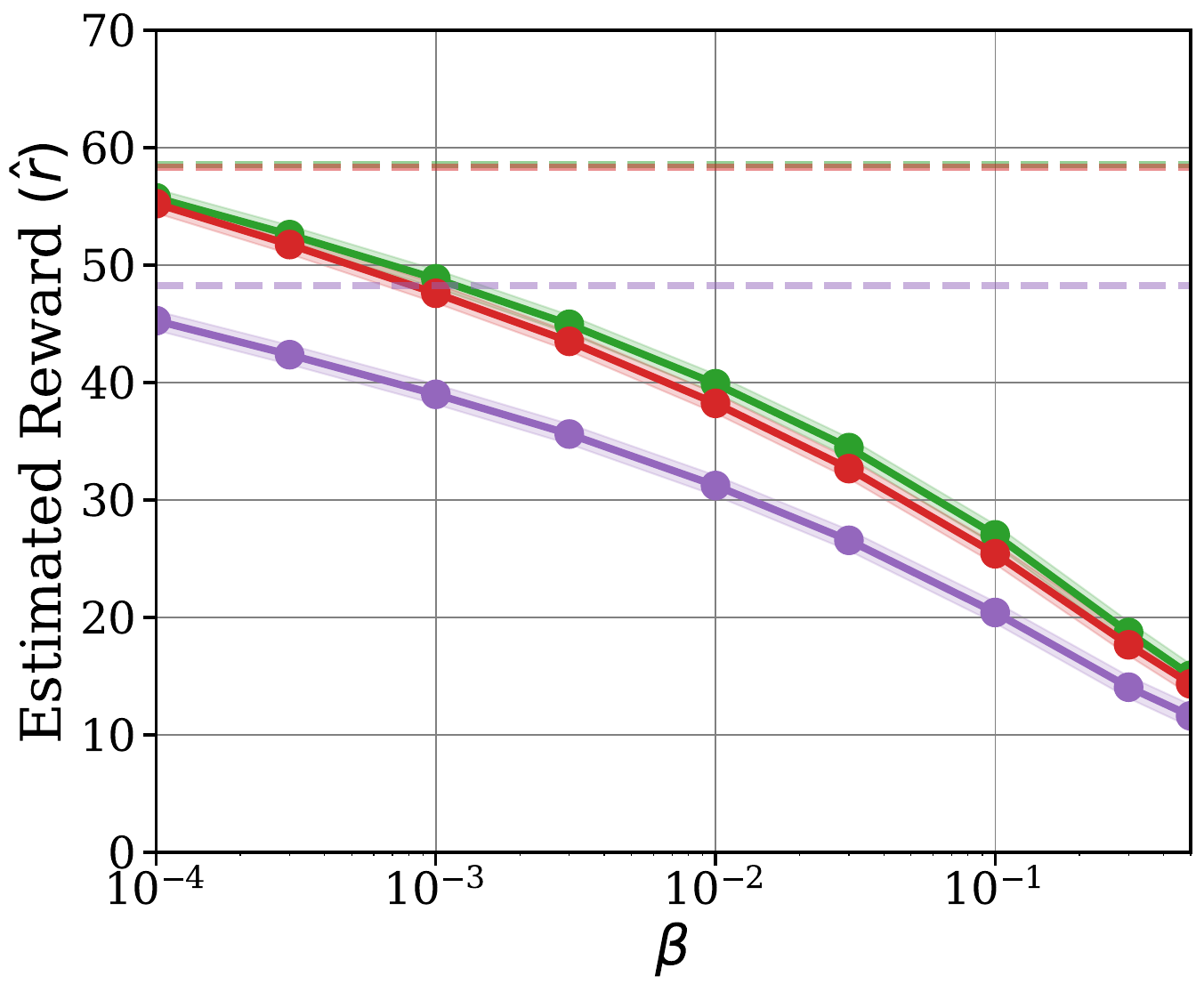}
        \label{sfig:betas-math-gemma-rhat} 
      }
      \hfill \subfigure[\llamarm]{
        \includegraphics[width=0.2\textwidth]{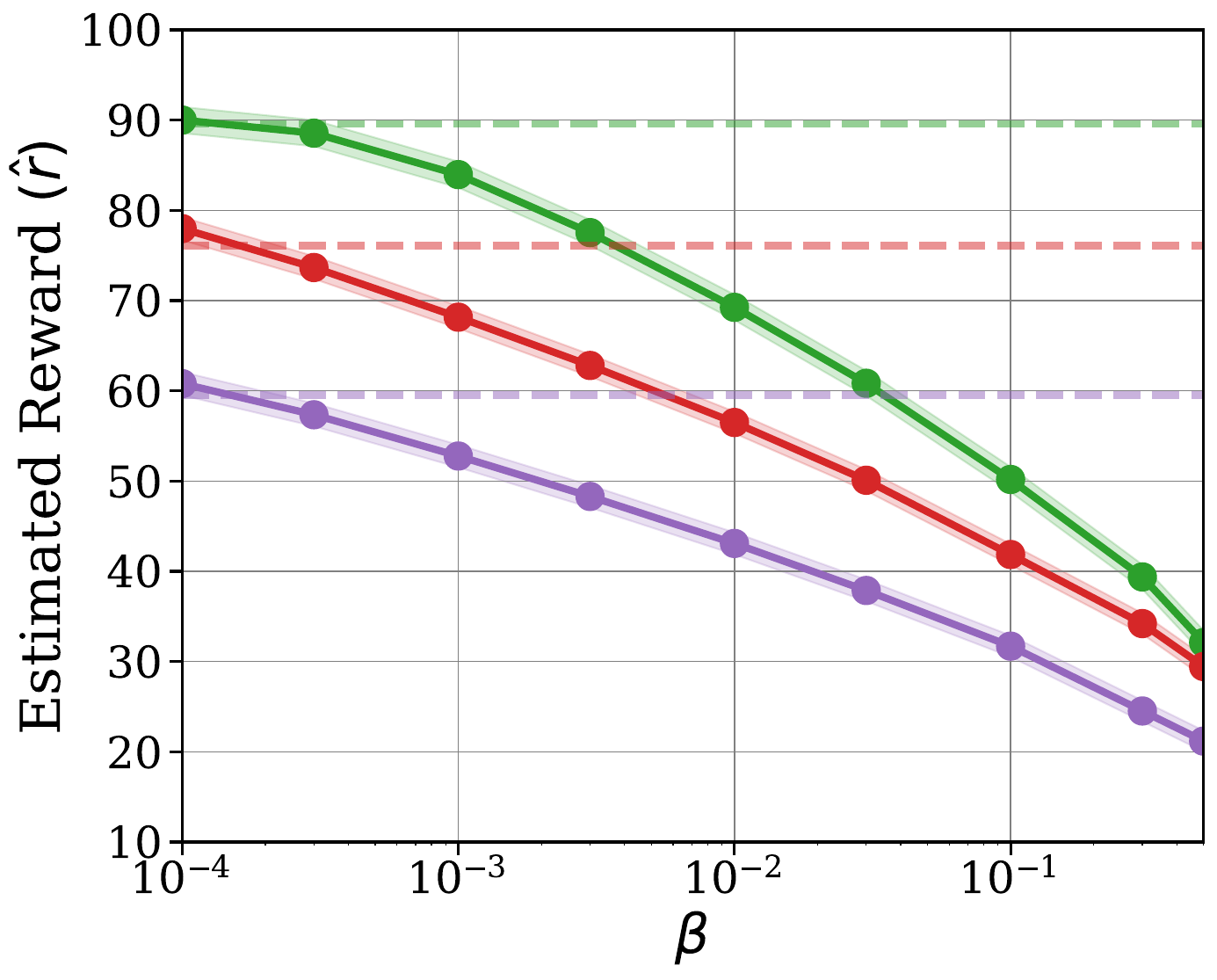}
        \label{sfig:betas-math-llama-rhat} 
      }
      \hfill \subfigure[\armorm]{
        \includegraphics[width=0.2\textwidth]{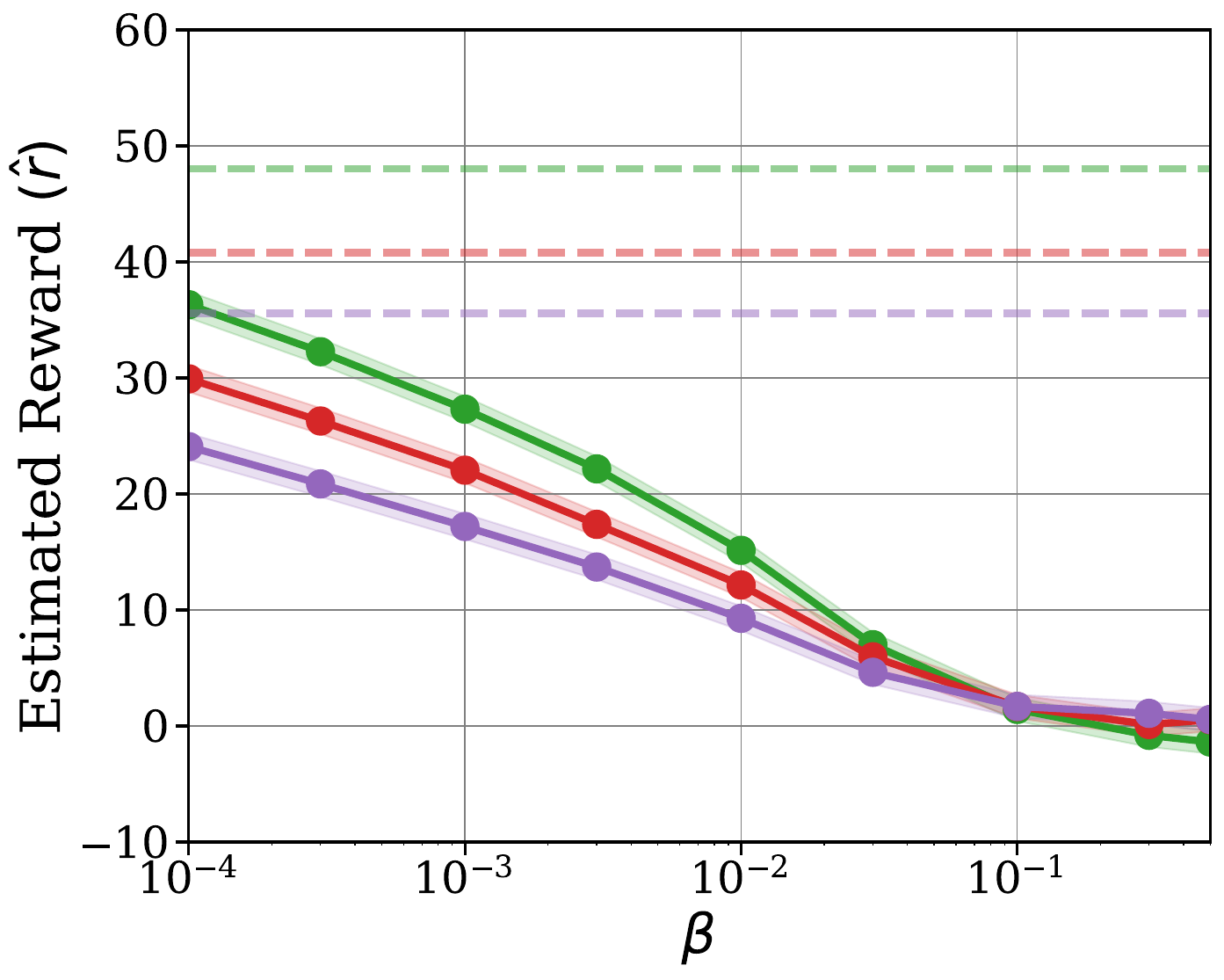}
        \label{sfig:betas-math-armo-rm-rhat} 
      }
    \caption{Compute-normalized comparison for $N = 2^{13}$ between \bonalg and \mainalg on \mathk for four reward models and choices of $\piref$, as a function of regularization $\beta$.}
    \label{fig:math-betas}
  \end{figure*}

  \begin{table*}[htp]
    \centering 
    \caption{Performance of $\piref=$ \phismall (\% Lift in Accuracy over $\piref$).}
    \begin{tabular}{lcccc}
        \toprule 
\textbf{Task} & \textbf{\oasst} & \textbf{\gemmarm} & \textbf{\llamarm} & \textbf{\armorm}  \\ \midrule 
\gsmk (Pessimism) & $0.25 \pm 0.88$ & $1.29 \pm 0.97$ & $5.43 \pm 0.93$ & $4.79 \pm 0.84$ \\ 
\gsmk (BoN) & $-3.06 \pm 1.21$ & $1.19 \pm 1.08$ & $5.71 \pm 0.95$ & $5.87 \pm 0.93$ \\ 
\rowcolor{lightgray} \mmlu (Pessimism) & $-2.18 \pm 5.37$ & $7.61 \pm 5.31$ & $14.47 \pm 5.60$ & $6.74 \pm 5.52$ \\ 
\rowcolor{lightgray} \mmlu (BoN) & $-3.65 \pm 5.66$ & $5.67 \pm 5.76$ & $14.12 \pm 5.64$ & $8.27 \pm 5.84$ \\ 
\mathk (Pessimism) & $-1.93 \pm 3.22$ & $9.94 \pm 3.40$ & $20.23 \pm 3.54$ & $12.71 \pm 3.37$ \\ 
\mathk (BoN) & $-4.85 \pm 3.45$ & $5.27 \pm 3.66$ & $19.99 \pm 3.61$ & $8.39 \pm 3.56$ \\ 
\bottomrule

    \end{tabular}
    \label{tab:phi3small}
\end{table*}

\begin{table*}[htp]
    \centering 
    \caption{Performance of $\piref=$ \mistral (\% Lift in Accuracy over $\piref$).}
    \begin{tabular}{lcccc}
        \toprule 
\textbf{Task} & \textbf{\oasst} & \textbf{\gemmarm} & \textbf{\llamarm} & \textbf{\armorm}  \\ \midrule 
\gsmk (Pessimism) & $4.15 \pm 1.77$ & $25.41 \pm 1.87$ & $57.30 \pm 1.77$ & $53.19 \pm 1.77$ \\ 
\gsmk (BoN) & $-12.10 \pm 1.96$ & $\mathbf{22.46 \pm 2.08}$ & $\mathbf{58.31 \pm 1.86}$ & $\mathbf{63.69 \pm 1.76}$ \\ 
\rowcolor{lightgray} \mmlu (Pessimism) & $1.28 \pm 7.26$ & $14.56 \pm 7.66$ & $35.01 \pm 9.16$ & $24.63 \pm 8.21$ \\ 
\rowcolor{lightgray} \mmlu (BoN) & $-1.70 \pm 8.84$ & $12.71 \pm 9.92$ & $37.43 \pm 9.70$ & $22.23 \pm 9.89$ \\ 
\mathk (Pessimism) & $10.32 \pm 6.51$ & $46.15 \pm 8.43$ & $129.58 \pm 11.55$ & $91.41 \pm 9.42$ \\ 
\mathk (BoN) & $-20.13 \pm 8.26$ & $47.49 \pm 9.12$ & $133.13 \pm 12.15$ & $102.10 \pm 10.32$ \\ 
\bottomrule

    \end{tabular}
    \label{tab:mistral}
\end{table*}

\begin{table*}[htp]
    \centering 
    \caption{Performance of $\piref=$ \llama (\% Lift in Accuracy over $\piref$).}
    \begin{tabular}{lcccc}
        \toprule 
        \textbf{Task} & \textbf{\oasst} & \textbf{\gemmarm} & \textbf{\llamarm} & \textbf{\armorm}  \\ \midrule 
        \gsmk (Pessimism) & $5.20 \pm 1.33$ & $14.38 \pm 1.32$ & $27.54 \pm 1.28$ & $29.66 \pm 1.18$ \\ 
        \gsmk (BoN) & $-4.35 \pm 1.94$ & $11.45 \pm 1.78$ & $27.43 \pm 1.52$ & $30.49 \pm 1.44$ \\ 
        \rowcolor{lightgray} \mmlu (Pessimism) & $16.82 \pm 10.21$ & $21.77 \pm 9.89$ & $46.55 \pm 10.49$ & $46.87 \pm 7.94$ \\ 
        \rowcolor{lightgray} \mmlu (BoN) & $16.82 \pm 10.44$ & $20.48 \pm 10.42$ & $44.13 \pm 10.71$ & $52.86 \pm 10.54$ \\ 
        \bottomrule

    \end{tabular}
    \label{tab:llama}
\end{table*}

\subsection{Further Experiments}\label{ssec:further_exp}
We performed several additional experiments to (i) validate the basic
modeling assumptions in our inference-time alignment
framework---particularly the assumed reward model accuracy bound in
\cref{eq:rm};, and  (ii) probe the
behavior and robustness of \mainalg, and. We begin by examining the
distribution of the reward model scores $\rhat(x,y)$ under $\piref$,
then explore the robustness of \mainalg to the choice of regularization parameter
$\beta$. we also present preliminary results on the \alpaca task in \cref{ssec:alpaca}.

\paragraph{Reward distribution under $\piref$}
In order to get a more fine-grained sense for the extent to which
reward overoptimization is a problem, in \Cref{fig:histograms} we plot
the distribution of the reward model value $\rhat(x,y)$ for a single
representative prompt from \gsmk for all \gemma-generated responses,
according to each of our four reward models, and conditioned on
whether or not the response is correct.  The more separated the
distributions are, and the further to the right the correct (blue)
distribution is, the better the reward model is at estimating the true
reward.  As we see, \armorm is by far the best reward model in this
respect.  In particular, one reason to expect that would not observe
reward overoptimization in \bonalg for \armorm with this task is the
fact that the maximal value in the support of the incorrect
distribution is, empirically, strictly smaller than that of the
correct distribution; thus for sufficiently large $N$, \bonalg will
always choose the correct answer, at least for the prompt we
visualize.  This observation is consistent with our theoretical
results, but suggests that in some cases our assumptions may be too
pessimistic; indeed, \mainalg may be overly conservative in situations
where $\rhat$ already underestimates $\rstar$ by a large margin (since
pessimism, or under-estimating the true reward value, is precisely
what the regularization in \mainalg is designed to enforce).

\begin{figure*}[htp]
    \centering
    \subfigure[]{
      \includegraphics[width=0.21\textwidth]{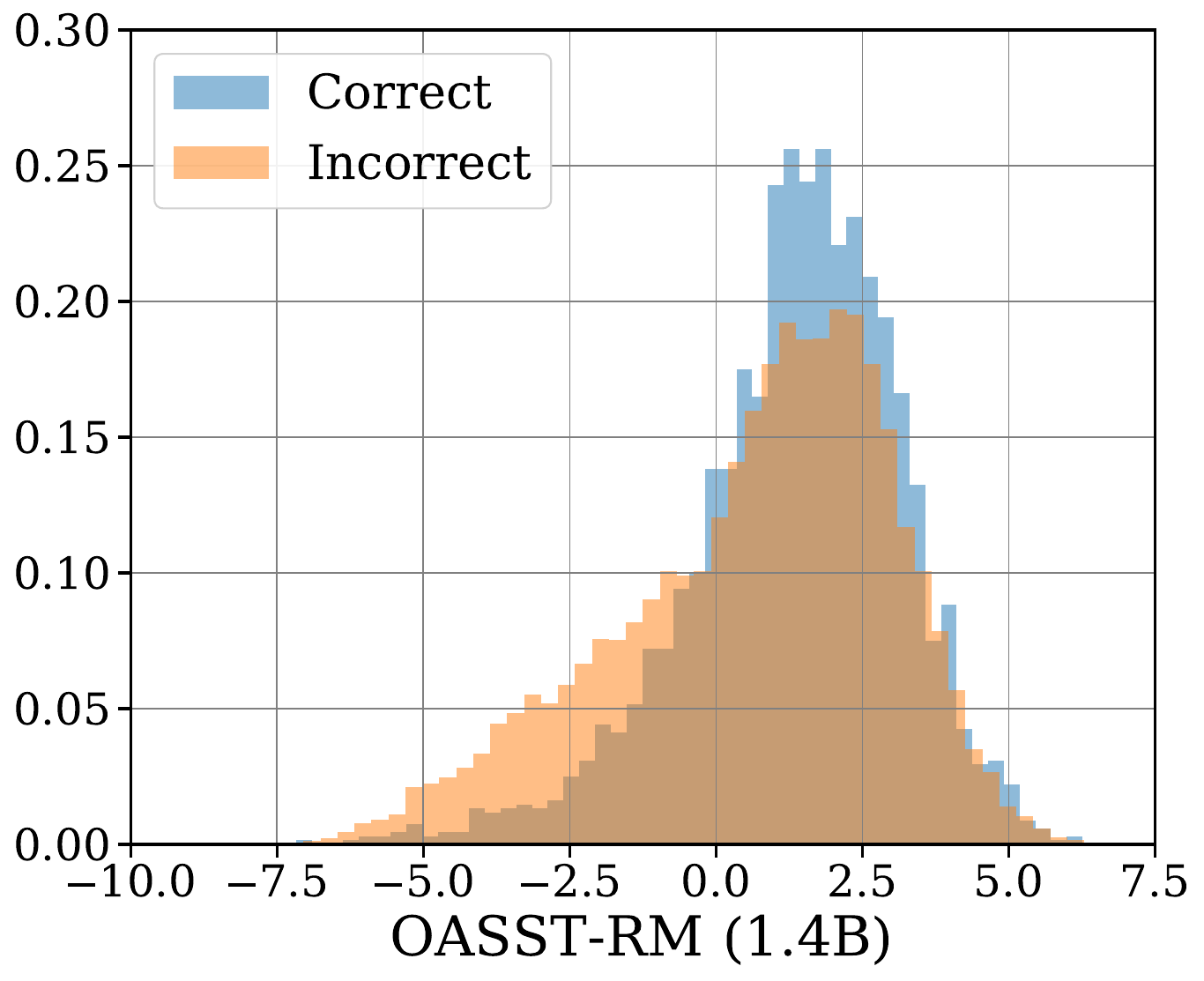}
      \label{sfig:hist-gsmk-oasst} 
    }
    \hfill \subfigure[]{
      \includegraphics[width=0.21\textwidth]{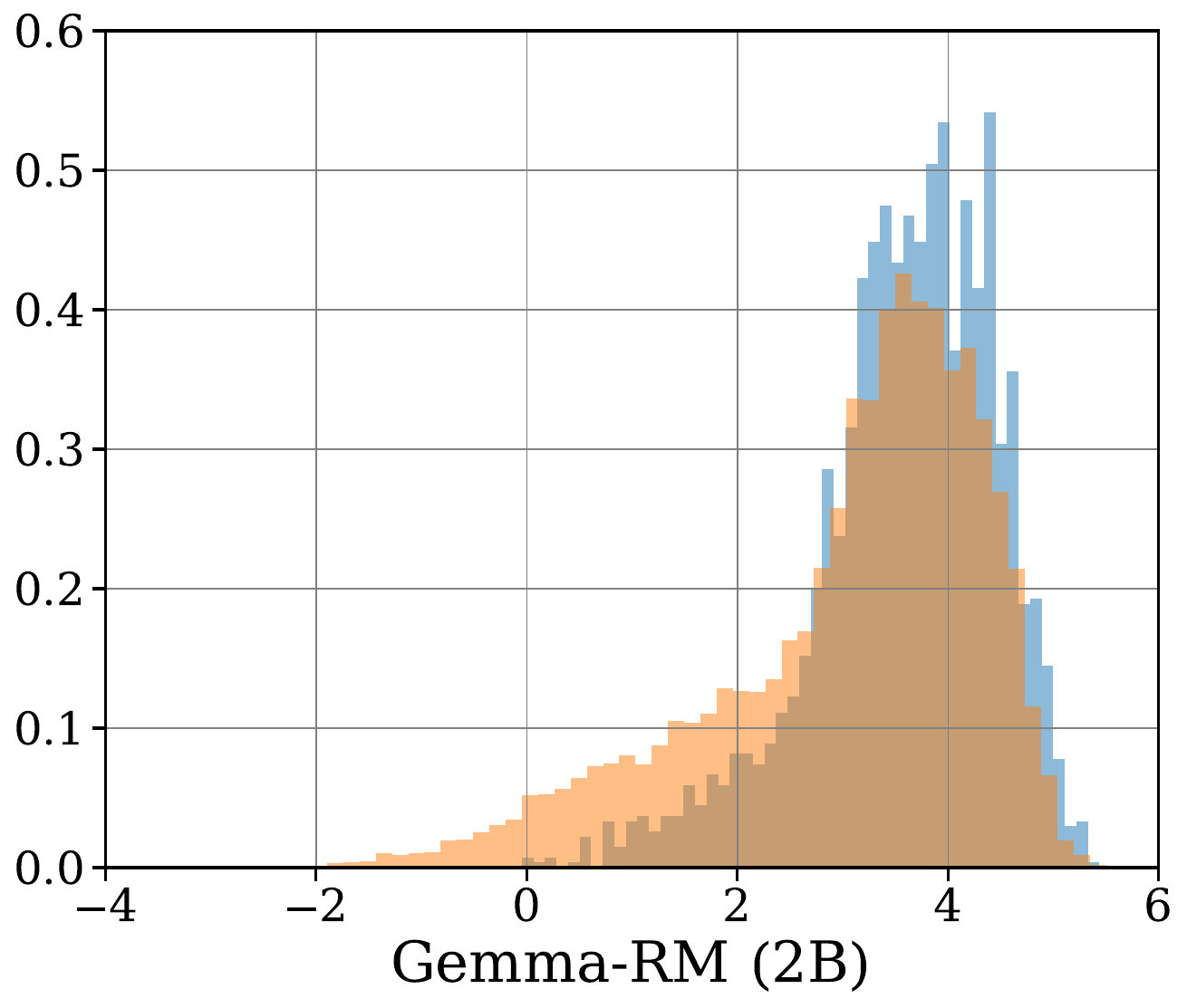}
      \label{sfig:hist-gsmk-gemma} 
    }
    \hfill \subfigure[]{
      \includegraphics[width=0.21\textwidth]{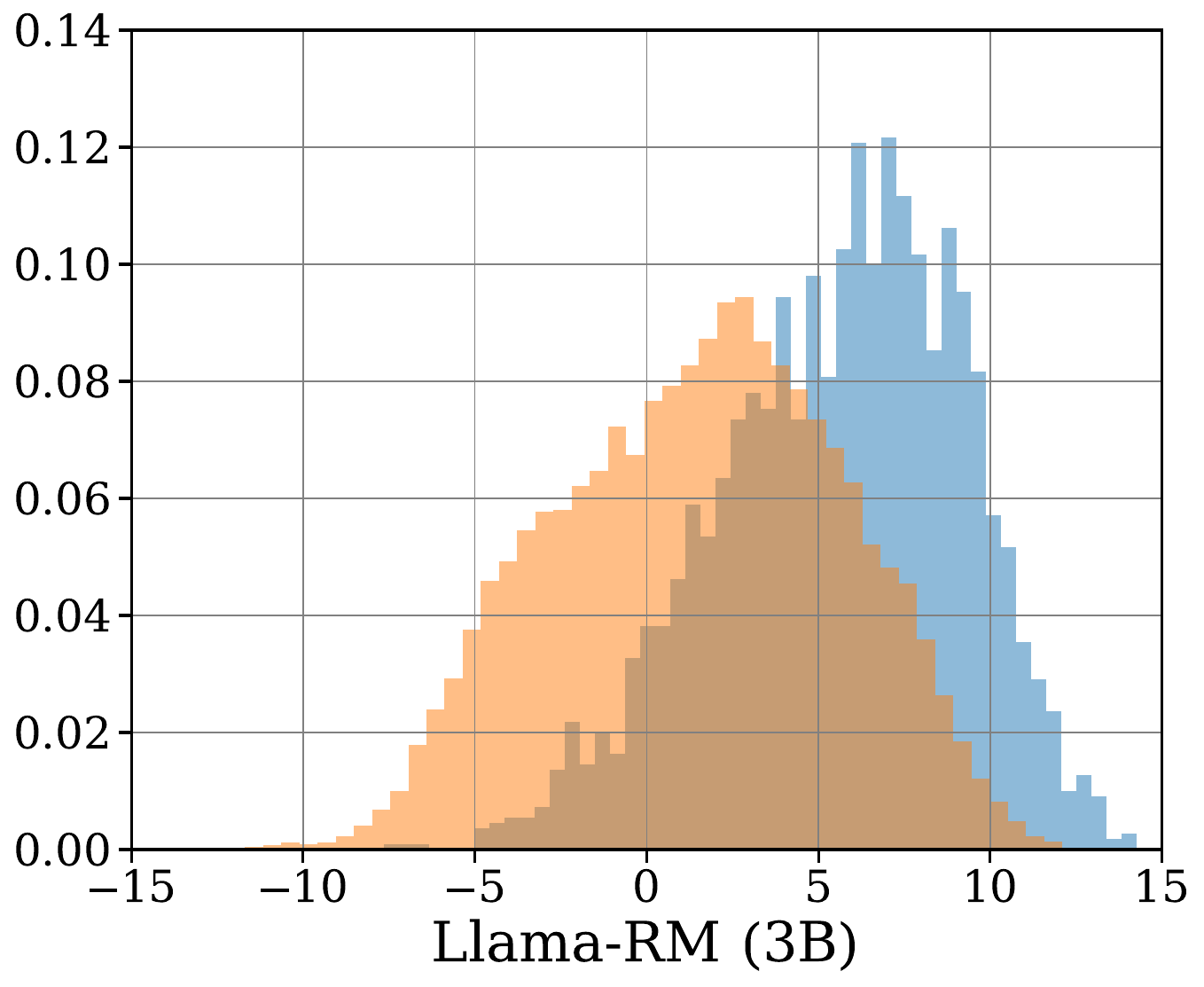}
      \label{sfig:hist-gsmk-llama} 
    }
    \hfill \subfigure[]{
        \includegraphics[width=0.21\textwidth]{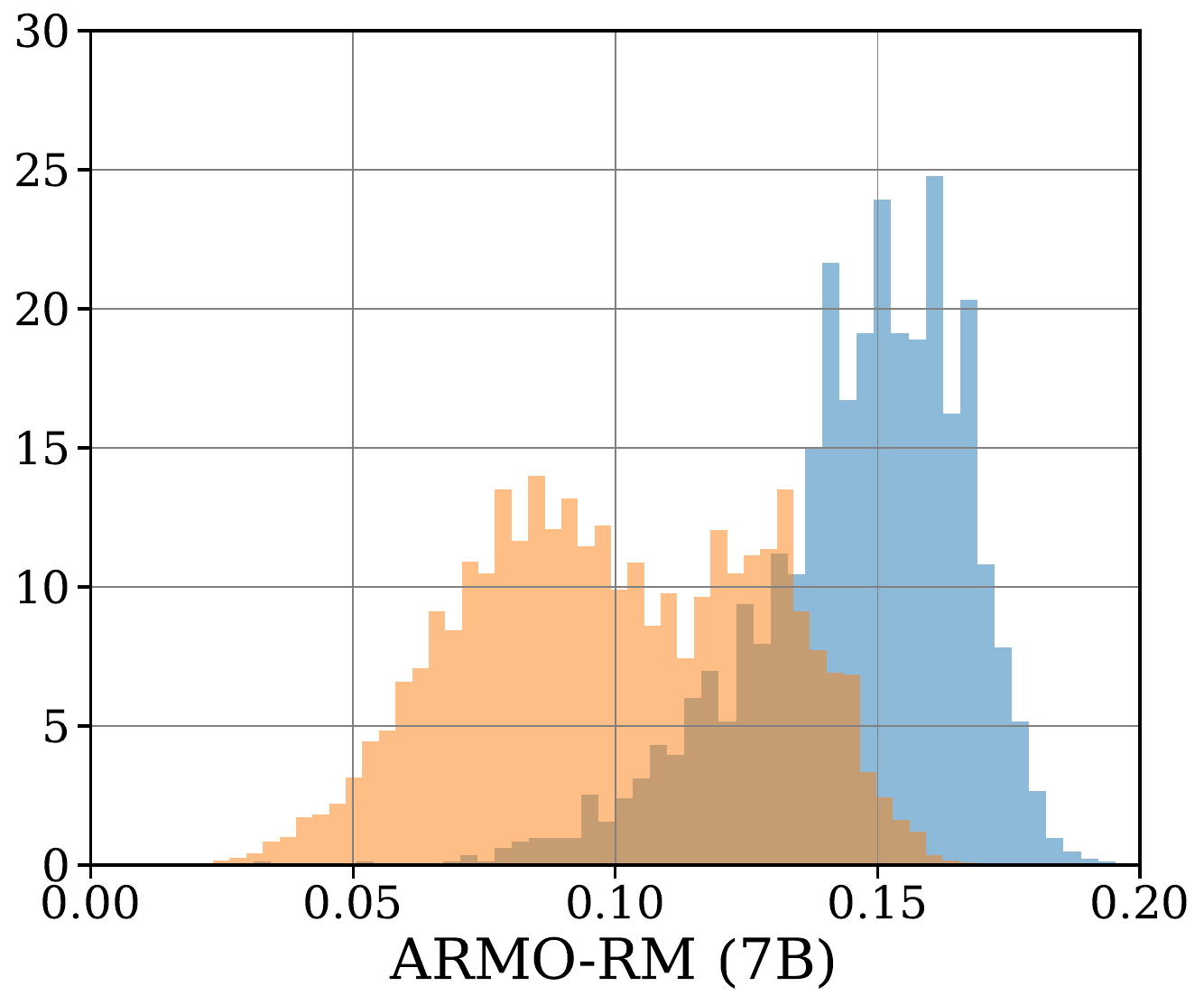}
        \label{sfig:hist-gsmk-armorm} 
      }
    \caption{Distribution of estimated rewards of responses generated by \gemma on \gsmk prompt number 100 conditioned on whether or not the response is correct for (a) \oasst, (b) \gemmarm, (c) \llamarm, and (d) \armorm.  Greater separation of the distributions with the correct (blue) further to the right than incorrect (orange) indicates the reward model is more informative.}
    \label{fig:histograms}
  \end{figure*}

\paragraph{Robustness of \mainalg to $\beta$} In addition to the
theoretical suboptimality under general notions of coverage, the
central drawback to \bonalg is the lack of monotonicity, requiring
careful tuning of the computational budget $N$ in order to avoid
over-optimization. In \Cref{fig:beta_robustness}, we plot the accuracy
of \mainalg on \mistral generations for a representative subset of the
tasks and reward models $\rhat$ we consider, evaluating the effect of
the regularization parameter $\beta$ on performance. We find (\Cref{sfig:robust-gsm8k-oasst-mistral,sfig:robust-math-oasst-mistral}) that
\mainalg experiences less over-optimization than \bonalg fairly
robustly across a range of $\beta$ values, though tuning $\beta$ is
typically required to avoid over-optimization entirely. We also
observe that in some cases, where the reward model remains in-distribution for
all task responses, \bonalg is monotonic without further interventions
(\Cref{sfig:robust-gsm8k-armo-mistral,sfig:robust-mmlu-grmllama-mistral}).

Note that for small values of the regularization $\beta$, we do observe a small dip in performance for \mainalg (cf. \Cref{sfig:robust-gsm8k-oasst-mistral,sfig:robust-math-oasst-mistral}). This is caused by our heuristic of defaulting to \bonalg when no responses are accepted by rejection sampling in \mainalg, which is more likely to occur when the computational budget $N$ is small relative to the inverse of regularization $1/\beta$ according to our theory (\Cref{lem:x-tv-upper-bound}).  We should like to remark that this is a reasonable heuristic, as it is precisely in the small $N$ regime that \bonalg is expected to still perform well according to our theoretical results.

\begin{figure*}[htp]
    \centering
   \subfigure[(\gsmk, \oasst)]{
        \includegraphics[width=0.21\textwidth]{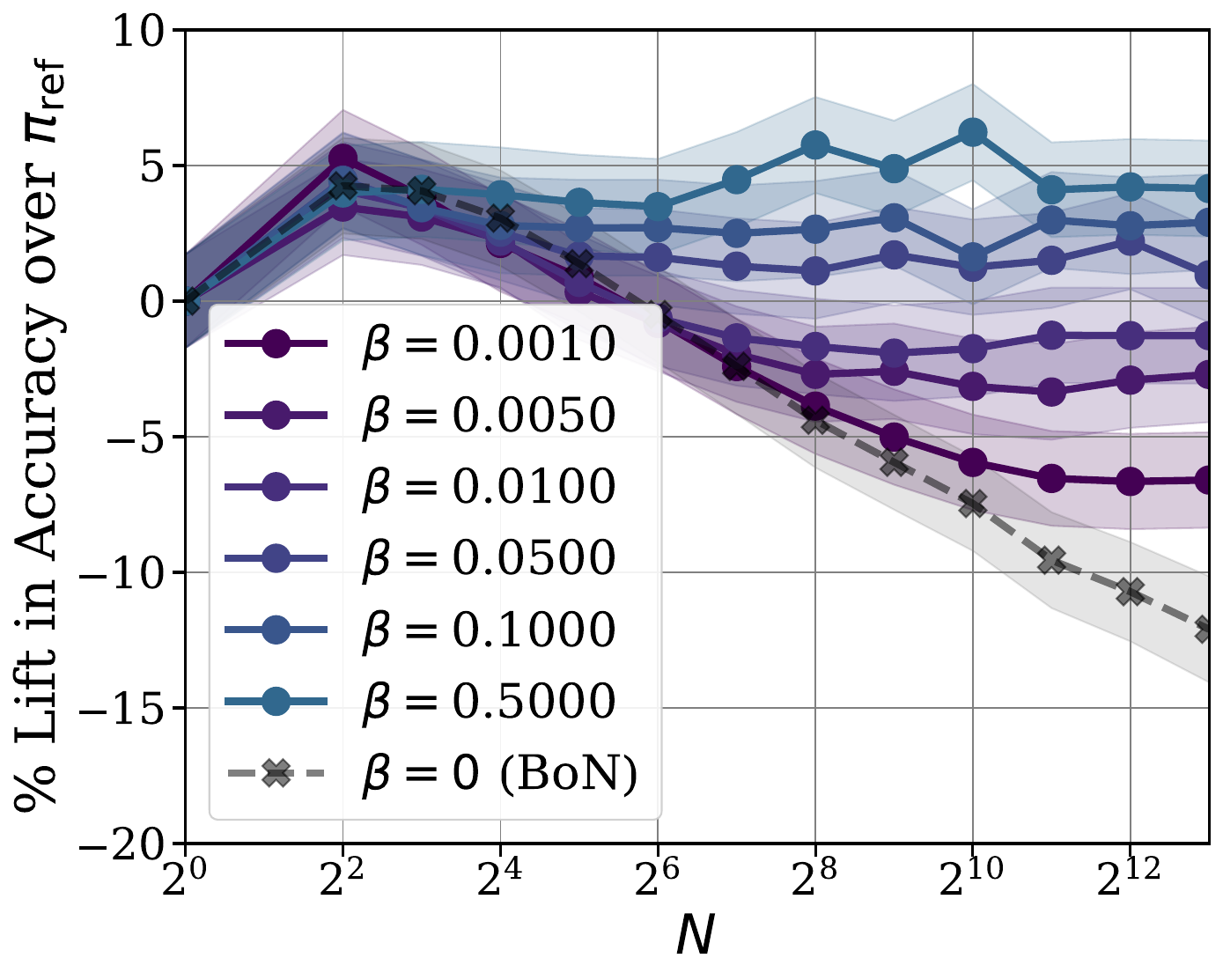}
        \label{sfig:robust-gsm8k-oasst-mistral} 
      }
    \hfill \subfigure[(\gsmk, \armorm)]{
        \includegraphics[width=0.21\textwidth]{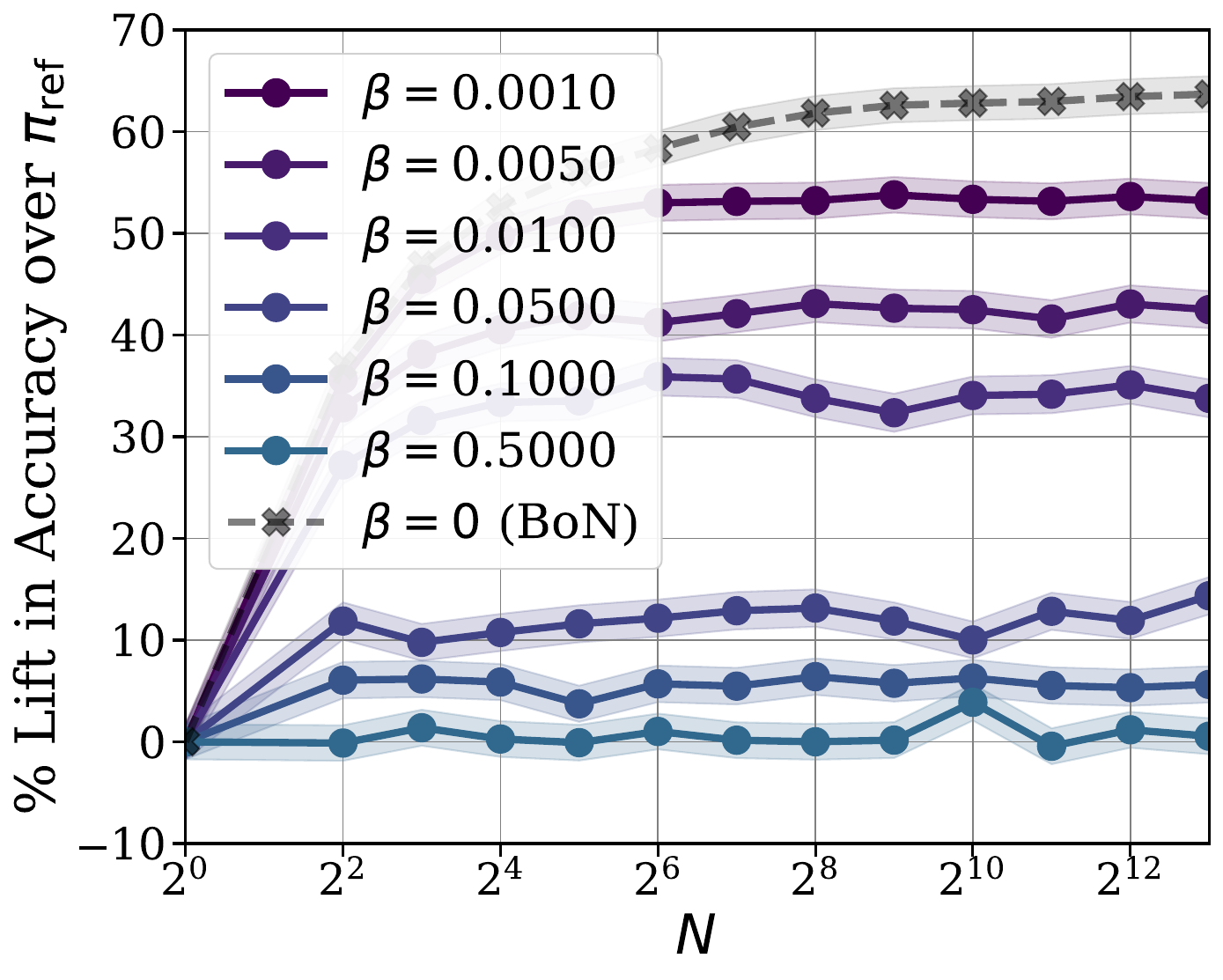}
        \label{sfig:robust-gsm8k-armo-mistral} 
      }
    \hfill \subfigure[(\mathk, \oasst)]{
      \includegraphics[width=0.21\textwidth]{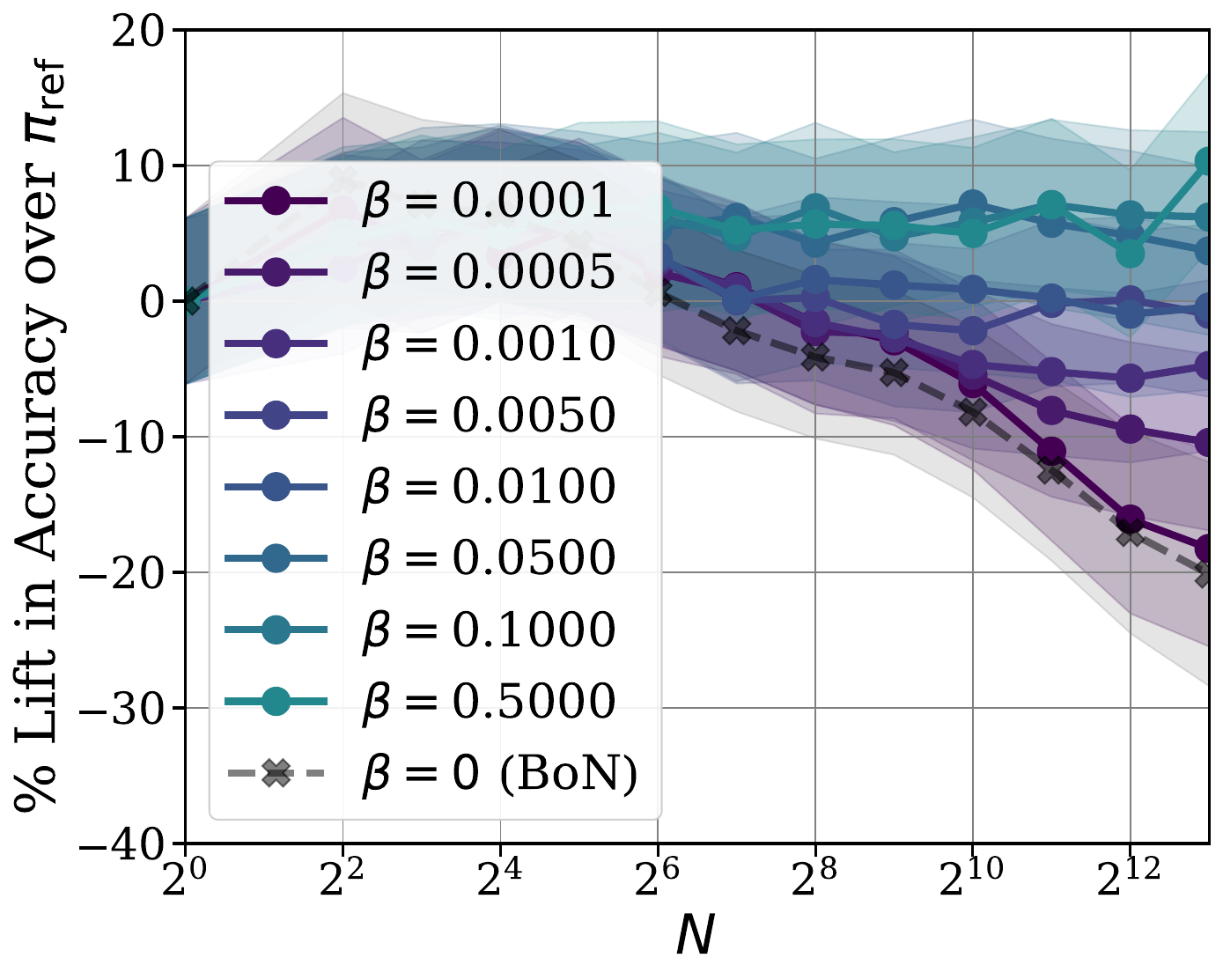}
      \label{sfig:robust-math-oasst-mistral} 
    }
    \hfill 
    \subfigure[(\mmlu, \llamarm)]{
      \includegraphics[width=0.21\textwidth]{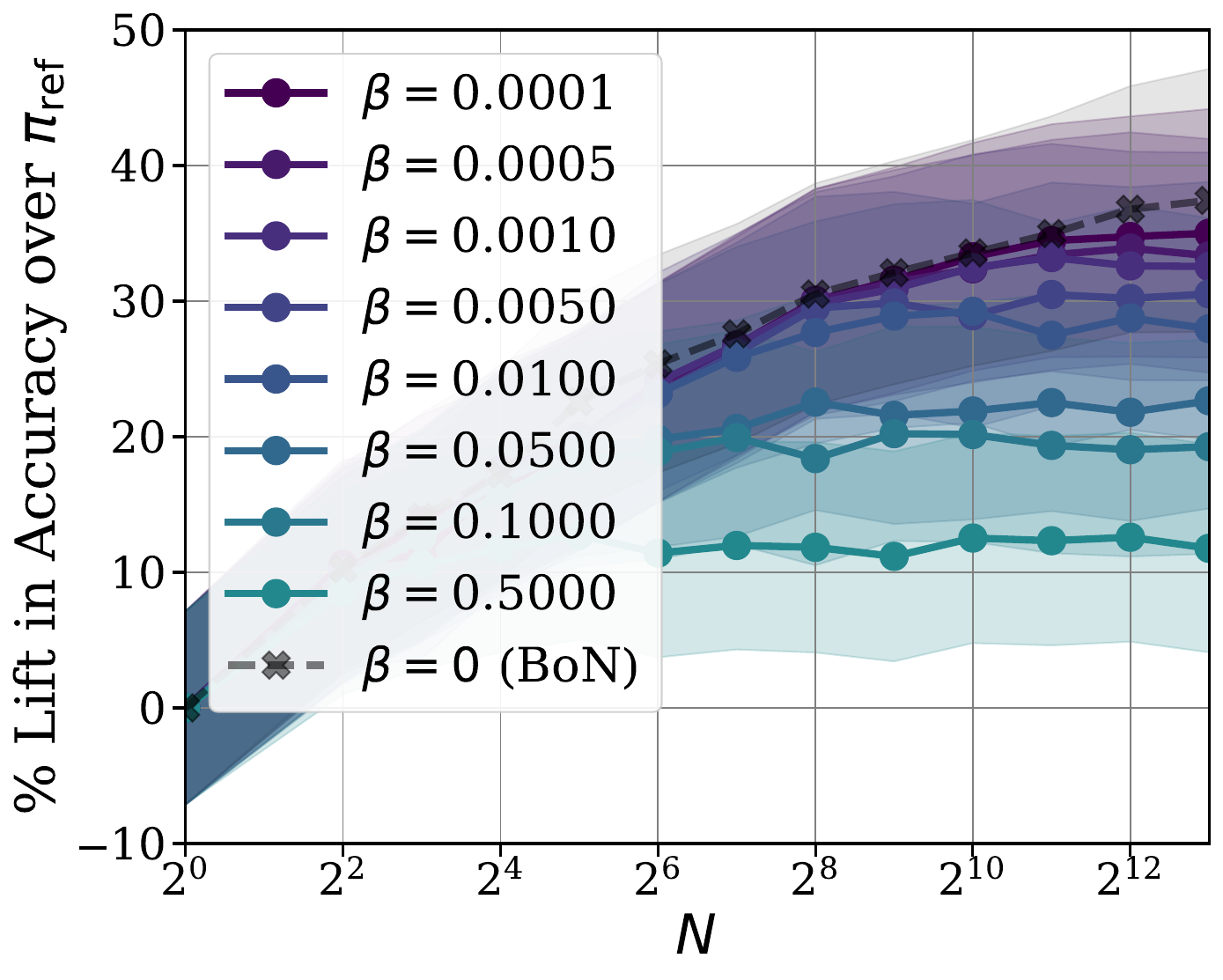}
      \label{sfig:robust-mmlu-grmllama-mistral} 
    }
    \caption{Demonstration that monotonicity is robust to the choice of regularization parameter $\beta$ for \mistral generations with a representative sample of tasks and estimated rewards $\rhat$.}
    \label{fig:beta_robustness}
  \end{figure*}

\subsection{Results for \alpaca}\label{ssec:alpaca}
In addition to the experiments with \gsmk, \mathk, and \mmlu, we
conduct a preliminary investigation on the performance of \mainalg on
\alpaca (cf. \Cref{ssec:exp_details} for an explanation of this task).
Because evaluation requires many queries to the proprietary OpenAI
models, we consider a significantly smaller scale of experiments for
this task, and use only 5 replicates for each prompt as opposed to 50.
The results are displayed in \Cref{fig:alpaca}.  Due to the noise of
the evaluation, coupled with the significantly smaller number of
replicates and prompts, it is difficult to separate the
performance of \bonalg and \mainalg in a statistically significant
way, although the broader trends agree with those found in our other
tasks.

\begin{figure*}[htp]
    \centering
   \subfigure[]{
        \includegraphics[width=0.45\textwidth]{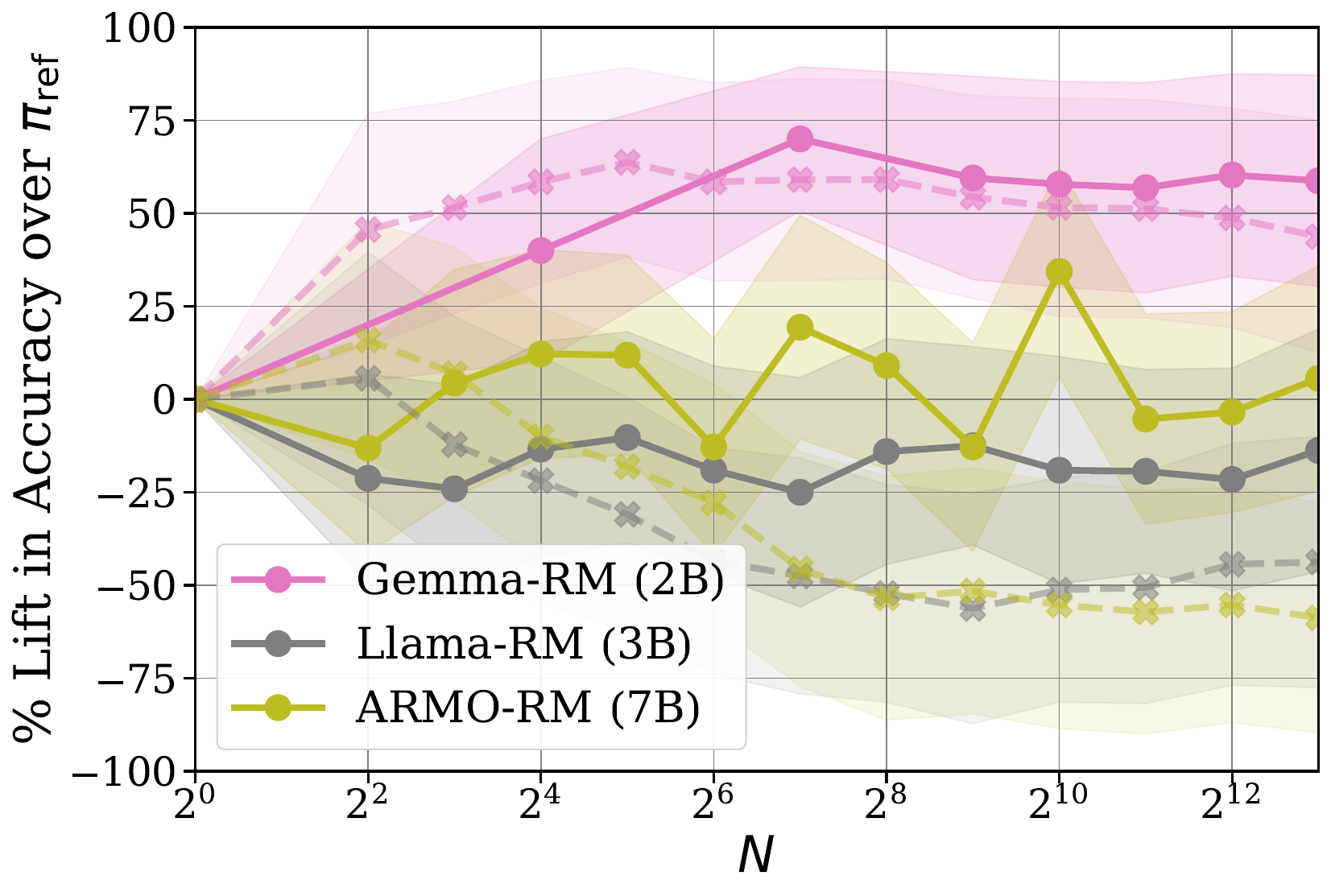}
        \label{sfig:alpaca_accuracy} 
      }
    \hfill \subfigure[]{
        \includegraphics[width=0.45\textwidth]{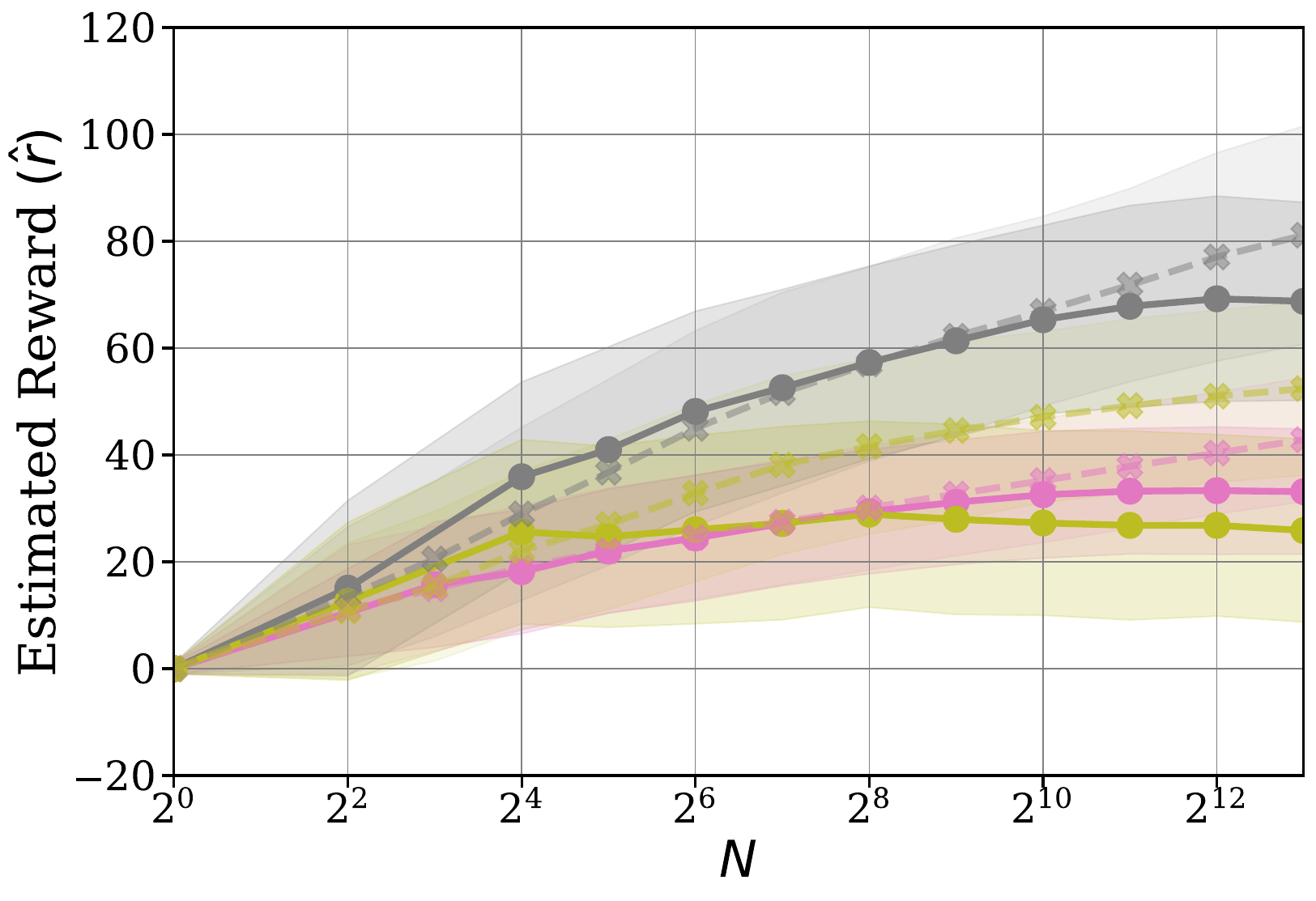}
        \label{sfig:alpaca_rhat} 
      }
   
    \caption{Performance of \bonalg and \mainalg on \alpaca with \gemma as $\piref$ for several reward models.}
    \label{fig:alpaca}
  \end{figure*}

\newpage
\part{Proofs}

\section{Normalization Constant Computation}
\label{app:normalization}

\icml{}
\arxiv{\arxiv{\begin{algorithm}[htp]}
  \icml{\begin{algorithm}[tp]}
  \caption{\arxiv{General normalization constant computation}\icml{\normalg for general base measures}}
  \label{alg:general_norm}
  \begin{minipage}{\linewidth}
  \begin{algorithmic}[1]
    \Statex[0] \multiline{{\bfseries input:}
      Prompt $x$, 
      reward model $\rhat$,
      reference policy $\piref$,
      regularization coefficient $\beta>0$,
      set of responses $\Yhat= (y_1,\ldots,y_N)$.
    }
    \State Sort and bucket $\Yhat$ into bins $\cY_{1},\ldots,\cY_{M}$ 
    according to the value of $\rhat(x,y)$ in ascending order, such that 
    \begin{align}
      &\rhat(x,y) = \rhat_i, \quad\forall y \in \cY_{i},\forall i\in[M]
      \\
      &\rhat_{i} < \rhat_{i+1}, \quad\quad\forall i\in[M],
    \end{align}
    and denote $\piref(\cY_{i}\mid{}x) = \sum_{y\in\cY_{i}} \piref(y\mid{}x)$
    \State Initialize $\rhat_0 = -\infty$, $J\leftarrow \sum_{i=1}^M \piref(\cY_{i}\mid{}x)\cdot\rhat_{i} $ and $Z\leftarrow \sum_{i=1}^M \piref(\cY_{i}\mid{} x)$.
    \For{$i=1\ldots M$}
      \vspace{0.5em}
      \State Set $\lambda \leftarrow   \frac{J - \beta}{Z}$.
      \vspace{0.5em}
      \If{$\rhat_{i-1} \leq \lambda < \rhat_{i}$ or $i = M$}
        \State \textbf{return} $\lambda$.
      \Else
        \State Update $J \leftarrow J -  \rhat_{i}\cdot\piref(\cY_{i}\mid{}x)$ and $Z \leftarrow Z - \piref(\cY_{i}\mid{}x)$.
      \EndIf
    \EndFor
\end{algorithmic}
\end{minipage}
\end{algorithm}
}

\subsection{Background}

This section gives guarantees for the \normalg subroutine
(\cref{alg:norm}) used within \mainalg. Given a set of response
$\Yhat$ and $x\in\cX$, the algorithm computes a normalization constant $\lambda$
such that
\[
\Phihat(\lambda)\ldef\frac{1}{N}\sum_{y\in\Yhat}\relu\prn*{\beta^{-1}(\rhat(x,y)-\lambda)} = 1
\]
in time $\bigoh(N\log{}N)$ using a dynamic programming-like procedure.
Note that such a $\lambda$ always exists because
$\Phihat(\lambda)$ is a continuous, piecewise linear function
that is decreasing in $\lambda$,
with $\Phihat(-\infty) = \infty$ and $\Phihat(\infty) = 0$.
In what follows, we state and prove the main guarantee for the algorithm.
En route, we will also prove a guarantee for a more general version of \normalg (\cref{alg:general_norm}),
which computes a normalization constant such that
$\Phi(\lambda)\ldef\sum_{y\in\cY}\piref(y\mid{}x)\relu\prn*{\beta^{-1}(\rhat(x,y)-\lambda)} = 1$
in time $\bigoh\rbr*{N \log N}$.

\subsection{Guarantee for \normalg}
\begin{lemma}[Main guarantee for \normalg]\label{lem:norm}
  For any $\rhat$, $\beta$, $\Yhat$, and $x\in\cX$,
  \cref{alg:norm} finds $\lambda$ such that
  $\Phihat(\lambda)\ldef \frac{1}{N} \sum_{y\in\Yhat}\relu\prn*{\beta^{-1}(\rhat(x,y)-\lambda)} = 1$
  in $O\rbr*{N\log N}$ time.
\end{lemma}

\begin{proof}[\pfref{lem:norm}]
  \cref{alg:norm} is equivalent to running \cref{alg:general_norm} 
  with $\rhat$, $\beta$, $\cY=\Yhat$ and $\piref'(y\mid{}x)=\frac{1}{N}\cdot\one{y\in\Yhat}$
  and so we can directly apply \cref{lem:general_norm}.
\end{proof}

\begin{lemma}[Guarantee for generalized \normalg]\label{lem:general_norm}
  For any $\rhat$, $\beta$, $\Yhat$, and $x\in\cX$,
  \cref{alg:general_norm} finds $\lambda$ such that $\wh\Phi(\lambda)\ldef\sum_{y\in\Yhat}\piref(y\mid{}x)\relu\prn*{\beta^{-1}(\rhat(x,y)-\lambda)} = 1$
  in $O\rbr*{N \log N}$ time.
\end{lemma}

\begin{proof}[\pfref{lem:general_norm}]
  Let $x\in\cX$ be fixed; we omit dependence on $x$ going forward to keep notation compact.
  We begin by defining a surrogate problem with response space
  $\Yhat[M]\ldef\cbr*{y_{1},\ldots,y_{M}}$,
  where $y_{i}$ is any response from $\cY_{i}$,
  surrogate reward function $\rhat'$ with $\rhat'(y_{i}) \rdef{} \rhat'_{i} = \rhat_{i}$,
  and surrogate reference policy $\piref'$, where for $i \in [M]$ we set
  \begin{align}
    \piref'(y_{i}) = \sum_{y \in \cY_{i}} \piref(y).
  \end{align}
  Note that all responses in this collection have unique values under a surrogate reward function $\rhat'$.
  Since \cref{alg:general_norm} sorts rewards in ascending order,
  we also have that the surrogate rewards are indexed in ascending order, i.e.,
  $\rhat'_i < \rhat'_{i+1}$ for all $i$.
  We also define the following function,
  which is the analog of $\wh\Phi(\lambda)$ defined over the $M$ surrogate responses,
  \begin{align}
    \Phihat'(\lambda) = \sum_{i=1}^M \piref'(y_{i})\sig{\rhat'_{i} - \lambda}.
  \end{align}
  For any $\lambda$, it can be seen that $\Phihat'(\lambda) = \Phihat(\lambda)$ since
  \[
    \Phihat'(\lambda) = \sum_{i=1}^M \rbr*{\sum_{y \in \cY_i} \piref(y)} \sig{\rhat_i  - \lambda}
    = \sum_{y \in \Yhat} \piref(y) \sig{\rhat(y) - \lambda}
    = \Phihat(\lambda),
  \]
  because each $y \in \Yhat$ is binned into one of $\cbr*{\cY_i}_{i=1}^M$,
  and within each $\cY_i$ all responses share the reward label $\rhat'_i$.

  Next, define $\lambda^\star$ to be such that $\Phihat'(\lambda^\star) = 1$, that is,
  the true constant that normalizes the distribution over the $M$ surrogate responses.
  Our goal is to show that \cref{alg:general_norm} computes $\lambda^\star$,
  which, given the previously shown equivalence, then proves the lemma statement since
  \[
    1 = \Phihat'(\lambda^\star) = \Phihat(\lambda^\star).
  \]
  Note that such a $\lambda^\star$ is guaranteed to exist
  because $\Phihat(\lambda)$ is continuous and decreasing in $\lambda$.
  Moreover, from the form of $\Phihat'$, it can be seen that there exists some index $j$ such that 
  $\lambda^\star \ge \rhat'_{j}$ and
  \begin{align}
    1 = \Phihat'(\lambda^\star) = \beta^{-1} \sum_{i > j} \piref'(y_{i}) \rbr*{\rhat'_{i} - \lambda^\star}.
  \end{align}
  In other words, there exists an index $j \in [M]$ such that
  $\lambda^\star \in [\rhat'_j, \rhat_{j+1}']$,
  which also shows that $\lambda^\star \in \sbr*{J(\piref) - \beta, \rhat'_M}$.
  Then to find $\lambda^\star$, that is,
  a $\lambda$ such that $\Phihat'(\lambda)=1$,
  it is sufficient to find an index $j$ such that for
  \begin{align}
    \lambda = \frac{\sum_{i > j} \piref'(y_{i})\rhat'_{i} - \beta}{\sum_{i > j}\piref'(y_{i})},
  \end{align}
    we have $\rhat'_j \leq \lambda < \rhat'_{j+1}$.
  This is exactly the output of \cref{alg:general_norm} run on $\Yhat[M], \piref', \rhat',\beta$, which breaks at an index that satisfies the above conditions, and outputs the corresponding $\lambda$. 
  By inspection, such a $\lambda$ satisfies $\Phihat'(\lambda) = 1$,
  and therefore also satisfies $\Phihat(\lambda) = 1$.

  Lastly, we discuss the computational complexity of \cref{alg:general_norm}. 
  Sorting $\Yhat$ and $\rhat$ require $O(N \log N)$ time, 
  while binning is $O(N)$.
  Finally, we make one forward pass through the responses,
  which requires at most $N$ iterations, each with $O(1)$ computations, before termination. 
\end{proof}

\newpage

\section{Rejection Sampling}
\label{app:rejection}

This section includes background on and guarantees for approximate rejection sampling, 
which forms the basis for our analysis for \bon and \mainalg.

The organization is as follows. First, in \cref{app:background_rejection}, 
we describe the approximate rejection sampling algorithm
(\cref{alg:rej})---which is used within \mainalg, as well as within
the analysis of \bonalg---and introduce the fundamental concepts 
that we will use 
to analyze the sample complexity of \rejalg. 
Then, in \cref{app:upper_rejection} we provide 
upper bounds on the sample complexity of \rejalg, 
with matching lower bounds in \cref{app:lower_rejection}, 
in terms of the aforementioned measure, 
demonstrating its fundamental nature and the tightness of our results.

\icml{}

\subsection{Background}\label{app:background_rejection}
The rejection sampling algorithm \rejalg is shown in \cref{alg:rej}.
The input parameters are the sample size $N$ and the \Mlong $M$,
a prompt $x$, 
and an importance sampling weight $w$. 

The algorithm first draws $N$ samples from the conditional distribution of $\piref$ 
for the fixed prompt $x$, 
i.e., it draws $\Yhat = \cbr*{y_{1},\ldots,y_{N}}$ where $y_{i}\sim \piref(\cdot\mid{}x)$
for each $i \in [N]$. 
Then, for each $y_{i} \in \Yhat$, it samples a Bernoulli random variable $\xi_{i}$
where the probability of observing $\xi_{i}=1$
is given by the \wlong $w(y_{i}\mid{}x)$ %
divided by the \Mlong $M$,
truncated to be at most 1. 
The algorithm returns any $y_{i}$ for which $\xi_{i}=1$,
and if no such event is observed, returns the first response 
(which is equivalent to randomly sampling from the base policy).

If $\pi(\cdot\mid{}x) = w(\cdot\mid{}x) \cdot \piref(\cdot\mid{}x)$ is a valid distribution, 
and the \Mlong $M$ upper bounds the \wlong (or now, likelihood ratio) uniformly, that is,
$M \ge \frac{\pi(y\mid{}x)}{\piref(y\mid{}x)}$
for all $x$ and $y$,
then \cref{alg:rej} is identical to classical rejection sampling, 
where it is known that the law of accepted samples matches the target distribution $\pi$, i.e. $\Pr\rbr*{y \mid \xi=1, x} = \pi(y\mid{} x)$.
If $M$ is not a uniform upper bound on the likelihood, the law of the accepted
samples is not identical to $\pi$.
This is because \cref{alg:rej} effectively truncates the distribution of $\pi$ to be at most $M \cdot \piref$,
and any mass above this threshold is effectively lost.
Nonetheless, the law of accepted responses is a close approximation to $\pi$ when $M$ is sufficiently large.
\emph{Approximate} rejection sampling under this regime was first analyzed
in \citet{block2023sample}, and our results in this section below borrow from their analysis.

For the remainder of this section, we will consider the problem of
sampling from a target distribution or policy $\pi : \cX \rightarrow \Delta(\cY)$,
by calling $\rejalgNM[\frac{\pi}{\piref} ]$, 
where $\frac{\pi}{\piref}$ is its \wlong (or here, likelihood ratio)
over the base policy. Concretely, we are concerned with analyzing 
how close the \rejalg response distribution is to $\pi$ as a function
of the truncation level $M$.
For example, if there exists $y$ such that $\ahreplace{\pi}{w}(y\mid{}x) > 0$
but $\piref(y\mid{}x) = 0$
then it is not possible information-theoretically 
to sample from this portion of the target distribution,
and similar reasoning applies to $y$ with poor coverage under $\piref$. \loose

\paragraph{Preliminaries: $\cE_{M}$-divergence}
Based on the intuition above, a central object in our analysis will be 
the \Edivlong in \cref{def:ediv} \citep{polyanskiy2010channel,block2023sample},
which, 
for an input \Mlong $M$, 
quantifies the mass of the target distribution $\pi$
that is lost by truncating the \wlong $\frac{\pi}{\piref} $ to $M$.

\begin{definition}[\Edivlong]\label{def:ediv}
  For a prompt $x$, 
  base policy $\piref: \cX \rightarrow \Delta(\cY)$, 
  and target policy $\pi: \cX \rightarrow \Delta(\cY)$,
  let $\pi(x) \ldef{} \pi(\cdot\mid{}x)$ refer to the distribution conditioned on $x$,
  and define $\piref(x)$ similarly.
  Then for any \Mlong $M \ge 1$, 
  the \Edivlong is defined as
  \begin{align}
    \Epix 
    \ldef{} \En_{y \sim \piref(x)}\sbr*{\rbr*{\frac{\pi(y\mid{}x)}{\piref(y\mid{}x)} - M}_+} 
    = \sum_{y \in \Youtx} \pi(y\mid{}x) - M \cdot \piref(y\mid{}x),
  \end{align}
  where $\Youtx = \cbr*{y \in \cY : \pi(y\mid{}x) > M \cdot \piref(y\mid{}x)}$. 
  In addition, the expected \Edivlong over the prompt distribution $\rho$ is defined as
  \begin{align}
    \Epi \ldef{} \En_{x\sim\rho}\sbr*{\Epix}.
  \end{align}
\end{definition}
Note that $\Epix$ and $\Epi$ are non-increasing in $M$ in the sense that 
they do not increase as $M$ gets larger, and we will show in the
sequel that they control the approximation error for \rejalg. 
In particular, the parameter $M$ strikes a bias-variance tradeoff
in the \rejalg procedure.
When $M$ is large, the algorithm requires a large number of samples in order to terminate; 
but for $M$ small, incurs larger bias from failing to sample from regions of the target density. 

For the results that follow, it will be useful to define the smallest parameter $M$
that guarantees $\Epix \le \veps$
for some error tolerance of interest
$\veps \in [0,1]$.

\begin{definition}\label{def:mstar}
  Given the base policy $\piref$, for a target policy $\pi$, 
  prompt $x$, 
  and any $\veps \in [0, 1]$, define the smallest \Mlong $M$ that ensures
  $\Epix \le \veps$ to be
  \begin{align}
    \Mxpi \ldef{} \min\cbr*{M \mid{} \Epix \le \veps}. 
  \end{align}
  Similarly, define the smallest \Mlong $M$ that ensures $\Epi \le \veps$ to be
  \begin{align}
    \Mpi \ldef{} \min\cbr*{M \mid{} \Epi \le \veps}. 
  \end{align}
\end{definition}

Though the \Edivlong is perhaps the most natural object by which to quantiy the error of aproximate rejection sampling, 
it can also be upper bounded by other information-theoretic divergences, 
such as $\Cinf$ and $\Cone$, 
which will be useful in the later analysis for \bonalg and \mainalg. 
We state the result below for $\Mxpi$ for a fixed $x$, 
which can be stated for $\Mpi$ in a similar manner. 

\begin{proposition}\label{prop:mstar}
  Given the base policy $\piref$, 
  for any target policy $\pi$ and prompt $x$, 
  recall that $\Cinf(x) = \sup_{y\in\cY} \frac{\pi(y\mid{}x)}{\piref(y\mid{}x)}$,
  and define $\Calpha(x) \ldef{} \frac{1}{\alpha} \En_{y \sim \pi}\sbr*{\rbr*{\frac{\pi(y\mid{}x)}{\piref(y\mid{}x)}}^{\alpha-1}}$ for any $\alpha > 1$,
  so that $\cC_{2}^\pi(x) = \Cone(x)$.
  Then for any $\veps\in[0,1]$ and $\alpha \in (1, \infty)$, we have 
  \[
    \Mxpi 
    \le 
    \min\rbr*{\Cinf(x),~~ \rbr*{\frac{\Calpha(x)}{\veps}}^{\frac{1}{\alpha - 1}}}.
  \]
\end{proposition}
In particular, for the coverage coefficient $\Cone(x) = \cC_2^\pi(x)$, 
the above result shows that $\Mxpi \le \frac{\Cone(x)}{\veps}$. 
The proof below utilizes \citet[Example 7]{block2023sample} and the fact that $\Calpha$ controls the Renyi divergence of order $\alpha$ between $\pi$ and $\piref$.  

\begin{proof}[\pfref{prop:mstar}]
  As the prompt is fixed, we omit $x$ dependencies for notational compactness. 
  Recall that for any $M$, 
  $
    \Epi = \pi(\Yout) - M\cdot\piref(\Yout).
  $
  The statement for $M = \Cinf$ follows directly from the definition of $\Epi$ since $\Epi[\Cinf] = 0$. For $\Calpha$, it can be seen that $\pi(\Yout) \le \frac{\Calpha}{M^{\alpha-1}}$ since 
  \begin{align}
    \Calpha \ge \sum_y \frac{\pi(y)^\alpha}{\piref(y)^{\alpha-1}} \cdot \one{y \in \Yout}
    > M^{\alpha-1} \sum_y \pi(y) \cdot \one{y \in \Yout} = M^{\alpha-1} \cdot \pi(\Yout). 
  \end{align}
  Then for $M$ to satisfy 
  \begin{align}
    \veps \le \Epi \le \frac{\Calpha}{M^{\alpha-1}} - M \cdot \piref(\Yout) \le \frac{\Calpha}{M^{\alpha-1}}, 
  \end{align}
  it suffices to have $M = \rbr*{\frac{\Calpha}{\veps}}^{\frac{1}{\alpha - 1}}$. 
\end{proof}

\subsection{Guarantee for \rejalg}
\label{app:upper_rejection}
This section contains sample complexity upper bounds for \cref{alg:rej},
which express the size of $N$
required to ensure that the total variation distance between 
the law of responses drawn from \cref{alg:rej}
and the target distribution $\pi$ is small.
The main result, \cref{lem:x-tv-upper-bound},  that expresses sample
complexity in terms of $M$ and $\Epix$.

For a fixed $N$,
\cref{lem:x-tv-upper-bound} shows that the total variation distance 
is bounded by the \Edivlong corresponding to the input \Mlong $M$,
as well as an exponential in $\frac{1}{N} $ term that 
bounds the error when the algorithm fails to terminate. 
As expected, the former term decreases as $M$ increases, 
while the lattter increases.
In addition, while increasing $N$ can reduce error from the latter term,
the former term is irreducible even as $N$ tends too infinity, 
and this too we expect given that $\Epi$ is an information-theoretic measure of error
that is a property of the distributions and the choice of $M$.

\begin{lemma}[Per-prompt rejection sampling upper bound; adapted from Theorem 3 in \citet{block2023sample}]\label{lem:x-tv-upper-bound}
  For any valid policy $\pi: \cX \rightarrow \Delta(\cY)$, $ N \in \bbZ $, and $M\in\bbR_+$,
  for a given prompt $x\in\cX$,
  let $\pirej(x) \in \Delta(\cY)$ be the law of responses induced by running 
  $\rejalgNM[\frac{\pi}{\piref} ]$ (\cref{alg:rej}).
  We have
  \begin{align}
    \Dtv{\pi(x)}{\pirej(x)}
    \le \Epix + \frac{1}{2} \exp\rbr*{ -\frac{N \cdot \rbr*{1 - \Epix}}{M}}.
  \end{align}
\end{lemma}

In particular, if $M = \Mxpi$ and $N \gtrsim \Mxpi \cdot \log\rbr*{\frac{1}{\veps} }$ for some $\veps \in \rbr*{0, \frac{1}{2}}$,  
we have that $\Dtv{\pi(x)}{\piref(x)} \lesssim \veps$.  

\begin{proof}[\pfref{lem:x-tv-upper-bound}]
  As the prompt is fixed, we omit $x$ dependencies for notational compactness. 
  Define the truncated pseudo-distribution $\wt\pi = \min\{\pi, M\cdot\piref\}$.
  Define $A_M \ldef{} \sum_{ y }{ \wt\pi(y) }$ to be the total mass of the truncated policy.
  Recalling that $\cY_M(x) = \cbr*{y \in \cY : \pi(y\mid{}x) > M \cdot \piref(y\mid{}x)}$,
  $A_M$ can be equivalent expressed as
  \begin{align}
    A_M
    &= \sum_y \min\{\pi(y), M\cdot\piref(y)\}
    \\
    &= \sum_{y \in \Yout} M \cdot \piref(y) + \sum_{y \notin \Yout}\pi(y)
    \\
    &= 1 - \Epi.
  \end{align}
  Recall that, for a single sample $y \sim \piref$,
  \cref{alg:rej} samples $\xi \sim \Ber(p_y)$, where
  $p_y \ldef{}  \min\cbr*{\frac{\pi(y)}{\piref(y) \cdot M}, 1}$ is a Bernoulli random variable such that
  $\Pr_{\xi \sim \Ber(p_y)}\rbr*{\xi = 1 \mid{} y} = p_y$.
  Then the expected probability of acceptance for a single sample is
  \begin{align}
    \Pr_{y\sim\piref, \xi\sim\Ber(p_y)}\rbr*{\xi = 1} 
    &= \En_{y \sim \piref}\sbr*{\Pr\rbr*{\xi = 1 \mid{} y}}
    \\
    &= \sum_{y} \piref(y) \cdot \min\cbr*{\frac{\pi(y)}{M \cdot \piref(y)}, 1}
    \\
    &= \frac{1}{M} \sum_y \min\cbr*{\pi(y), M \cdot \piref(y)}
    \\
    &= \frac{A_M}{M},
  \end{align}
  and it can be seen that the law of the response conditioned on acceptance,
  $\sum_{y \in \Yout} M \cdot \piref(y) + \sum_{y \notin \Yout}\pi(y)$,
  is equivalent to the normalized version of $\wt\pi$ since
  \begin{align}
    \Pr_{y' \sim \piref}\rbr*{y' = y \mid{} \xi=1}
    = \frac{\Pr_{y' \sim \piref, \xi \sim \Ber(p_y)}\rbr*{y' = y, \xi = 1}}{\Pr_{y' \sim \piref, \xi \sim \Ber(p_y)}(\xi = 1)}
    = \frac{\piref(y) \cdot \min\cbr*{\frac{\pi(y)}{M \cdot \piref(y)}, 1}}{ A_M / M}
    = \frac{\wt\pi(y)}{A_M}.
  \end{align}
  Next, given $\Yhat = \cbr*{y_{1},\ldots,y_N}$,
  let $\xiN = \cbr*{\xi_1, \ldots, \xi_N}$ be the random draw of Bernoulli random variables,
  where $\xi_i \sim \Ber(p_{y_i})$ for all $i \in [N]$.
  For short, we write $\PrN(\cdot) \equiv \Pr_{\Yhat\sim\piref, \xi_1 \sim\Ber(y_1),\ldots,\xi_N\sim\Ber(y_N)}(\cdot)$.
  Define also the following event, which is a random variable over the draw of $\Yhat$ and $\xiN$,
  under which rejection sampling accepts one of the $N$ responses
  and \cref{alg:rej} outputs the $\istar$th response in \lineref{line:rej-term},
  \begin{equation}\label{eq:rej-term}
    \term \ldef{} \cbr*{ \exists \istar{}\in [N] \text{ s.t. } \xi_{\istar} = 1}.
  \end{equation}
  Let $\yhat = \rejalgNM[\frac{\pi}{\piref}]$ be the response output by \cref{alg:rej},
  which is a random variable over draws of $\Yhat$ and $\xiN$.
  Due to independence, we have that the law of responses is given by
  \begin{align}\label{eq:term-law}
    \PrN\rbr*{ \yhat = y \mid{} \term }
    &= \PrN\rbr*{ y_{\istar} = y \mid{} \exists \istar\in[N] \st \xi_{\istar} = 1 }
    = \frac{\wt\pi(y)}{A_M}.
  \end{align} 
  It can be observed that via union bound,
  \begin{align}
    \PrN\rbr*{\neg\term}
    &= \PrN\rbr*{\xi_i = 0,~\forall i\in[N]}
    \\
    &= \rbr*{1 - \Pr_{y\sim\piref, \xi\sim\Ber(p_y)}(\xi = 1)}^N
    \\
    &= \rbr*{1 - \frac{A_M}{M}}^{N}
    \\
    &\le e^{-\frac{N \cdot A_M}{M}}.
    \label{eq:noterm-lb}
  \end{align}
  Recall that $\pirej(y) = \PrN\rbr*{y}$.
  Now for the upper bound, we first decompose the total variation distance using $\wt\pi$, 
  \begin{align}
    \Dtv{\pi}{\pirej}
    = \frac{1}{2} \sum_{y} {} \abr*{\pi(y) - \pirej(y)}
    \le \frac{1}{2} \underbrace{\sum_y {} \abr*{\pi(y) - \wt\pi(y)}}_{\uterm{1}}
    + \frac{1}{2} \underbrace{\sum_y {} \abr*{\pirej(y) - \wt\pi(y)}}_{\uterm{2}},
  \end{align}
  where, via \cref{def:ediv}, the first term is equivalent to
  \begin{align}
    \uterm{1}
    = \sum_y {} \pi(y) - \min\cbr*{\pi(y), M\cdot\piref(y)}
    = \sum_y {} \pi(y) \cdot \one{\pi(y) > M \cdot \piref(y)} 
    = \Epi.
  \end{align}
  and we further bound the second term as 
  \begin{align}
    \uterm{2}
    &= \sum_{y}{} \abr*{\PrN\rbr*{\yhat=y \mid{} \term}
    + \PrN\rbr*{\yhat=y \mid{} \neg\term}
    - \pitil(y)}
    \\
    &\le \underbrace{\sum_{y}{} \abr*{\PrN\rbr*{\yhat=y \mid{} \term} - \pitil(y)}}_{\uterm{3}}
    + \PrN(\neg\term)
  \end{align}
  From \cref{eq:noterm-lb}, have $\PrN(\neg\term) \le e^{-\frac{N \cdot A_M}{M}}$.
  In addition, using \cref{eq:term-law} and the fact that $A_M \le 1$,
  we have that
  \begin{align}
    \uterm{3}
    &= \sum_{y}{} \abr*{\PrN\rbr*{\yhat=y \mid{} \term} - \pitil(y)}
    \\
    &= \sum_{y}{} \abr*{\frac{\pitil(y)}{A_M} - \pitil(y)}
    = \sum_{y}{} {\frac{\pitil(y)}{A_M} - \pitil(y)}
    \\
    &= 1 - A_M = \Epi,
  \end{align}
  so that 
  \begin{align}
    \uterm{2} \le \Epi +  \exp^{-\frac{N \cdot A_M}{M}}.
  \end{align}
  Combining $\uterm{1}$ and $\uterm{2}$,
  \begin{align}
    \Dtv{\pi}{\piref}
    &\le \frac{1}{2} \rbr*{\uterm{1} + \uterm{2}} 
    \\
    &\le \Epi + \frac{1}{2}  \exp^{-\frac{N \cdot A_M}{M}},
  \end{align}
  and substituting the expression for $A_M = 1 - \Epi$ gives the result.
\end{proof}

\subsection{Lower Bounds}
\label{app:lower_rejection}

For a fixed prompt $x$, 
\cref{lem:approx-tv-lower-bound} shows $N \gtrsim \Mxpi \log\rbr*{\veps^{-1}}$
samples are sufficient to guarantee that $\Dtv{\pi(x)}{\piref(x)} \le \veps$.
The information-theoretic lower bound in \cref{lem:approx-tv-lower-bound}
shows that this dependence is tight
for any selection strategy that selects a response from $\Yhat$,
defined more formally below.

\begin{definition}[Selection strategy $\cA$]\label{def:tv-selection}
  A selection strategy $\cA$ is any method that,
  given a prompt $x$ and $N$ responses $\Yhat = \rbr*{y_{1}, \ldots, y_N}$ sampled i.i.d. from $\piref$,
  returns some (possibly random) $y \in \Yhat$.
\end{definition}

\cref{lem:approx-tv-lower-bound} states that selection strategy (\cref{def:tv-selection})
must obtain at least $\Mxpi$ samples in order to induce a response distribution 
that is $\veps$-close to $\pi$.
We will later use this result as a component
within our main regret lower bounds (\cref{thm:lower-bon,thm:lower}).

\begin{lemma}[TV lower bound, adapted from the proof of Theorem 5 in \citet{block2023sample}]\label{lem:approx-tv-lower-bound}
  Fix the base policy $\piref$, 
  target policy $\pi$, 
  prompt $x$, 
  and sample size $N$.
  Let $\cA$ be any selection algorithm (\cref{def:selection}) 
  that, given $\Yhat\sim\piref(\cdot\mid{}x)$ in the sample-and-evaluate framework (\cref{def:oracle-model}), outputs a response $y \in \Yhat$. Let $\pi_\cA : \cX \rightarrow \Delta(\cY)$ denote the distribution of responses induced by $\cA$.
  Then if $N < \Mxpi$, we have 
  \[
    \Dtv{\pi_\cA(x)}{\pi(x)} > \veps.
  \]

\end{lemma}
\begin{proof}[\pfref{lem:approx-tv-lower-bound}]
  In the proof below we omit $x$ dependencies,
  including in $\Mxpi \equiv \Mpi$, and $\Epix \equiv \Epi$. First fix $M$. Suppose $N < M$. Then 
  \begin{align}
    2\cdot\Dtv{\pi_{\cA}}{\pi} \ge&~ \sum_{y \in \Yout} \abr{\pi(y) - \pi_{\cA}(y)} 
    \\
    \ge&~ \sum_{y \in \Yout} \pi(y) - \Pr\rbr*{y \in \Yhat}
    \\
    \ge&~ \sum_{y \in \Yout} \pi(y) - N \cdot \mu(y)
    \\
    \ge&~ \sum_{y \in \Yout} \pi(y) - M \cdot \mu(y) 
    \\
    =&~ 2 \cdot \Epi
  \end{align}
  It follows that if $N < M \le \Mpi$, we have 
  \begin{align}
    \Dtv{\pi_{\cA}}{\pi} > \veps
  \end{align}
  from the definition of $\Mpi$, since $\Epi$ is non-decreasing as $M$ decreases.  
\end{proof}

\newpage

\section{Proofs from  Section \ref*{sec:framework}}
\label{app:framework}

  \begin{proof}[\pfref{prop:coverage}]
    Let us write
    $J_{r}(\pi,x)\ldef{}\En_{y\sim\pi(\cdot\mid{}x)}\brk*{r(x,y)}$ to
    denote the expected reward under a function $r(x,y)$. We first
    state an elementary technical lemma.
    \begin{lemma}
      \label{lem:coverage}
      For any prompt $x\in\cX$, reward model
      $\rhat:\cX\times\cY\to\bbR$, and policies
      $\pistar,\pihat:\cX\to\Delta(\cY)$ and $\veps>0$, there exists
      $\rstar:\cX\times\cY\to\bbR$ such that
  \begin{align}
    J_{\rstar}(\pistar;x) - J_{\rstar}(\pihat;x) \geq
    J_{\rhat}(\pistar;x) - J_{\rhat}(\pihat;x) + \veps\cdot{}\sqrt{\sum_{y\in\cY}\frac{(\pistar(y\mid{}x)-\pihat(y\mid{}x))^2}{\piref(y\mid{}x)}}
  \end{align}
  and $\En_{y\sim\piref(\cdot\mid{}x)}\brk*{(\rhat(x,y)-\rstar(x,y))^2}\leq\veps^2$.
\end{lemma}
\begin{proof}[\pfref{lem:coverage}]
We use the choice
  \[
\rstar(x,y) = \rhat(x,y) + \veps\cdot \frac{\pistar(y\mid{}x)-\pihat(y\mid{}x)}{\piref(y\mid{}x)}\cdot{}\prn*{\sum_{y\in\cY}\frac{(\pistar(y\mid{}x)-\pihat(y\mid{}x))^2}{\piref(y\mid{}x)}}^{-1/2},
\]
from which the result is immediate.
\end{proof}

Going forward, we omit dependence on $x\in\cX$ to keep notation
compact. Define
$\Cmax=\max_{\pi:\cX\to\Delta(\cY)}\Cone[\pi]$ and
$\pimax=\argmax_{\pi:\cX\to\Delta(\cY)}\Cone[\pi]$, and let $\Cmax\geq\Cpar\geq{}8$ be given. Throughout the proof, we will use the fact
that $\Cone\geq{}1$ for all $\pi$. We set $\rhat(x,y)=0$ and consider
two cases for the analysis.

\paragraph{Case 1: $\Chat>\frac{1}{8}\Cpar$}
In this case, we invoke \cref{lem:coverage} with $\pistar=\piref$ and $\veps=\vepsrm$,
which shows that there exists some $\rstar$ for which
\[
    J_{\rstar}(\pistar) - J_{\rstar}(\pihat) \geq
\vepsrm\cdot{}\sqrt{\sum_{y\in\cY}\frac{(\piref(y)-\pihat(y))^2}{\piref(y)}}.
\]
The result now follows by noting that
\[
  \sum_{y\in\cY}\frac{(\piref(y)-\pihat(y))^2}{\piref(y)}
  = \Cone[\pihat]-1
  \geq \frac{1}{8}\Cpar-1 \geq \frac{1}{16}\Cpar.
\]

\paragraph{Case 2: $\Chat\leq\frac{1}{8}\Cpar$}
In this case, we set
\[
\pistar = \lambda\pimax + (1-\lambda)\piref
\]
for $\lambda^2\ldef{}\frac{\Cpar}{2\Cmax}\in(0,1)$. We compute directly
that
\[
\Cstar = \lambda^2\Cmax + (1-\lambda^2),
\]
so that $\Cstar \in\brk*{\frac{1}{2}\Cpar, \Cpar}$. We invoke
\cref{lem:coverage} with $\pistar$ and $\veps=\vepsrm$, which gives that there exists some $\rstar$ for which
\[
    J_{\rstar}(\pistar) - J_{\rstar}(\pihat) \geq
\vepsrm\cdot{}\sqrt{\sum_{y\in\cY}\frac{(\pistar(y)-\pihat(y))^2}{\piref(y)}}.
\]
We further compute that
\begin{align}
  \sum_{y\in\cY}\frac{(\pistar(y)-\pihat(y))^2}{\piref(y)}
  = \Cone[\pistar] + \Cone[\pihat]
  - 2\sum_{y\in\cY}\frac{\pistar(y)\pihat(y)}{\piref(y)}
  \geq{} \frac{1}{2}\Cone[\pistar] - \Cone[\pihat],
\end{align}
where the last step follows by the AM-GM inequality, i.e.
\begin{align}
  \sum_{y\in\cY}\frac{\pistar(y)\pihat(y)}{\piref(y)}
  \leq{} \frac{1}{4}\sum_{y\in\cY}\frac{(\pistar(y))^2}{\piref(y)}
  + \sum_{y\in\cY}\frac{(\pihat(y))^2}{\piref(y)}.
\end{align}
By the assumption for this case, we have $\Chat\leq\frac{1}{8}\Cpar$,
so that
\begin{align}
  \frac{1}{2}\Cone[\pistar] - \Cone[\pihat]
  \geq{} \frac{1}{2}\Cone[\pistar] - \frac{1}{8}\Cpar
  \geq{} \frac{1}{4}\Cpar - \frac{1}{8}\Cpar = \frac{1}{8}\Cpar, 
\end{align}
completing the result.

  \end{proof}

\newpage

\section{Proofs from Section \ref*{sec:bon}}
\label{app:bon}
This section gives proofs for the guarantees for \bonalg in \cref{sec:bon}. These results in this section build on the guarantees for rejection sampling in \cref{app:rejection}, 
and are organized as follows:
\begin{itemize}
\item \cref{sec:bon_general,sec:bon_lower} provide general tools to
  analyze \bonalg. In \cref{sec:bon_general}, we provide a general upper bound on the
  regret of \bonalg in terms of the \emph{\Edivlong} introduced
  in \cref{app:rejection}, and in \cref{sec:bon_lower} we give general
  lower bounds on the regret.
\item In \cref{sec:bon1,sec:bon2,sec:bon3}, we instantiate these
  results to prove the main theorems in \cref{sec:bon}. Namely, these results leverage \cref{prop:mstar} to translate 
  the \Edivlong to the 
  $\Cone$ and $\Cinf$ coverage coefficients,
  which are standard measures of distribution shift in offline alignment (cf. \cref{prop:coverage}).
  \icml{\item \cref{sec:bon_iid} contains additional regret bounds for
    \bonalg that handle the setting where the prompt $x$ is drawn
    \iid, and \cref{sec:bon_additional} contains further extensions.}
\end{itemize}

\subsection{General Regret Decomposition for \bonalg}
\label{sec:bon_general}

Recall that the \Edivlong (\cref{def:ediv}) is defined for a
parameter $M>0$ via
\begin{align}
    \Epix \ldef{} \sum_{y \in \Youtx}{\pi(y\mid{}x) - M \cdot \piref(y\mid{}x)} = \En_{y \sim \piref(x)}\sbr*{\rbr*{\frac{\pi(y\mid{}x)}{\piref(y\mid{}x)} - M}_+}.
\end{align}

Our central technical result for the analysis of \bonalg is 
\cref{lem:x-regret-upper-general} below, which quantifies the regret of the \bon policy
given $N$ samples. This result will later 
be instantiated to prove \cref{thm:bon} and \cref{thm:bon_uniform}. Even though \cref{lem:x-regret-upper-general} concerns \bonalg, its analysis makes use of
the rejection sampling algorithm (\cref{alg:rej}) as a tool to analyze certain intermediate quantities.
As a result, the lemma statement contains an extra parameter $M$, 
which corresponds to a rejection sampling threshold (cf. \cref{alg:rej}), and the regret upper bound is expressed in terms of 
the information-theoretic \Edivlong (\cref{def:ediv}) which appears in the analysis of rejection sampling in \cref{app:rejection}.\loose

\begin{lemma}[\Edivlong regret bound for \bonalg]\label{lem:x-regret-upper-general}
  Fix a prompt $x$. 
  For any comparator policy $\pistar$ and $N \in \bbZ$ and $M\in\bbR_{+}$ such that $\Epistx \le \frac{1}{2} $, the \bon policy $\piN$ satisfies
  \begin{align}
    J(\pistar;x) - J(\piN;x) 
    &\le \Rmax \cdot \rbr*{\Epistx + \exp\prn*{ -\frac{N}{M} \cdot \rbr*{1-\Epistx}}} 
    \\
    &~+ 2\cdot\sqrt{ \Cone[\pistar](x)\cdot \vepsrms(x) } + \frac{\vepsrm(x)}{2} + \sqrt{N\cdot \vepsrms(x)}
  \end{align}
\end{lemma}

Because \cref{lem:x-regret-upper-general} holds 
for \emph{any choice} of $M\in\bbR_+$, 
$M$ can be viewed as a parameter (within the analysis)
that trades off between the best regret achievable
---which is upper bounded by $\Epistx$, which decreases with $M$---
and the sample complexity required to achieve it, 
which increases with $M$.
Indeed, when we later prove \cref{thm:bon} and \cref{thm:bon_uniform} later, 
we will choose $M$ to make the RHS of \cref{lem:x-regret-upper-general}
tight when $\Epix$ is translated to our coverage coefficients of interest. 

\paragraph{Proof sketch} The high-level idea of the proof is to use as 
an intermediate comparator the response distribution $\pirej$
induced simulating sampling from $\pistar$ via 
$\rejalg_{N,M}(\frac{\pistar}{\piref}\midsem{}\piref,x)$, 
with the parameter $M$ from the lemma statement. 
The utility of using such a comparator is that, 
because the \bon procedure always chooses the response with largest reward label, 
$\piN$ is always able to compete with $\pirej$ under $\rhat$ 
(in the sense of having larger expected estimated reward). 

We can then translate this observation to the desired bound 
by recalling that $\pirej$ approximates $\pistar$ in total variation distance 
(from \cref{lem:x-tv-upper-bound}),
which means that $\piN$ is also approximately as good as $\pistar$
under $\rhat$.
Lastly, we translate the performance under $\rhat$
to the performance under the true reward $\rstar$, 
which is penalized by the reward estimation error $\vepsrm(x)$ 
on the \bon response distribution, and the gap between the two is quantified by the reward estimation error 
under the distribution of responses drawn from $\piN$, 
which is the source of the $\sqrt{ N}$ term.

\begin{proof}[Proof of \cref{lem:x-regret-upper-general}]
  For simplicity in this proof,
  we will assume WLOG that $\rhat(y)$ is unique for all $y \in \cY$,
  otherwise we can perturb $\rhat(y)$ with a miniscule $\veps(y) \ll \vepsrm$,
  as long as it is within floating-point precision
  to break ties. 

  Let $\pisrej(x)$ be the distribution of responses induced by running
  $\rejalgNM[\frac{\pistar}{\piref}]$ 
  which we use to decompose the regret as follows:
  \begin{align}
    J(\pistar;x) - J(\piN;x) 
    & = J(\pistar;x) - J(\pisrej;x) + J(\pisrej;x) - J(\piN;x) 
    \\
    & \le \Rmax \cdot \Dtv{\pistar(x)}{\pisrej(x)} + J(\pisrej;x) - J(\piN;x) .
  \end{align}
  We next bound the last pair of terms as a function of the reward
  estimation error.
  Below, we use 
  $\En_{\Yhat\sim\piref(x)}\sbr*{\cdot}$ 
  to refer to expectations over $N$ samples $\Yhat=(y_1,\ldots,y_N)$ drawn \iid from $\piref(x)$, 
  and, given $\Yhat$,  
  we use $\En_{y\sim\pisrej(x)\mid{}\Yhat}[\cdot]$ to refer to 
  the expectation over responses induced by running\newline
  $\rejalgNM[\nicefrac{\pistar}{\piref}]$
  conditioned on the realization of the set $\Yhat=(y_1,\ldots,y_N)$ drawn by the algorithm. We define $\En_{y\sim\piN(x)\mid{}\Yhat}[\cdot]$ analogously.\loose

  We begin by decomposing the second term above as follows:
  \begin{align}
    J(\pisrej;x) - J(\piN;x) 
    &= \En_{y \sim \pisrej(x)}\sbr*{\rstar(x,y) - \rhat(x,y)}
    + \En_{\Yhat\sim\piref(x)}\sbr*{\En_{y\sim\pisrej(x)\mid{}\Yhat}\sbr*{\rhat(x,y)} - \En_{y\sim\piN(x)\mid{}\Yhat}\sbr*{\rhat(x,y)}}
    \\
    &\quad + \En_{y\sim\piN(x)}\sbr*{\rhat(x,y) - \rstar(x,y)}
    \\
    &\le \En_{y \sim \pisrej(x)}\sbr*{\rstar(x,y) - \rhat(x,y)} + \En_{y\sim\piN(x)}\sbr*{\rhat(x,y) - \rstar(x,y)}
    \\
    &\le  \underbrace{\En_{y \sim \pisrej(x)}\sbr*{\abr*{\rstar(x,y) - \rhat(x,y)}}}_{\mathrm{(T1)}} + \underbrace{\En_{y\sim\piN(x)}\sbr*{\abr*{\rhat(x,y) - \rstar(x,y)}}}_{\mathrm{(T2)}}.
  \end{align} 
  Note that above we use the linearity of expectation in the first inequality, so that 
  \[\En_{\Yhat\sim\piref(x)}\sbr*{\En_{y\sim\pisrej(x)\mid{}\Yhat}\sbr*{\rhat(x,y)} - \En_{y\sim\piN(x)\mid{}\Yhat}\sbr*{\rhat(x,y)}}\]
  couples the set $\Yhat$ drawn by the two algorithms, 
  and compares the performance of $\piN$ and $\pisrej$
  for a fixed set of $N$ responses.
  Then, because the \bon policy always chooses the response with 
  the largest value under $\rhat$ for any fixed $\Yhat$, i.e., 
  $\one{y\in\supp(\piN(\cdot\mid{}x))} = \one{\rhat(x,y) \ge \rhat(x,y'),~\forall y' \in \Yhat}$,
  it can be seen that 
  \begin{align}
   \En_{y\sim\pirej(x)\mid{}\Yhat}\sbr*{\rhat(x,y)} - \En_{y\sim\piN(x)\mid{}\Yhat}\sbr*{\rhat(x,y)} \le 0.
  \end{align} 
  
  For (T2), we first show that $\Cone[\piN](x) \le \Cinf[\piN](x) \le N$.
  Letting $\sum_{\Yhat\sim\piref(x)}$ refer to the sum over all possible
  sequences of $N$ responses $\Yhat = \prn*{y_{1},\ldots,y_N} \in \cY^N$,  
  and $\piref(y_{1},\ldots,y_N\mid{}x)= \prod_{i=1}^N \piref(y_{i}\mid{}x)$ to be its probability,
  we can express the \bon policy in closed form as
  \begin{align}
      \piN(y\mid{}x) = \sum_{\Yhat\sim\piref(x)} \piref(y_{1},\ldots,y_N\mid{}x)\cdot \one{y \in \Yhat} \cdot \one{\rhat(x,y)  \ge\rhat(x,y'), ~\forall y' \in \Yhat},
  \end{align}
  since, conditioned on a set of samples $\Yhat$,
  the \bon algorithm deterministically outputs the one with the largest $\rhat$ value. 
  The base policy $\piref$ can be written in a similar form, by marginalizing over the process through which we sample $\Yhat$, then sample $y$ uniformly from this set:
  \begin{align}
      \piref(y\mid{}x) = \sum_{\Yhat\sim\piref(x)} \piref(y_{1},\ldots,y_N\mid{}x) \cdot \sum_{y' \in \Yhat}\frac{\one{y' = y}}{N}.
  \end{align}
  Then for any $y$, we can upper bound the likelihood ratio between 
  the $\bon$ policy and $\piref$ by $N$,  
  \begin{align}
      \frac{\piN(y\mid{}x)}{\piref(y\mid{}x)} =&~ N \cdot \frac{\sum_{\Yhat\sim\piref(x)} \piref(y_{1},\ldots,y_N\mid{}x) \cdot \one{y \in \Yhat} \cdot \one{\rhat(x,y)  \ge\rhat(x,y'), ~\forall y' \in \Yhat}}
      {\sum_{\Yhat\sim\piref(x)} \piref(y_{1},\ldots,y_N\mid{}x) \cdot \sum_{y' \in \Yhat}\one{y' = y}}
      \\
      \le&~ 
      N \cdot \frac{\sum_{\Yhat\sim\piref(x)} \piref(y_{1},\ldots,y_N\mid{}x) \cdot \one{y \in \Yhat}\cdot \one{\rhat(x,y)  \ge\rhat(x,y'), ~\forall y' \in \Yhat}}{\sum_{\Yhat\sim\piref(x)} \piref(y_{1},\ldots,y_N\mid{}x) \cdot \one{y \in \Yhat}}
      \\
      \le&~ N.
  \end{align}
  Now, to bound the reward estimation error in (T2), 
  we combine this result with the 
  Cauchy-Schwarz inequality, giving
  \begin{align}
    \mathrm{(T2)} \le \sqrt{\Cone[\piN](x)\cdot\En_{\piref(x)}\sbr*{\rbr*{\rhat(x,y) - \rstar(x,y)}^2}} \le \sqrt{N \cdot \vepsrms(x)}.
  \end{align}
  For (T1), we leverage results from \cref{app:rejection}. 
  Recall the random event from \cref{eq:rej-term} random draws of $\Yhat$
  and $\xiN = (\xi_1,\ldots,\xi_N)$,
  under which \cref{alg:rej} returns a response in \lineref{line:rej-term},
  \begin{equation}
    \term \ldef{} \cbr*{ \exists \istar{}\in [N] \text{ s.t. } \xi_{\istar} = 1}.
  \end{equation}
  From the proof of \cref{lem:x-tv-upper-bound} (\cref{app:upper_rejection}),
  recall that we can write the induced policy as
  \[
    \pisrej(y \mid{} x)
    = \PrN\rbr*{y \mid{} x, \term}\cdot\PrN(\term \mid{} x)
    + \PrN\rbr*{y \mid{} x, \neg\term} \cdot \PrN(\neg\term \mid{} x).
  \]
  On the event of $\neg\term$, \cref{alg:rej} returns a randomly drawn response
  $y_{N+1} \sim \piref(\cdot \mid{}x)$ in \lineref{line:rej-noterm}, thus 
  \[
     \PrN\rbr*{y \mid{} x, \neg\term} = \piref(y \mid{} x)
  \]
  As a result,
  \begin{align}
    \pisrej(y\mid{}x) 
    &= \PrN\rbr*{\term\mid{}x}\cdot\PrN\rbr*{y\mid{}x,\term}
    + \PrN\rbr{\neg\term\mid{}x} \cdot \piref(y\mid{}x)
    \\
    & \le \PrN\rbr*{y\mid{}x,\term} + \frac{1}{2} \cdot \piref(y\mid{}x)
    \\
    &= \frac{\min\cbr*{\pistar(y\mid{}x), M\cdot\piref(y\mid{}x)}}{1 - \Epistx} + \frac{1}{2} \cdot \piref(y\mid{}x)
    \\
    &\le 2\cdot\min\cbr*{\pistar(y\mid{}x), M\cdot\piref(y\mid{}x)} + \frac{1}{2} \cdot \piref(y\mid{}x)
  \end{align}
  where in the last inequality we have used the assumption that
  $\Epistx\le\frac{1}{2}$, and in the first we use the observation that
  $\PrN(\neg\term\mid{}x)\le\frac{1}{2}$ since
  $\PrN\rbr*{\term\mid{}x} = \frac{1 - \Epistx}{M}$ and $M \ge 1$.
  We can then use Cauchy-Schwarz to bound
  \begin{align}
   \mathrm{(T1)}
   &= \En_{y \sim \pisrej(x)}\sbr*{\abr*{\rstar(x,y) - \rhat(x,y)}}
   \\
   & \le 2\cdot\En_{y \sim \pistar(x)}\sbr*{\abr*{\rstar(x,y) - \rhat(x,y)}} + \frac{1}{2}\cdot\En_{y \sim \piref(x)}\sbr*{\abr*{\rstar(x,y) - \rhat(x,y)}} 
   \\
   & \le 2\cdot\sqrt{ \Cone[\pistar](x)\cdot \vepsrms(x) } + \frac{\vepsrm(x)}{2}
  \end{align}
  Combining all the preceding bounds, we obtain 
  \begin{align}
    J(\pisrej;x) - J(\piN;x) 
    \le 2\cdot\sqrt{ \Cone[\pistar](x)\cdot \vepsrms(x) } + \frac{\vepsrm(x)}{2} + \sqrt{N\cdot \vepsrms(x)}
  \end{align}
  thus the regret is bounded as
  \begin{align}
    J(\pistar;x) - J(\piN;x) 
    \le \Rmax \cdot \Dtv{\pistar(x)}{\pisrej(x)} + 2\cdot\sqrt{ \Cone[\pistar](x)\cdot \vepsrms(x) } + \frac{\vepsrm(x)}{2} + \sqrt{N\cdot \vepsrms(x)}.
  \end{align}
  Finally, we apply \cref{lem:x-tv-upper-bound}, which bounds
  \begin{align}
    \Dtv{ \pistar(x) }{ \pisrej(x) } \le \Epistx + \exp\prn*{ -\frac{N}{2M} \cdot \rbr*{1-\Epistx} }
  \end{align}
  to give the lemma statement.
\end{proof}

\subsection{General Lower Bounds on Regret}
\label{sec:bon_lower}

This section contains two regret lower bounds that apply
to both \bon and \mainalg across a range of parameter values, 
for a single prompt $x$. 
Each bound contains an information-theoretic component
that applies to \emph{any selection algorithm}, 
which is defined formally in \cref{def:selection},
and takes the general form that if $N < \mathrm{(threshold)}$, 
the regret of any selection algorithm will be 
at least $\poly(\Cone[\pistar](x), \vepsrm(x))$.
The results also have a component that is specific to \bon, 
that when $N \ge \mathrm{(threshold)}$, 
\bon has at least $\sqrt{ N \cdot \vepsrms(x)}$ regret. 
\begin{itemize}
  \item \cref{thm:cinf-regret-lower} in \cref{app:inf_lower_bon}
  shows that $\mathrm{(threshold)} \propto \Cinf(x)$,
  and for smaller $N$ any algorithm pays $\sqrt{\Cinf[\pistar](x) \cdot \vepsrms(x)}$ regret, 
  and for larger $N$ \bon incurs $\sqrt{N \cdot \vepsrms(x)} \ge \sqrt{\Cinf[\pistar](x)\cdot\vepsrms(x)}$ regret. 
  \item \cref{thm:cone-regret-lower} in \cref{app:one_lower_bon} 
  utilizes a construction where $\Cinf[\pistar](x)$ is exponentially larger than $\Cone[\pistar](x)$, 
  and any algorithm has regret at least 
  $\rbr*{\Cone[\pistar](x)\cdot \veps}^{p}$  for a range of $\veps \ge \vepsrm(x)$,
  unless $N \ge \mathrm{(threshold)} \propto \rbr*{\vepsrms(x)}^{-p}$.
  \bon again pays $\sqrt{N \cdot \vepsrms(x)}$ in this regime. 
\end{itemize}
  The latter result is later used to prove \cref{thm:lower-bon},
  where $p=\frac{1}{3} $ to balance the terms,
  and \cref{thm:lower}, 
  where $p=\frac{1}{2} $.

\paragraph{Proof techniques} 
The information-theoretic component of the results 
reflects the difficulty of simulating a sample from the target policy's distribution,
which is required for any selection algorithm $\cA$
to be able to compete with the target policy. 
Recall that the lower bound for rejection sampling (\cref{lem:approx-tv-lower-bound})
states that any selection algorithm requires 
at least $\Mxpist$ samples to be $\veps$-close 
to the target distribution in TV distance.
To convert this result to regret lower bounds, we 
a) construct a pair of distributions for which the conversion 
from $\Mxpist$ to coverage coefficient $\Cone(x)$ 
in \cref{prop:mstar} is tight, and 
b) specify reward functions $\rstar$ and $\rhat$ so
that the regret maximally witnesses 
the reward estimation error $\vepsrm(x)$
where rejection sampling fails to approximate $\pistar$, 
and where $\piN$ overfits to $\rhat$.

While the construction used for the $\Cinf[\pistar](x)$ result is 
relatively simple and has $\abr*{\cY}=O(1)$, 
the construction for the $\Cone[\pistar](x)$ utilizes a countable infinite response space $\cY$, 
which is necessary to create the exponential separation between $\Cone[\pistar](x)$ and $\Cinf[\pistar](x)$.
If we label the responses $i = 1, 2, 3, \ldots$, 
the ratio $\frac{\pistar}{\piref}$ increases exponentially in $i$,
and $\rstar$ increases in $i$
while the reward model $\rhat$ decreasess.
Because $\pistar$ here is an exponential tilting of $\piref$ with respect to the true reward,
the lower bound construction reflects the structure of language modeling,
where policies are parameterized as softmax functions and the response space is exponentially large, 
which we believe may be of independent interest.  

\paragraph{Preliminaries}
The lower bound constructions below utilize only a single prompt $x$, 
and in the proofs (not the theorem statements), we drop the $x$ dependence for notational compactness.
For example, $\pi(y) \equiv \pi(y\mid{}x)$ since $\cX= \cbr*{x}$, 
$\Cone \equiv \Cone(x)$, etc. 
Lastly, we formally define what we mean by ``any selection algorithm''
in the sample-and-evaluate framework.
\begin{definition}[Inference-time selection algorithm $\cA$]\label{def:selection}
  Under the sample-and-evaluate framework (\cref{def:oracle-model}), 
  an inference-time selection algorithm $\cA$ is any mapping from
  $\Yhat=(y_1,\ldots,y_N)\sim\piref(\cdot\mid{}x)$ and $\cbr*{\rhat(x, y_{i})}_{i \in [N]}$ 
  to a response $y \in \Yhat$;
   we define $\pi_\cA : \cX \rightarrow \Delta(\cY)$ to be 
  the law over responses that the algorithm induces. 
\end{definition}

\subsubsection{Lower Bounds under $L_\infty$-Coverage}
\label{app:inf_lower_bon}

\begin{theorem}[Regret lower bound for $\cC^{\pistar}_{\infty}$]\label{thm:cinf-regret-lower}
  For any $\vepsrms \in [0, 1]$ and $C \ge 1$, 
  there exists a comparator policy $\pistar$ over contexts $\cX = \{x\}$ and responses $\cY$
  for which the following statements hold.
  \begin{enumerate}
    \item 
    There exists a problem instance $(\piref, \rstar, \rhat)$ with 
    $\vepsrms(x) \le \vepsrm$ and
    \[
      \Cinf[\pistar](x) = \Cone[\pistar](x) = C
    \]
    such that, for any $N < \frac{C}{2}$,
    any inference-time selection algorithm $\cA$ (\cref{def:selection})
    has regret
    \[
      J(\pistar;x) - J(\pi_\cA;x) > \min\cbr*{2\sqrt{\Cinf[\pistar](x) \cdot \vepsrms(x)},~ 1}.
    \]
    \item For any $N \ge 1$,
    there exists a problem instance $(\piref, \rstar, \rhat)$ with 
    $\vepsrms(x) \le \vepsrm$ and
    \[
      \Cinf[\pistar](x) = \Cone[\pistar](x) = C
    \]
    such that \bonalg suffers regret
    \[
      J(\pistar;x) - J(\piN;x) \ge c\cdot\min\cbr*{\sqrt{N \cdot \vepsrms(x)} ,~ 1},
    \]
    where $c$ is a universal constant.
  \end{enumerate}
\end{theorem}

\begin{proof}[Proof of \cref{thm:cinf-regret-lower}]
  Because we condition on a single prompt $x$,
  we omit $x$ dependencies for ease of presentation; in particular, we abbreviate $\Mpistar \equiv \Mxpist$, and $\Epist \equiv \Epistx$.

  Fix $N$ and $\vepsrm$ and $C$.
  The response space is $\cY=\{y_0, \ystar, y_{\bad}\}$,
  the comparator policy $\pistar(y) = \one{y = \ystar}$ plays $\ystar$ deterministically,
  and the reference policy is
  \[
    \begin{array}{l l l}
      \piref(y_0) = 1 - \frac{1}{2N} - \frac{1}{C}, & \piref(\ystar) = \frac{1}{C}, & \piref(y_{\bad}) = \frac{1}{2N},
    \end{array}
  \]
  for which we have $\Cinf[\pistar] = \Cone[\pistar] = C$. 
  Next, for some $\veps \ge 0$ recall \cref{def:mstar},
  \[
    \Mpistar \ldef{} \min\cbr*{M \mid{} \Epist \le \veps}. 
  \]
  For our choice of policies we can compute $\Mpistar$ in closed form,
  since for any $M$, 
  \begin{align}
    \Epistar = 1 - M \cdot \piref(\ystar) = 1 - \frac{M}{C}.
  \end{align}
  Then any $M \ge C \cdot \rbr{1 - \veps}$ is sufficient to have $\Epistar \le \veps$, therefore
  \begin{equation}
    \label{eq:mpix1}
    \Mpistar = C \cdot \rbr{1 - \veps}.
  \end{equation}
  
  \paragraph{Part 1 (Small $N$)} Define the reward functions   %
  \[
  \begin{array}{l l l}
    \rstar(y_0) = 0 &\rstar(\ystar) = 1 &\rstar(\ybad) = 0
    \\
    \rhat(y_0) = 0 & \rhat(\ystar) = 1- \min\cbr*{\sqrt{C \cdot \vepsrms}, 1} & \rhat(\ybad) = 0
  \end{array}
  \]
  It is easy to see that 
  \begin{align}
    \En_{\piref}\sbr*{\rbr*{\rhat(y) - \rstar(y)}^2} \le \frac{C \cdot \vepsrms}{C} = \vepsrms. 
  \end{align}
  For some $\veps$ that we will set shortly, when $N < \Mxpistar$ the regret is lower bounded as 
  \begin{align}
    J(\pistar) - J(\piN)
    & = \pistar(\ystar) - \piN(\ystar)
    \\
    &\ge 1 - \Pr\rbr*{\ystar \in \Yhat}
    \\
    &\ge 2\Epistar[\Mpistar].
  \end{align}
  Then setting $\veps = \min\cbr*{\sqrt{C \cdot \vepsrms}, \frac{1}{2}}$ 
  and using the closed form of $\Mxpistar$ in \cref{eq:mpix1}, 
  \cref{lem:approx-tv-lower-bound} states that when
  \[
    N < \Mpistar = C \cdot \rbr*{ 1 -\min\cbr*{\sqrt{C \cdot \vepsrms},\frac{1}{2}} }
    = C \cdot \max\cbr*{1 - \sqrt{C \cdot \vepsrms} ,~ \frac{1}{2}},
  \] 
  we have $\Epistar[{\Mpistar}] > \veps = \min\cbr*{\sqrt{C \cdot \vepsrms}, \frac{1}{2}}$, and thus 
  \[
    J(\pistar) - J(\piN) > \min\cbr*{2\sqrt{C \cdot \vepsrms}, 1}~. 
  \]

  \paragraph{Part 2 (Large $N$) } Define the gap on $\ybad$ to be $\Delta \ldef{} \min\cbr*{1, \sqrt{N \cdot \vepsrms}}$, and set the reward functions 
  \[
    \begin{array}{l l l}
      \rstar(y_0) = 0 &\rstar(\ystar) = 1 &\rstar(\ybad) = 1 - \Delta
      \\
      \rhat(y_0) = 0 & \rhat(\ystar) = 1- \min\cbr*{\sqrt{\frac{C}{2} \cdot \vepsrms}, 1} & \rhat(\ybad) = 1
    \end{array}
  \]
  and we can check that 
  \begin{align}
    \En_{\piref}\sbr*{\rbr*{\rhat(y) - \rstar(y)}^2} 
    \le \frac{C\cdot \vepsrms}{2C} + \frac{\Delta^2}{2N}
    = \frac{\vepsrms}{2} + \frac{\min\cbr*{1, N \cdot \vepsrms}}{2N}
    \le \vepsrms.
  \end{align}
  Note that for this construction, $\rhat(\ybad)$ is the largest reward under $\rhat$, so $\piN$
  will always play $\ybad$ if $\ybad \in \Yhat$, which is an event
  that occurs with at least constant probability under our choice for $\piref$
  \begin{align}
    \Pr(\ybad \in \Yhat) = 1 - \rbr*{1 - \frac{1}{2N}}^{N} \ge 1 - e^{-\frac{1}{2}}. %
  \end{align}
  Using this, we lower bound the regret as 
  \begin{align}
    J(\pistar) - J(\piN) 
    & = \Pr(\ybad \in \Yhat) \cdot \En_{\Yhat}\sbr*{J(\pistar) - J(\piN) \mid{} \ybad \in \Yhat} 
    \\
    &~+ \Pr(\ybad \notin \Yhat) \cdot \En_{\Yhat}\sbr*{J(\pistar) - J(\piN) \mid{} \ybad \notin \Yhat}
    \\
    & > \Pr(\ybad \in \Yhat) \cdot \En_{\Yhat}\sbr*{J(\pistar) - J(\piN) \mid{} \ybad \in \Yhat}
    \\
    & = \Pr(\ybad \in \Yhat) \cdot \Delta
    \\
    & \ge c\cdot \min \cbr*{1, \sqrt{N \cdot \vepsrms}},
  \end{align}
  where in the first inequality we use the fact that $\pistar$ is optimal for $\rstar$, and in the last inequality we plug in the definition of $\Delta$.     
\end{proof}

\subsubsection{Lower Bounds under $L_1$-Coverage}
\label{app:one_lower_bon}

\begin{theorem}[Regret lower bound for $\cC^{\pistar}$]\label{thm:cone-regret-lower}
  For any $\vepsrm \in (0, 1/4]$ and $C \ge \vepsrm^{-1}$,
  there exists a comparator policy $\pistar$ over
  contexts $\cX = \{x\}$ and response space $\cY = \bbZ^+$,
  and universal constants $c_1, c_2, c_3$ such that
  the following statements hold for any $p \in (0, 1/2]$.
  \begin{enumerate}
    \item For any $\veps \in [\vepsrm, \tfrac{1}{4}]$,
    there exists a problem instance $(\piref, \rstar, \rhat)$
    with $\vepsrms(x) \le \vepsrm$ and
    \begin{align}
      &\Cone[\pistar](x) = O\rbr{\log C},
      \\
      &\Cinf[\pistar](x) = O\rbr{C},
    \end{align}
    such that, for any $N < c_{1} \cdot \rbr*{\Cone[\pistar](x) \cdot \veps^2}^{-p}$,
    any selection algorithm $\cA$ (\cref{def:selection}) suffers regret
    \[
      J(\pistar;x) - J(\pi_{\cA};x) > c_2 \cdot \rbr*{\Cone[\pistar](x)
        \cdot \veps^2}^p.
    \]
    \item For any $N \gtrsim 1$,
    there exists a problem instance $(\piref, \rstar, \rhat)$ with $\vepsrms(x) \le \vepsrm$ and
    \begin{align}
      &\Cone[\pistar](x) = O\rbr{\log C}
      \\
      &\Cinf[\pistar](x) = O\rbr{C},
    \end{align}
    such that the \bonalg policy $\piN$ has regret
    \[
      J(\pistar;x) - J(\piN;x) > c_3 \cdot \sqrt{N \cdot \vepsrms(x)}.
    \]
  \end{enumerate}
\end{theorem}

\begin{proof}[\pfref{thm:cone-regret-lower}]
  As in the proof of \cref{thm:cinf-regret-lower}, $x$-dependencies are ommitted in the proof below, 
  inclusive of complexity measures such as $\Epistx$ and $\vepsrm(x)$, 
  since there is only a single prompt.

  \textbf{Part 1 ($N$ small).} We prove the statement for $\Cone[\pistar] \le \vepsrm^{-2}$, otherwise $\sqrt{\Cone[\pistar] \cdot \vepsrms} = \bigom(1)$. 

  For all $\veps\in\sbr*{\vepsrm, \tfrac{1}{4}}$,
  we will define $\piref$ and $\pistar$ to be the distributions from the construction in \cref{lem:c1-tv-lower-bound} with $C$.
  These policies are defined as follows for all $i \in \cY = \cbr*{1,2,\ldots}$
  (where we use $i$ instead of $y$ to index responses).
  \begin{align}
    \piref(i) &= \frac{3}{4^i}
    \\
    \pistar(i) &=
    \begin{cases}
      \frac{2^i}{3} \cdot \piref(i), &\text{~if }i \le I \ldef{} \lceil \log C \rceil,
      \\
      2^I \cdot \piref(i), &\text{~otherwise.}
    \end{cases}
  \end{align}
  From the proof of \cref{lem:c1-tv-lower-bound}, we know that
  \[
    \Cone[\pistar] = O(\log C),\quad\text{and}\quad \Cinf[\pistar] = O(C).
  \]

  Now fix a choice of $\veps \in \sbr*{\vepsrm, \tfrac{1}{4}}$.
  We will now define the reward functions for the construction.
  Let $k_\veps \in \cY$ is an index that will be specified shortly,
  and, for $i\in\cY$, define
  $\Delta_{\veps}(i) = 2^i \cdot \sqrt{\frac{\vepsrms}{4 \cdot k_\veps}}$, and
  \begin{align}
    \rstar(i) = 
    \begin{cases}
      0 &\text{if } i < k_\veps,
      \\ 
      \frac{1}{2} + \frac{\Delta_{\veps}(k_\veps)}{2} &\text{otherwise}.
    \end{cases}
    \qquad
    \rhat(i) =
    \begin{cases}
      \Delta_\veps(i) &\text{if } i < k_\veps,
      \\
      \frac{1}{2} - \frac{\Delta_{\veps}(k_\veps)}{2} &\text{otherwise}.
    \end{cases}
  \end{align}
  For this choice of $\rhat$ and $\rstar$,
  we verify that the estimation error is upper bounded as $\vepsrms$,
  \begin{align}
    \En_{\piref}\sbr*{\rbr*{\rhat(i) - \rstar(i)}^2} 
    & = 3 \rbr*{\sum_{i=1}^{k_\veps} 4^{-i} \cdot \Delta^2_{\veps}(i)
    + \Delta^2_{\veps}(k_\veps) \sum_{i=k_\veps+1}^\infty 4^{-i}}
    \\
    & = 3 \cdot k_{\veps} \cdot \frac{\vepsrms}{4 \cdot k_\veps}
    + \rbr*{\frac{4^{k_\veps} \cdot \vepsrms }{4 \cdot k_\veps}} \cdot 4^{-k_\veps}
    \\
    & =  \frac{\vepsrms}{4 }\rbr*{3 + \frac{1}{k_\veps}} 
    \\
    & \le \vepsrms.  
  \end{align}

  Next, we set 
  \[
    k_\veps = \left\lfloor \log\rbr*{\rbr*{\Cone[\pistar] \cdot \veps^2}^{-p}} \right\rfloor.
  \]
  We can check that $k_\veps \le I = \lceil \log C \rceil$,
  again using the precondition that $C \ge \frac{1}{\vepsrm} \ge \frac{1}{\veps}$,
  \begin{align}
    k_\veps
    \le \left\lfloor \log\rbr*{\frac{1}{\rbr*{\Cone[\pistar] \cdot \vepsrms}^p}} \right\rfloor
    \le \left\lfloor \log \rbr*{\frac{1}{\vepsrm}}  \right\rfloor 
    \le \left\lfloor \log(C) \right\rfloor.
  \end{align}
  
  Then \cref{lem:approx-tv-lower-bound} applied with
  $\veps' = 2^{-k_\veps} = c \cdot \rbr*{\Cone[\pistar]\cdot \veps^2}^p$,
  where $c$ is an absolute constant, states that,
  if $ N < c\cdot \rbr*{\Cone[\pistar]\cdot \veps^2}^{-p}$,
  any selection algorithm (\cref{def:selection}) has 
  \begin{align}\label{eq:epi2}
    \Dtv{\pistar}{\pi_\cA} 
    \ge \Epistar[{\frac{1}{\veps'}}]
    \gtrsim \rbr*{\Cone[\pistar]\cdot\veps^2}^p.
  \end{align}
  We conclude by lower bounding the regret using \cref{eq:epi2}.
  When $N <  c\cdot \rbr*{\Cone[\pistar]\cdot \veps^2}^{-p}$,
  \begin{align}
    J(\pistar) - J(\pi_\cA) 
    & = \sum_{i=1}^{k_\veps} \rstar(i) \cdot \rbr*{\pistar(i) - \pi_\cA(i)}
    + r^\star(k_{\veps}) \cdot \sum_{i=k_{\veps}+1}^\infty \rbr*{\pistar(i) - \pi_\cA(i)}
    \\
    & = r^\star(k_\veps) \cdot \sum_{i=k_\veps+1}^\infty \rbr*{\pistar(i) - \pi_\cA(i)} 
    \\
    & \ge \rbr*{1 + \frac{\Delta_\veps(k_\veps)}{2}}  \cdot \Epistar[\frac{1}{\veps'} ]  
    \\
    & > c_2 \cdot \rbr*{\Cone[\pistar] \cdot \veps^2}^p,
  \end{align}
  where we have applied \cref{eq:epi2} in the last line,
 and in the first inequality we use the definition of $\Epistar$ with
 $M = \frac{1}{\veps'} = 2^{k_\veps}$,
 since $\{k_\veps+1,\ldots\} = \{i : \frac{\pistar(i)}{\piref(i)} \ge 2^{k_\veps}\}$.

  \textbf{Part 2 ($N$ large).}
  Fix $N \in \bbZ^+$.
  We prove the result for $N \lesssim \frac{1}{\vepsrms}$, otherwise the stated bound holds trivially.
  Let $k_N = \lfloor \log_4(N) \rfloor \le I \ldef{} \lceil \log C \rceil$, 
  so that $\piref(k_N) \ge \frac{1}{N}$.

  Next, let
  $\Delta_N(i) \ldef{} 2^i \sqrt{\frac{\vepsrms}{8 \cdot k_N}}$,
  and define the reward functions 
  \begin{align}
    \rstar(i) = 
    \begin{cases}
      \frac{1}{2} + \frac{\Delta_N(i) }{2} &\text{if } i < k_N,
      \\
      \frac{1}{2} - \frac{\sqrt{k_N} \cdot \Delta_N(k_N) }{2} &\text{if } i = k_N,
      \\
      \frac{1}{2} + \frac{\sqrt{k_N} \cdot \Delta_N(k_N) }{2} &\text{otherwise}. 
    \end{cases}
    \qquad
    \rhat(i) =
    \begin{cases}
      \frac{1}{2} - \frac{\Delta_N(i)}{2} &\text{if } i < k_N, 
      \\
      \frac{1}{2} + \frac{\sqrt{k_N} \cdot \Delta_N(k_N) }{2} &\text{if } i = k_N,
      \\
      \frac{1}{2} - \frac{\sqrt{k_N} \cdot \Delta_N(k_N) }{2} &\text{otherwise}. 
    \end{cases}
  \end{align}

  First, we check the reward ranges, and observe that
  $\sqrt{k_N}\cdot\Delta_N(k_N) = \sqrt{\frac{1}{8}\cdot N \cdot \vepsrms} \le 1$.
  For the reward model error, we have
  \begin{align}
    \En_{\piref}\sbr*{\rbr*{\rhat(i) - \rstar(i)}^2} 
    &= 3\rbr*{\sum_{i=1}^{k_N-1} 4^{-i} \cdot \Delta_N^2(i)
    + k_N\cdot\Delta_{k_N}^2(k_N) \sum_{i=k_N}^\infty 4^{-i}}
    \\
    & =  \frac{3(k_N-1)\cdot\vepsrms}{8 k_N}
    + k_N \cdot \Delta_N^2(k_N) \cdot 4^{1-k_N}
    \\
    & =  \frac{\vepsrms}{8}\rbr*{\frac{3(k_N - 1)}{k_N} + 4} 
    \\
    & \le \vepsrms.
  \end{align}

  Second, we show that $J(\pistar) - J(\piref) \ge 0$ as long as $k_N \ge 4$ by computing the policies in closed form.
  Recall that $\pistar(i) = 2^{-i}$ if $i \le I$,
  and $\pistar(i) = 3 \cdot 2^I \cdot 4^{-i}$ otherwise.
  Then by plugging in this expression and the definition of $\rstar$ above,
  we calculate its return as 
  \begin{align}
    J(\pistar) 
    &= \sum_{i=1}^{k_N - 1} 2^{-i} \cdot\rbr*{\frac{1}{2} + \frac{\Delta_N(i)}{2}}
    + \frac{1}{2}\cdot\rbr*{2^{-k_N} + \sum_{i=k_N + 1}^I 2^{-i} + 3 \cdot 2^I \sum_{i=I+1}^\infty 4^{-i}}
    \\
    &+ \frac{\sqrt{k_N} \cdot \Delta_N(k_N)}{2} \cdot \rbr*{\sum_{i=k_N+1}^I 2^{-i} + 3 \cdot 2^I \sum_{i=I+1}^\infty 4^{-i} - 2^{-k_N}}
  \end{align}
  Recall that
  \[
    1 = \sum_{i \in \cY}{} \pistar(y)
    = \sum_{i=1}^{I} 2^{-i} + 3 \cdot 2^I \sum_{i=I+1}^\infty 4^{-i}.
  \]
  Then grouping terms and substituting this identity,
  \begin{align}
    J(\pistar)
    &= \frac{1}{2}\rbr*{1 + \sum_{i=1}^{k_N - 1}{} 2^{-i}\Delta_N(i)
    + \sqrt{k_N} \cdot \Delta_N(k_N) \cdot \rbr*{1 - \sum_{i=1}^{k_N} 2^{-i} - 2^{-k_N}}}
    \\
    &= \frac{1}{2}\rbr*{1 + \sum_{i=1}^{k_N - 1}{} 2^{-i} \cdot 2^i \sqrt{\frac{\vepsrms}{8\cdot k_N}}
    + 2^{k_N} \sqrt{\frac{\vepsrms}{8}} \cdot \rbr*{1 - \sum_{i=1}^{k_N} 2^{-i} - 2^{-k_N}}}
    \\
    &= \frac{1}{2}\rbr*{1 + \rbr*{k_N-1} \sqrt{\frac{\vepsrms}{8 \cdot k_N}}
    + 2^{k_N} \sqrt{\frac{\vepsrms}{8}} \cdot \rbr*{2^{-k_N} - 2^{-k_N}}}
    \\
    &= \frac{1}{2} + \rbr*{k_N-1}\sqrt{\frac{\vepsrms}{32 \cdot k_N}}.
  \end{align}
  Also, for $\piref$ we have
  \begin{align} 
    J(\piref) 
    &= 3 \rbr*{\sum_{i=1}^{k_N-1} 4^{-i} \cdot \rbr*{\frac{1}{2} + \frac{\Delta_N(i)}{2}}
    + \frac{1}{2}\cdot\rbr*{4^{-k_N} + \sum_{i=k_N+1}^\infty 4^{-i}}
    + \frac{\sqrt{k_N} \cdot \Delta_N(k_N)}{2}\cdot \rbr*{\sum_{i=k_N+1}^\infty 4^{-i}- 4^{-k_N}}}
    \\
    &= \frac{1}{2} + 3\sqrt{\frac{\vepsrms}{32k_N}}\cdot \sum_{i=1}^{k_N-1} 2^{-i}
    + \frac{\sqrt{k_N} \cdot \Delta_N(k_N)}{2} \cdot \rbr*{4^{-k_N} - 4^{-k_N}}
    \\
    &= \frac{1}{2} + 3\sqrt{\frac{\vepsrms}{32k_N}}\cdot \rbr*{1 - 2^{1-k_N}}
    \\
    &\le  \frac{1}{2} + 3\sqrt{\frac{\vepsrms}{32k_N}}. 
  \end{align}

  Together, we obtain that
  \begin{align}
    J(\pistar) - J(\piref)
    \ge (k_N - 4)\cdot\sqrt{\frac{\vepsrms}{32\cdot k_N}},
  \end{align}
  which is nonnegative whenever $k_N = \lfloor \log_4 N \rfloor > 4$, or $N \ge 256$. 
  
  Lastly, we lower bound the regret by considering two cases. 
  \begin{itemize}
    \item If $k_N \in \Yhat$, then \bon will choose $k_N$ since $k_N = \argmax_i \rhat(i)$, and 
    \[
      \En\sbr*{J(\pistar) - J(\piN) \mid{} k_N \in \Yhat} 
      \ge \frac{1}{2} - \rbr*{\frac{1}{2}- \frac{\sqrt{k_N}\cdot \Delta_k}{2}} 
      = \frac{\sqrt{k_N}\cdot \Delta_N(k_N)}{2}
      = \sqrt{\frac{1}{32}\cdot N\cdot \vepsrms}.
    \]
    \item If $k_N \notin \Yhat$, our design of $\rstar$ and $\rhat$ ensures that \bon will always choose the response $y\in\Yhat$
      with the smallest reward $\rstar$, while  $\piref$ is equivalent to choosing $y \in \Yhat$ uniformly, thus 
    \[
      \En\sbr*{J(\piref) - J(\piN) \mid{} k_N \notin \Yhat} \ge 0.
    \] 
  \end{itemize}
  Formally, combining these cases gives
  \begin{align}
    J(\pistar) - J(\piN) 
    & = \Pr(k_N \in \Yhat)\cdot \En\sbr*{J(\pistar) - J(\piN) \mid{} k_N \in \Yhat}  + \Pr\rbr*{k_N \notin\Yhat} \cdot \En\sbr*{J(\pistar) - J(\piN) \mid{} k_N \notin \Yhat}
    \\
    & \ge \Pr(k_N \in \Yhat) \cdot \frac{\sqrt{k_N}\cdot \Delta_N(k_N)}{2}   + \Pr\rbr*{k_N \notin\Yhat} \cdot \rbr*{J(\pistar) - J(\piref)}
    \\
    & > \Pr(k_N \in \Yhat) \cdot \frac{\sqrt{k_N}\cdot \Delta_N(k_N)}{2}
    \\
    & \ge \rbr*{1-e^{-3}} \cdot\sqrt{\frac{1}{32}\cdot N\cdot \vepsrms}
  \end{align}
  where in the last inequality we use the fact that, since $\piref(k_N) \ge \frac{3}{N}$,
  \[
    \Pr(k \in \Yhat) = 1 - \Pr(k \notin \Yhat) = 1 - \rbr*{1 - \frac{3}{N}}^N > 1 - e^{-3}. 
  \] 
  
\end{proof}

\subsubsection{Supporting Lemmas}

The following lemma identifies a pair of distributions $\pi$ and
$\piref$ where the upper
bound $\Mpi \le \frac{\Cone}{\veps}$ is tight up to logarithmic
factors, and where it is information-theoretically hard to
  approximate $\pi$ using samples from $\piref$. In particular, the
distributions exhibit $\Mpi = O\rbr*{\veps^{-1}}$ and $\Cone[\pi] = O\rbr*{\log(\Cinf[\pi])}$.

\begin{lemma}\label{lem:c1-tv-lower-bound}
  For any $\veps \in (0, \frac{1}{4}]$ and $C \ge \frac{1}{2 \cdot
    \veps}$, there exists a prompt space $\cX=\{x\}$, response space $\cY$, and two distributions $\pi,\piref: \cX \rightarrow \Delta(\cY)$ with 
  \begin{itemize}
    \item $\Cone(x) = O(\log (C))$; and
    \item $\Cinf(x) = O(C)$;
  \end{itemize}
  such that if $N < \frac{1}{12 \cdot \veps}$, any selection algorithm
  $\cA$ (\cref{def:selection})
    has $\Dtv{\pi(x)}{\pi_{\cA}(x)} > \veps$. 
\end{lemma}

\begin{proof}[\pfref{lem:c1-tv-lower-bound}]
We omit all $x$ dependencies given that there is a single prompt, and all instances of $\log$ are base
2. The construction is as follows. The response space $\cY=\bbN$ is countably infinite and indexed by $i \in \{1,
2, 3, \ldots\}$. Let $I \ldef{} \lceil \log(C) \rceil$ (the
preconditions on $\veps$ and $C$ are to ensure that $I \ge 1$), and
define  %
\begin{align}
  \pi(i) =
  \begin{cases}
    \frac{2^{-i}}{Z_\pi} &\text{if } i \le I, 
    \\
    \frac{2^I}{Z_\pi}\cdot \piref(i) &\text{if } i > I,
  \end{cases}
  \mathand
  \piref(i) = \frac{4^{-i}}{Z_{\piref}}, 
\end{align}
where
\begin{align}
  Z_{\piref} \icml{&}\ldef \sum_{i=1}^\infty 4^{-i} = \frac{1}{3},\mathand
\icml{  \\}
  Z_\pi \icml{&}= \sum_{i=1}^I 2^{-i} + 2^I \cdot \sum_{i=I+1}^\infty \frac{4^{-i}}{Z_{\piref}} = \rbr*{1 - 2^{-I}} + 2^I \cdot 4^{-I} = 1. 
\end{align}
The coverage coefficient has $\Cone = O(\log C)$ since 
\begin{align}
  \Cone 
  & = \sum_{i=1}^I \frac{\pi^2(i)}{\piref(i)} + \frac{4^I }{(Z_\pi)^2}\cdot\sum_{i=I+1}^\infty \piref(i)
  = \frac{Z_{\piref}}{(Z_\pi)^2} \cdot I + \frac{4^I \cdot Z_{\piref}}{(Z_\pi)^2} \cdot 4^{-I}
  = \frac{Z_{\piref}}{(Z_\pi)^2} \rbr*{I + 1},
\end{align}
while $\Cinf = O(C)$, since 
$
  \Cinf = \frac{Z_{\piref}}{Z_\pi} \cdot 2^{\lceil \log C \rceil}. 
$
Next, recall that for any $M \ge 1$, 
\begin{align}
  \Epi = \sum_{i=1}^\infty \piref(i) \cdot \rbr*{\frac{\pi(i)}{\piref(i)} - M}_+ 
  = \sum_{i=1}^\infty \frac{4^{-i}}{Z_{\piref}} \cdot \rbr*{\frac{2^{\rbr{i \wedge I}} \cdot Z_{\piref}}{Z_\pi} - M}_+.
\end{align}
Note that $\frac{\pi(i)}{\piref(i)}$ increases with $i$. Motivated by this, we will set $M = M_k \ldef{} \frac{2^{k} \cdot Z_{\piref}}{Z_\pi}$ for some index $k \le I$ to be specified shortly, which effectively zeroes out all terms $i \le k$ in the sum above: %
\begin{align}
  \Epi[M_k] 
  & = \sum_{i=1}^\infty \frac{4^{-i}}{Z_{\piref}} \cdot \rbr*{\frac{2^{\rbr{i \wedge I}} \cdot Z_{\piref}}{Z_\pi} - M_k}_+.
  \\
  & = \sum_{i=k+1}^\infty \frac{4^{-i}}{Z_{\piref}} \cdot \rbr*{\frac{2^i \cdot Z_{\piref}}{Z_\pi} - M_k} 
  \\
  & = \sum_{i=k+1}^\infty \frac{2^{-i}}{Z_\pi} - M_k \cdot \sum_{i=k}^\infty \frac{4^{-i}}{Z_{\piref}}
  \\
  & = 2^{-k} - M_k \cdot 4^{-k}
  \\
  &= 2^{-k} - \rbr*{\frac{2^{k} \cdot Z_{\piref}}{Z_\pi}}\cdot 4^{-k}
  \\
  &= 2^{-k} \cdot \rbr*{1 - \frac{Z_{\piref}}{Z_\pi}}.
\end{align}
Now let $Z = \frac{Z_{\piref}}{Z_\pi} = \frac{1}{3}$ for short, and set $k = \lfloor \log\rbr*{\frac{1}{2\veps}} \rfloor$. We check that $k \le \log\rbr*{\frac{1}{2 \veps}} \le \log(C) \le I$, since $ C \ge \frac{1}{2 \veps}$ by assumption. For this choice of $k$, we have 
\[
  \Epi[M_k] \ge 2^{-\log\rbr*{\frac{1}{2\veps}} } \cdot (1-Z) = \frac{4}{3} \cdot \veps.
\]
Recall that $\Mpi = \min\cbr*{M \mid{} \Epi[M] \le \veps}$. Since larger $M$ has smaller $\Epi[M]$, we conclude that 
\[
  \Mpi \ge M_k \ge  2^{\log\rbr*{\frac{1}{4\veps}}} \cdot Z = \frac{Z}{4\veps} = \frac{1}{12 \cdot \veps}~. 
\]
The theorem statement then follows by applying
\cref{lem:approx-tv-lower-bound}, which states that when $N <
\frac{1}{12 \cdot \veps} $, any selection strategy $\cA$ must have $\Dtv{\pi}{\pi_{\cA}} > \veps$. 

\end{proof}

\subsection{Proof of Theorem \ref*{thm:bon}}
\label{sec:bon1}

\begin{proof}[Proof of \cref{thm:bon}] 
  Recall from \cref{def:ediv} that for a given prompt $x$ and target policy $\pi$, 
  \begin{align}
    \Epix \ldef{} \sum_{y \in \Yout(x)}{\pi(y\mid{}x) - M \cdot \piref(y\mid{}x)} = \En_{y \sim \piref(\cdot\mid{}x)}\sbr*{\rbr*{\frac{\pi(y\mid{}x)}{\piref(y\mid{}x)} - M}_+}. 
  \end{align}
  
  \cref{prop:mstar} states that for any $M$ we can upper bound 
  $
      \Epistx \le \frac{\Cone[\pistar](x)}{M}
  $, 
  which when combined with \cref{lem:x-regret-upper-general} results in
  \begin{align}
    & J(\pistar;x) - J(\piN;x)
    \\
    & \quad\le \Rmax \cdot \rbr*{\frac{\Cone[\pistar](x)}{M}+ \exp\prn*{ -\frac{N}{M} \cdot \rbr*{1-\Epistx}}} 
    + 2\cdot\sqrt{ \Cone[\pistar](x)\cdot \vepsrms(x) } + \frac{\vepsrm(x)}{2} + \sqrt{N\cdot \vepsrms(x)}
  \end{align}
  as long as $M$ is large enough such that $  \Epistx \le \frac{\Cone[\pistar](x)}{M}\leq1/2$.
  We will set $M = \frac{N}{\log\rbr*{4\Rmax^{2}/\vepsrms(x)}}$;
  we can check that as long as $N \ge 2\cdot\Cone[\pistar](x)\cdot\log\rbr*{4\Rmax^{2}/\vepsrms(x)}$
  then $\Epistx \le \frac{1}{2}$ since
  \begin{align}
    \Epistx \le \frac{\Cone[\pistar](x)}{M} = \frac{\Cone[\pistar](x)\log\rbr*{4\Rmax^{2}/\vepsrms(x)}}{N} \le \frac{1}{2}. 
  \end{align}
  Then for any $N \ge 2 \cdot\Cone[\pistar](x) \cdot \log\rbr*{4\Rmax^{2}/\vepsrms(x)}$, we have 
  \begin{align}
    J(\pistar;x) - J(\piN;x) 
    &\le \Rmax \cdot \rbr*{\frac{\Cone[\pistar](x)}{M}+ \exp\prn*{ -\frac{N}{2M}}} + 2\cdot\sqrt{ \Cone[\pistar](x)\cdot \vepsrms(x) } + \frac{\vepsrm(x)}{2} + \sqrt{N\cdot \vepsrms(x)}
    \\
    &= \Rmax \cdot \frac{\Cone[\pistar](x)\log\rbr*{\frac{4\Rmax^{2}}{\vepsrm^{2}(x)}}}{N}+  2\cdot\sqrt{ \Cone[\pistar](x)\cdot \vepsrms(x) } + \vepsrm(x) + \sqrt{N \cdot \vepsrms(x)}, 
    \label{eq:bon-cone-n-regret}
  \end{align}
  which completes the first statement of the theorem.
  The second part of the theorem statement follows from by setting $N$
  to minimize the RHS of the regret bound above, by balancing the two $N$-dependent terms.  
  Namely, if
  $N \asymp \rbr*{\frac{\Cone[\pistar](x)\cdot\Rmax\cdot\log\rbr*{\frac{4\Rmax^{2}}{\vepsrms(x)}}}{\vepsrm(x)}}^{\frac{2}{3}}$, then
  \begin{align}\label{eq:bon-cone-regret}
    J(\pistar;x) - J(\piN;x) 
    \lesssim \rbr*{\Rmax\cdot \Cone[\pistar](x)\cdot \vepsrms(x)\cdot\log\rbr*{\frac{\Rmax}{\vepsrm(x)}}}^{\frac{1}{3}}. 
  \end{align} 
\end{proof}

\subsection{Proof of Theorem \ref*{thm:lower-bon}}
\label{sec:bon2}

\begin{proof}[Proof of Theorem \ref*{thm:lower-bon}]
  The first part of the theorem statement follows from \cref{thm:cinf-regret-lower} with $C = 2$, which then states that 
  \begin{align}
    J(\pistar;x) - J(\piN) \ge c \cdot \min\cbr*{\sqrt{N \cdot \vepsrms(x)}, ~1}
  \end{align}
  as long as $N \ge 1$, where $c$ is a universal constant. 

  The second part of the theorem statement follows from invoking \cref{thm:cone-regret-lower} with $p = \frac{1}{3}$ and $C=\vepsrm(x)^{-1}$ %
  which states that if $N = \wt O\rbr*{\vepsrm^{-\frac{2}{3} }(x) }$ then 
  \begin{align}
    J(\pistar;x) - J(\piN;x) > c_{1} \cdot \rbr*{ \vepsrms(x) \cdot \log\rbr*{\vepsrm(x)}}^{\frac{1}{3} } 
  \end{align}
  while if $N = \wt\Omega\rbr*{\vepsrm^{-\frac{2}{3} }(x) }$ then 
  \begin{align}
    J(\pistar;x) - J(\piN;x) & > c_{2} \cdot \sqrt{N \cdot \vepsrms(x)} 
  \geq \bigomt\rbr*{\vepsrm^{\frac{2}{3}}(x)} ~.
  \end{align}
\end{proof}

\subsection{Proof of Theorem \ref*{thm:bon_uniform}}
\label{sec:bon3}

\begin{proof}[Proof of \cref{thm:bon_uniform}]
  When $M = \Cinf[\pistar](x)$ we have $\Epistarx = 0$, 
    and using this choice of $M$ in \cref{lem:x-regret-upper-general} gives
  \begin{align}\label{eq:bon-cinf-n-regret}
    J(\pistar;x) - J(\piN;x) 
    \le \Rmax \cdot \exp\prn*{ -\frac{N}{\Cinf[\pistar](x)} } + 2\cdot\sqrt{ \Cone[\pistar](x)\cdot \vepsrms(x) } + \frac{\vepsrm(x)}{2} + \sqrt{N\cdot \vepsrms(x)},
  \end{align}
  which proves the first part of the theorem statement. For the second, 
  we observe for any
  $N \ge \Cinf[\pistar](x)\log(2\Rmax/\vepsrm(x))$, we have
  \begin{align}
    J(\pistar;x) - J(\piN;x) 
    \approxleq \vepsrm(x) + \sqrt{N \cdot \vepsrms(x)}, 
  \end{align} 
  and if $N \asymp \Cinf[\pistar](x)\cdot\log(\Rmax/ \vepsrm(x))$ additionally, then 
  \begin{align}\label{eq:bon-cinf-regret}
    J(\pistar;x) - J(\piN;x) 
    \lesssim \sqrt{ \Cinf[\pistar](x) \cdot \vepsrms(x) \cdot\log(\Rmax/ \vepsrm(x))}~. 
  \end{align} 
\end{proof}

\icml{\subsection{Guarantees for IID Prompts}
\label{sec:bon_iid}

\cref{thm:lower-bon} shows that \bonalg is suboptimal on a per-prompt
basis. In this section, we consider the behavior of the algorithm in
the setting where prompts $x\sim\rho$ are drawn \iid from a 
distribution $\rho$, with the aim of achieving low regret on average
over the draw of the prompt (cf. \cref{sec:offline}). Here, an additional concern is that
since the optimal choice of $N$ in \cref{thm:bon} depends on the
specific prompt $x$, there may not a single fixed choice for $N$ that
leads to low regret on average. Fortunately, the following result
shows that this is not the case, and we can derive guarantees similar
to \cref{thm:bon} (with the coverage coefficient and reward model
error replaced with average-case counterparts) using a single $N$ for
all prompts.

To state the result, we define
$J(\pi)=\En_{x\sim\rho}\brk*{J(\pi\midsem{}x)}$,
  $\vepsrms=\En_{x\sim\rho}\brk*{\vepsrms(x)}$,
  $\Cone[\pistar]=\En_{x\sim\rho}\brk*{\Cone[\pistar](x)}$, and $\Cinf[\pistar]\ldef{}\sup_{x\in\cX}\Cinf[\pistar](x)$.
  \begin{corollary}\label{cor:bon}
    For any comparator policy $\pistar$, whenever $N \geq
    \frac{\Cone[\pistar]}{2}$, the \bonalg policy $\piN$ satisfies
    \icml{\begin{small}}
      \begin{align}
        \label{eq:bon_average1}
        J(\pistar) - J(\piN) 
        \le 
        \Rmax \cdot \min\cbr*{~\frac{\Cone[\pistar]\log\rbr*{\frac{4\Rmax^{2}}{\vepsrms}}}{N}~, ~~\exp\prn*{ -\frac{N}{\Cinf[\pistar]}}\,} 
        + 2\cdot\sqrt{ \Cone[\pistar] \cdot \vepsrms} + 2\cdot\sqrt{N \cdot \vepsrms} ~.
      \end{align}
  \icml{\end{small}}%
  \arxiv{  and as long as
  $N \asymp \rbr*{\frac{\Rmax\cdot\Cone[\pistar]\log\rbr*{\Rmax /
        \vepsrm}}{\vepsrm}}^{\frac{2}{3}}$, we have that   
  \begin{align}
    \label{eq:bon_average2}
    J(\pistar) - J(\piN) 
    \lesssim
        \prn*{\Rmax\cdot \Cone[\pistar]\cdot \vepsrms\cdot\log\rbr*{\frac{\Rmax}{\vepsrm}}}^{1/3}.
  \end{align}}%
\end{corollary}
\icml{Note that whenever  $N \asymp \prn[\Big]{\frac{\Rmax\cdot\Cone[\pistar]\log\prn[\big]{\Rmax /
        \vepsrm}}{\vepsrm}}^{\frac{2}{3}}$, this gives regret at most $\prn[\big]{\Rmax\cdot \Cone[\pistar]\cdot
    \vepsrms\cdot\log\prn[\big]{\frac{\Rmax}{\vepsrm}}}^{1/3}$, and
  whenever $N\asymp\Cinf[\pistar]\log\prn[\big]{\Rmax /
        \vepsrm}$, this gives regret at most
  $\sqrt{ \Cinf[\pistar] \cdot \vepsrms
    \cdot\log\rbr*{\frac{\Rmax}{\vepsrm}}}$. The latter bound is
  perhaps somewhat pessimistic, since $\Cinf[\pistar]$ is defined as
  the maximum over the $L_\infty$-Coverage coefficient over all prompts.
}

\begin{proof}[\pfref{cor:bon}]
  \cref{cor:bon} follows by combining the statements in
  \cref{cor:bon_} and \cref{cor:bon_uniform}, stated and proven below.
\end{proof}

\begin{corollary}\label{cor:bon_}
  For any $N \in \bbZ$ and comparator policy $\pistar$ such that $N \ge 2\cdot \Cone[\pistar]$, the \bon policy $\piN$ satisfies
  \begin{align}
    J(\pistar) - J(\piN) 
    &\le \Rmax \cdot \frac{ \Cone[\pistar] \log\rbr*{ \frac{4\Rmax^2}{\vepsrms}}}{N} + 2\cdot\sqrt{ \Cone[\pistar]\cdot \vepsrms } + \vepsrm + \sqrt{N\cdot \vepsrms}~,
  \end{align}
  and if $N \asymp \rbr*{\frac{ \Cone[\pistar] \Rmax \log\rbr*{\frac{4\Rmax^2}{\vepsrms} }}{\vepsrms}}^{\frac{2}{3} }$, then 
  \begin{align}
    J(\pistar) - J(\piN) \lesssim \rbr*{\Rmax \cdot \Cone[\pistar] \cdot \vepsrms\cdot \log\rbr*{\frac{\Rmax}{\vepsrms}}}^{\frac{1}{3} }~.
  \end{align} 
\end{corollary}

\begin{proof}[Proof of \cref{cor:bon_}]
  The proof proceeds in a similar manner as did the proof of \cref{thm:bon}. Given $N$, we will set $M = \frac{N}{\log\rbr*{ 4\Rmax^2 / \vepsrms }}$ in \cref{lem:regret-upper-general}. As long as $N \ge 2 \cdot \Cone[\pistar]$, %
  \begin{align}
    \Epistar \le \frac{ \Cone[\pistar]}{M} = \frac{ \Cone[\pistar] \log\rbr*{ 4\Rmax^2 /\vepsrms }}{N} \le \frac{1}{2} . 
  \end{align}
  Then applying \cref{lem:regret-upper-general} (stated and proven in the sequel) to this choice of $M$ gives 
  \begin{align}
    J(\pistar) - J(\piN) 
    &\le \Rmax \cdot \frac{ \Cone[\pistar] \log\rbr*{ \frac{4\Rmax^2}{\vepsrms}}}{N} + 2\cdot\sqrt{ \Cone[\pistar]\cdot \vepsrms } + \vepsrm + \sqrt{N\cdot \vepsrms}~,
  \end{align}
  and setting $N \asymp \rbr*{\frac{ \Cone[\pistar] \Rmax \log\rbr*{\frac{\Rmax^2}{\vepsrms} }}{\vepsrms}}^{\frac{2}{3} }$ balances the two $N$-dependent terms on the RHS such that
  \begin{align}
    J(\pistar) - J(\piN) \lesssim \rbr*{\Rmax^2 \cdot \Cone[\pistar] \cdot \vepsrms\cdot \log\rbr*{\frac{\Rmax}{\vepsrms}}}^{\frac{1}{3} }~.
  \end{align}
\end{proof}

\begin{corollary}\label{cor:bon_uniform}
  For any $N \in \bbZ$, 
  \begin{align}
    J(\pistar) - J(\piN) 
    &\le \Rmax \cdot \exp\prn*{-\frac{N}{ \Cinf[\pistar]}} + 2\cdot\sqrt{ \Cone[\pistar]\cdot \vepsrms } + \frac{\vepsrm}{2} + \sqrt{N\cdot \vepsrms}~
  \end{align}
  and if $N \asymp \Cinf[\pistar]\cdot\log\rbr*{\frac{\Rmax}{\vepsrm}}$ then 
  \begin{align}
    J(\pistar) - J(\piN) 
    \lesssim \sqrt{ \Cinf[\pistar] \cdot \vepsrms \cdot \log\rbr*{\frac{\Rmax}{\vepsrm}}}
  \end{align}
\end{corollary}

\begin{proof}[Proof of \cref{cor:bon_uniform}]
  For a fixed $N$, we set $M = \Cinf[\pistar]$ in \cref{lem:tv-upper-bound}, 
  and the first line of its proof yields
  \begin{align}
    J(\pistar) - J(\piN) 
    &\le \Rmax \cdot \exp\prn*{-\frac{N}{ \Cinf[\pistar]}} + 2\cdot\sqrt{ \Cone[\pistar]\cdot \vepsrms } + \frac{\vepsrm}{2} + \sqrt{N\cdot \vepsrms},
  \end{align}
  from which the second theorem statement immediately follows after setting $N \asymp \Cinf[\pistar]\cdot\log\rbr*{\frac{\Rmax}{\vepsrm}}$, and using the fact that $ \Cone[\pistar] \le \Cinf[\pistar]$. 
\end{proof}

\subsubsection{Supporting Lemmas}

The following lemma is an \iid-promp counterpart to \cref{lem:x-regret-upper-general}.

\begin{lemma}\label{lem:regret-upper-general}
  For any comparator policy $\pistar$, query size $N\in\bbZ$, and truncation level $M\in\bbR_{+}$ such that $N > M$ and $\Epistar \le \frac{1}{2}$, the \bon policy $\piN$ satisfies
  \begin{align}
    J(\pistar) - J(\piN) 
    &\le \Rmax \cdot \rbr*{\Epistar  + \exp\prn*{-\frac{N}{2M}}} + 2\cdot\sqrt{ \Cone[\pistar]\cdot \vepsrms } + \frac{\vepsrm}{2} + \sqrt{N\cdot \vepsrms}~.
  \end{align}
\end{lemma}
\begin{proof}[Proof of \cref{lem:regret-upper-general}]
  As in the proof of \cref{lem:x-regret-upper-general}, we begin with the following decomposition, now with expectation over prompts:
  \begin{align}
    J(\pistar) - J(\piN) 
    & \le \Rmax \cdot \En_{x\sim\rho}\sbr*{\Dtv{\pistar(x)}{\pisrej(x)}} + J(\pisrej) - J(\piN) .
  \end{align}
  When $M < N$ and $\Epistar \le \frac{1}{2}$, \cref{lem:tv-upper-bound} states that %
  \begin{align}
    \En_{x \sim \rho}\sbr*{ \Dtv{\pi(x)}{\pirej(x)} } \le \Epistar + \exp\rbr*{ -\frac{N}{2M} }, 
  \end{align}
  and plugging into the previous inequlaity we obtain
  \begin{align}
    J(\pistar) - J(\piN) \le \Rmax \cdot \rbr*{ \Epistar + \exp\rbr*{ -\frac{N}{2M} } } + J(\pisrej) - J(\piN).
  \end{align}
  From the proof of \cref{lem:x-regret-upper-general} we know that
  \begin{align}
    J(\pisrej) - J(\piN) 
    & = \En_{x\sim \rho}\sbr*{ J(\pisrej;x) - J(\piN;x)}
    \\
    & \le \En_{x \sim \rho}\sbr*{  2\cdot\sqrt{ \Cone[\pistar](x)\cdot \vepsrms(x) } + \frac{\vepsrm(x)}{2} + \sqrt{N\cdot \vepsrms(x)} }
    \\
    & \le 2\cdot\En_{x \sim \rho}\sbr*{  \sqrt{ \Cone[\pistar](x)\cdot \vepsrms(x) }} + \frac{\vepsrm}{2} + \sqrt{ N\cdot\vepsrms }~,
  \end{align}
  since $z \mapsto \sqrt{ z } $ is a concave function. For the remaining expectation, applying the Cauchy-Schwarz inequality gives
  \begin{align}
    \En_{x \sim \rho}\sbr*{  \sqrt{ \Cone[\pistar](x)\cdot \vepsrms(x) }}
    & = \sum_{x} \sqrt{ \Cone[\pistar](x) \cdot \rho(x) } \cdot \sqrt{ \vepsrms(x) \cdot \rho(x) }  
    \le \sqrt{ \Cone[\pistar] \cdot \vepsrms }~.
  \end{align}
  Then combining the above inequalities, we obtain 
  \begin{align}
    J(\pistar) - J(\piN) \le \Rmax \cdot \rbr*{ \Epistar + \exp\rbr*{ -\frac{N}{2M} } } 
    + 2\cdot \sqrt{ \Cone[\pistar] \cdot \vepsrms } +  \frac{\vepsrm}{2} + \sqrt{ N\cdot\vepsrms }~.
  \end{align}
\end{proof}
}

\subsection{Additional Results}
\label{sec:bon_additional}

Here, as an additional result, we prove a regret guarantee for \bonalg in terms of 
the information-theoretic quantities $\Mxpi$ and $\Mpi$ (\cref{def:mstar}).
Our main theorems, \cref{thm:bon} and \cref{thm:bon_uniform},
can be viewed as more interpretable relaxations 
that isolate the dependencies on coverage coefficients and $N$.

\begin{theorem}[\bon regret with $\Mpistar$]\label{thm:bon_general}
  Given prompt $x$, for any $\veps \in [0, \frac{1}{2}]$ and $N \in \bbZ$ and comparator policy $\pistar$ , the \bon policy $\piN$ satisfies
  \begin{align}
    J(\pistar;x) - J(\piN;x) 
    \le \Rmax\cdot \rbr*{ \veps + \exp\rbr*{ -\frac{N}{2\Mxpistar[\veps]} }} + 2\cdot\sqrt{ \Cone[\pistar](x)\cdot \vepsrms(x) } + \frac{\vepsrm(x)}{2} + \sqrt{N\cdot \vepsrms(x)}. 
  \end{align}
  In particular, if $N = 2\Mxpistar[\veps]\cdot\log\rbr*{\frac{\vepsrm(x)}{2\Rmax}}$, 
  \begin{align}
    J(\pistar;x) - J(\piN;x) 
    &\lesssim 
    \Rmax\cdot\veps + \sqrt{ \Cone[\pistar](x)\cdot \vepsrms(x) } + \sqrt{\Mxpistar[\veps] \cdot \vepsrms(x) \cdot \log\rbr*{\frac{\vepsrm(x)}{2\Rmax}}} .
  \end{align}
\end{theorem}

\begin{proof}[Proof of \cref{thm:bon_general}]
  We will apply \cref{lem:x-regret-upper-general} with $M = \Mxpistar[\veps]$ for $\veps\in[0, \frac{1}{2} ]$. Whenever $\veps \le \frac{1}{2} $, we have $\Epistarx \le \frac{1}{2}$ by definition, since  
  \[
    \Epistarx[{\Mxpistar[\veps]}] \le \veps \le \frac{1}{2}.
  \]
  As a result, $1-\Epistarx[{\Mxpistar[\veps]}] \ge \frac{1}{2} $, and \cref{lem:x-regret-upper-general} states that, for any $N$,  
  \begin{align}
    J(\pistar;x) - J(\piN;x) 
    \le \Rmax\cdot \rbr*{ \veps + \exp\rbr*{ -\frac{N}{2\Mxpistar[\veps]}}} + 2\cdot\sqrt{ \Cone[\pistar](x)\cdot \vepsrms(x) } + \frac{\vepsrm(x)}{2} + \sqrt{N\cdot \vepsrms(x)}.
  \end{align}
  Setting $N = 2\Mxpistar[\veps] \log\rbr*{\frac{\vepsrm(x)}{2\Rmax}}$, we obtain 
  \begin{align}
    J(\pistar;x) - J(\piN;x) 
    &\le \Rmax\cdot\veps + 2\cdot\sqrt{ \Cone[\pistar](x)\cdot \vepsrms(x) } + \vepsrm(x) + \sqrt{2\Mxpistar[\veps] \cdot \vepsrms(x) \cdot \log\rbr*{\frac{\vepsrm(x)}{2\Rmax}}}
    \\
    &\lesssim 
    \Rmax\cdot\veps + \sqrt{ \Cone[\pistar](x)\cdot \vepsrms(x) } + \sqrt{\Mxpistar[\veps] \cdot \vepsrms(x) \cdot \log\rbr*{\frac{\vepsrm(x)}{2\Rmax}}}.
  \end{align}
\end{proof}

\newpage

\section{Proofs from Section \ref*{sec:algorithm}}
\label{app:algorithm}
\icml{

The proofs below are conditioned on a single prompt $x$, 
and as a result we omit $x$ dependencies to simplify presentation throughout.
We refer to $x$-dependent quantities via their $x$-independent analogs, 
e.g., $J(\pi;x)\equiv J(\pi)$, $\lambda(x)\equiv\lambda$, etc. 

\subsection{Proof of Theorem \ref*{thm:main}}

\begin{proof}[Proof of \cref{thm:main}]
  In the proof below,
  for a reward model $r$ we define the expected return under it to be
  $J_r(\pi;x) \ldef{} \En_{y\sim \pi(x)}\sbr*{r(x,y)}$, 
  and using this definition we can also write $J(\pi;x)=\Jr[\rstar](\pi;x)$. 
  With prompt dependencies ommitted, we equivalently write $J_r(\pi;x)\equiv J_r(\pi)$ .

  Given a reward function $\rhat$, define for a normalization constant $\lambda$ the following functions:
  \begin{equation}\label{eq:pi-lambda}
    \pi_{\lambda}(y) = \frac{ \piref(y) \sig{\rhat(y) - \lambda}}{\Phi(\lambda)},
  \end{equation}
  \begin{equation}\label{eq:phi-lam}
    \Phi(\lambda) = \En_{ \piref }\sbr*{ \sig{\rhat(y) - \lambda} },
  \end{equation}
and
  \begin{equation}\label{eq:hat-phi-lam}
    \wh\Phi(\lambda) = \frac{1}{N} \sum_{i=1}^{N}\sig{ \rhat(y_{i}) - \lambda }.
  \end{equation}

  We analyze the regret using the following decomposition, 
  where $\lambdahat \equiv \lambdahat(x)$ is the output of 
  \cref{eq:normalization_approx} in \cref{alg:main}.
  \begin{align}
    J(\pistar) - J(\pihat)
    = \underbrace{J(\pistar) - \Jr(\pistar)}_{\mathrm{(T1)}} 
    + \underbrace{\Jr(\pistar) - \Jr({\pil[\lambdahat]})}_{\mathrm{(T2)}}  
    + \underbrace{\Jr({\pil[\lambdahat]}) - J({\pil[\lambdahat]}) }_{\mathrm{(T3)}} 
    + \underbrace{J({\pil[\lambdahat]}) - J(\pihat)}_{\mathrm{(T4)}}.
  \end{align}

  For (T2), recall that $\lambdahat$ is computed using
  using \cref{alg:norm} in \cref{eq:normalization_approx}, 
  which \cref{lem:norm} shows obtains $\lambdahat$ such that $\wh\Phi(\lambda) = 1$.
  Further, the procedure in \cref{alg:norm} is equivalent to solving the optimization problem in \cref{lem:chis-alpha}
  with $\piref = \unif\rbr*{\Yhat}$
  and $\alpha=1$, 
  and, as a result, we have $\lambdahat\in\sbr*{-\beta, \Rmax-\beta}$.
  Now for a fixed $N$, the bound in \cref{lem:lambda-concentration} states that 
  with probability at least $1-\delta$, for any $\delta\in(0,1)$ we have 
  \begin{align}
    \frac{7}{9} + \frac{8}{9}\cdot \veps_N \le \Phi(\lambdahat) \le \frac{9}{7} + \frac{8}{7} \cdot \veps_N,
  \end{align}
  where 
  $\veps_N \ldef{} 12\rbr*{\frac{\Rmax + \beta}{\beta}}\frac{\log\rbr*{\frac{60\Rmax}{\beta\delta} }}{N} $
  is the expression in the RHS of the bound.
  Then to have  $\Phi(\lambdahat) \in \sbr*{\frac{1}{2} , \frac{3}{2} }$
  with probability $\ge 1-\delta$, it is sufficient to choose $N$ large enough such that
  $\veps_N \le \frac{1}{4} $, which is satisfied when 
  \begin{align}
    N \ge 48\rbr*{\frac{\Rmax + \beta}{\beta}}\log\rbr*{\frac{60\Rmax}{\beta\delta} }.
  \end{align}

  Next, let $\cE = \cbr*{\Phi(\lambdahat) \in \sbr*{\frac{1}{2} , \frac{3}{2} }}$
  denote the aforementioned high-probability event. Using \cref{lem:sensitivity}, we may bound \loose
  \begin{align}
    \Jr(\pistar) - \Jr({\pil[\lambdahat]}) 
    & \le \Pr\rbr*{\cE}\cdot\En\sbr*{\Jr(\pistar)-\Jr({\pil[\lambdahat]})\mid{}\cE}  +
    \Pr\rbr*{\neg\cE}\cdot\En\sbr*{\Jr(\pistar)-\Jr({\pil[\lambdahat]})\mid{}\neg\cE}  
    \\
    & \le \En\sbr*{\Jr(\pistar)-\Jr({\pil[\lambdahat]})\mid{}\cE}  +
    \delta\cdot\Rmax
    \\
    & \le \frac{3 \beta}{4} \cdot \Cone[\pistar] - \frac{ \beta}{4}\cdot \Cone[{\pil[\lambdahat]}] + \delta\cdot \Rmax, 
  \end{align}
  By setting 
  $\delta = \frac{\vepsrm}{\Rmax}$,
  we obtain that
  \begin{align}
  \mathrm{(T2)} = \Jr(\pistar) - \Jr({\pil[\lambdahat]})
    & \le \frac{3 \beta}{4} \cdot \Cone[\pistar] - \frac{ \beta}{4}\cdot \Cone[{\pil[\lambdahat]}] + \vepsrm
  \end{align}
  if $N = \Omega\rbr*{\rbr*{1+\frac{\Rmax}{\beta}}\log\rbr*{\frac{\Rmax}{\beta\cdot\vepsrm}}}$.

  For (T4), from \cref{lem:x-tv-upper-bound} with $M = \frac{\Rmax}{\beta}$, we have 
  \begin{align}
    \Dtv{{\pil[\lambdahat]}}{\pihat} \le \exp\rbr*{-\frac{N\cdot \beta}{2 \Rmax}}
  \end{align} 
  thus as long as $N \gtrsim \frac{\Rmax}{\beta}\log\rbr*{ \frac{\Rmax}{ \vepsrm} }$ we have 
  \begin{equation}
    \Dtv{{\pil[\lambdahat]}}{\pihat} \le \frac{ \vepsrm }{\Rmax}. 
  \end{equation}
  As a result, 
  \begin{align}
    \mathrm{(T4)} = J({\pil[\lambdahat]}) - J(\pihat) &\le \Rmax \cdot \Dtv{{\pil[\lambdahat]}}{\pihat} \le \vepsrm.
  \end{align}

  For (T1), we know from the standard Cauchy-Schwarz bound that
  \begin{equation}
  \mathrm{(T1)} = J(\pistar) - \Jr(\pistar) 
    = \En_{y\sim\pistar}\sbr*{ \rstar(y) - \rhat(y) }
    \le \sqrt{ \Cone[\pistar]\cdot \vepsrms}
  \end{equation}

  Lastly, for (T3), Cauchy-Schwarz and the AM-GM inequality imply that
  \begin{align}
      \mathrm{(T3)} = \Jr({\pil[\lambdahat]}) - J({\pil[\lambdahat]})
    = \En_{{y\sim\pil[\lambdahat]}}\sbr*{ \rhat(y) - \rstar(y) }
    \le \sqrt{ \Cone[{\pil[\lambdahat]}]\cdot \vepsrms }
    \le \frac{ \beta}{4} \cdot \Cone[{\pil[\lambdahat]}] + \frac{\vepsrms}{\beta}
  \end{align}

  Putting things together, as long as $N \gtrsim \wt\Omega\rbr*{\rbr*{1 + \frac{\Rmax}{\beta}}\log\rbr*{\frac{\Rmax}{ \beta\cdot\vepsrm}  }}$ we have
  \begin{align}
    J(\pistar) - J(\pihat)
    \le \sqrt{ \Cone[\pistar] \cdot \vepsrms }
    + \frac{3\beta}{4} \cdot \Cone[\pistar]
    + \frac{\vepsrms}{\beta}
    + 2\vepsrm.
  \end{align}
  We set $\beta =\sqrt{ \frac{ \vepsrms}{ \Cone[\pistar]} }$ to balance the terms for the final result. 
\end{proof}

\subsubsection{Supporting Lemmas}

Throughout this section, we consider a fix prompt $x\in\cX$ and omit dependence on $x$ as above.

\begin{lemma}\label{lem:sensitivity}
  Suppose we have $\lambda$ such that $\Phi(\lambda) = \alpha$ for some
  $\alpha > 0$ (cf. \cref{eq:phi-lam}). Then for ${\pil[\lambda]}$ defined in \cref{eq:pi-lambda}, it holds for any policy $\pi$ that 
  \begin{equation}
    \Jr(\pi) -  \Jr({\pil[\lambda]}) \le \frac{\alpha \beta}{2} \cdot \Cone - \frac{\alpha \beta}{2}\cdot \Cone[{{\pil[\lambda]}}].
  \end{equation}  
  In particular, if $\alpha \in \sbr*{\frac{1}{2} , \frac{3}{2} }$, then  
  \begin{equation}
    \Jr(\pi) - \Jr({\pil[\lambda]}) \le \frac{3 \beta}{4} \cdot \Cone - \frac{ \beta}{4}\cdot \Cone[{{\pil[\lambda]}}]
  \end{equation}
\end{lemma}

\begin{proof}[\pfref{lem:sensitivity}]
  Let $\alpha = \Phi(\lambda)$ and define ${\pil[\lambda]}' = \sig{\rhat(y)-\lambda}
  = \alpha\cdot{\pil[\lambda]}$, so that $\pil[\lambda]'\in \Dela(\cY)$ (cf. \cref{eq:delta_alpha}).   
  Further, for the comparator policy $\pi$, define $\pi' = \alpha\cdot \pi \in \Dela(\cY)$. We know from \cref{lem:chis-alpha} that 
  \begin{equation}
    \Jr(\pi') - \Jr({\pil[\lambda]}') \le \frac{\beta}{2} \cdot \Cone[\pi'] - \frac{\beta}{2} \Cone
  \end{equation} 
  It can be observed that $ \Cone[\pi'] = \alpha^{2} \cdot \Cone $ and $ \Cone[{\pil[\lambda]}'] = \alpha^{2} \cdot \Cone[{\pil[\lambda]}]$, as well as $\alpha\rbr*{\Jr(\pi) - \Jr({\pil[\lambda]})} = \Jr(\pi') - \Jr({\pil[\lambda]}')$, therefore
  \begin{equation}
   \Jr(\pi) - J({\pil[\lambda]}) \le \frac{\alpha \beta}{2}\cdot \Cone - \frac{\alpha \beta}{2} \Cone[{\pil[\lambda]}] .
  \end{equation}  
  For the second statement, for any $\alpha \in \sbr*{\frac{1}{2} , \frac{3}{2} }$ we have
  \begin{equation}
   \Jr(\pi) - J({\pil[\lambda]}) \le \frac{\frac{3}{2} \cdot \beta}{2}\cdot \Cone - \frac{\frac{1}{2} \cdot \beta}{2} \Cone[{\pil[\lambda]}] 
   = \frac{3 \beta}{4} \cdot \Cone[\pistar] - \frac{ \beta}{4}\cdot \Cone
  \end{equation}
\end{proof}

\begin{lemma}\label{lem:chis-alpha}
  Let a reward function $r$ and parameter $\alpha>0$ be given, and define 
  \begin{align}
    \label{eq:delta_alpha}
    \Delta_{\alpha}(\cY) = \cbr*{\pi\in\bbR_{+}^{\cY} \mid{} \textstyle\sum_{y\in\cY} \pi(y) = \alpha}.
  \end{align} 
  Then there exists a choice of $\lambda\in\brk{J(\piref)-\alpha\beta,\max_{y\in\cY}r(y)-\alpha\beta}$
  such that $\sum_{y\in\cY}\pi(y)=\alpha$. Furthermore, given any $\lambda$ such that $\sum_{y\in\cY}\pi(y)=\alpha$, 
  \begin{align}
    \pi(y) = \piref(y)\cdot\sig{r(y)-\lambda}
  \end{align}
  is the optimal solution to
  \begin{align}
    \label{eq:chis_alpha}
    \pi = \argmax_{\pi\in\Dela(\cY)}\sbr*{J(\pi)-\frac{\beta}{2}\Cone}.
  \end{align}

\end{lemma}

\begin{proof}[\pfref{lem:chis-alpha}]
  First we rewrite the objective in \cref{eq:chis_alpha} in primal form, 
  \begin{align}
    \pi = &\argmin_{\pi\in\bbR^\cY} \cbr*{-J(\pi) + \frac{\beta}{2} \Cone}.
    \\
    &\text{s.t.}~ 
    \sum_{y\in\cY}\pi(y) = \alpha
    \\
    &\quad-\pi(y) \le 0,~\forall y\in\cY
  \end{align}
  and we can verify that Slater's condition holds, because the objective is convex, since $J(\pi)$ is affine and $\Cone$ is (strongly) convex, and there exists at least one strictly feasible point, an example being the function $\pi'$ that sets $\pi'(y) = \frac{\alpha}{|\cY|}$ for all $y\in\cY$. 

  Under strong duality, the KKT conditions are both necessary and sufficient for optimality; 
  further, the objective has a unique minimum due to strong convexity,  
  and therefore, to prove the theorem statement, it is sufficient to show that the proposed $\pi$ satisfies the KKT conditions. 

  For primal variable $\pi$ and dual variables $(\nu, \lambda)$, the Lagrangian is given by 
  \begin{align}
    \min_{\pi\in\bbR^\cY}{}\max_{\lambda\in\bbR, \nu\in\bbR_{+}^{\cY}}{}\cL(\pi, \nu, \lambda) = -\sum_{y\in\cY}\pi(y)r(y) + \frac{\beta}{2} \Cone + \lambda \rbr*{\sum_{y\in\cY} \pi(y) - \alpha} - \sum_{y\in\cY} \nu(y)\pi(y)
  \end{align}
  and under strong duality, we know that the optimal primal and dual variables $(\pi,\nu,\lambda)$ satisfy 
  \begin{itemize}
  \item $\pi\geq{}0$
  \item $\sum_{y}\pi(y)=\alpha$
  \item $\nu\geq{}0$
  \item $\pi(y)\cdot{}\nu(y)=0$ for all $y$ (complementary slackness).
  \item $\grad{}_\pi\cL(\pi, (\nu,\lambda))=0$ (first-order condition).
  \end{itemize}
  From the first-order condition $\grad{}_\pi\cL(\pi,
  \nu,\lambda)=0$, we know that $\pi$ satisfies for all $y\in\cY$:
  \begin{align}
    \label{eq:first_order}
    r(y) - \beta\frac{\pi(y)}{\piref(y)} - \lambda + \nu(y) = 0,
  \end{align}
  or after rearranging,
  \begin{align}
    \label{eq:lag1}
    \pi(y) = \piref(y)\cdot{}\beta^{-1}(r(y) - \lambda + \nu(y)).
  \end{align}
  Now consider a fixed $y\in\cY$. We consider three cases:
  \begin{itemize}
  \item If $r(y)-\lambda<0$, then we must have
    $\nu(y)\geq{}-(r(y)-\lambda)>0$ to satisfy $\pi(y)\geq{}0$ by
    \cref{eq:lag1}. By complementary slackness, this implies that
    $\pi(y)=0$.
    \item If $r(y)-\lambda=0$, then \cref{eq:lag1} gives
      $\pi(y)=\piref(y)\beta^{-1}\nu(y)$, which implies $\pi(y)=0$ by
      complementary slackness.
    \item if $r(y)-\lambda>0$, then $r(y)-\lambda+\nu(y)>0$ (since
      $\nu(y)\geq{}0$), which in turn gives $\pi(y)>0$ by
      \cref{eq:lag1}. This implies that $\nu(y)=0$ by complementary slackness.
    \end{itemize}
    Combining these cases, we conclude that
    \begin{align}
      \pi(y) = \left\{
      \begin{array}{ll}
        0,&\quad r(y) - \lambda \leq{} 0,\\
        \piref(y)\beta^{-1}(r(y)-\lambda)&\quad r(y) - \lambda > 0.
      \end{array}
      \right.
    \end{align}
    This is equivalent to $    \pi(y) =
    \piref(y)\cdot\relu(\beta^{-1}(r(y)-\lambda))$, which is also optimal under strong duality. Finally, the
    condition $\sum_{y}\pi(y)=\alpha$ implies that $\lambda$ must be chosen
    such $\pi\in\Dela(\cY)$ normalizes to $\alpha$.
\end{proof}

\begin{lemma}\label{lem:lambda-concentration}
  Recall $ \Phi(\lambda) = \En_{ y\sim\piref }\sbr*{ \sig{r(y) - \lambda} } $, and given $N$ samples from $y_1,\ldots,y_N\sim\piref $, define 
  \[
    \wh\Phi(\lambda) = \frac{1}{N} \sum_{i=1}^{N}\sig{ r(y_{i}) - \lambda }.
  \] 
  Fix any $\gamma \in \bbR_{+}$. With probability at least $1-\delta$, for all $\lambda \in \sbr*{-\gamma, \Rmax-\gamma}$, we have  
  \begin{align}
    \max \cbr*{\frac{7}{8} \Phi(\lambda) - \wh\Phi(\lambda), \wh\Phi(\lambda) - \frac{9}{8} \Phi(\lambda)}  
    &\le \frac{1}{8} 
    +  12\rbr*{\frac{\Rmax + \gamma}{\beta}}\frac{\log\rbr*{\frac{60\Rmax}{\beta\delta} }}{N}.
  \end{align}
\end{lemma}

\begin{proof}[\pfref{lem:lambda-concentration}]
  We start with Bernstein's inequality, which states that for a bounded random variable $X\in[a,b]$, with probability at least $1-\delta'$,  
  the empirical mean $\wh\En[X] = \frac{1}{N} \sum_{i=1}^N X_{i}$ satisfies
  \begin{align}
    \abr*{\wh\En[X] - \En[X]} \le 
      2\sqrt{\frac{\bbV[X]\log(\frac{2}{\delta})}{N}} + \frac{4\rbr*{b-a}\log(\frac{2}{\delta} )}{N},
  \end{align}
where $\mathbb{V}[X]$ is the variance of $X$.
  Now fix $\lambda$, and let the random variable be $X=\relu\rbr*{r(y) - \lambda}$.
  Since $\lambda\in\sbr*{-\gamma, \Rmax-\gamma}$, 
  the random variable $X \in \sbr*{0, \frac{\Rmax + \gamma}{\beta}}$, 
  and henceforth we refer to the range as $R = \frac{\Rmax + \gamma}{\beta}$. 
  We can bound the variance of this random variable as 
  \begin{align}
    \bbV[X] \le \En[X^2] \le R \cdot \En[X] = R \cdot \Phi(\lambda). 
  \end{align}
  Then with probability at least $1-\delta'$, we have
  \begin{align}
    \abr*{\Phi(\lambda) - \wh\Phi(\lambda)} 
    & \le 2\sqrt{\frac{R\Phi(\lambda)\log(\frac{2}{\delta'})}{N}} + \frac{4R\log(\frac{2}{\delta'} )}{N} 
    \\
    & \le \frac{\Phi(\lambda)}{8} + \frac{8R\log\rbr*{\frac{2}{\delta'} }}{N} + \frac{4R\log\rbr*{\frac{2}{\delta'} }}{N} 
    \\
    & = \frac{\Phi(\lambda)}{8} + \frac{12R\log\rbr*{\frac{2}{\delta'} }}{N},   
  \end{align}
  where we use the AM-GM inequality in the second inequality. This then implies that 
  \begin{align}
    \Phi(\lambda) - \wh\Phi(\lambda) 
    & \le \frac{\Phi(\lambda)}{8} + \frac{12R\log\rbr*{\frac{2}{\delta'} }}{N},\mathand
    \\
    \wh\Phi(\lambda) - \Phi(\lambda) 
    & \le  \frac{\Phi(\lambda)}{8} + \frac{12R\log\rbr*{\frac{2}{\delta'} }}{N},
  \end{align}
  which after rearranging and plugging in the expression for $R$, results in
  \begin{align}\label{eq:phi-bernstein}
    \frac{7}{8} \Phi(\lambda) - \wh\Phi(\lambda) &\le 12\rbr*{\frac{\Rmax + \gamma}{\beta}}\frac{\log\rbr*{\frac{2}{\delta'} }}{N},\mathand
    \\
    \wh\Phi(\lambda) - \frac{9}{8} \Phi(\lambda) & \le 12\rbr*{\frac{\Rmax + \gamma}{\beta}}\frac{\log\rbr*{\frac{2}{\delta'} }}{N}, 
  \end{align}
  which proves that the desired inequality holds for any $\lambda$ fixed a-priori.
  
  We now convert this guarantee into a uniform-in-$\lambda$ bound. Consider the interval $\sbr*{-\gamma, \Rmax-\gamma}$, and, for some fixed $\veps$,  let 
  $\Lambda_\veps = \cbr*{-\gamma + \veps\cdot i : i=1,\ldots,\lceil\frac{\Rmax}{\veps}\rceil }$, 
  which has cardinality $\abr*{\Lambda_{\veps}} \le \frac{2\Rmax}{\veps}$. 
  Then for any $\lambda\in\sbr*{-\gamma, \Rmax-\gamma}$, 
  we can always find $\lambda' \in \Lambda_{\veps}$ such that 
  $\abr*{\lambda - \lambda'} \le \veps$, for which 
  \begin{align}
     \abr*{\sig{r(y)-\lambda} - \sig{r(y)-\lambda'}}
     \le \beta^{-1} \cdot \abr*{\lambda-\lambda'}
     \le \beta^{-1} \cdot \veps. 
  \end{align}
  Applying \cref{eq:phi-bernstein} and taking a union bound, with probability at least $1-\delta$ we have that for all $\lambda'\in\Lambda_{\veps}$, 
  \begin{align}
    \frac{7}{8} \Phi(\lambda') - \wh\Phi(\lambda') &\le 12\rbr*{\frac{\Rmax + \gamma}{\beta}}\frac{\log\rbr*{\frac{4\Rmax}{\veps\delta} }}{N}   
    \\
    \wh\Phi(\lambda') - \frac{9}{8} \Phi(\lambda') & \le 12\rbr*{\frac{\Rmax + \gamma}{\beta}}\frac{\log\rbr*{\frac{4\Rmax}{\veps\delta} }}{N}, 
  \end{align}
  thus for all $\lambda\in\sbr*{-\gamma, \Rmax-\gamma}$, 
  \begin{align}
    \frac{7}{8} \Phi(\lambda) - \wh\Phi(\lambda) &\le \frac{15}{8}\frac{\veps}{\beta}  + 12\rbr*{\frac{\Rmax + \gamma}{\beta}}\frac{\log\rbr*{\frac{4\Rmax}{\veps\delta} }}{N}   
    \\
    \wh\Phi(\lambda) - \frac{9}{8} \Phi(\lambda) & \le \frac{15}{8}\frac{\veps}{\beta}  + 12\rbr*{\frac{\Rmax + \gamma}{\beta}}\frac{\log\rbr*{\frac{4\Rmax}{\veps\delta} }}{N}, 
  \end{align}
  and choosing $\veps=\frac{\beta}{15}$ results in 
  \begin{align}
    \frac{7}{8} \Phi(\lambda) - \wh\Phi(\lambda) 
    &\le \frac{1}{8} 
    +  12\rbr*{\frac{\Rmax + \gamma}{\beta}}\frac{\log\rbr*{\frac{60\Rmax}{\beta\delta} }}{N}   
    \\
    \wh\Phi(\lambda) - \frac{9}{8} \Phi(\lambda) 
    &\le \frac{1}{8} 
    +  12\rbr*{\frac{\Rmax + \gamma}{\beta}}\frac{\log\rbr*{\frac{60\Rmax}{\beta\delta} }}{N}   
  \end{align}
  which when combined proves the lemma statement. 
\end{proof}

\subsection{Proof of Theorem \ref*{thm:lower}}

\begin{proof}[Proof of \cref{thm:lower}]
  Fix $N \lesssim \frac{1}{\vepsrm}$.
  We apply the first part of \cref{thm:cone-regret-lower} with
  $\vepsrm=\vepsrm$, $p=\frac{1}{2}$,
  $C = O(\vepsrm^{-1})$,
  and $\veps = \frac{1}{c_1 \cdot N\cdot\sqrt{2\Cone[\pistar](x)}}$.
  Note that the construction has $\Cone[\pistar](x) = O(\log C)$.
  With this set of parameters, any algorithm $\cA$ using $N'$ such that
  \begin{align}
    N'
    < c_1 \cdot \rbr*{\Cone[\pistar](x) \cdot \veps^2}^{-\frac{1}{2}}
    = c_1 \cdot \rbr*{\frac{1}{2 c_1^2 N^2} }^{-\frac{1}{2}}
    = 2N
  \end{align}
  must suffer regret at least
  \begin{align}
    J(\pistar) - J(\pihat_\cA)
    &> c_2 \cdot \rbr*{\Cone[\pistar](x) \cdot \veps^2}^{\frac{1}{2}}
    \\
    &= c_2 \cdot \rbr*{\frac{\Cone[\pistar](x)}{2 c_1^2\cdot\Cone[\pistar](x)\cdot N^2}}^{\frac{1}{2}}
    \\
    &= \frac{c_2}{2 c_1} \cdot \frac{1}{N},
  \end{align}
  which is therefore a lower bound that applies to $N' = N$.

\end{proof}

}
\arxiv{
  
}

\end{document}